\documentclass[12pt,a4]{amsbook}
\usepackage{amsfonts}
\usepackage{graphicx}
\usepackage{epsfig}

\usepackage{times,amssymb,amsmath,amsfonts,epsfig,graphicx,mathrsfs}
\newcommand{\Ocap}{\mbox{\textcircled{$\scriptstyle{\cap}$}}}
	
\newcommand{\lpr}{{\underline{P}}}
\newcommand{\upr}{{\overline{P}}}
\newcommand{\cred}{{\mathcal{P}}}

\newtheorem{theorem}{Theorem}
\newtheorem{lemma}{Lemma}
\newtheorem{definition}{Definition}
\newtheorem{proposition}{Proposition}
\newtheorem{corollary}{Corollary}	
\newtheorem{axiom}{Axiom}

\begin{document}

\setcounter{page}{1}
% frontespizio
\thispagestyle{empty} \null\def\baselinestretch{1}

\vskip 90mm
\begin{center}
\large {\huge \bf Visions of a Generalized Probability Theory}\\ \vskip 10mm \large Fabio Cuzzolin 
\end{center}

\vskip 90mm

\begin{center}\rm\normalsize October 18, 2018\end{center}
%\newpage
\thispagestyle{empty}

\normalsize

\chapter*{Preface}

\emph{Computer vision} is an ever growing discipline whose ambitious goal is to equip machines with the intelligent visual skills humans and animals are provided by Nature, allowing them to interact effortlessly with complex, dynamic environments. Designing automated visual recognition and sensing systems typically involves tackling a number of challenging tasks, and requires an impressive variety of sophisticated mathematical tools. In most cases, the knowledge a machine has of its surroundings is at best incomplete -- missing data is a common problem, and visual cues are affected by imprecision. The need for a coherent mathematical `language' for the description of uncertain models and measurements then naturally arises from the solution of computer vision problems. 

The \emph{theory of evidence} (sometimes referred to as `evidential reasoning', `belief theory' or `Dempster-Shafer theory') is, perhaps, one of the most successful approaches to uncertainty modelling, and arguably the most straightforward and intuitive approach to a generalized probability theory. Emerging in the late Sixties as a profound criticism of the more classical Bayesian theory of inference and modelling of uncertainty, evidential reasoning stimulated in the following decades an extensive discussion of the epistemic nature of both subjective `degrees of beliefs' and frequentist `chances', or relative frequencies. More recently a renewed interest in \emph{belief functions}, the mathematical generalization of probabilities which are the object of study of the theory of evidence, has seen a blossoming of applications to a variety of fields of human knowledge.

In this Book we are going to show how, indeed, the fruitful interaction of computer vision and belief calculus is capable of stimulating significant advances in both fields.\\ From a methodological point of view, novel theoretical results concerning the geometric and algebraic properties of belief functions as mathematical objects are illustrated and discussed in Part II, with a focus on both a perspective `geometric approach' to uncertainty and an algebraic solution to the issue of conflicting evidence.\\ In Part III we show how these theoretical developments arise from important computer vision problems (such as articulated object tracking, data association and object pose estimation) to which, in turn, the evidential formalism is able to provide interesting new solutions.\\ Finally, some initial steps towards a generalization of the notion of \emph{total probability} to belief functions are taken, in the perspective of endowing the theory of evidence with a complete battery of estimation and inference tools to the benefit of all scientists and practitioners.

\newpage

\null\def\baselinestretch{1} \noindent \emph{`'La vera logica di questo mondo  il calcolo delle probabilit\`a ... Questa
branca della matematica, che di solito viene ritenuta favorire il gioco d'azzardo, quello dei dadi e delle scommesse, e
quindi estremamente immorale, \`e la sola `matematica per uomini pratici', quali noi dovremmo essere. Ebbene, come la
conoscenza umana deriva dai sensi in modo tale che l'esistenza delle cose esterne  \`e inferita solo dall'armoniosa (ma
non uguale) testimonianza dei diversi sensi, la comprensione, che agisce per mezzo delle leggi del corretto
ragionamento, assegner\`a a diverse verit\`a (o fatti, o testimonianze, o comunque li si voglia chiamare) diversi gradi di
probabilit\`a."\vskip 4mm
\begin{flushright}{James Clerk Maxwell}\end{flushright}}
\thispagestyle{empty}

\tableofcontents \thispagestyle{empty}

\chapter{Introduction} \label{cha:introduction}

In the wide river of scientific research, seemingly separate streams often intertwine, generating novel, unexpected results. The fruitful interaction between mathematics and physics, for example, marked in the Seventeenth Century the birth of modern science in correspondence with the publication of Newton's \emph{Philosophia Naturalis Principia Mathematica}\footnote{Isaac Newton, 1687}. The accelerated accumulation of human knowledge which characterized the last century has, on the one hand, much increased the possibility of such fecund encounters taking place - on the other hand, this very growth has caused as a side effect a seemingly unstoppable trend towards extreme specialization.

The aim of this Book is to provide a significant example of how crossing the traditional boundaries between disciplines can lead to novel results and insights that would have never been possible otherwise.

As mentioned in the Preface, computer vision is an interesting case of a booming discipline involving a panoply of difficult problems, most of which involve the handling of various sources of uncertainty for decision making, classification or estimation. Indeed the latter are crucial problems in most applied sciences \cite{bell96generalized}, as both people and machines need to make inferences about the state of the external world, and take appropriate actions. Traditionally, the (uncertain) state of the world is assumed to be described by a probability distribution over a set of alternative, disjoint hypotheses. Making appropriate decisions or assessing quantities of interest requires therefore estimating such a distribution from the available data. 

Uncertainty \cite{klir95principles,Walley96,Kong86b,resconi93integration,dubois82several} is normally handled in the literature within the Bayesian framework \cite{shafer1981b,black87shafer}, a fairly intuitive and easy to use setting capable of providing a number of `off the shelf' tools to make inferences or compute estimates from time series. Sometimes, however, as in the case of extremely rare events (e.g., a volcanic eruption or a catastrophic nuclear power plant meltdown), few statistics are available to drive the estimation. Part of the data can be missing. Furthermore, under the law of large numbers, probability distributions are the outcome of an infinite process of evidence accumulation, drawn from an infinite series of samples, while in all practical cases the available evidence can only provide some sort of constraint on the unknown probabilities governing the process. All these issues have led to the recognition of the need for a coherent mathematical theory of uncertainty under partial data \cite{sheridan1991,Kohlas94,Levi83,zadeh1986,Smets88,Shafer85a,yager99modeling,cozman00reasoning}.

Different kinds of constraints are associated with different generalizations of probabilities \cite{polkowski96mereology,shafer78bernoulli}, formulated to model uncertainty at the level of  distributions \cite{klir1988,krause1993}. The simplest approach consists in setting upper $u(x)$ and lower $l(x)$ bounds to the probability values of each element $x$ of the sample space, yielding what is usually called a `probability interval'. A more general approach allows the (unknown) distribution to belong to an entire convex set of probability distributions -- a `credal set'. Convexity (as a mathematical requirement) is a natural consequence, in these theories, of rationality axioms such as coherence. A battery of different uncertainty theories has indeed been developed in the last century or so \cite{goodman85uncertainty,smithson1989,Grabish95,klir95book,km95book,shafer01book,halpern03book}, starting from De Finetti's pioneering work \cite{definetti74,kuhr2007finetti}. Among the most powerful and successful frameworks it is worth mentioning possibility-fuzzy set theory \cite{dubois88possibility}, the theory of random sets \cite{matheronrandom,ross86random}, and that of imprecise probabilities \cite{walley91book}, without forgetting other significant contributions such as monotone capacities \cite{hendon96product,denneberg00totally}, Choquet integrals \cite{wang97choquet}, rough sets, hints, and more recent approaches based on game theory \cite{pal92uncertainty1,pal92uncertainty2}. 

G. Shafer's theory of belief functions \cite{Dubois92,Fine77,zadeh84review,diaconis78review,lemmers86confidence,dubois87principle}, in particular, allows us to express partial belief by providing lower and upper bounds to probability values on all events \cite{stein93ds,spillman90managing,neapolitan93interpretation,strat89explaining,Wasserman92}. According to A. Dempster’s seminal work \cite{aitchinson68discussion}, a belief function is a lower probability measure induced by the application of a multi-valued mapping to a classical probability distribution. The term `belief function' was coined when Shafer proposed to adopt these mathematical objects to represent evidence in the framework of subjective probability, and gave an axiomatic definition for them as non-additive probability measures. In a rather controversial interpretation, belief functions can be also seen as a special case of credal set, for each of them determines a convex set of probabilities `dominating' its belief values.

Belief functions carried by different bodies of evidence can be combined using the so-called \emph{Dempster's rule}, a direct generalization of classical Bayes' rule. This combination rule is an attractive tool which has made the fortune of the theory of evidence, for it allows us to merge different sources of information prior to making decisions or estimating a quantity of interest. Many other combination rules have been proposed since, to address paradoxes generated by the indiscriminate application of Dempster's rule to all situations or to better suit cases in which the sources of evidence to combine are not independent or entirely reliable (as requested by the original combination rule).

%The theory of evidence was born as an application of standard probability theory (consult Dempster's seminal papers on \emph{upper and lower probabilities}, \cite{dempster67multivariate,Dempster68a}). Dempster's rule is itself a consequence of assuming the independence of the underlying probability distributions inducing the belief functions to combine. Glenn Shafer later reorganized it on an axiomatic basis by Glenn Shafer in his 1976 essay \cite{Shafer76}. These two interpretations - belief functions as generalized probabilities, and belief functions as mathematical descriptions of partial subjective beliefs - still fuel a raging debate at the present day.

\subsection*{Why a theory of evidence?} \label{sec:why}

Despite its success the theory of evidence, along with other non-`mainstream' uncertainty theories, is often the object of a recurring criticism: why investing effort and intellectual energy on learning a new and (arguably) rather more complex formalism only to satisfy some admittedly commendable philosophical curiosity? The implication being that classical probability theory is powerful enough to tackle any real-world application. Indeed people are often willing to acknowledge the greater naturalness of evidential solutions to a variety of problems, but tend also to point out belief calculus' issue with computational complexity while failing to see its practical edge over more standard solutions. 

Indeed, as we will see in this Book, the theory of belief functions does address a number of complications associated with the mathematical description of the uncertainty arising from the presence of partial or missing evidence (also called `ignorance'). It makes use of all and only the available (partial) evidence. It represents ignorance in a natural way, by assigning `mass' to entire sets of outcomes (in our jargon `focal elements'). It explicitly deals with the representation of evidence and uncertainty on domains that, while all being related to the same problem, remain distinct. It copes with missing data in the most natural of ways. As a matter of fact it has been shown that, when part of the data used to estimate a desired probability distribution is missing, the resulting constraint is a credal set of the type associated with a belief function \cite{zaffalon04incomplete}. Furthermore, evidential reasoning is a straightforward generalization of probability theory, one which does not require abandoning the notion of event (as is the case for Walley's imprecise probability theory). It contains as special cases both fuzzy set and possibility theory.

In this Book we will also demonstrate that belief calculus has the potential to suggest novel and arguably more `natural' solutions to real-world problems, in particular within the field computer vision, while significantly pushing the boundaries of its mathematical foundations.\\ 
A word of caution. We will neglect here almost completely the evidential interpretation of belief functions (i.e., the way Shafer's `weights of the evidence' induce degrees of belief), while mostly focussing on their mathematical nature of generalized probabilities. We will not attempt to participate in the debate concerning the existence of \emph{a} correct approach to uncertainty theory. Our belief, supported by many scientists in this field (e.g. by Didier Dubois), is that uncertainty theories form a battery of useful complementary tools among which the most suitable must be chosen depending on the specific problem at hand.

\subsection*{Scope of the Book} \label{sec:scope}

The theory of evidence is still a relatively young field. For instance, a major limitation (in its original formulation) is its being tied to {finite} decision spaces, or `frames of discernment', although a number of efforts have been brought forward to generalize belief calculus to continuous domains (see Chapter \ref{cha:state}). With this Book we wish to contribute to the completion of belief theory's mathematical framework, whose greater complexity (when compared to standard probability theory) is responsible, in addition, for the existence of a number of problems which do not have any correspondence in Bayesian theory.

We will introduce and discuss theoretical advances concerning the {geometrical} and {algebraic} properties of belief functions and the domains they are defined on, and formalize in the context of the theory of evidence a well-known notion of probability theory -- that of \emph{total function}. In the longer term, our effort is directed towards equipping belief calculus with notions analogous to those of {filtering} and {random process} in probability theory. Such tools are widely employed in all fields of applied science, and their development is, in our view, crucial to making belief calculus a valid alternative to the more classical Bayesian formalism.

We will show how these theoretical advances arise from the formulation of evidential solutions to classical computer vision problems. We believe this may introduce a novel perspective into a discipline that, in the last twenty years, has had the tendency to reduce to the application of kernel-based support vector machine classification to images and videos.

\subsection*{Outline of the Book} \label{sec:outline}

Accordingly, this Book is divided into three Parts. 

In Part I we recall the core definitions and the rationale of the theory of evidence (Chapter \ref{cha:toe}), along with the notions necessary to the comprehension of what follows. As mentioned above, many theories have been formulated with the goal of integrating or replacing classical probability theory, their rationales ranging from the more philosophical to the strictly application-orientated. Several of these methodologies for the mathematical treatment of {uncertainty} are briefly reviewed in Chapter \ref{cha:state}. We will not, however, try to provide a comprehensive view of the matter, which is still evolving as we write.

Part II is the core of this work. Starting from Shafer's axiomatic formulation of the theory of belief functions, and motivated by the computer vision applications later discussed in Part III, we propose novel analyses of the geometrical and algebraic properties of belief functions as set functions, and of the domains they are defined on.\\
In particular, in Chapter \ref{cha:geo} the geometry of belief functions is investigated by analyzing the convex shape of the set of all the belief functions defined over the same frame of discernement (which we call `{belief space}'). The belief space admits two different geometric representations, either as a simplex (a higher-dimensional triangle) or as a (recursive) fiber bundle. It is there possible to give both a description of the effect of conditioning on belief functions and a geometric interpretation of Dempster's rule of combination itself. In perspective, this geometric approach has the potential to allow us to solve problems such as the canonical decomposition of a belief function in term of its simple support components (see Chapter \ref{cha:toe}). The problem of finding a {probability transformation} of belief functions that is respectful of the principles of the theory of evidence, and may be used to make decisions based on classical utility theory, also finds a natural formulation in this geometric setting.\\
Stimulated by the so-called `conflict' problem, i.e., the fact that each and every collection of belief functions (representing, for example, a set of image features in the pose estimation problem of Chapter \ref{cha:pose}) is not combinable, in Chapter \ref{cha:alg} we analyze the algebraic structure of families of compatible frames of discernment, proving that they form a special class of lattices. In Chapter \ref{cha:independence} we build on this lattice structure to study Shafer's notion of `independence of frames', seen as elements of a semimodular lattice. We relate independence of frames to classical matroidal independence, and outline a future solution to the conflict problem based on a pseudo-Gram-Schmidt orthogonalization algorithm.

In Part III we illustrate in quite some detail two computer vision applications whose solution originally stimulated the mathematical developments of Part II.\\ In Chapter \ref{cha:total} we propose an evidential solution to the {model-based data association} problem, in which correspondences between feature points in adjacent frames of a video associated with the positions of markers on a moving articulated object are sought at each time instant. Correspondences must be found in the presence of occlusions and missing data, which induce uncertainty in the resulting estimates. Belief functions can be used to integrate the logical information available whenever a topological model of the body is known, with the predictions generated by a battery of classical Kalman filters. In this context the need to combine {conditional} belief functions arises, leading to the evidential analogue of the classical total probability theorem. This is the first step, in our view, towards a theory of filtering in the framework of generalized probabilities. Unfortunately, only a partial solution to this \emph{total belief} problem is given in this Book, together with sensible predictions on the likely future directions of this investigation.

Chapter \ref{cha:pose} introduces an evidential solution to the problem of estimating the configuration or `pose' of an articulated object from images and videos, while solely relying on a training set of examples. A sensible approach consists in learning maps from image features to poses, using the information provided by the training set. We present therefore a `Belief Modeling Regression' (BMR) framework in which, given a test image, its feature measurements translate into a collection of belief functions on the set of training poses. These are then combined to yield a belief estimation, equivalent to an entire family of probability distributions. From the latter, either a single central pose estimate (together with a measure of its reliability) or a set of extremal estimates can be computed. We illustrate BMR's performance in an application to human pose recovery, showing how it outperforms our implementation of both Relevant Vector Machine and Gaussian Process Regression. We discuss motivation and advantages of the proposed approach with respect to its competitors, and outline an extension of this technique to fully-fledged tracking.

Finally, some reflections are proposed in the Conclusions of Part IV on the future of belief calculus, its relationships with other fields of pure and applied mathematics and statistics, and a number of future developments of the lines of research proposed in this Book.

\part{Belief calculus}

\chapter{Shafer's mathematical theory of evidence} \label{cha:toe}

\begin{center}
\includegraphics[width = 0.6 \textwidth]{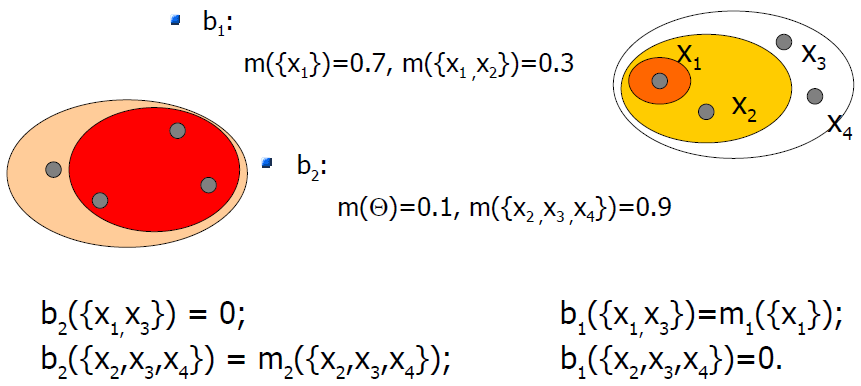}
\end{center}
\vspace{5mm}

The \emph{theory of evidence} \cite{Shafer76} was introduced in the Seventies by Glenn Shafer as a way of representing epistemic knowledge, starting from a sequence of seminal works (\cite{Dempster67}, \cite{Dempster68a}, \cite{Dempster69}) by Arthur Dempster, Shafer's advisor \cite{Dempster08b}. In this formalism the best representation of chance is a \emph{belief function} (b.f.) rather than a classical probability distribution. Belief functions assign probability values to \emph{sets} of outcomes, rather than single events: their appeal rests on their ability to naturally encode evidence in favor of {propositions}.\\ The theory embraces the familiar idea of assigning numbers between 0 and 1 to measure degrees of support but, rather than focusing on \emph{how} these numbers are determined, it concerns itself with the mechanisms driving the \emph{combination} of degrees of belief.\\ The formalism provides indeed a simple method for merging the evidence carried by a number of distinct sources (called \emph{Dempster's rule} \cite{hajek92deriving}), with no need for any prior distributions \cite{wilson91prior}. In this sense, according to Shafer, it can be seen as a theory of probable reasoning. The existence of different levels of granularity in knowledge representation is formalized via the concept of \emph{family of compatible frames}.

The most popular theory of probable reasoning is perhaps the \emph{Bayesian} framework \cite{bernardo94,dempster90bayes}, due to the English clergyman Thomas Bayes (1702-1761). There, all degrees of beliefs must obey the rule of chances (i.e., the proportion of time in which one of the possible outcomes of a random experiment tends to occur). Furthermore, its fourth rule dictates the way a Bayesian function must be updated whenever we learn that a certain proposition $A$ is true:
\begin{equation} \label{eq:bayes}
P(B|A) = \frac{P(B\cap A)}{P(A)}.
\end{equation}
This so-called \emph{Bayes' rule} is inextrably related to the notion of `conditional' probability $P(B|A)$ \cite{lewis76}. 

As we recall in this Chapter, the Bayesian framework is actually contained in the theory of evidence as a special case, since:
\begin{enumerate}
\item Bayesian functions form a special class of belief functions, and 
\item Bayes' rule is a special case of Dempster's rule of combination.
\end{enumerate}
In the following we will neglect most of the emphasis Shafer put on the notion of `weight of evidence', which in our view is not strictly necessary to the comprehension of what follows.

\section{Belief functions} \label{sec:belief-functions}

Let us first review the classical definition of probability measure, due to Kolmogorov \cite{kolmogorov}.

\subsection{Belief functions as superadditive measures}

\begin{definition}\label{def:prob}
A \emph{probability measure} over a $\sigma$-field or $\sigma$-algebra\footnote{\label{foot:sigma} Let $X$ be some set, and let $2^X$ represent its power set. Then a subset $\Sigma \subset 2^X$ is called a \emph{$\sigma$-algebra} if it satisfies the following three properties \cite{Rudin}:
\begin{itemize}
\item
$\Sigma$ is non-empty: there is at least one $A \subset X$ in $\Sigma$;
\item
$\Sigma$ is closed under complementation: if $A$ is in $\Sigma$, then so is its complement, $X \setminus A$;
\item
$\Sigma$ is closed under countable unions: if $A_1, A_2, A_3, ...$ are in $\Sigma$, then so is $A = A_1 \cup A_2 \cup A_3 \cup\cdots$.
\end{itemize}
From these properties, it follows that the $\sigma$-algebra is also closed under countable intersections (by De Morgan's laws).} $\mathbf{F}\subset 2^{\Omega}$ associated with a sample space
$\Omega$ is a function $p:\mathbf{F} \rightarrow [0,1]$ such that
\begin{itemize}
\item $p(\emptyset)=0$;
\item $p(\Omega)=1$;
\item if $A\cap B = \emptyset,\; A,B\in \mathbf{F}$ then $p(A\cup B)=p(A)+p(B)$
(\emph{additivity}).
\end{itemize}
\end{definition}

If we relax the third constraint to allow the function $p$ to meet additivity only as a lower bound, and restrict ourselves to finite sets, we obtain what Shafer \cite{Shafer76} called a \emph{belief function}.

\begin{definition}\label{def:bel1} Suppose $\Theta$ is a finite set, and let $2^{\Theta} = \{ A \subseteq \Theta \}$ denote the set of all subsets of $\Theta$. A \emph{belief function} (b.f.) on $\Theta$ is a function $b:2^{\Theta} \rightarrow [0,1]$ from the power set $2^\Theta$ to the real interval $[0,1]$ such that:
\begin{itemize}
\item $b(\emptyset)=0$;
\item $b(\Theta)=1$;
\item for every positive integer $n$ and for every collection $A_1,...,A_n\in 2^{\Theta}$ we have that:
\begin{equation} \label{eq:superadditivity}
\displaystyle b(A_1\cup ... \cup A_n)\geq \sum_i b(A_i) - \sum_{i<j} b(A_i\cap A_j)+ ... + (-1)^{n+1} b(A_1\cap ... \cap A_n) .
\end{equation}
\end{itemize}
\end{definition}

Condition (\ref{eq:superadditivity}), called \emph{superadditivity}, obviously generalizes Kolmogorov's additivity (Definition \ref{def:prob}). Belief functions can then be seen as {generalizations} of the familiar notion of (discrete) probability measure. The domain $\Theta$ on which a belief function is defined is usually interpreted as the set of possible answers to a given problem, exactly one of which is the correct one. For each subset (`event') $A\subset \Theta$ the quantity $b(A)$ takes on the meaning of \emph{degree of belief} that the truth lies in $A$.

\emph{Example: the Ming vase.} A simple example (from \cite{Shafer76}) can clarify the notion of degree of belief. We are looking at a vase that is represented as a product of the Ming dynasty, and we are wondering whether the vase is genuine. If we call $\theta_1$ the possibility that the vase is original, and $\theta_2$ the possibility that it is indeed counterfeited, then
\[
\Theta=\{\theta_1,\theta_2\}
\]
is the set of possible outcomes, and
\[
\big \{ \emptyset, \Theta,\{\theta_1\},\{\theta_2\} \big \}
\]
is the (power) set of all its subsets. A belief function $b$ over $\Theta$ will represent the degree of belief that the vase is genuine as $b(\{\theta_1\})$, and the degree of belief the vase is a fake as $b(\{\theta_2\})$ (note we refer to the \emph{subsets} $\{\theta_1\}$ and $\{\theta_2\}$). Axiom 3 of Definition \ref{def:bel1} poses a simple constraint over these degrees of belief, namely: $b(\{\theta_1\}) + b(\{\theta_2\}) \leq 1$. The belief value of the whole outcome space $\Theta$, therefore, represents evidence that cannot be committed to any of the two precise answers $\theta_1$ and $\theta_2$ and is therefore an indication of the level of uncertainty about the problem.

\subsection{Belief functions as set functions}

Following Shafer \cite{Shafer76} we call the finite set of possibilities/outcomes \emph{frame\footnote{For a note about the intuitionistic origin of this denomination see \emph{Rosenthal, Quantales and their applications} \cite{Rosenthal}.} of discernment} (FOD). 

\subsubsection{Basic probability assignment}

\begin{definition}\label{def:bpa}
A \emph{basic probability assignment} (b.p.a.) \cite{Augustin96} over a FOD $\Theta$ is a set function \cite{dubois86set,denneberg99interaction,dubois86set} $m : 2^\Theta\rightarrow[0,1]$ defined on the collection $2^\Theta$ of all subsets of $\Theta$ such that:
\[
m(\emptyset)=0, \; \sum_{A\subset\Theta} m(A)=1.
\]
\end{definition}
The quantity $m(A)$ is called the \emph{basic probability number} or `mass' \cite{kruse91tool,kruse91reasoning} assigned to $A$, and measures the belief committed exactly to $A\in 2^{\Theta}$. The elements of the power set $2^\Theta$ associated with non-zero values of $m$ are called the \emph{focal elements} of $m$ and their union is called its \emph{core}:
\begin{equation} \label{eq:core}
\mathcal{C}_m \doteq \bigcup_{A \subseteq \Theta : m(A)\neq 0} A.
\end{equation}
Now suppose that empirical evidence is available so that a basic probability assignment can be introduced over a specific FOD $\Theta$.
\begin{definition} \label{def:bel2}
The \emph{belief function} associated with a basic probability assignment $m : 2^\Theta\rightarrow[0,1]$ is the set function $b : 2^\Theta\rightarrow[0,1]$ defined as:
\begin{equation}\label{eq:belief}
b(A) = \sum_{B\subseteq A} m(B).
\end{equation}
\end{definition}
It can be proven that \cite{Shafer76}:
\begin{proposition}
Definitions \ref{def:bel1} and \ref{def:bel2} are equivalent formulations of the notion of belief function.
\end{proposition}

The intuition behind the notion of belief function is now clearer: $b(A)$ represents the \emph{total} belief committed to a set of possible outcomes $A$ by the available evidence $m$. As the Ming vase example illustrates, belief functions readily lend themselves to the representation of ignorance, in the form of the mass assigned to the whole set of outcomes (FOD). Indeed, the simplest belief function assigns all the basic probability to the whole frame $\Theta$ and is called \emph{vacuous} belief function. 

Bayesian theory, in comparison, has trouble with the whole idea of encoding ignorance, for it cannot distinguish between `lack of belief' in a certain event $A$ ($1 - b(A)$ in our notation) and `disbelief' (the belief in the negated event $\bar{A} = \Theta \setminus A$). This is due to the additivity constraint: $P(A) + P(\bar{A}) = 1$.\\ The Bayesian way of representing the complete absence of evidence is to assign an equal degree of belief to every outcome in $\Theta$. As we will see in this Chapter, Section \ref{sec:priors}, this generates incompatible results when considering different descriptions of the same problem at different levels of granularity.

\subsubsection{Moebius inversion formula}

Given a belief function $b$ there exists a unique basic probability assignment which induces it. The latter can be recovered by means of the
\emph{Moebius inversion formula}\footnote{See \cite{Stern} for an explanation in term of the theory of monotone functions over partially ordered sets.}:
\begin{equation} \label{eq:moebius}
m(A)=\sum_{B\subset A} (-1)^{|A \setminus B|}b(B).
\end{equation}
Expression (\ref{eq:moebius}) establishes a 1-1 correspondence between the two set functions $m$ and $b$ \cite{grabish06moebius}.

\subsection{Plausibility functions or upper probabilities}

Other expressions of the evidence generating a given belief function $b$ are what can be called the \emph{degree of doubt} $d(A) \doteq b(\bar{A})$ on an event $A$ and, more importantly, the \emph{upper probability} of $A$:
\begin{equation}\label{eq:upper-probability}
pl(A) \doteq 1 - d(A) = 1 - b(\bar{A}),
\end{equation}
as opposed to the \emph{lower probability} of $A$, i.e., its belief value $b(A)$.
The quantity $pl(A)$ expresses the `plausibility' of a proposition $A$ or, in other words, the amount of evidence \emph{not against} $A$ \cite{cuzzolin10ida}. Once again the \emph{plausibility function} $pl : 2^\Theta \rightarrow [0,1]$ conveys the same information as $b$, and can be expressed as
\[
pl(A) = \sum_{B\cap A\neq \emptyset} m(B) \geq b(A).
\]

\subsubsection{Example} \label{sec:example-belief-function}

As an example, suppose a belief function on a frame $\Theta = \{ \theta_1, \theta_2, \theta_3 \}$ of cardinality three has two focal elements $B_1 = \{ \theta_1, \theta_2 \}$ and $B_2 = \{ \theta_1 \}$ as in Figure \ref{fig:classes}, with b.p.a. $m(B_1) = 1/3$, $m(B_2) = 2/3$. 

Then, for instance, the belief value of $A = \{ \theta_1, \theta_3 \}$ is: 
\begin{equation} \label{eq:example-belief-value}
b(A) = \sum_{B\subseteq \{ \theta_1, \theta_3 \}} m(B) = m(\{ \theta_1 \}) = 2/3, 
\end{equation}
while $b(\{ \theta_2 \}) = m(\{ \theta_2 \}) = 0$ and $b(\{ \theta_1, \theta_2 \}) = m(\{ \theta_1 \}) + m(\{ \theta_1, \theta_2 \}) = 2/3 + 1/3 = 1$ (so that the `core' of the considered belief function is $\mathcal{C} = \{ \theta_1, \theta_2 \}$).
\begin{figure}[ht!]
\begin{center}
\includegraphics[width = 0.5 \textwidth]{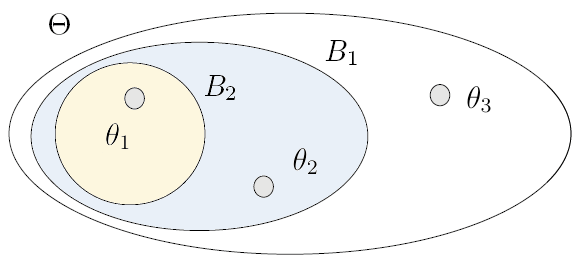} 
\caption{ \label{fig:classes} An example of (consonant, see Section \ref{sec:consonant}) belief function on a frame of discernment $\Theta = \{ \theta_1, \theta_2, \theta_3 \}$ of cardinality 3, with focal elements $B_2 = \{ \theta_1 \} \subset B_1 = \{ \theta_1, \theta_2 \}$. }
\end{center}
\end{figure}

To appreciate the difference between belief (lower probability) and plausibility (upper probability) values, let us focus in particular on the event $A' = \{ \theta_1, \theta_3 \}$. Its belief value (\ref{eq:example-belief-value}) represents the amount of evidence which \emph{surely supports} $\{ \theta_1, \theta_3 \}$, and is guaranteed to involve \emph{only} elements of $A'$.

On the other side, its plausibility value:
\[
pl (\{\theta_1, \theta_3 \}) = 1 - b(\{ \theta_1, \theta_3 \}^c) = \sum_{B \cap \{\theta_1, \theta_3\} \neq \emptyset } m (B) = m(\{ \theta_1 \}) + m(\{ \theta_1, \theta_2 \}) = 1
\]
accounts for the mass that \emph{might} be assigned to some element of $A'$, and measures the evidence \emph{not surely against} it. 

\subsection{Bayesian theory as a limit case} \label{sec:bayesian-belief-functions}

Confirming what said when discussing the superadditivity axiom (\ref{eq:superadditivity}), in the theory of evidence a (finite) probability function is simply a belief function satisfying the additivity rule for disjoint sets.
\begin{definition}
A \emph{Bayesian} belief function $b: 2^\Theta \rightarrow [0,1]$ meets the additivity condition:
\[
b(A) + b(\bar{A}) = 1
\]
whenever $A \subseteq \Theta$.
\end{definition}
Obviously, as it meets the axioms of Definition \ref{def:bel1}, a Bayesian belief function is indeed a belief function. It can be proved that \cite{Shafer76}:
\begin{proposition}
A belief function $b: 2^\Theta \rightarrow [0,1]$ is Bayesian if and only if $\exists \ p:\Theta\rightarrow [0,1]$ such that $\sum_{\theta\in\Theta} p(\theta)=1$ and:
\[
b(A) = \sum_{\theta\in A}p(\theta)\;\;\;\forall A\subseteq \Theta.
\]
\end{proposition}

\section{Dempster's rule of combination} \label{sec:dempster}

Belief functions representing distinct bodies of evidence can be combined by means of \emph{Dempster's rule of combination} \cite{dubois86unicity}, also called \emph{orthogonal sum}.

\subsection{Definition}

\begin{definition} \label{def:dempster}
The \emph{orthogonal sum} $b_1 \oplus b_2 : 2^\Theta \rightarrow [0,1]$ of two belief functions $b_1 : 2^\Theta \rightarrow [0,1]$, $b_2 : 2^\Theta \rightarrow [0,1]$ defined on the same FOD $\Theta$ is the unique belief function on $\Theta$ whose focal elements are all the possible intersections of focal elements of $b_1$ and $b_2$, and whose basic probability assignment is given by:
\begin{equation} \label{eq:dempster}
\displaystyle m_{b_1 \oplus b_2}(A) = \frac{ \displaystyle \sum_{i,j : A_i\cap B_j=A} m_1(A_i) m_2(B_j)} { \displaystyle 1-\sum_{i,j :
A_i\cap B_j=\emptyset} m_1(A_i) m_2(B_j)},
\end{equation}
where $m_i$ denotes the b.p.a. of the input belief function $b_i$.
\end{definition}

Figure \ref{fig:dempdiagram} pictorially expresses Dempster's algorithm for computing the basic probability assignment of the combination $b_1 \oplus b_2$ of two belief functions. Let a unit square represent the total, unitary probability mass one can assign to subsets of $\Theta$, and associate horizontal and vertical strips with the focal elements $A_1,...,A_k$ and $B_1,...,B_l$ of $b_1$ and $b_2$, respectively. If their width is equal to their mass value, then their area is also equal to their own mass $m(A_i)$, $m(B_j)$. The area of the intersection of the strips related to any two focal elements $A_i$ and $B_j$ is then equal to the product $m(A_i)\cdot m(B_j)$, and is committed to the intersection event $A_i \cap B_j$. As more than one such rectangle can end up being assigned to the same subset $A$ (as different pairs of focal elements can have the same intersection) we need to sum up all these contributions, obtaining:
\[
m_{b_1 \oplus b_2}(A) \propto \sum_{i,j : A_i\cap B_j=A} m_1(A_i) m_2(B_j).
\]
Finally, as some of these intersections may be empty, we need to discard the quantity
\[
\sum_{i,j : A_i\cap B_j=\emptyset} m_1(A_i) m_2(B_j)
\]
by normalizing the resulting basic probability assignment, obtaining (\ref{eq:dempster}).

Note that, by Definition \ref{def:dempster} \emph{not all pairs of belief functions admit an orthogonal sum} -- two belief functions are combinable if and only if their cores (\ref{eq:core}) are not disjoint: $\mathcal{C}_1 \cap \mathcal{C}_2 \neq \emptyset$ or, equivalently, iff there exist a f.e. of $b_1$ and a f.e. of $b_2$ whose intersection is non-empty. $A_1 \; A_2 \; A_3 \; A_4 \; B_1 \; B_2 \; B_3$
\begin{figure}[ht!]
\begin{center}
\includegraphics[width = 0.53 \textwidth]{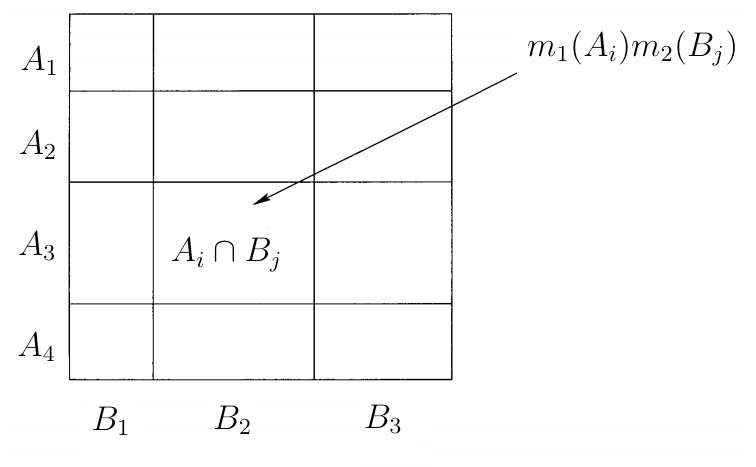}
\caption{\label{fig:dempdiagram} Graphical representation of Dempster's rule of combination: the sides of the square are divided into strips associated with the focal elements $A_i$ and $B_j$ of the belief functions $b_1$, $b_2$ to combine.}
\end{center}
\end{figure}

\begin{proposition}
\cite{Shafer76} If $b_1, b_2 : 2^\Theta \rightarrow [0,1]$ are two belief functions defined on the same frame $\Theta$, then the following conditions are equivalent:
\begin{itemize}
\item their Dempster's combination $b_1\oplus b_2$ does not exist;
\item their cores  (\ref{eq:core}) are disjoint, $\mathcal{C}_{b_1} \cap \mathcal{C}_{b_2} = \emptyset$;
\item $\exists A\subset \Theta$ s.t. $b_1(A) = b_2(\bar{A}) = 1$.
\end{itemize}
\end{proposition}

\subsubsection{Example of Dempster's combination}

\begin{figure}[ht!]
\begin{center}
\begin{tabular}{c}
\includegraphics[width = 0.7 \textwidth]{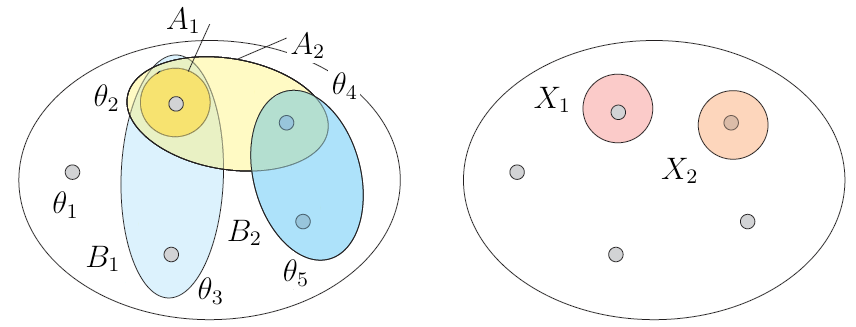} 
\end{tabular}
\caption{ \label{fig:example-dempster} Example of Dempster's sum. The belief functions $b_1$ with focal elements $A_1,A_2$ and $b_2$ with f.e.s $B_1,B_2$ (left) are combinable via Dempster's rule. This yields a new belief function $b_1\oplus b_2$ (right) with focal elements $X_1$ and $X_2$.}
\end{center}
\end{figure}

Consider a frame of discernment $\Theta = \{ \theta_1, \theta_2, \theta_3, \theta_4, \theta_5 \}$. We can define there a belief function $b_1$ with basic probability assignment: 
\[
\begin{array}{cc}
m_1(\{ \theta_2 \}) = 0.7, & m_1(\{ \theta_2, \theta_4 \}) = 0.3. 
\end{array}
\]
Such a b.f. has then two focal elements $A_1 = \{ \theta_2 \}$ and $A_2 = \{ \theta_2, \theta_4 \}$. As an example, its belief values on the events $\{\theta_4\}$, $\{\theta_2,\theta_5\}$, $\{\theta_2,\theta_3,\theta_4\}$ are respectively $b_1(\{\theta_4\}) = m_1(\{\theta_4\})=0$, $b_1(\{\theta_2,\theta_5\}) = m_1(\{\theta_2\}) + m_1(\{\theta_5\}) + m_1(\{\theta_2,\theta_5\}) = 0.7 + 0 + 0 = 0.7$ and $b_1(\{\theta_2,\theta_3,\theta_4\}) = m_1(\{\theta_2\}) + m_1(\{\theta_2,\theta_4\}) = 0.7 + 0.3 = 1$ (so that the core of $b_1$ is $\{\theta_2,\theta_4\}$). 

Now, let us introduce another belief function $b_2$ on the same FOD, with b.p.a.: 
\[
\begin{array}{cc}
m_2(B_1) = m_2(\{ \theta_2, \theta_3 \}) = 0.6, & m_2(B_2) = m_2(\{ \theta_4, \theta_5 \}) = 0.4.
\end{array}
\]

The pair of belief functions are combinable, as their cores $\mathcal{C}_1 = \{\theta_2,\theta_4\}$ and $\mathcal{C_2} = \{ \theta_2, \theta_3, \theta_4, \theta_5 \}$ are clearly not disjoint.

Dempster's combination (\ref{eq:dempster}) yields a new belief function on the same FOD, with focal elements (Figure \ref{fig:example-dempster}-right) $X_1 = \{ \theta_2 \} = A_1 \cap B_1 = A_2 \cap B_1$ and $X_2 = \{\theta_4\} = A_2 \cap B_2$ and b.p.a.: 
\[
\begin{array}{lll}
m(X_1) & = & \displaystyle \frac{m_1(\{ \theta_2 \}) \cdot m_2(\{ \theta_2, \theta_3 \}) + m_1(\{ \theta_2, \theta_4 \}) \cdot m_2(\{ \theta_2, \theta_3 \}) }{1 - m_1(\{ \theta_2 \}) \cdot m_2(\{ \theta_4, \theta_5 \})} = \frac{0.7 \cdot 0.6 + 0.3 \cdot 0.6}{1 - 0.7 \cdot 0.4} = 5/6, \\ \\ m(X_2) & = & \displaystyle \frac{m_1(\{ \theta_2, \theta_4 \}) \cdot m_2(\{ \theta_4, \theta_5 \})}{1 - m_1(\{ \theta_2 \}) \cdot m_2(\{ \theta_4, \theta_5 \})} = \frac{0.3 \cdot 0.4}{1 - 0.7 \cdot 0.4} = 1/6.
\end{array}
\]

Note that the resulting b.f. $b_1 \oplus b_2$ is Bayesian.

\subsection{Weight of conflict} \label{sec:weight-of-conflict}

The normalization constant in (\ref{eq:dempster}) measures the \emph{level of conflict} between the two input belief
functions, for it represents the amount of evidence they attribute to contradictory (i.e., disjoint) subsets.
\begin{definition}\label{def:conf}
We call \emph{weight of conflict} $\mathcal{K}(b_1,b_2)$ between two belief functions $b_1$ and $b_2$ the logarithm of the normalisation constant in their Dempster's combination:
\[
\mathcal{K} = \log\frac{1}{1-\sum_{i,j : A_i\cap B_j = \emptyset} m_1(A_i) m_2(B_j)}.
\]
\end{definition}

Dempster's rule can be trivially generalised to the combination of $n$ belief functions. It is interesting to note that, in that case, weights of conflict combine additively.
\begin{proposition}
Suppose $b_1,...,b_{n+1}$ are belief functions defined on the same frame $\Theta$, and assume that $b_1 \oplus \cdots \oplus b_{n+1}$ exist. Then:
\[
\mathcal{K}(b_1,...,b_{n+1}) = \mathcal{K}(b_1,..., b_{n}) + \mathcal{K}(b_1 \oplus \cdots \oplus b_{n}, b_{n+1}).
\]
\end{proposition}

\subsection{Conditioning belief functions}

Dempster's rule describes the way the assimilation of new evidence $b'$ changes our beliefs previously encoded by
a belief function $b$, determining new belief values given by $b \oplus b'(A)$ for all events $A$. In this formalism, \emph{a new body of evidence is not constrained to be in the form of a single proposition $A$ known with certainty}, as it happens in Bayesian theory.\\ Yet, the incorporation of new certainties is permitted as a special case. In fact, this special kind of evidence is represented by belief functions of the form:
\[
b'(A) = \left\{ \begin{array}{ll} 1 & if B\subset A\\ \\ 0 & if B \not \subset A \end{array} \right. ,
\]
where $B$ is the proposition known with certainty. Such a belief function is combinable with the original b.f. $b$ as long as $b(\bar{B}) < 1$, and the result has the form:
\[
b(A | B) \doteq b \oplus b' = \frac{b(A\cup \bar{B}) - b(\bar{B})}{1 - b(\bar{B})}
\]
or, expressing the result in terms of upper probabilities/plausibilities (\ref{eq:upper-probability}):
\begin{equation} \label{eq:dempster-conditioning}
pl(A | B) = \frac{pl(A \cap B)}{pl(B)}.
\end{equation}
Expression (\ref{eq:dempster-conditioning}) strongly reminds us of Bayes's rule of conditioning (\ref{eq:bayes}) -- Shafer calls it \emph{Dempster's rule of conditioning}.

\subsection{Combination vs conditioning}

Dempster's rule (\ref{eq:dempster}) is clearly symmetric in the role assigned to the two pieces of evidence $b$ and $b'$ (due to the commutativity of set-theoretical intersection). In Bayesian theory, instead, we are constrained to represent new evidence as a true proposition, and condition a Bayesian prior probability on that proposition. There is no obvious symmetry, but even more importantly {we are forced to assume that the consequence of any new piece of evidence is to support a single proposition with certainty}!

\section{Simple and separable support functions} \label{sec:separable}

In the theory of evidence a body of evidence (a belief function) usually supports more than one proposition (subset) of a frame of discernment. The simplest situation, however, is that in which the evidence points \emph{to a single non-empty subset} $A\subset \Theta$.\\
Assume $0\leq \sigma \leq 1$ is the degree of support for $A$. Then, the degree of support for a generic subset $B\subset\Theta$ of the frame is given by:
\begin{equation}\label{eq:simple}
b(B) = \left\{\begin{array}{ll} 0 & if \; B\not\supset A\\ \\ \sigma & if \; B\supset A,\;B\neq\Theta\\ \\ 1 & if \; B = \Theta.
\end{array} \right.
\end{equation}

\begin{definition} \label{def:simple-support-function}
The belief function $b : 2^{\Theta} \rightarrow [0,1]$ defined by Equation (\ref{eq:simple}) is called a \emph{simple support function} focused on $A$. Its basic probability assignment is: $m(A) = \sigma$, $m(\Theta) = 1 - \sigma$ and $m(B) = 0$ for every other $B$.
\end{definition}

\subsection{Heterogeneous and conflicting evidence}

We often need to combine evidence pointing towards different subsets, $A$ and $B$, of our frame of discernment. When
$A\cap B\neq \emptyset$ these two propositions are compatible, and we say that the associated belief functions represent \emph{heterogeneous} evidence.\\ In this case, if $\sigma_1$ and $\sigma_2$ are the masses committed respectively to $A$ and $B$ by two simple support functions $b_1$ and $b_2$, we have that their Dempster's combination has b.p.a.:
\[
m(A\cap B)=\sigma_1 \sigma_2, \;m(A)=\sigma_1(1-\sigma_2),\;m(B)=\sigma_2(1-\sigma_1),\;m(\Theta)=(1-\sigma_1)(1-\sigma_2).
\]
Therefore, the belief values of $b = b_1 \oplus b_2$ are as follows:
\[
b(X) = b_1 \oplus b_2 (X) = \left\{\begin{array}{ll} 0 & X\not\supset A\cap B\\ \\ \sigma_1\sigma_2 & X\supset A\cap B,\;X\not\supset A,B\\ \\ \sigma_1 & X\supset A,\; X\not\supset B\\ \\ \sigma_2 & X\supset B,\; X\not\supset A\\ \\ 1-
(1-\sigma_1)(1-\sigma_2) & X\supset A,B,\;X\neq \Theta\\ \\ 1 & X=\Theta.
\end{array} \right.
\]
As our intuition would suggest, the combined evidence supports $A\cap B$ with degree $\sigma_1\sigma_2$.

When the two propositions have empty intersection $A\cap B = \emptyset$, instead, we say that the evidence is \emph{conflicting}. In this situation the two bodies of evidence contrast the effect of each other.\\
The following example is also taken from \cite{Shafer76}.

\subsubsection{Example: the alibi}

A criminal defendant has an alibi: a close friend swears that the defendant was visiting his house at the time of the crime. This friend has a good reputation: suppose this commits a degree of support of $1/10$ to the innocence of the defendant ($I$). On the other side, there is a strong, actual body of evidence providing a degree of support of $9/10$ for his guilt ($G$).

To formalize this case we can build a frame of discernment $\Theta=\{G,I\}$, so that the defendant's friend provides a simple support function focused on $\{I\}$ with $b_I(\{I\}) = 1/10$, while the hard piece of evidence corresponds to another simple support function $b_G$ focused on $\{G\}$ with $b_G(\{G\}) = 9/10$.\\ Their orthogonal sum $b = b_I \oplus b_G$ yields then:
\[
b(\{I\}) = 1/91,\hspace{5mm} b(\{G\}) = 81/91.
\]
The effect of the testimony has mildly eroded the force of the circumstantial evidence.

\subsection{Separable support functions and decomposition}

In general, belief functions can support more than one proposition at a time.\\ The next simplest class of b.f.s is that of `separable support functions'.

\begin{definition} \label{def:separable-support-function}
A \emph{separable support function} $b$ is a belief function that is either simple or equal to the orthogonal sum of two or more simple support functions, namely:
\[
b = b_1 \oplus \cdots \oplus b_n,
\]
where $n \geq 1$ and $b_i$ is simple $\forall \; i=1,...,n$.
\end{definition}
\noindent A separable support function $b$ can be decomposed into simple support functions in different ways. More precisely, given one such decomposition $b = b_1 \oplus \cdots \oplus b_n$ with foci $A_1,...,A_n$ and denoting by $\mathcal{C}$ the core of $b$, each of the following
\begin{itemize}
\item $b = b_1 \oplus \cdots \oplus b_n \oplus b_{n+1}$ whenever $b_{n+1}$ is the vacuous belief function on the same frame;
\item $b = (b_1 \oplus b_2) \oplus \cdots \oplus b_n$ whenever $A_1 = A_2$;
\item $b = b'_1 \oplus \cdots \oplus b'_n$, whenever $b'_i$ is the simple support function focused on $A'_i \doteq A_i \cap \mathcal{C}$ such that $b'_i(A'_i) = b_i(A_i)$, if $A_i\cap \mathcal{C} \neq \emptyset$ for all $i$;
\end{itemize}
is a valid decomposition of $b$ in terms of simple belief functions. On the other hand,

\begin{proposition} \label{pro:canonical}
If $b$ is a non-vacuous, separable support function with core $\mathcal{C}_b$ then there exists a unique collection $b_1,...,b_n$ of non-vacuous simple support functions which satisfy the following conditions:
\begin{enumerate}
\item $n\geq 1$;
\item $b = b_1$ if  $n=1$, and $b = b_1 \oplus \cdots \oplus b_n$ if  $n\geq 1$;
\item $\mathcal{C}_{b_i} \subset \mathcal{C}_b$;
\item $\mathcal{C}_{b_i} \neq \mathcal{C}_{b_j}$ if $i\neq j$.
\end{enumerate}
\end{proposition}

This unique decomposition is called the \emph{canonical decomposition} of $b$ -- we will reconsider it later in the Book.\\
An intuitive idea of what a separable support function represents is provided by the following result.

\begin{proposition} \label{pro:separable}
If $b$ is a separable belief function, and $A$ and $B$ are two of its focal elements with $A\cap B \neq \emptyset$, then $A\cap B$ is a focal element of $b$.
\end{proposition}
The set of f.e.s of a separable support function is closed under set-theoretical intersection. Such a n.f. $b$ is coherent in the sense that if it supports two propositions, then it must support the proposition `naturally' implied by them, i.e., their intersection. Proposition \ref{pro:separable} gives us a simple method to check whether a given belief function is indeed a separable support function.

\subsection{Internal conflict} \label{sec:internal-conflict}

Since a separable support function can support pairs of disjoint subsets, it flags the existence of what we can call `internal' conflict.
\begin{definition}
The \emph{weight of internal conflict} $\mathcal{W}_b$ for a separable support function $b$ is defined as:
\begin{itemize}
\item 0 if $b$ is a simple support function;
\item $\inf \mathcal{K}(b_1,...,b_n)$ for the various possible decompositions of $b$ into simple support functions $b = b_1 \oplus \cdots \oplus b_n$ if $b$ is not simple.
\end{itemize}
\end{definition}
\noindent It is easy to see (see \cite{Shafer76} again) that $\mathcal{W}_b = \mathcal{K}(b_1,...,b_n)$ where
$b_1\oplus \cdots \oplus b_n$ is the canonical decomposition of $b$.

\section{Families of compatible frames of discernment} \label{sec:families-of-frames}

\subsection{Refinings}

One appealing idea in the theory of evidence is the simple, sensible claim that our knowledge of any given problem is inherently imperfect and imprecise. As a consequence, new evidence may allow us to make decisions on more detailed decision spaces (represented by frames of discernments). All these frames need to be `compatible' with each other, in a sense that we will precise in the following.\\ One frame can certainly be assumed compatible with another if it can be obtained by introducing new distinctions, i.e., by analyzing or splitting some of its possible outcomes into finer ones. This idea is embodied by the notion of \emph{refining}.
\begin{definition} \label{def:refining}
Given two frames of discernment $\Theta$ and $\Omega$, a map $\rho :2^\Theta \rightarrow2^\Omega$ is said to be a \emph{refining} if it
satisfies the following conditions:
\begin{enumerate}
\item $\rho(\{\theta\})\neq\emptyset\;\forall \theta\in\Theta$;
\item $\rho(\{\theta\})\cap\rho(\{\theta'\})=\emptyset\; if\;
\theta\neq\theta'$;
\item $\displaystyle \cup_{\theta\in\Theta}\rho(\{\theta\})=\Omega$.
\end{enumerate}
\end{definition}
In other words, a refining maps the coarser frame $\Theta$ to a disjoint partition of the finer one $\Omega$ (see Figure \ref{fig:refining}).
\begin{figure}[ht!]
\begin{center}
\includegraphics[width = 0.48 \textwidth]{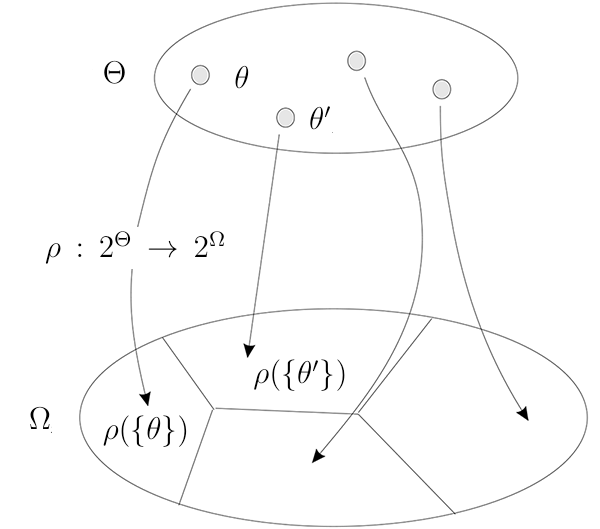}
\caption{\label{fig:refining} A refining between two frames of discernment.}
\end{center}
\end{figure}

\noindent The finer frame is called a \emph{refinement} of the first one, and we call $\Theta$ a \emph{coarsening} of $\Omega$. Both frames represent sets of admissible answers to a given decision problem (see Chapter \ref{cha:state} as well) -- the finer one is nevertheless a more detailed description, obtained by splitting each possible answer $\theta \in \Theta$ in the original frame. The image $\rho(A)$ of a subset $A$ of $\Theta$ consists of all the outcomes in $\Omega$ that are obtained by splitting an element of $A$.

Proposition \ref{pro:refining} lists some of the properties of refinings \cite{Shafer76}.
\begin{proposition} \label{pro:refining}
Suppose $\rho:2^{\Theta}\rightarrow 2^{\Omega}$ is a refining. Then
\begin{itemize}
\item $\rho$ is a one-to-one mapping;
\item $\rho(\emptyset)=\emptyset$;
\item $\rho(\Theta)=\Omega$;
\item $\rho(A\cup B)=\rho(A)\cup\rho(B)$;
\item $\rho(\bar{A})=\overline{\rho(A)}$;
\item $\rho(A\cap B)=\rho(A)\cap \rho(B)$;
\item if $A,B \subset \Theta$ then $\rho(A)\subset\rho(B)$ iff $A\subset B$;
\item if $A,B \subset \Theta$ then $\rho(A)\cap\rho(B)=\emptyset$ iff $A\cap B=\emptyset$.
\end{itemize}
\end{proposition}
\noindent A refining $\rho:2^{\Theta}\rightarrow 2^{\Omega}$ is not, in general, onto; in other words, there are
subsets $B\subset \Omega$ that are not images of subsets $A$ of $\Theta$. Nevertheless, we can define two different ways of associating each subset of the more refined frame $\Omega$ with a subset of the coarser one $\Theta$.
\begin{definition}
The \emph{inner reduction} associated with a refining $\rho:2^{\Theta}\rightarrow 2^{\Omega}$ is the map
$\underline{\rho} : 2^\Omega \rightarrow 2^\Theta$ defined as:
\begin{equation} \label{eq:inner-reduction}
\underline{\rho}(A) = \Big \{ \theta\in\Theta \Big | \rho(\{\theta\}) \subseteq A \Big \}.
\end{equation}
The \emph{outer reduction} associated with $\rho$ is the map $\bar{\rho} : 2^\Omega \rightarrow 2^\Theta$
given by:
\begin{equation} \label{eq:outer-reduction}
\bar{\rho}(A) = \Big \{ \theta\in\Theta \Big | \rho(\{\theta\}) \cap A \neq \emptyset \Big \}.
\end{equation}
\end{definition}
Roughly speaking, $\underline{\rho}(A)$ is {the largest subset of $\Theta$ that implies $A\subset \Omega$}, while $\bar{\rho}(A)$ is {the smallest subset of $\Theta$ that is implied by $A$}. As a matter of fact:
\begin{proposition}
Suppose $\rho:2^{\Theta}\rightarrow 2^{\Omega}$ is a refining, $A\subset \Omega$ and $B\subset\Theta$. Let $\bar{\rho}$ and $\underline{\rho}$ the related outer and inner reductions. Then $\rho(B)\subset A$ iff $B\subset \underline{\rho}(A)$, and $A\subset \rho(B)$ iff $\bar{\rho}(A)\subset B$.
\end{proposition}

\subsection{Families of frames}

The existence of distinct admissible descriptions at different levels of granularity of a same phenomenon is encoded in the theory of evidence by the concept of \emph{family of compatible frames} (see \cite{Shafer76}, pages 121-125), whose building block is the notion of refining (Definition \ref{def:refining}).

\begin{definition}\label{def:1}
A non-empty collection of finite non-empty sets $\mathcal{F}$ is a \emph{family of compatible frames of discernment} with refinings $\mathcal{R}$, where $\mathcal{R}$ is a non-empty collection of refinings between pairs of frames in $\mathcal{F}$, if $\mathcal{F}$ and $\mathcal{R}$ satisfy the following requirements:
\begin{enumerate}
\item 
\emph{composition of refinings}: if $\rho_1:2^{\Theta_1}\rightarrow2^{\Theta_2}$ and $\rho_2:2^{\Theta_2}\rightarrow2^{\Theta_3}$ are in $\mathcal{R}$, then $\rho_2 \circ \rho_1 : 2^{\Theta_1} \rightarrow 2^{\Theta_3}$ is in $\mathcal{R}$;
\item 
\emph{identity of coarsenings}: if $\rho_1:2^{\Theta_1}\rightarrow2^{\Omega}$, $\rho_2:2^{\Theta_2}\rightarrow2^{\Omega}$ are in $\mathcal{R}$ and $\forall \theta_1\in\Theta_1$ $\exists \theta_2\in\Theta_2$ such that $\rho_1(\{\theta_1\}) = \rho_2(\{\theta_2\})$, then $\Theta_1=\Theta_2$ and $\rho_1=\rho_2$;
\item 
\emph{identity of refinings}: if $\rho_1:2^{\Theta}\rightarrow2^{\Omega}$ and $\rho_2:2^{\Theta}\rightarrow2^{\Omega}$ are in $\mathcal{R}$, then $\rho_1=\rho_2$;
\item 
\emph{existence of coarsenings}: if $\Omega\in\mathcal{F}$ and $A_1,...,A_n$ is a disjoint partition of $\Omega$ then there is a coarsening in $\mathcal{F}$ corresponding to this partition;
\item 
\emph{existence of refinings}: if $\theta\in\Theta\in\mathcal{F}$ and $n\in \mathbb{N}$ then there exists a refining $\rho:2^{\Theta}\rightarrow2^{\Omega}$ in $\mathcal{R}$ and $\Omega\in{\mathcal{F}}$ such that $\rho(\{\theta\})$ has $n$ elements;
\item 
\emph{existence of common refinements}: every pair of elements in $\mathcal{F}$ has a common refinement in $\mathcal{F}$.
\end{enumerate}
\end{definition}
Roughly speaking, two frames are compatible if and only if they concern propositions which can be both expressed in terms of propositions of a common, finer frame.\\ By property (6) each collection of compatible frames has many common refinements. One of these is particularly simple.
\begin{theorem}\label{the:minimal}
If $\Theta_1,...,\Theta_n$ are elements of a family of compatible frames $\mathcal{F}$, then there exists a unique frame $\Theta\in\mathcal{F}$ such that:
\begin{enumerate}
\item $\exists$ a refining $\rho_i:2^{\Theta_i}\rightarrow2^\Omega$ for all $i=1,...,n$;
\item $\forall \theta\in\Theta \; \exists \; \theta_i\in\Theta_i \;
for\; i=1,...,n \; such \; that$
\[
\{\theta\}=\rho_1(\{\theta_1\})\cap...\cap\rho_n(\{\theta_n\}).
\]
\end{enumerate}
\end{theorem}
This unique frame is called the \emph{minimal refinement} $\Theta_1 \otimes \cdots \otimes \Theta_n$ of the collection
$\Theta_1,...,\Theta_n$, and is the simplest space in which we can compare propositions pertaining to different compatible frames. Furthermore:
\begin{proposition}\label{pro:minimal}
If $\Omega$ is a common refinement of $\Theta_1,...,\Theta_n$, then $\Theta_1\otimes\cdots\otimes\Theta_n$ is a coarsening of $\Omega$. Furthermore, $\Theta_1\otimes\cdots\otimes\Theta_n$ is the only common refinement of $\Theta_1,...,\Theta_n$ that is a coarsening of every other common refinement.
\end{proposition}

\subsubsection{Example: number systems}

Figure \ref{fig:familyexample} illustrates a simple example of compatible frames. A real number $r$ between 0 and 1 can be expressed, for instance, using either binary or base-5 digits. Furthermore, even within a number system of choice (for example the binary one), the real number can be represented with different degrees of approximation, using for instance one or two digits. Each of these quantized versions of $r$ is associated with an interval of $[0,1]$ (red rectangles) and can be expressed in a common frame (their common refinement, Definition \ref{def:1}, property (6)), for example by selecting a 2-digit decimal approximation.\\ Refining maps between coarser and finer frames are easily interpreted, and are depicted in Figure \ref{fig:familyexample}.

\begin{figure}[ht!]
\begin{center}
\includegraphics[width = 0.55 \textwidth]{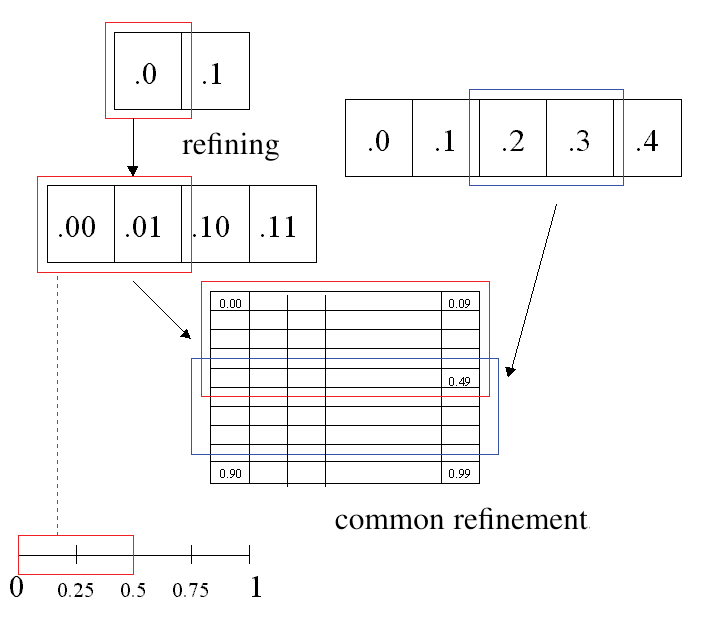}
\caption{\label{fig:familyexample} The different digital representations of the same real number $r \in [0,1]$ constitute a simple example of family of compatible frames.}
\end{center}
\end{figure}

\subsection{Consistent belief functions}

If $\Theta_1$ and $\Theta_2$ are two compatible frames, then two belief functions $b_1 : 2^{\Theta_1}\rightarrow [0,1]$, $b_2 : 2^{\Theta_2}\rightarrow [0,1]$ can potentially be expression of the same body of evidence. This is the case only if $b_1$ and $b_2$ agree on those propositions that are discerned by both $\Theta_1$ and $\Theta_2$, i.e., they represent the same subset of their minimal refinement.
\begin{definition}
Two belief functions $b_1$ and $b_2$ defined over two compatible frames $\Theta_1$ and $\Theta_2$ are said to be \emph{consistent} if
\[
b_1(A_1) = b_2(A_2)
\]
whenever
\[
\begin{array}{ll}
A_1\subset\Theta_1, \; A_2\subset\Theta_2 \; and \; \rho_1(A_1)=\rho_2(A_2), & \rho_i:2^{\Theta_i}\rightarrow
2^{\Theta_1\otimes \Theta_2},
\end{array}
\]
where $\rho_i$ is the refining between $\Theta_i$ and the minimal refinement $\Theta_1\otimes \Theta_2$ of $\Theta_1$ and $\Theta_2$.
\end{definition}

A special case is that in which the two belief functions are defined on frames connected by a refining $\rho : 2^{\Theta_1} \rightarrow 2^{\Theta_2}$ (i.e., $\Theta_2$ is a refinement of $\Theta_1$). In this case $b_1$ and $b_2$ are consistent iff:
\[
b_1(A) = b_2(\rho(A)), \;\;\; \forall A\subseteq \Theta_1.
\]
The b.f. $b_1$ is called the \emph{restriction} of $b_2$ to $\Theta_1$, and their mass values are in the following relation:
\begin{equation} \label{eq:restriction}
m_1(A) = \sum_{A = \bar{\rho}(B)} m_2(B),
\end{equation}
where $A \subset \Theta_1$, $B \subset \Theta_2$ and $\bar{\rho}(B) \subset \Theta_1$ is the inner reduction (\ref{eq:inner-reduction}) of $B$.

\subsection{Independent frames}

Two compatible frames of discernment are \emph{independent} if no proposition discerned by one of them trivially implies a proposition discerned by the other. Obviously we need to refer to a common frame: by Proposition \ref{pro:minimal} what common refinement we choose is immaterial.

\begin{definition}\label{def:indep}
Let $\Theta_1,...,\Theta_n$ be compatible frames, and $\rho_i:2^{\Theta_i}\rightarrow
2^{\Theta_1\otimes\cdots\otimes \Theta_n}$ the corresponding refinings to their minimal refinement. The frames $\Theta_1,...,\Theta_n$ are said to be \emph{independent} if
\begin{equation}
\label{eq:2}\rho_1(A_1)\cap\cdots\cap\rho_n(A_n)\neq \emptyset
\end{equation}
whenever $\emptyset\neq A_i\subset\Theta_i$ for $i=1,...,n$.
\end{definition}

Equivalently, condition (\ref{eq:2}) can be expressed as follows:
\begin{itemize}
\item if $A_i\subset\Theta_i$ for $i=1,...,n$ and $\rho_1(A_1)\cap \cdots \cap \rho_{n-1}(A_{n-1}) \subset \rho_n(A_n)$ then $A_n = \Theta_n$ or one of the first $n-1$ subsets $A_i$ is empty.
\end{itemize}
The notion of independence of frames is illustrated in Figure \ref{fig:if}.

In particular, it is easy to see that if $\exists j \in [1,..,n]$ s.t. $\Theta_j$ is a coarsening of some other frame $\Theta_i$, $|\Theta_j|>1$, then $\{ \Theta_1,...,\Theta_n \}$ are \emph{not} independent. Mathematically, families of compatible frames are collections of Boolean subalgebras of their common refinement \cite{Sikorski}, as Equation (\ref{eq:2}) is nothing but the independence condition for the associated Boolean sub-algebras
\footnote{The following material comes from \cite{Sikorski}.
\begin{definition}
A \emph{Boolean algebra} is a non-empty set $\mathcal{U}$ provided with three internal operations
\[
\begin{array}{ccc}
\begin{array}{cccc}
\cap : & {\mathcal{U}}\times {\mathcal{U}} & \longrightarrow & {\mathcal{U}}\\ & A,B & \mapsto & A\cap B
\end{array}\begin{array}{cccc} \cup : & {\mathcal{U}}\times {\mathcal{U}} & \longrightarrow & {\mathcal{U}}\\ & A,B &
\mapsto & A\cup B \end{array}\begin{array}{cccc} \neg : & {\mathcal{U}} & \longrightarrow & {\mathcal{U}}\\ & A &
\mapsto & \neg A\end{array}
\end{array}
\]
called respectively \emph{meet}, \emph{join} and \emph{complement}, characterized by the following properties:
\[
\begin{array}{cc}
A \cup B = B \cup A, & A \cap B = B \cap A \\ \\ A \cup (B \cup C)= (A \cup B) \cup C, & A \cap (B \cap C)= (A \cap B) \cap C \\ \\ (A \cap B) \cup B = B, & (A \cup B) \cap B = B \\ \\ A \cap (B \cup C)= (A \cap B)\cup (A \cap C),& A \cup (B \cap C)= (A \cup B)\cap (A \cup C) \\ \\ (A \cap \neg A)\cup B = B,& (A \cup \neg A)\cap B = B
\end{array}
\]
\end{definition}

As a special case, the collection $(2^S, \subset)$ of all the subsets of a given set $S$ is a Boolean algebra.

\begin{definition} \label{def:boolean-subalgebras}
$\mathcal{U}'$ is a subalgebra of a Boolean algebra $\mathcal{U}$ iff whenever $A, B \in \mathcal{U}'$ it follows that $A \cup B, A \cap B$ and $\neg A$ are all in $\mathcal{U}'$.
\end{definition}

The `zero' of a Boolean algebra $\mathcal{U}$ is defined as: $0 = \cap_{A \in \mathcal{U}} A$.

\begin{definition}
A collection $\{{\mathcal{U}}_t\}_{t\in T}$ of subalgebras of a Boolean algebra ${\mathcal{U}}$ is said to be \emph{independent} if
\begin{equation} \label{eq:independence-boolean}
A_1\cap\cdots\cap A_n \neq 0
\end{equation}
whenever $0 \neq A_j\in{\mathcal{U}}_{t_j}$, $t_j\neq t_k$ for $j\neq k$.
\end{definition}
Compare expressions (\ref{eq:independence-boolean}) and (\ref{eq:2}).
}.

\begin{figure}[ht!]
\centering
\includegraphics[scale = 0.4]{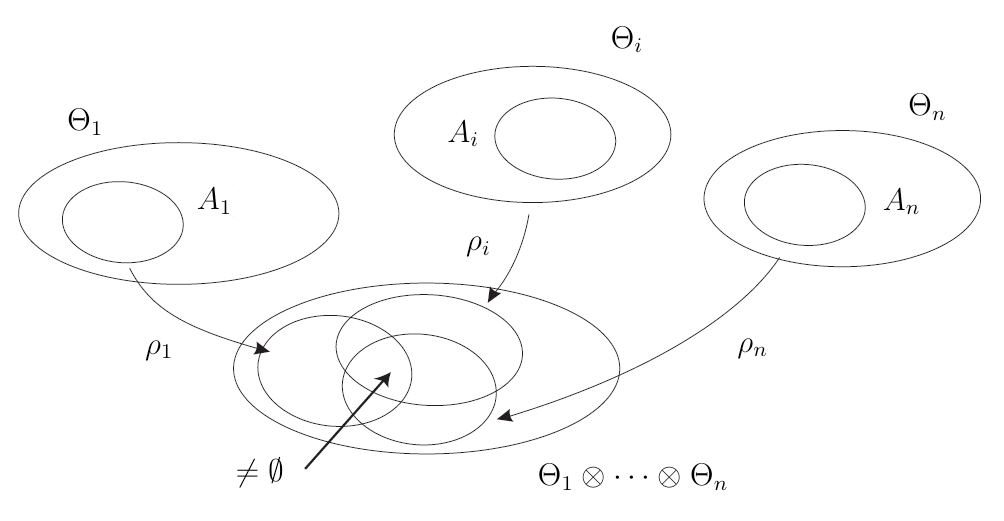}
\caption{\label{fig:if} Independence of frames.}
\end{figure}

\section{Support functions} 

\noindent Since Dempster's rule of combination is applicable only to set functions satisfying the axioms of belief functions (Definition \ref{def:bel1}), we are tempted to think that the class of separable belief functions is sufficiently large to describe the impact of a body of evidence on any frame of a family of compatible frames. This is, however, not the case as \emph{not all} belief functions are separable ones. 

Let us consider a body of evidence inducing a separable b.f. $b$ over a certain frame $\Theta$ of a family $\mathcal{F}$: the `impact' of this evidence onto a coarsening $\Omega$ of $\Theta$ is naturally described by the restriction $b | 2^{\Omega}$ of $b$ (Equation \ref{eq:restriction}) to $\Omega$.

\begin{definition}\label{def:support}
A belief function $b : 2^{\Theta}\rightarrow [0,1]$ is a \emph{support function} if there exists a refinement $\Omega$ of $\Theta$ and a separable support function $b' : 2^{\Omega}\rightarrow [0,1]$ such that $b = b'|2^{\Theta}$.
\end{definition}

In other words, a support function \cite{kohlas93c} is the restriction of some separable support function.\\ As it can be expected, not all support functions are separable support functions. The following Proposition gives us a simple equivalent condition.
\begin{proposition}\label{pro:support}
Suppose $b$ is a belief function, and $\mathcal{C}$ its core. The following conditions are equivalent:
\begin{itemize}
\item $b$ is a support function;
\item $\mathcal{C}$ has a positive basic probability number, $m(\mathcal{C})>0$.
\end{itemize}
\end{proposition}
Since there exist belief functions whose core has mass zero, Proposition \ref{pro:support} tells us that not all the belief functions are support ones (see Section \ref{sec:quasi-support}).

\subsection{Vacuous extension}

There are occasions in which the impact of a body of evidence on a frame $\Theta$ is fully discerned by one of its coarsening $\Omega$, i.e., no proposition discerned by $\Theta$ receives greater support than what is implied by propositions discerned by $\Omega$.
\begin{definition}\label{def:vacuous}
A belief function $b : 2^{\Theta}\rightarrow [0,1]$ on $\Theta$ is the \emph{vacuous extension} of a second belief function $b_0 : 2^{\Omega}\rightarrow [0,1]$, where $\Omega$ is a coarsening of $\Theta$, whenever:
\[
b(A) = \max_{B \subset \Omega,\; \rho(B) \subseteq A} b_0 (B) \hspace{5mm} \forall A \subseteq \Theta.
\]
\end{definition}
We say that $b$ is `carried' by the coarsening $\Omega$. We will make use of this all important notion in our treatment of two computer vision problems in Part III, Chapter \ref{cha:total} and Chapter \ref{cha:pose}.

\section{Impact of the evidence} \label{sec:impact}

\subsection{Families of compatible support functions}

In its 1976 essay \cite{Shafer76} Glenn Shafer distinguishes between a `subjective' and an `evidential' vocabulary, keeping distinct objects with the same mathematical description but different philosophical interpretations.

Each body of evidence $\mathcal{E}$ supporting a belief function $b$ (see \cite{Shafer76}) simultaneously affects the whole family $\mathcal{F}$ of compatible frames of discernment the domain of $b$ belongs to, determining a support function over every element of $\mathcal{F}$. We say that $\mathcal{E}$ determines a \emph{family of compatible support functions} $\{ s^{\Theta}_{\mathcal{E}}\}_{\Theta\in\mathcal{F}}$.\\ The complexity of this family depends on the following property.

\begin{definition}
The evidence $\mathcal{E}$ \emph{affects $\mathcal{F}$ sharply} if there exists a frame $\Omega \in \mathcal{F}$ that carries $s^{\Theta}_{\mathcal{E}}$ for every $\Theta \in \mathcal{F}$ that is a refinement of $\Omega$. Such a frame $\Omega$ is said to \emph{exhaust the impact} of $\mathcal{E}$ on $\mathcal{F}$.
\end{definition}
Whenever $\Omega$ exhausts the impact of $\mathcal{E}$ on $\mathcal{F}$, $s^{\Omega}_{\mathcal{E}}$ determines the whole family $\{ s^{\Theta}_{\mathcal{E}} \}_{\Theta \in \mathcal{F}}$, for any support function over any given frame $\Theta \in\mathcal{F}$ is the restriction to $\Theta$ of $s^{\Omega}_{\mathcal{E}}$'s vacuous extension (Definition \ref{def:vacuous}) to $\Theta \otimes \Omega$.\\ A typical example in which the evidence affects the family sharply is statistical evidence, in which case both frames and evidence are highly idealized \cite{Shafer76}.

\subsection{Discerning the interaction of evidence}

It is almost a commonplace to affirm that, by selecting particular inferences from a body of evidence and combining them with particular inferences from another body of evidence, one can derive almost arbitrary conclusions. In the evidential framework, in particular, it has been noted that Dempster's rule {may produce inaccurate results when applied to `inadequate' frames of discernment}.

Namely, let us consider a frame $\Theta$, its coarsening $\Omega$, and a pair of support functions $s_1, s_2$ on $\Theta$ determined by two distinct bodies of evidence. Applying Dempster's rule directly on $\Theta$ yields the following support function on its coarsening $\Omega$:
\[
(s_1\oplus s_2)|2^{\Omega},
\]
while its application on the coarser frame $\Theta$ after computing the restrictions of $s_1$ and $s_2$ to it yields:
\[
(s_1 | 2^{\Omega}) \oplus (s_2|2^{\Omega}).
\]
In general, the outcomes of these two combination strategies will be different. Nevertheless, a condition on the refining linking $\Omega$ to $\Theta$ can be imposed which guarantees their equivalence.

\begin{proposition}
Assume that $s_1$ and $s_2$ are support functions over a frame $\Theta$, their Dempster's combination $s_1\oplus s_2$ exists, $\bar{\rho}:2^{\Theta}\rightarrow 2^{\Omega}$ is an outer reduction, and
\begin{equation}
\bar{\rho}(A\cap B) = \bar{\rho}(A)\cap \bar{\rho}(B)
\end{equation}
holds wherever $A$ is a focal element of $s_1$ and $B$ is a focal element of $s_2$. Then
\[
(s_1\oplus s_2)|2^{\Omega} = (s_1|2^{\Omega}) \oplus (s_2|2^{\Omega}).
\]
\end{proposition}
In this case $\Omega$ is said to \emph{discern the relevant interaction} of $s_1$ and $s_2$. Of course if $s_1$ and $s_2$ are carried by a coarsening of $\Theta$ then this latter frame discerns their relevant interaction.

The above definition generalizes to entire bodies of evidence.

\begin{definition}
Suppose $\mathcal{F}$ is a family of compatible frames, $\{ s^{\Theta}_{\mathcal{E}_1}\}_{\Theta \in \mathcal{F}}$ is the family of support functions determined by a body of evidence $\mathcal{E}_1$, and $\{ s^{\Theta}_{\mathcal{E}_2}\}_{\Theta \in \mathcal{F}}$ is the family of support functions determined by a second body of evidence $\mathcal{E}_2$.\\ Then, a particular frame $\Omega \in \mathcal{F}$ is said to discern the relevant interaction of $\mathcal{E}_1$ and $\mathcal{E}_2$ if:
\[
\bar{\rho}(A\cap B) = \bar{\rho}(A)\cap \bar{\rho}(B)
\]
whenever $\Theta$ is a refinement of $\Omega$, where $\bar{\rho}:2^{\Theta}\rightarrow 2^{\Omega}$ is the associated outer reduction, $A$ is a focal element of $s_{\mathcal{E}_1}^{\Theta}$ and $B$ is a focal element of $s_{\mathcal{E}_2}^{\Theta}$.
\end{definition}

\section{Quasi support functions} \label{sec:quasi-support}

Not every belief function is a support function. The question remains of how to characterise in a precise way the class of belief functions which are \emph{not} support functions.\\ Let us consider a finite power set $2^{\Theta}$. A sequence $f_1,f_2,...$ of set functions on $2^{\Theta}$ is said to tend to a limit function $f$ if
\begin{equation} \label{eq:limit}
\lim_{i\rightarrow \infty} f_i(A)=f(A)\;\;\;\forall A\subset \Theta.
\end{equation} 
It can be proved that \cite{Shafer76}:
\begin{proposition} \label{pro:limit}
If a sequence of belief functions has a limit, then the limit is itself a belief function.
\end{proposition}
In other words, {the class of belief functions is closed} with respect to the limit operator (\ref{eq:limit}). The latter provides us with an insight into the nature of non-support functions.
\begin{proposition}
If a belief function $b:2^\Theta \rightarrow [0,1]$ is not a support function, then there exists a refinement $\Omega$ of $\Theta$ and a sequence $s_1,s_2,...$ of separable support functions over $\Omega$ such that:
\[
b = \Big ( \lim_{i\rightarrow \infty} s_i \Big ) \Big | 2^{\Theta}.
\]
\end{proposition}

\begin{definition}
We call belief functions of this class \emph{quasi-support functions}.
\end{definition}
It should be noted that
\[
\Big ( \lim_{i\rightarrow \infty} s_i \Big ) \Big | 2^{\Theta} = \lim_{i\rightarrow \infty} (s_i|2^{\Theta}),
\]
so that we can also say that $s$ is a limit of a sequence of support functions.

The following proposition investigates some of the properties of quasi-support functions.
\begin{proposition}\label{pro:negation}
Suppose $b: 2^\Theta \rightarrow [0,1]$ is a belief function over $\Theta$, and $A\subset \Theta$ a subset of $\Theta$. If $b(A)>0$ and $b(\bar{A})>0$, with $b(A) + b(\bar{A}) = 1$, then $b$ is a quasi-support function.
\end{proposition}
It easily follows that Bayesian b.f.s are quasi-support functions, unless they commit all their probability mass to a single element of the frame.
\begin{proposition}
A Bayesian belief function $b$ is a support function iff there exists $\theta\in\Theta$ such that
$b(\{\theta\})=1$.
\end{proposition}
Furthermore, it is easy to see that vacuous extensions of Bayesian belief functions are also quasi-support functions.

As Shafer remarks, people used to think of beliefs as chances can be disappointed to see them relegated to a peripheral role, as beliefs that cannot arise from actual, finite evidence. On the other hand, statistical inference already teaches us that chances can be evaluated only after infinitely many repetitions of independent random experiments.\footnote{Using the notion of \emph{weight of evidence} Shafer gives a formal explanation of this intuitive observation by showing that a Bayesian b.f. indicates an \emph{infinite} amount of evidence in favor of \emph{each} possibility in its core \cite{Shafer76}.}

\subsection{Bayes' theorem}

Indeed, as it commits an infinite amount of evidence in favor of each possible element of a frame of discernment, a Bayesian belief function tends to obscure much of the evidence additional belief functions may carry with them.
\begin{definition} \label{def:relative-plausibilities}
A function $l:\Theta\rightarrow [0,\infty)$ is said to express the \emph{relative plausibilities} of singletons
under a support function $s : 2^\Theta \rightarrow [0,1]$ if
\[
l(\theta) = c\cdot pl_s(\{\theta\})
\]
for all $\theta \in \Theta$, where $pl_s$ is the plausibility function for $s$ and the constant $c$ does not depend on
$\theta$.
\end{definition}
\begin{proposition}\label{pro:bayes} (\textbf{Bayes' theorem})
Suppose $b_0$ and $s$ are a Bayesian belief function and a support function on the same frame $\Theta$, respectively. Suppose $l:\Theta\rightarrow [0,\infty)$ expresses the relative plausibilities of singletons under $s$. Suppose also that their Dempster's sum $b' = s \oplus b_0$ exists. Then $b'$ is Bayesian, and
\[
b'(\{\theta\}) = K \cdot b_0(\{\theta\}) l(\theta) \hspace{5mm} \forall \theta\in\Theta,
\]
where
\[
K = \bigg( \sum_{\theta \in \Theta} b_0(\{\theta\}) l(\theta) \bigg )^{-1}.
\]
\end{proposition}
This implies that the combination of a Bayesian b.f. with a support function requires nothing more than the latter's relative plausibilities of singletons.\\ It is interesting to note that the latter functions behave multiplicatively under combination, 
\begin{proposition} If $s_1,...,s_n$ are combinable support functions, and $l_i$ represents the relative plausibilities of singletons under $s_i$ for $i=1,...,n$, then $l_1\cdot l_2\cdot \cdots \cdot l_n$ expresses the relative
plausibilities of singletons under $s_1\oplus\cdots \oplus s_n$.
\end{proposition}
providing a simple algorithm to combine any number of support functions with a Bayesian b.f.

\subsection{Incompatible priors}\label{sec:priors}

Having an established convention on how to set a Bayesian {prior} would be useful, as it would prevent us from making arbitrary and possibly unsupported choices that could eventually affect the final result of our inference process. Unfortunately, the only natural such convention (a uniform prior) is strongly dependent on the frame of discernment at hand, and is sensitive to both refining and coarsening operators.

More precisely, on a frame $\Theta$ with $n$ elements it is natural to represent our ignorance by adopting an uninformative uniform prior assigning a mass $1/n$ to every outcome $\theta\in\Theta$. However, the same convention applied to a different compatible frame $\Omega$ of the same family may yield a prior that is incompatible with the first one. As a result, the combination of a given body of evidence with one arbitrary such prior can yield almost any possible result \cite{Shafer76}.

\subsubsection{Example: Sirius' planets}

A team of scientists wonder whether there is life around Sirius. Since they do not have any evidence concerning this question, they adopt a vacuous belief function to represent their ignorance on the frame
\[
\Theta = \{ \theta_1,\theta_2 \},
\]
where $\theta_1,\theta_2$ are the answers ``there is life" and ``there is no life". They can also consider the
question in the context of a more refined set of possibilities. For example, our scientists may raise the question of
whether there even exist planets around Sirius. In this case the set of possibilities becomes
\[
\Omega = \{ \zeta_1,\zeta_2,\zeta_3 \},
\]
where $\zeta_1,\zeta_2,\zeta_3$ are respectively the possibility that there is life around Sirius, that there are
planets but no life, and there are no planets at all. Obviously, in an evidential setup our ignorance still needs to be represented by a vacuous belief function, which is exactly the vacuous extension of the vacuous b.f. previously defined on $\Theta$.

From a Bayesian point of view, instead, it is difficult to assign consistent degrees of belief over $\Omega$ and $\Theta$ both symbolizing the lack of evidence. Indeed, on $\Theta$ a uniform prior yields $p(\{\theta_1\}) = p(\{\theta_1\}) = 1/2$, while on $\Omega$ the same choice will yield $p'(\{\zeta_1\})=p'(\{\zeta_2\})=p'(\{\zeta_3\})=1/3$. $\Omega$ and $\Theta$ are obviously compatible (as the former is a refinement of the latter): the vacuous extension of $p$ onto $\Omega$ produces a Bayesian distribution
\[
p(\{\zeta_1\}) = 1/3,\hspace{5mm} p(\{\zeta_1,\zeta_2\}) = 2/3
\]
which is inconsistent with $p'$!

\section{Consonant belief functions} \label{sec:consonant}

To conclude this brief review of evidence theory we wish to recall a class of belief functions which is, in some sense, opposed to that quasi-support functions -- that of \emph{consonant} belief functions.
\begin{definition} \label{def:consonant}
A belief function is said to be consonant if its focal elements $A_1,...,A_m$ are nested: $A_1 \subset A_2 \subset \cdots \subset A_m$.
\end{definition}
The following Proposition illustrates some of their properties.
\begin{proposition}\label{pro:cons}
If $b$ is a belief function with upper probability function $pl$, then the following conditions
are equivalent:
\begin{enumerate}
\item $b$ is consonant;
\item $b(A\cap B) = \min(b(A),b(B))$ for every $A,B\subset\Theta$;
\item $pl(A\cup B) = \max(pl(A),pl(B))$ for every $A,B\subset\Theta$;
\item $pl(A) = \max_{\theta\in A} pl(\{\theta\})$ for all non-empty $A \subset \Theta$;
\item there exists a positive integer $n$ and a collection of simple support functions $s_1,...,s_n$ such that $b = s_1 \oplus \cdots \oplus s_n$ and the focus of $s_i$ is contained in the focus of $s_j$ whenever $i<j$.
\end{enumerate}
\end{proposition}
Consonant b.f.s represent collections of pieces of evidence \emph{all pointing towards the same direction}. Moreover,
\begin{proposition}
Suppose $s_1,...,s_n$ are non-vacuous simple support functions with foci $\mathcal{C}_{s_1},...,\mathcal{C}_{s_n}$ respectively, and $b = s_1\oplus\cdots\oplus s_n$ is consonant. If $\mathcal{C}_b$ denotes the core of $b$, then all the sets $\mathcal{C}_{s_i}\cap \mathcal{C}_b$, $i=1,...,n$ are nested.
\end{proposition}
By condition (2) of Proposition \ref{pro:cons} we have that:
\[
0 = b(\emptyset) = b(A\cap \bar{A}) = min(b(A),b(\bar{A})),
\]
i.e., either $b(A) = 0$ or $b(\bar{A}) = 0$ for every $A\subset \Theta$. Comparing this result to Proposition \ref{pro:negation} explains in part why consonant and quasi-support functions can be considered as representing diametrically opposed subclasses of belief functions.

%\newpage
%\thispagestyle{empty}

\chapter{State of the art} \label{cha:state}

\begin{center}
\includegraphics[width = 0.5 \textwidth]{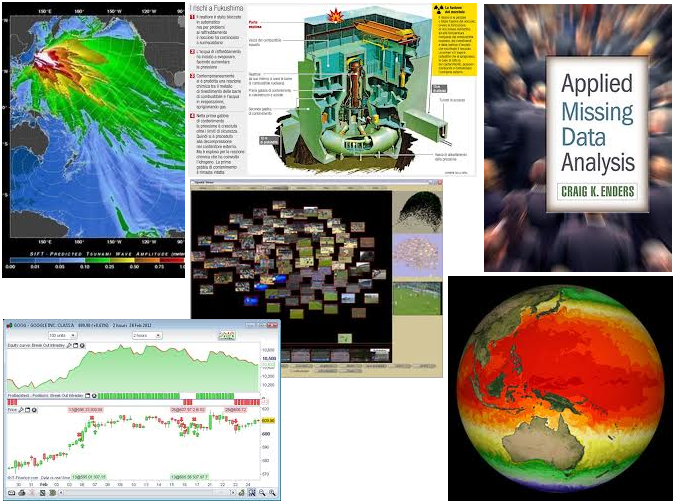}
\end{center}
\vspace{5mm}

It the almost forty years since its formulation the theory of evidence has obviously evolved quite substantially, thanks to the work of several talented researchers \cite{hajek92state}, and now this denomination refers to a number of slightly different interpretations of the idea of generalized probability. Some people have proposed their own framework as the `correct version of the evidential reasoning, partly in response to strong criticisms brought forward by important scientists (compare for instance Judea Pearl's contribution in \cite{pearl89reasoning}, later recalled in \cite{lee88comparison}, and \cite{Pearl90}). Several generalizations of the initial finite-space formulation to continuous frames of discernment have been proposed \cite{yen92computing}, although none of them has been yet acknowledged as `the' ultimate answer to the limitations of Shafer's original formulation. 

In the same period of time, the number of applications of the theory of evidence to engineering \cite{getler92failure,rakar99transferable,demotier06smcc} and applied sciences \cite{lesh86evidential,boucher90speech,cleynenbreugel91road} has been steadily increasing: its diffusion, however, is still relatively limited when compared to that of classical probability \cite{Fua86} or fuzzy methods \cite{cross91compatibilityIEEE}. A good (albeit outdated) survey on the topic, from a rather original point of view, can be found in \cite{Kohlas94}. Another comparative review about texts on evidence theory is presented in \cite{ramer96comparative}.

\subsection*{Scope of the Chapter}

In this Chapter, we wish to give a more up-to-date survey of the current state of the art of the theory of evidence, including the various associated frameworks proposed during the years, the theoretical advances achieved since its inception, and the algorithmic schemes \cite{kofler94algorithmic} (mainly based on propagation networks \cite{dempster90construction}) proposed to cope with the inherent computational complexity which comes with dealing with an exponential number of focal elements \cite{wang92exponential}. Many interesting new results have been achieved of late\footnote{The work of Roesmer \cite{roesmer00nonstandard} deserves a special mention for its original connection between nonstandard analysis and theory of evidence.}, showing that the discipline is evolving towards maturity. Here we would just like to briefly mention some of those results concerning major open problems in belief calculus, in order to put into context  the work we ourselves are going to develop in Part II. The most accredited approaches to decision making and inference with belief functions are also reviewed \cite{jaffray89coherent}, and a brief overview of the various proposals for a continuous generalization of Shafer's belief functions is given. The relationships between Dempster-Shafer theory and other related uncertainty theories are shortly outlined, and a very limited sample of the various applications of belief calculus to the most disparate fields is discussed.

\section{The alternative interpretations of belief functions} \label{sec:origins}

The axiomatic set up that Shafer gave originally to his work could seem quite arbitrary at a first glance \cite{shenoy90axioms,klawonn90axiomatic}. For example, Dempster's rule \cite{Ginsberg84} is not really given a convincing justification in his seminal book \cite{Shafer76}, leaving the reader wondering whether a different rule of combination could be chosen instead \cite{zadeh86simple,fagin91new,dasilva92algorithms,sudkamp92consistency,smets86bayes}. This question has been posed by several authors (e.g. \cite{Voorbraak91}, \cite{Zadeh86}, \cite{Shafer86} and \cite{Wilson92} among the others), most of whom tried to provide an axiomatic support to the choice of this mechanism for combining evidence.

\subsection{Upper and lower probabilities, multi-valued mappings and compatibility relations} \label{sec:multivalued}

As a matter of fact the notion of belief function \cite{Shafer81,Shafer87d} originally derives from a series of Dempster's works on upper and lower probabilities induced by multi-valued mappings, introduced in \cite{Dempster67}, \cite{Dempster68a} and \cite{Dempster69}. Shafer later reformulated Dempster's work by identifying his upper and lower probabilities with epistemic probabilities or `degrees of belief', i.e., the quantitative assessments of one's belief in a given fact or proposition. The following sketch of the nature of belief functions is abstracted from \cite{Shafer90}: another debate on the relation between b.f.s and upper and lower probabilities is developed in \cite{smets87versus}.

Let us consider a problem in which we have probabilities (coming from arbitrary sources, for instance subjective judgement or objective measurements) for a question $Q_1$ and we want to derive degrees of belief for a related question $Q_2$. For example, $Q_1$ could be the judgement on the reliability of a witness, and $Q_2$ the decision about the truth of the reported fact. In general, each question will have a number of possible answers, only one of them being correct.\\ Let us call $\Omega$ and $\Theta$ the sets of possible answers to $Q_1$ and $Q_2$ respectively. So, given a probability measure $P$ on $\Omega$ we want to derive a degree of belief $b(A)$ that $A\subset \Theta$ contains the correct response to $Q_2$ (see Figure \ref{fig:compat}).

\begin{figure}[ht!]
\begin{center}
\includegraphics[width = 0.55 \textwidth]{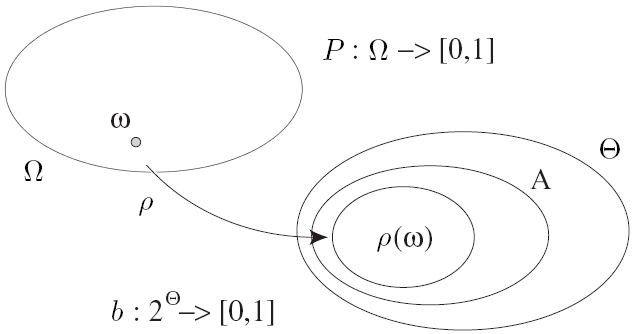} 
\caption{\label{fig:compat} Compatibility relations and multi-valued mappings. A probability measure $P$ on $\Omega$ induces a belief function $b$ on $\Theta$ whose values on the events $A$ of $\Theta$ are given by (\ref{eq:belvalue}).}
\end{center}
\end{figure}

If we call $\Gamma(\omega)$ the subset of answers to $Q_2$ compatible with $\omega\in\Omega$, each element $\omega$
tells us that the answer to $Q_2$ is somewhere in $A$ whenever
\[
\Gamma(\omega) \subset A.
\]
The degree of belief $b(A)$ of an event $A\subset\Theta$ is then the total probability (in $\Omega$) of all the answers
$\omega$ to $Q_1$ that satisfy the above condition, namely:
\begin{equation} \label{eq:belvalue}
b(A) = P(\{ \omega | \Gamma(\omega) \subset A \}).
\end{equation}
The map $\Gamma : \Omega \rightarrow 2^{\Theta}$ (where $2^\Theta$ denotes, as usual, the collection of subsets of $\Theta$) is
called a \emph{multi-valued mapping} from $\Omega$ to $\Theta$. Each of those mappings $\Gamma$, together with a probability measure $P$ on $\Omega$, induce a \emph{belief function} on $\Theta$:
\[
\begin{array}{lllll}
b & : & 2^\Theta & \rightarrow & [0,1]\\ & & A\subset\Theta & \mapsto & \displaystyle b(A) \doteq \sum_{\omega\in\Omega
: \Gamma(\omega)\subset A} P(\omega).
\end{array}
\]
Obviously a multi-valued mapping is equivalent to a \emph{relation}, i.e., a subset $C$ of $\Omega \times \Theta$. The
\emph{compatibility relation} associated with $\Gamma$:
\begin{equation} \label{eq:compat}
C = \{(\omega,\theta) | \theta \in \Gamma(\omega)\}
\end{equation}
describes indeeed the subset of answers $\theta$ in $\Theta$ compatible with a given $\omega\in\Omega$.

As Shafer himself admits in \cite{Shafer90}, compatibility relations are only a new name for multivalued mappings. Nevertheless, several authors (among whom Shafer \cite{Shafer87c}, Shafer and Srivastava \cite{Srivastava89}, Lowrance \cite{lowrance88automated} and Yager \cite{yagernonmonotonicity}) chose this approach to build the mathematics of belief functions.

\subsection{Random sets} \label{sec:random-sets}

Having a multi-valued mapping $\Gamma$, a straightforward step is to consider the probability value $P(\omega)$ as attached to the subset $\Gamma(\omega) \subset \Theta$: what we obtain is a \emph{random set} in $\Theta$, i.e., a probability measure on a collection of subsets (see \cite{goutsias97random,Goutsias98,matheronrandom} for the most complete introductions to the matter). The degree of belief $b(A)$ of an event $A$ then becomes the total probability that the random set is contained in $A$. This approach has been emphasized in particular by Nguyen (\cite{goodman85uncertainty}, \cite{Nguyen97,Nguyen78}) and Hestir \cite{Hestir91}, and resumed in \cite{Shafer87b}.

Consider a multi-valued mapping $\Gamma: \Omega \rightarrow 2^\Theta$. The \emph{lower inverse} of $\Gamma$ is defined as:
\begin{equation}\label{eq:lower}
\begin{array}{llll}
\Gamma_*: & 2^\Theta & \rightarrow & 2^\Omega\\ & A & \mapsto & \Gamma_*(A)\doteq \{ \omega \in \Omega : \Gamma(\omega)
\subset A, \Gamma(\omega) \neq \emptyset \},
\end{array}
\end{equation}
while its \emph{upper inverse} is
\begin{equation}\label{eq:upper}
\begin{array}{llll}
\Gamma^*: & 2^\Theta & \rightarrow & 2^\Omega \\ & A & \mapsto & \Gamma^*(A)\doteq \{ \omega \in \Omega :
\Gamma(\omega)\cap A \neq \emptyset \}.
\end{array}
\end{equation}

Given two $\sigma$-fields (see Chapter \ref{cha:toe}, Footnote \ref{foot:sigma}) $\mathcal{A},\mathcal{B}$ on $\Omega, \Theta$ respectively, $\Gamma$ is said \emph{strongly measurable} iff $\forall B\in\mathcal{B}$, $\Gamma^*(B)\in\mathcal{A}$. The \emph{lower probability} measure on $\mathcal{B}$ is defined as $P_*(B)\doteq P(\Gamma_*(B))$ for all $B\in\mathcal{B}$. By Equation (\ref{eq:belvalue}) the latter is nothing but a belief function.

Nguyen proved that, if $\Gamma$ is strongly measurable, the probability distribution $\hat{P}$ of the random set \cite{Nguyen78} coincides with the lower probability measure:
\[
\begin{array}{cc}
\hat{P}[I(B)] = P_*(B) & \forall B\in\mathcal{B},
\end{array}
\]
where $I(B)$ denotes the interval $\{C\in\mathcal{B}, C\subset B\}$.\\ In the finite case the probability distribution of the random set $\Gamma$ is precisely the basic probability assignment (Definition \ref{def:bpa}) associated with the lower probability or belief function $P_*$.

An extensive analysis of the relations between Smets' Transferable Belief Model and the theory of random sets can be found in \cite{smets92TBMrandom}.

\subsection{Inner measures} \label{sec:inner-measures}

Belief functions can also be assimilated to \emph{inner measures}.
\begin{definition} \label{eq:inner-probability}
Given a probability measure $P$ defined over a $\sigma$-field of subsets $\mathcal{F}$ of a finite set $\mathcal{X}$, the \emph{inner probability} of $P$ is the function $P_*$ defined by:
\begin{equation}
\begin{array}{ccc}
P_*(A) = \max\{ P(B) | B\subset A,\; B \in \mathcal{F} \}, & & A\subset \mathcal{X}
\end{array}
\end{equation}
for \emph{each} subset $A$ of $\mathcal{X}$, not necessarily in $\mathcal{F}$.
\end{definition}
The inner probability value $P_*(A)$ represents the degree to which the available probability values of $P$ suggest us to believe in $A$.\\ Now, let us define as domain $\mathcal{X}$ of the inner probability function (\ref{eq:inner-probability}) the compatibility relation $C$ (\ref{eq:compat}) associated with a multi-valued mapping $\Gamma$, and choose as $\sigma$-field $\mathcal{F}$ on $C$ the collection:
\begin{equation} \label{eq:effe}
\mathcal{F} = \{ C \cap (E \times \Theta), \forall E\subset\Omega \}.
\end{equation}
Each element of $\mathcal{F}$ is the collection of all pairs in $C$ which relate a point of $E\subset\Omega$ to a subset of $\Theta$. It is then natural to define a probability measure $Q$ over the $\sigma$-field (\ref{eq:effe}) which depends on the original measure $P$ on $\Omega$:
\[
\begin{array}{lllll}
Q & : & \mathcal{F} & \rightarrow & [0,1]\\
& & C\cap(E\times\Theta) & \mapsto & P(E).
\end{array}
\]
The inner probability measure associated with $Q$ is then the function on $2^C$:
\[
\begin{array}{lllll}
Q_* : & 2^C & \rightarrow & [0,1] \\
& \mathcal{A} \subset C & \mapsto & Q_*(\mathcal{A}) = \max\{ P(E) | E \subset \Omega,\; C\cap (E \times
\Theta))\subset \mathcal{A} \}.
\end{array}
\]
We can then compute the inner probability of the subset $\mathcal{A} = C \cap (\Omega \times A)$ of $C$ which
corresponds to a subset $A$ of $\Theta$ as:
\[
\begin{array}{lll}
Q_*(C\cap (\Omega \times A)) & = & \max \{ P(E) | E \subset \Omega,\;C\cap (E \times \Theta) \subset C \cap (\Omega \times A )\} \\ &
= & \max \{ P(E) | E \subset \Omega, \; \omega \in E \wedge (\omega,\theta)\in C \Rightarrow \theta \in A \} \\ & =  & P(\{
\omega | (\omega,\theta) \in C \Rightarrow \theta \in A \})
\end{array}
\]
which, by definition of compatibility relation, becomes:
\[
= P(\{ \omega : \Gamma(\omega) \subset A \}) = b(A),
\]
i.e., the classical definition (\ref{eq:belvalue}) of the belief value of $A$ induced by a multi-valued mapping $\Gamma$. This connection between inner measures and belief functions appeared in the literature in the second half of the Eighties (\cite{Ruspini87,Ruspini92}, \cite{Fagin88}).

\subsection{Belief functions as credal sets} \label{sec:credal-sets}

The interpretation of belief values as lower bounds to the true unknown probability value of an event (Section \ref{sec:random-sets}) generates, in turn, an additional angle of the nature of belief functions \cite{laskey87beliefs}. Belief functions admit the following order relation:
\[
\begin{array}{ccc}
b \leq b' \equiv b(A) \leq b'(A) & & \forall A \subset \Theta,
\end{array}
\]
called \emph{weak inclusion}. A b.f. $b$ is weakly included in $b'$ whenever its belief values are dominated by those of $b'$ for all the events of $\Theta$.

A probability distribution $P$ in which a belief function $b$ is weakly included ($P(A)\geq b(A)$ $\forall A$) is said to be \emph{consistent} with $b$ \cite{kyburg87bayesian}. Each belief function $b$ then uniquely identifies a lower envelope of the set of probabilities consistent with it:
\begin{equation} \label{eq:consistent}
P[b] = \{ P \in \mathcal{P} : P(A) \geq b(A) \},
\end{equation}
i.e., the set of probability measures whose values dominate that of $b$ on all events $A$. Accordingly, the theory of evidence is seen by some authors as a special case of robust statistics \cite{seidenfeld97some}. This position has been heavily critised along the years.

Convex sets of probabilities are often called \emph{credal sets} \cite{levi80enterprise,zaffalon-treebased,cuzzolin2010credal,antonucci10ipmu}. Of course not all credal sets `are' belief functions. The set (\ref{eq:consistent}) is a polytope in the simplex of all probabilities we can define on $\Theta$. Its vertices are all the distributions $p^\pi$ induced by any permutation $\pi = \{ x_{\pi(1)}, ..., x_{\pi(|\Theta|)} \}$ of the singletons of $\Theta$ of the form \cite{chateauneuf1989,cuzzolin08jelia}:
\begin{equation} \label{eq:prho}
p^\pi[b](x_{\pi(i)}) = \sum_{\substack{A \ni x_\pi(i); \; A \not\ni x_\pi(j)\; \forall j<i}} m(A),
\end{equation}
assigning to a singleton element put in position $\pi(i)$ by the permutation $\pi$ the mass of all focal elements containing it, but not containing any elements preceeding it in the permutation order \cite{wallner2005}.

\subsection{The debate on the foundations} \label{sec:debate}

A number of researchers have fueled a debate about the nature and foundations of the notion of belief function \cite{Lowrance82,Ruspini87,Fagin88,benferhat95belief,smets94what,smets91updating,williams82discussion} and its relation to other theories \cite{williams1978,Ruspini92,wilson92howmuch,provan92validity,smets97normative,shapley71cores,smets1993no,wakker99dempster,Shafer04comments}, in particular standard probability theory and the Bayesian approach to statistical inference \cite{smets88versus,shafer82bayes,shenoy88axiomatic}. We only mention a few here.

Halpern and Fagin \cite{halpern92twoviews}, for instance, underlined two different views of belief functions - as generalized probabilities (corresponding to the inner measures of Section \ref{sec:inner-measures}), and as mathematical representations of evidence (perspective which we completely neglected in our brief overview of Chapter \ref{cha:toe}). They maintained that many issues with the practical use of belief functions can be explained as consequences of confusing these two interpretations. As an example, they cite Pearl and other authors' remarks that the belief function approach leads to incorrect or counterintuitive answers in a number of situations \cite{pearl89reasoning,Pearl90}.

In \cite{smets93quantifying} Smets gave an axiomatic justification of the use of belief functions to quantify partial beliefs. In \cite{smets92concept}, instead, he tried to precise the notion of distinct pieces of evidence combined via Dempster's rule. He also responded in \cite{smets92resolving} to Pearl's criticisms appeared in \cite{pearl89reasoning}, by accurately distinguishing the different epistemic interpretations of the theory of evidence (resounding Halpern \emph{et al}. \cite{halpern92twoviews}), focusing in particular on his Transferable Belief Model (see Section \ref{sec:tbm}).

\section{Frameworks and approaches} \label{sec:approaches}

The theory of evidence has been elaborated upon by a number of researchers \cite{den99reasoning,campos05nlp,grabish06lattices,Lowrance90,Zarley88a,Laskey88,Hajek96}, who have proposed original angles on the formalism \cite{baldwin90general,baldwin91combining,wang94robust,wang94polar} which later developed into proper frameworks \cite{an93relation,Kramosil96nonnumerical,andersen96linear,smarandache2005introduction,Lowrance86,lamata94calculus}. We mention here a couple of significant examples.

\subsection{Smets' Transferable Belief Model} \label{sec:tbm}

In his 1990's seminal work \cite{smets90PAMI} P. Smets introduced his \emph{Transferable Belief Model} (TBM) as a framework for quantifying degrees of belief based on Shafer's belief functions. In \cite{smets94transferable} and \cite{smets95axiomatic} (but also \cite{smets98quantified} and \cite{smets97TBMbelief}) an extensive analysis of the major traits of the TBM can be found. To summarise them briefly, within the TBM \emph{positive basic probability values can be assigned to the empty set}, originating `unnormalized' belief
functions (see also \cite{ubf,cuzzolin14tfs}) whose nature is analyzed in \cite{smets92nature}. In \cite{smets91other} Smets compared the Transferable Belief Model with other interpretations of the theory of evidence. In \cite{smets90constructing} he axiomatically derived a `pignistic' transform which can be used to map a belief function to a probability function and make decisions in any uncertain context via classical utility theory.

Smets applied the TBM to diagnostic \cite{smets98application} and reliability \cite{smets92reliability} problems, among others. Dubois \emph{et al} later used the TBM approach on an illustrative example: the assessment of the value of a candidate
\cite{dubois01using}. The vulnerability of the TBM to Dutch books was investigated by Snow in \cite{snow98vulnerability}.

\subsection{Kramosil's probabilistic interpretation of the Dempster-Shafer theory} \label{sec:kramosil}

We have seen that the theory of evidence can be developed in an axiomatic way quite independent of probability theory. Such axioms express the requirements any uncertainty calculus intuitively ought to meet. Nevertheless, D-S theory can also be seen as a sophisticated application of probability theory in a random set (see Section \ref{sec:random-sets}) context.

From a similar angle I. Kramosil \cite{Kramosil96a} proposed to exploit measure theory to expand the theory of belief functions beyond its original scope. The scope of his investigation ranges from Boolean and non-standard valued belief functions \cite{kramosil93boolean,kramosil97nonstandard}, with application to expert systems \cite{Kramosil96nonnumerical}, to the extension of belief functions to countable sets \cite{kramosil94definability}, to the introduction of a {strong law of large numbers} for random sets \cite{kramosil94stronglaw}.

Of particular interest is the notion of \emph{signed belief function} \cite{kramosil97belief}, in which frames of discernment are replaced by measurable spaces equipped by a signed measure. The latter is a $\sigma$-additive set function which can take values also outside the unit interval, including the negative and infinite ones. An assertion analogous to Jordan's decomposition theorem for signed measures is stated and proven \cite{kramosil96jordan}, according to which each signed belief function restricted to its finite values can be defined by a linear combination of two classical probabilistic belief functions, assuming that the domain is finite. A probabilistic analysis of Dempster's rule is also developed by Kramosil \cite{kramosil97dempster}, and its version for signed belief functions is formulated in
\cite{kramosil98dempstersigned}.

A detailed review of Kramosil's work on his measure-theoretic approach to belief functions can be found in two technical reports of the Academy of Sciences of the Czech Republic (\cite{kramosil97probabilistic}, \cite{kramosil98probabilistic}).

\section{Conditional belief functions} \label{sec:conditional-belief-functions}

In the original model in which belief functions are induced by multi-valued mappings of probability distributions, Dempster's conditioning can be judged inappropriate from a Bayesian point of view.

\subsection{Dempster's conditioning}

Recall that in Dempster's approach, conditional belief functions with respect to an arbitrary event $A$ are obtained by simply combining the original b.f. with a `categorical' (in Smets' terminology) belief function focussed on $A$, by means of Dempster's rule of combination: $b(.|A) = b \oplus b_A$. In an interesting work, Kyburg \cite{Kyburg87} indeed analyzed the links between Dempster's conditioning of belief functions and Bayesian conditioning of closed, convex sets of probabilities (of which belief functions are a special case). He arrived at the conclusion that the probability intervals \cite{lemmer91conditions} generated by Dempster updating were included in those generated by Bayesian updating.

As a consequence, several theories of and approaches to conditioning of belief functions \cite{yu94conditional} have been proposed along the years \cite{Chateauneuf89,fagin91new,Jaffray92,gilboa93updating,denneberg94conditioning,itoh95new} to address this criticism.

\subsection{Fagin's conditional belief}

In the framework of credal sets and lower probabilities, Fagin and Halpern defined a new notion of \emph{conditional belief} \cite{fagin91new}, different from Dempster's definition, as the lower envelope of a family of conditional probability functions, and provided a closed-form expression for it.  This definition is related to the idea of inner measure (see Section \ref{sec:inner-measures})

\subsection{Spies' conditional events}

On his side, Spies \cite{spies94conditional} established a link between conditional events and discrete random sets. Conditional events were defined as \emph{sets of equivalent events under the conditioning relation}. By applying to them a multivalued mapping he gave a new definition of conditional belief function. Finally, an updating rule (that is equivalent to the law of total probability is all beliefs are probabilities) was introduced.

Slobodova also conducted some early studies on the issue of conditioning, quite related to Spies' work. In particular, she introduced a multi-valued extension of conditional b.f.s and examined its properties \cite{slobodova97lncs}. In \cite{slobodova94conditional}, in particular, she described how conditional belief functions (defined as in Spies' approach) fit in the framework of valuation-based systems.

\subsection{Smets' conditional approaches}

One way of dealing with the Bayesian criticism of Dempster's rule is to abandon all notions of multivalued mapping to define belief directly in terms of basis belief assignments, as in Smets' Transferable Belief Model \cite{smets93belief} (Section \ref{sec:tbm}). The unnormalized conditional belief function $b_U(.|B)$ with b.b.a. $m_U(.|B)$\footnotemark[1]\footnotetext[1]{Author's notation.}
\[
m_U(.|B) = \left \{ \begin{array}{ll} \displaystyle \sum_{X \subseteq B^c} m(A \cup X) & if \; A \subseteq B \\ 0 & elsewhere  \end{array} \right .
\]
is the `minimal commitment' \cite{hsia91characterizing} specialization of $b$ such that $pl_b(B^c|B) = 0$ \cite{klawonn92dynamic}.\\ In \cite{xu94evidential,xu96reasoning}, Xu and Smets used conditional belief functions to represent relations between variables in evidential networks, and presented a propagation algorithm for such networks. In \cite{smets93jeffrey}, Smets pointed out the distinction between revision \cite{gilboa93updating,perea09amodel} and focussing in the conditional process, and how they led to unnormalized and geometric \cite{suppes1977} conditioning 
\begin{equation} \label{eq:geometric-conditioning}
b_G(A|B) = \frac{b(A \cap B)}{b(B)}, 
\end{equation}
respectively. Note the strong resemblance between (\ref{eq:geometric-conditioning}) and classical Bayes' rule (\ref{eq:bayes}). In these two scenarios he proposed generalizations of Jeffrey's rule of conditioning \cite{jeffrey65book,shafer81jeffrey,jeffrey1988conditioning} 
\[
P(A|P',\mathbb{B}) = \sum_{B \in \mathbb{B}} \frac{P(A \cap B)}{P(B)} P'(B) 
\]
to the case of belief functions.

\subsection{Other contributions}

In an interesting work, Klopotek and Wierzchon \cite{klopotek99lncs} provided a frequency-based interpretation for conditional belief functions. More recently, Tang and Zheng \cite{tang05dempster} discussed the issue of conditioning in a multi-dimensional space. Quite related to the topic of Chapter \ref{cha:geo}, Lehrer \cite{lehrer05updating} proposed a geometric approach to determine the conditional expectation of non-additive probabilities. Such conditional expectation was then applied to updating, whenever information became available, and to introduce a notion of independence.

\section{Statistical inference and estimation}

The question of how to transform the available data (typically coming in the form of a series of trials) into a belief function (the so-called \emph{inference problem}) \cite{krantz83priors,chateauneuf00ambiguity} is crucial to allow practical statistical inference with belief functions \cite{Smets97FAPR2,bryson98qualitative,ngwenyama98generating}. The data can be of different nature: statistical \cite{seidenfeld78,durham92statistical}, logical \cite{watada94logical}, expressed in terms of mere preferences, subjective. The problem has been studied by scholars of the caliber of Dempster \cite{dempster1966}, Shafer, Seidenfeld, Walley, Wasserman \cite{Wasserman90prior} and others \cite{Beran71,dutta1985,Chen95,chen1995}, who delivered an array of approaches to the problem. Unfortunately, the different approaches to the inference problem produce different belief functions from the same statistical data. A very general exposition by Chateauneuf and Vergnaud providing some foundation for a belief revision process, in which both the initial knowledge and the new evidence is a belief function, can be found in \cite{chateauneuf00ambiguity}. We give here a very brief survey of the topic.

Given a parametric model of the data as a function of a number of parameters, we want to identify (or compute the support for) the parameter values which better describe the available data. Shafer's initial proposal for inferring a belief functions from the data via a likelihood-based support function \cite{Shafer76} was supported by Seidenfeld \cite{seidenfeld78}, but led him to criticise Dempster's rule as an appropriate way of combining different pieces of statistical evidence. Wasserman \cite{Wasserman90} showed that the likelihood-based belief function can indeed be used to handle partial prior information, and related it to robust Bayesian inference.

Later \cite{shafer82parametric} Shafer illustrated three different ways of doing statistical inference in the belief framework, according to the nature of the available evidence. He stressed how the strength of belief calculus is really about allowing inference under partial knowledge or ignorance, when simple parametric models are not available.

Many others have also contributed to the debate on the inference issue \cite{aregui08constructing}.\\ In the late Eighties Walley \cite{Walley87} characterized the classes of belief and commonality functions for which statistical independent observations can be combined by Dempster's rule, and those for which Dempster's rule is consistent with Bayes' rule. Van den Acker \cite{acker00belief} designed a method to represent statistical inference as belief functions, designed for application in an audit context \cite{srivastava94integrating}. An original paper by Hummel and Landy \cite{hummel88statistical} proposed a new interpretation of Dempster's rule of combination in terms of statistics of opinions of experts. Liu \emph{et al}. \cite{liu97method} described an algorithm for inducting implication networks from empirical data samples. The values in the implication networks were predicted by applying the belief updating scheme and then compared to Pearl's stochastic simulation method, demonstrating that evidential-based inference has a much lower computational cost.

As far as non-quantitative data are concerned, Bryson \emph{et al}. \cite{ngwenyama98generating,bryson98qualitative} presented an approach to the generation of quantitative belief functions that included linguistic quantifiers to avoid the premature use of numeric measures. Similarly, Wong and Lingras \cite{wong94representation} proposed to generate belief functions from symbolic information such as qualitative preferences \cite{wong90generation} of users in decision systems.

\section{Decision making} \label{sec:decision-making}

Decision making is, together with estimation, the natural final link of any inference chain \cite{seidenfeld07isipta}. Indeed, decision making in the presence of partial evidence and subjective assessment \cite{einhorn1986} is the original rationale for the development of the theory of evidence. Consequently, decision making with belief functions has been studied throughout the last three decades \cite{smets2002decision,smets2005decision,Caselton92,Jaffray94,Strat90,klir94dynamic}, originating a number of different approaches to the problem \cite{horiuchi98decision}. As an example, a work by Beynon \emph{et al}. \cite{beynon00alternative} explored the potential of the theory of evidence as an alternative approach to multi-criteria decision modeling.\\ The situation quite resembles that surrounding the argument around inference. There is no lack of proposal solutions, but rather too many of them have been proposed, and a clear notion of what option is the most sensible in what context is needed. A recent discussion on the meaning of belief functions in the context of decision making can be found in \cite{smets01decision}.

Perhaps the first one who noted the lack in Shafer's theory of belief functions of a formal procedure for making decision was Strat \cite{strat94decisionanalysis}. He proposed a simple assumption that disambiguates decision problems in the context of belief calculus, by enforcing a separation between the evidence carrying information \cite{maluf97monotonicity} about the decision problem and the assumptions that need to be made in order to disambiguate the decision outcomes. He also showed how to generalize the methodology for decision analysis employed in probabilistic reasoning to the use of belief functions, allowing their use within the framework of decision trees. Schubert \cite{schubert94thesis,schubert95onrho} subsequently studied the influence of the $\rho$ parameter in Strat's decision apparatus. Elouedi, Smets \emph{et al}. \cite{elouedi00classification,elouedi00decision} adapted this decision tree technique to the presence of uncertainty about the class value, represented by a belief function.

A decision system based on the Transferable Belief Model was developed \cite{Xu92a,xu96transferable,xu96decision} and applied to a waste disposal problem by Xu \emph{et al}. The framework makes use of classical expected utility theory \cite{vonneumann44}.\\ There, given a utility function $u(x)$ on the possible outcomes $x \in \Theta$, a decision maker can pick one among a number of `lotteries' (probability distributions) $p_i(x)$, in order to maximize the expected return or utility 
\[
E(p_i) = \sum_x u(x) p_i(x) 
\]
of the lottery. In \cite{smets05ijar}, the author proved the necessity of the linearity axiom (and therefore of the pignistic transform) by maximizing the following expected utility (our notation), where $p = BetP$ is the pignistic approximation (\ref{eq:pignistic-probability}) of a belief function $b$:
\[
E[u] = \sum_{x\in\Theta} u(a,x) p(x).
\]
The set of possible actions (decisions) $a \in \mathcal{A}$ and the set $\Theta$ of the possible outcomes $x$ of the problem are distinct, and the utility function $u(a,x)$ is defined on $\mathcal{A} \times \Theta$.\\ A significant contribution to the application of linear utility theory to belief functions is due to Jaffray \cite{yaffray88application,jaffray89linear,yaffray94dynamic}.

A number of decision rules not based on the application of utility theory to the result of a probability transform have also been proposed, for instance by Troffaes \cite{troffaes07}. Most of those proposals are based on order relations between uncertainty measures \cite{denoeux08ai}, in particular the `least commitment principle' \cite{ristic04ipmu,caron06ijar,denoeux08flairs}, the analogous of maximum entropy in belief function theory \cite{chau93upper}. Xu and Yang \cite{yang94evidential} proposed instead a decision calculus in the framework of {valuation based systems} \cite{xu97valuation}, and showed that decision problems can be solved using local computations. Wakker \cite{wakker99dempster} argued on the central role that the so-called `principle of complete ignorance' plays in the evidential approach to decision problems.

\section{Efficient implementation of belief calculus} \label{sec:efficient}

The complexity of Dempster's rule of computation is inherently exponential, due to having to consider all possible subsets of a frame of discernment. Indeed, Orponen \cite{orponen90dempster} proved that the problem of computing the orthogonal sum of a finite set of belief functions is $\mathcal{NP}$-complete. This has encouraged much research on the efficient implementation \cite{clarke91efficient} of the theory of evidence \cite{Barnett81,guan94computational,kennes91fast,kennes91computational,kennes92computational,Xu92,km91,kohlas89b,Xu94,Thoma91,bissig97fastdivision,gordon85method}, and Dempster's rule in particular \cite{Shafer87a,Lehmann99,yager86arithmetic}.

\subsection{Graphical models} \label{sec:graphical-models}

In their 1987's work \cite{Shafer87b} Shafer, Shenoy and Mellouli posed the problem in the lattice of partitions of a fixed frame of discernment. Different questions were represented as different partitions of this frame, and their relations were represented as qualitative conditional independence or dependence relations between partitions.\\ They showed that an efficient implementation of Dempster's rule is possible if the questions are arranged in a \emph{qualitative Markov tree} \cite{Xu94a,xu95computing}, as in such a case belief functions can be propagated through the tree \cite{almond95book,yaghlane06ipmu}. Multivariate belief functions on graphical models \cite{bergsten93dempster} were explored by Kong in her PhD thesis \cite{Kong86a}, but also by Mellouli \cite{Mellouli86,mellouli97pooling}. The fusion and propagation of graphical belief models was studied by Almond \cite{almond90thesis}. The close relation of Shafer-Shenoy's architecture with the contents of Chapter \ref{cha:alg}, where we will discuss the algebraic structure of families of frames, is worth noticing. Indeed, in \cite{Shafer87b} the analysis was limited to a lattice of partitions, rather than encompassing entire families of frames.

In related work, Bissig, Kohlas and Lehmann proposed a so-called \emph{Fast-Division architecture} \cite{bissig97fastdivision} for Dempster's rule computation. The latter has the advantage, over the Shenoy-Shafer and the Lauritzen-Spiegelhalter architectures \cite{lepar98uai}, of guaranteeing all intermediate results to be belief functions. Each of them has a Markov tree as the underlying computational structure.\\ When the evidence is ordered in a \emph{complete direct acyclic graph} it is possible to formulate algorithms with lower computational complexity \cite{bergsten93dempster}. Credal networks \cite{zaffalon-treebased,cozman2000} have also been proposed.

\subsection{Monte-Carlo methods} \label{sec:monte-carlo}

Monte-Carlo methods are extremely useful in Bayesian estimation when the need arises to represent complex, multi-modal probability distributions at an acceptable computational cost. Some work on MonteCarlo methods for belief functions has been done in the past by Wilson and Moral \cite{wilson91montecarlo,moral99montecarlo}. Kramosil \cite{kramosil98montecarlo} also worked on adapting Monte-Carlo estimation to belief functions. Resconi \emph{et al}. \cite{resconi98speed-up} achieved a speed-up of Monte-Carlo methods by using a physical model of belief measures. This approach could potentially provide the ultimate solution to efficient implementation of belief calculus. In particular, estimation techniques based on the notion of `particle filtering' \cite{Deutscher99,Deutscher00} may be very useful in the context of the theory of belief functions, as a way or reducing the damning computational complexity of handling belief functions.

\subsection{Transformation approaches} \label{sec:transformation}

One approach to efficient belief calculus that has been explored since the late Eighties, known as `probability transformation' \cite{daniel06ijis} consists in approximating belief functions by means of appropriate probability measures prior to using them for making decisions. A number of distinct transformations have been introduced \cite{weiler94approximation,Kramosil95,bauer97approximation,Yaghlane01ecsqaru,Denoeux01ijufk,Denoeux02ijar,Haenni02ijar}. It is worth noticing that different approximations appear to be aimed at different goals, and do not necessarily seek an efficient implementation of the rule of combination \cite{BaroniV03}.

Given a frame of discernment $\Theta$, let us denote by $\mathcal{B}$ the set of all belief functions on $\Theta$, and by $\mathcal{P}$ the set of all probability measures on $\Theta$.\\ According to \cite{daniel06on}, we call a \emph{probability transform} \cite{daniel04ipmu} of belief functions an operator 
\[
pt : \mathcal{B} \rightarrow \mathcal{P}, \hspace{5mm} b \mapsto pt[b] 
\]
mapping belief measures onto probability measures, such that 
\begin{equation} \label{eq:probability-transform}
b(x) \leq pt[b](x) \leq pl_b(x) = 1 - b(\{x\}^c). 
\end{equation}

Note that (\ref{eq:probability-transform}) requires the probability which results from the transform to be compatible with the upper and lower bounds the original b.f. $b$ enforces \emph{on the singletons only}, and not on all the focal sets as in Equation (\ref{eq:consistent}). This is a minimal, sensible constraint which does not require probability transforms to adhere to the credal semantics of belief functions (Section \ref{sec:credal-sets}). As a matter of fact, important such transforms are not compatible with such semantics.

As mentioned above, in Smets' `Transferable Belief Model' \cite{smets88beliefversus,smets94transferable} decisions are made by resorting to the \emph{pignistic probability}:
\begin{equation}\label{eq:pignistic-probability}
BetP[b](x) = \sum_{A\supseteq \{x\}} \frac{m_b(A)}{|A|},
\end{equation}
generated by what he called the \emph{pignistic transform} $BetP : \mathcal{B} \rightarrow \mathcal{P}$, $b \mapsto BetP[b]$ \cite{wilson93pignistic}. Justified by a `linearity' axiom \cite{smets94transferable}, the pignistic probability is the result of a redistribution process in which the mass of each focal element $A$ is re-assigned to all its elements $x \in A$ on an equal basis, and is perfectly compatible with the upper-lower probability semantics of belief functions, as it is the center of mass of the polytope (\ref{eq:consistent}) of consistent probabilities \cite{chateauneuf1989}. Generalizations of the pignistic transform for partial bet have been recently proposed by Burger \cite{burger09gene} and Dezert \cite{dezert2004generalized}.

Originally developed by Voorbraak \cite{voorbraak89efficient} as a probabilistic approximation intended to limit the computational cost of operating with belief functions in the Dempster-Shafer framework, the \emph{plausibility transform} \cite{Cobb03isf} has later been supported by Cobb and Shenoy \cite{Cobb03ecsqaru} in virtue of its desirable commutativity properties with respect to Dempster's sum. Although initially defined in terms of commonality values, the plausibility transform $\tilde{pl} : \mathcal{B} \rightarrow \mathcal{P}$, $b \mapsto \tilde{pl}[b]$ maps each belief function $b$ onto the probability distribution $\tilde{pl}[b] = \tilde{pl}_b$ obtained by normalizing the plausibility values $pl_b(x)$\footnotemark[1]\footnotetext[1]{With a harmless abuse of notation we denote the values of b.f.s and pl.f.s on a singleton $x$ by $m_b(x), pl_b(x)$ instead of $m_b(\{x\}),pl_b(\{x\})$.} of the element of $\Theta$:
\begin{equation}\label{eq:relplaus}
\tilde{pl}_b(x) = \frac{pl_b(x)}{\sum_{y\in\Theta} pl_b(y)}.
\end{equation}
We call the output $\tilde{pl}_b$ (\ref{eq:relplaus}) of the plausibility transform \emph{relative plausibility of singletons} \cite{cuzzolin10amai}. Voorbraak proved that his (in our terminology) relative plausibility of singletons $\tilde{pl}_b$ is a perfect representative of $b$ when combined with other probabilities $p \in \mathcal{P}$ through Dempster's rule $\oplus$: 
\[
\tilde{pl}_b \oplus p = b\oplus p \hspace{5mm} \forall p \in\mathcal{P}.
\]

Dually, a \emph{relative belief transform} $\tilde{b} : \mathcal{B} \rightarrow \mathcal{P}$, $b \mapsto \tilde{b}[b]$ mapping each belief function to the corresponding \emph{relative belief of singletons} $\tilde{b}[b] = \tilde{b}$ \cite{cuzzolin08pricai,cuzzolin2008semantics,cuzzolin2008dual,cuzzolin08unclog-semantics,haenni08aggregating,daniel06on} such that:
\begin{equation}\label{eq:btilde}
\tilde{b}(x) = \frac{b(x)}{\sum_{y \in \Theta} b(y)}
\end{equation}
can be defined. This notion (under the name of `normalized belief of singletons') was first proposed by Daniel \cite{daniel06on}. Some initial analyses of the relative belief transform and its close relationship with the (relative) plausibility transform were presented in \cite{cuzzolin08pricai,cuzzolin08unclog-semantics}.

More recently, other proposals have been brought forward by Dezert et al. \cite{dezert} and Sudano \cite{sudano01fusion,sudano01IPT,sudano03icif}, based on redistribution processes similar to that of the pignistic transform. Two new Bayesian approximations of belief functions have been derived by the author of this Book from purely geometric considerations \cite{cuzzolin07smcb} in the context of the geometric approach to the ToE \cite{cuzzolin08smcc}, in which belief and probability measures are represented as points of a Cartesian space (see Part II, Chapter \ref{cha:geo}). Consonant transformation approaches have also been proposed \cite{cuzzolin11isipta-consonant}.

\subsection{Reducing the number of focal elements} \label{sec:reducing}

Probability (and possibility \cite{dubois93possibility-probability}) transforms reduce the number of focal elements to store to $O(N)$ by re-distributing the mass assignment of a belief function to size-1 subsets or chains of subsets, respectively. An alternative approach to efficiency can be sought by re-distributing all the mass to subsets of size up to $k$, obtaining a \emph{$k$-additive belief function} \cite{grabish97fss,miranda06ejor,burger10approx,burger09gene,burger2010barycenters}. Some approaches to probability transformation explicitly aim at reducing the complexity of belief calculus. Tessem \cite{tessem93approximations}, for instance, incorporates only the highest-valued focal elements in his $m_{klx}$ approximation. A similar approach inspires the {summarization} technique formulated by Lowrance \emph{et al.} \cite{Lowrance86}.

\section{Continuous belief functions} \label{sec:continuous-formulation}

The original formulation of the theory of evidence summarized in Chapter \ref{cha:toe} is inherently linked to finite frames of discernment. Since the late Seventies, the need for a general formulation of the theory of evidence to continuous domains has been recognized. Numerous efforts have been made since then in order to extend the theory of belief functions to infinite sets of possibilities. None of them has been found entirely convincing yet (see \cite{Kohlas94,wang92continuous}). Nevertheless, they all contributed to a clearer vision of this issue -- we summarize here the most significant ones.

\subsection{Shafer's allocation of probabilities} \label{sec:allocation}

The first attempt (1979) is due to Shafer himself, and goes under the name of \emph{allocation of probabilities} \cite{shafer79allocations}. Shafer proved that every belief function can be represented as an allocation of probability, i.e., a $\cap$-homomorphism into a positive and completely additive probability algebra, deduced from the integral representation due to Choquet. Canonical continuous extensions of belief functions defined on `multiplicative subclasses' $E$ to an arbitrary power set can be introduced by allocation of probability: for every belief function there exists a complete Boolean algebra $M$, a positive measure $\mu$ and a mapping $\rho$ between $E$ and $M$ such that $f = \mu \circ \rho$.

In \cite{shafer79allocations} the concepts of \emph{continuity} and \emph{condensability} are defined for belief functions, and it is shown how to extend a b.f. defined on an algebra of subsets to the whole power set. Canonical extensions satisfy Shafer's notion of belief function definable on infinitely many compatible frames, and show significant resemblance with the notions of inner measure (Section \ref{sec:inner-measures}) and extension of capacities \cite{honda06entropy}. This approach was later reviewed by Jurg Kohlas \cite{kohlas97allocation}, who conducted an algebraic study of argumentation systems (\cite{kb95},\cite{kb94}) as ways of defining numerical degrees of support of hypotheses by means of allocation of probability.

\subsection{From belief functions to random sets}

Possibly the most elegant formalism in which to formulate a continuous version of the theory of belief functions is the theory of random sets \cite{matheronrandom,goutsias97random,ross86random,Nguyen97}, i.e., probability measures over power sets, of which traditional belief functions are indeed a special case (recall Section \ref{sec:random-sets}). A serious obstacle, however, is given by the formulation of aggregation operators for random sets. Neither Shafer's allocations of probability approach nor Nguyen's random set interpretation seemed to be much concerned with combination rules, and not much progress seems to have been made since.

\subsection{Belief functions on Borel intervals} \label{sec:borel-intervals}

Almost in syncronous, Strat \cite{Strat84} and Smets \cite{smets05real} had the idea of making the problem of generalising belief functions to continuous frames tractable via standard calculus, by \emph{allowing only focal elements which are closed intervals of the real line}. Such extension of belief functions to mere \emph{Borel sets of the real line} \cite{kuhr2007finetti} has demonstrated in time its fertility. Generalizations of combination and conditioning rules follow quite naturally \cite{smets05real}. Inference mechanisms with predictive b.f.s on real numbers have been proposed \cite{aregui07isipta}. The computation of a pignistic probability for b.f.s on Borel intervals is straightforward, allowing the formulation of a theory of decision making with continuous belief functions.

\subsection{Theory of hints} \label{sec:hints}

Kohlas and Monney proposed a very similar definition of belief functions on real numbers. Indeed, some of the relations introduced in  \cite{smets05real} and \cite{kohlas95foundations} had already appeared in \cite{Dempster68a}. This led to the so-called \emph{mathematical theory of hints} \cite{kohlas95foundations,kohlas97allocation,kohlas95foundations,kohlas94representation,vakili93} (see the monograph \cite{km95book} for a detailed exposition). Hints \cite{kohlas94representation} are bodies of information inherently imprecise and uncertain, that do not point to precise answers but are used to judge hypotheses, leading to support and plausibility functions similar to those introduced by Shafer. They allow a logical derivation of Dempster's rule, and originate a theory valid for general, infinite frames of discernment. Among others, hints have been applied to model-based diagnostic \cite{kohlas95modelbased}.

\subsection{Monotone capacities and Choquet integrals} \label{sec:monotone}

\emph{Monotone capacities} \cite{chateauneuf1989,maccheroni05annals,miranda03extreme} have been also suggested as a general framework for the mathematical description of uncertainty \cite{bruning02stat}.
\begin{definition}
Let $S$ be a finite set, and $2^{S}$ be the power set of $S$. Then $v : 2^{S}\rightarrow [0,1]$ is a
\emph{capacity} on $S$ if $v(\emptyset)=0$, $v(S)=1$ and:
\begin{itemize}
\item if $E\subseteq F$ then $v(E)\leq v(F)$, for $E,F\in\mathcal{S}$ (\emph{monotonicity}).
\end{itemize}
\end{definition}
Obviously \cite{denneberg00totally}:
\begin{proposition} Belief functions are totally monotone capacities.
\end{proposition}
Hendon \emph{et al}. \cite{hendon96product} examined the question of defining the product of two independent capacities. In particular, for the product of two belief functions as totally monotone capacities, there is a unique minimal product belief function.

The Choquet's integral \cite{gilboa94additive} of monotone set functions (such as belief functions) is a generalization of the Lebesgue integral with respect to $\sigma$-additive measures. Wang and Klir investigated the relations between Choquet integrals and belief measures in \cite{wang97choquet}.

\section{Other theoretical developments} \label{sec:other-advances}

\subsection{Inverting Dempster's rule: Canonical decomposition} \label{sec:canonical}

The quest for an inverse operation to Dempster's combination rule has a natural appeal and an intuitive interpretation. If Dempster's rule reflects a modification of one's system of belief when the subject becomes familiar with the degrees of beliefs of another subject and accepts the arguments on which these degrees are based, the inverse operation would enable to \emph{erase the impact} of this modification, and to return back to one's original degrees of beliefs, supposing that the reliability of the second subject is put into doubts. This inversion problem, called also `canonical decomposition', was solved in an algebraic framework by Smets in \cite{smets95canonical}. Kramosil also proposed a solution to the inversion problem, within his measure-theoretic approach \cite{kramosil97measure}.

\subsection{Frequentist formulations} \label{sec:frequentist}

The theory of evidence was born as an attempt to formulate a mathematical theory of subjective belief, in a rather incompatible approach to theories of chance in which probabilities are the result of series of empirical trials. To our knowledge only Walley has tried, in an interesting even if not very recent paper \cite{Walley82frequentist}, to formulate a frequentist theory of upper and lower probabilities (see also \cite{denoeux06ipmu}). Namely, he considered models for independent repetitions of experiments described by \emph{interval probabilities}, and suggested generalizations of the familiar concepts of independence and asymptotic behavior.

\subsection{Gaussian belief functions} \label{sec:gaussian}

The notion of \emph{Gaussian belief function} \cite{liu95model,shafer92note} is an interesting effort to extend Dempster-Shafer theory to represent mixed knowledge, some of which logical and some uncertain. The notion of Gaussian b.f. was proposed by A. Dempster and formalized by L. Liu in \cite{liu96theory}. Technically, a Gaussian belief function is a Gaussian distribution over the members of the parallel partition of an hyperplane. By adapting Dempster's rule to the continuous case, Liu derives a rule of combination and proves its equivalence to Dempster's geometrical description \cite{dempster90normal}. In \cite{liu99local}, Liu proposed a join-tree computation scheme for expert systems using Gaussian belief functions, for he proved their rule of combination satisfies the axioms of Shenoy and Shafer \cite{shenoy90axioms}.

\section{Relation with other mathematical theories of uncertainty} \label{sec:other-theories}

Currently several different mathematical theories of uncertainty compete to be adopted by practitioners of all field of applied science \cite{walley91book,yao98interpretations,shafer01book,halpern03book,maass06philosophical,Gardenfors}, a process resulting in a growing number of applications of these frameworks. The consensus is that there no such a thing as \emph{the} best mathematical description of uncertainty (compare \cite{joslyn98towards}, \cite{Fagin88}, \cite{sheridan1991}, \cite{klir95principles,klir99fuzzy} and \cite{denoeux99reasoning} to cite a few), and that the choice of the most suitable methodology depends on the actual problem at hand (an extensive presentation of a possible unified theory of imprecise probability can be found in \cite{Walley91coherent,walley00towards}). Smets (\cite{smets98which}, \cite{smets94what}) also contributed to the analysis of the difference between imprecision and uncertainty \cite{Keynes21pir}, and compared the applicability of various models of uncertainty. Whenever a probability measure can be estimated, most authors suggest the use of a classical Bayesian approach. If probability values cannot be reliably estimate, upper and lower probabilities should instead be preferred.

Here we briefly survey the links between the theory of evidence and other approaches to uncertainty theory.

\subsection{Lower probabilities} \label{sec:lower-probabilities}

A \emph{lower probability} $\lpr$ is a function from $2^\Theta$, the power set of $\Theta$, to the unit interval $[0,1]$. With any lower probability $\lpr$ is associated a dual upper probability function $\upr$, defined for any $A \subseteq \Theta$ as $\overline{P}(A)=1-\underline{P}(A^c)$, where $A^c$ is the complement of $A$. With any lower probability $\lpr$ we can associate a (closed convex) set
\begin{equation}\label{eq:credal}
\cred(\lpr)= \Big \{p: P(A) \geq \lpr(A), \forall A \subseteq \Theta \Big \}
\end{equation}
of probability distributions $p$ whose measure $P$ dominates $\lpr$. Such a polytope or convex set of probability distributions is usually called a \emph{credal set} \cite{levi80enterprise}. Note, however, that not all convex sets of probabilities can be described by merely focusing on events (see Walley~\cite{Walley91}).

A lower probability $\lpr$ is called `consistent' if $\cred(\lpr)\neq \emptyset$ and `tight' if
\[
\inf_{p \in \cred(\lpr)}P(A) = \lpr(A)
\]
(respectively $\lpr$ `avoids sure loss" and $\lpr$ is `coherent' in Walley's~\cite{Walley91} terminology). Consistency means that the lower bound constraints $\lpr(A)$ can indeed be satisfied by some probability measure, while tightness indicates that $\lpr$ is the lower envelope on subsets of $\cred(\lpr)$.

Belief functions are indeed a special type of {coherent lower probabilities}, which in turn can be seen as a special class of \emph{lower previsions} (consult \cite{Walley91}, Section 5.13). Walley has proved that coherent lower probabilities are closed under convex combination: the relationship of belief functions with convexity will be discussed in Part II.

\subsection{Probability intervals} \label{sec:probability-intervals}

Dealing with general lower probabilities defined on $2^\Theta$ can be difficult when $\Theta$ is large: it may then be interesting for practical applications to focus on simpler models.

A \emph{set of probability intervals} or \emph{interval probability system} \cite{tessem92interval,decampos94} is a system of constraints on the probability values of a probability distribution $p:\Theta \rightarrow [0,1]$ on a finite domain $\Theta$ of the form:
\begin{equation} \label{eq:credal-interval}
\mathcal{P}(l,u) \doteq \Big \{p: l(x) \leq p(x) \leq u(x), \forall x \in \Theta \Big \}.
\end{equation}
Probability intervals were introduced as a tool for uncertain reasoning in \cite{decampos94,moral93partially}, where combination and marginalization of intervals were studied in detail. The authors also studied the specific constraints such intervals ought to satisfy in order to be consistent and tight.

As pointed out for instance in \cite{unclog08book}, probability intervals typically arise through measurement errors. As a matter of fact, measurements can be inherently of interval nature (due to the finite resolution of the instruments). In that case the \emph{probability} interval of interest is the class of probability measures consistent with the \emph{measured} interval.

A set of constraints of the form (\ref{eq:credal-interval}) also determines a credal set: credal sets generated by probability intervals are a sub-class of all credal sets generated by lower and upper probabilities \cite{smets87versus}. Note that given a set of bounds $\mathcal{P}(l,u)$ we can obtain lower and upper probability values $\lpr(A)$ on any subset $A \subseteq \Theta$ by using the following simple formulas:
\begin{equation}\label{eq:lowup-from-int}
\lpr(A)=\max \left \{ \sum_{x \in A} l(x), 1- \sum_{x \not\in A} u(x) \right \}, \quad \lpr(A)=\min \left \{ \sum_{x \in A} u(x), 1- \sum_{x \not\in A} l(x) \right \}.
\end{equation}

Belief functions are also associated with a set of lower and upper probability constraints of the form (\ref{eq:credal-interval}): they correspond therefore to a special class of probability intervals, associated with credal sets of a specific form.

\subsection{Possibility theory} \label{sec:possibility-theory}

\emph{Possibility theory} \cite{dubois83unfair,dubois88possibility,dubois87properties} is a theory of uncertainty based on the notion of \emph{possibility measure}.

\begin{definition} \label{def:possibility-measure}
A \emph{possibility measure} on a domain $\Theta$ is a function $Pos: 2^\Theta \rightarrow [0,1]$ such that $Pos(\emptyset) =0$, $Pos(\Theta)=1$ and 
\[
Pos(\bigcup_i A_i)=\sup_i Pos(A_i) 
\]
for any family $\{A_i|A_i\in 2^\Theta, i\in I\}$ where $I$ is an arbitrary set index.
\end{definition}

Each possibility measure is uniquely characterized by a \emph{membership function} or \emph{possibility distribution} $\pi: \Theta \rightarrow [0,1]$ s.t. $\pi(x)\doteq Pos(\{x\})$ via the formula $Pos(A)=\sup_{x\in A} \pi(x)$. Its dual $Nec(A) = 1 - Pos(A^c)$ is called \emph{necessity measure}.

A number of studies have pointed out that necessity measures coincide in the theory of evidence with the class of consonant belief functions (Definition \ref{def:consonant}). Let us call \emph{plausibility assignment} $\bar{pl}_b$ \cite{joslyn91towards} the restriction of the plausibility function to singletons $\bar{pl}_b(x) = pl_b(\{x\})$. From Condition 4 of Proposition \ref{pro:cons} it follows immediately that:
\begin{proposition} \label{pro:conso}
The plausibility function $pl_b$ associated with a belief function $b$ on a domain $\Theta$ is a possibility measure iff $b$ is consonant, in which case the membership function coincides with the plausibility assignment: $\pi = \bar{pl}_b$.\\ Equivalently, a b.f. $b$ is a necessity measure iff $b$ is consonant.
\end{proposition}
Possibility theory (in the finite case) is then embedded in the ToE as a special case.

The links between the Transferable Belief Model and possibility theory have been briefly investigated by Ph. Smets in \cite{smets90possibility}, while Dubois and Prade \cite{dubois90} have worked extensively on consonant approximations of belief functions \cite{cuzzolin09ecsqaru-outer,cuzzolin10fss}. Their work has been later considered in \cite{cliff92minimal,joslyn97possibilistic}.

\subsection{Fuzzy measures} \label{sec:fuzzy-measures}

While evidential reasoning generalises both standard probability and possibility theory, a further generalization of the class of belief measures is expressed by \emph{fuzzy measures} \cite{cuzzolin04ipmu}.
\begin{definition} \label{def:fuzzy-measures}
Given a domain $\Theta$ and a non-empty family $\mathcal{C}$ of subsets of $\Theta$, a \emph{fuzzy measure} $\mu$ on
$\langle \Theta, \mathcal{C} \rangle$ is a function $\mu : \mathcal{C} \rightarrow [0,1]$ which meets the following conditions:
\begin{itemize}
\item $\mu(\emptyset)=0$;
\item if $A\subseteq B$ then $\mu(A)\leq \mu(B)$, for every $A,B\in\mathcal{C}$;
\item for any increasing sequence $A_1\subseteq A_2 \subseteq \cdots$ of subsets in $\mathcal{C}$,
\[
if \bigcup_{i=1}^{\infty}A_i\in\mathcal{C},\;then\;\lim_{i\rightarrow \infty} \mu(A_i) = \mu\Big(\bigcup_{i=1}^{\infty}A_i\Big)
\]
(\emph{continuity from below});
\item  for any decreasing sequence $A_1\supseteq A_2 \supseteq \cdots$ of subsets in $\mathcal{C}$,
\[if \bigcap_{i=1}^{\infty}A_i\in\mathcal{C}\;and \;\mu(A_1)<\infty,\;then\;\lim_{i\rightarrow \infty}\mu(A_i) = \mu\Big(\bigcap_{i=1}^{\infty}A_i \Big) \]
(\emph{continuity from above}).
\end{itemize}
\end{definition}
Clearly, from Definition \ref{def:bel1} a belief measure is also a fuzzy measure \cite{smets81degree}.

Klir \emph{et al}. published an excellent discussion \cite{klir97constructing} on the relations between belief and possibility theory \cite{feriet1982,lee95fuzzy}, and examined different methods for constructing fuzzy measures in the context of expert systems. Authors like Heilpern \cite{heilpern97representation}, Yager \cite{yager99class}, Palacharla \cite{palacharla94understanding}, Romer \cite{romer95applicability} and others \cite{Renaud99} also studied the connection between fuzzy numbers and Dempster-Shafer theory. Lucas and Araabi  proposed in \cite{lucas99generalization} their own generalization of the Dempster-Shafer theory \cite{yen90generalizing} to a fuzzy valued measure, while Mahler \cite{mahler95combining} formulated his own `fuzzy conditioned Dempster-Shafer (FCDS) theory', as a probability-based calculus for dealing with possibly imprecise and vague evidence. Palacharla and Nelson \cite{palacharla94understanding,Palacharla94b}, instead, focused on comparing their application to data fusion problems in transportation engineering.

Ronal Yager \cite{yager86entailment,yager99class}, one of the main contributors to fuzzy logic, proposed \cite{yager95including} a combined fuzzy-evidential framework for fuzzy modeling. In another work, Yager investigated the issue of normalization (i.e., the assignment of non-zero values to empty sets as a consequence of the combination of evidence) in the fuzzy Dempster-Shafer theory of evidence,
proposing a technique called \emph{smooth normalization} \cite{yager96normalization}.

\subsection{Probabilistic logic} \label{sec:probabilistic-logic}

Generalizations of classical logic in which propositions are assigned probability values have been proposed in the past. As belief functions naturally generalize probability measures, it is quite natural to define non-classical logic frameworks in which propositions are assigned \emph{belief values} \cite{smets91patterns}, rather than probability values \cite{cholvy2009using,HRWW08a,paris08unclog,batens00}. This approach has been brought forward in particular by Saffiotti \cite{Saffiotti_abelief-function,Saffiotti90,Saffiotti91}, Haenni \cite{haenni05isipta}, and others \cite{kramosil93boolean,provan90logicbased,harmanec94qualitative,kohlas87b}.

In propositional logic, propositions or formulas are either true or false, i.e., their truth value $T$ is either 0 or 1 \cite{Mates72}. Formally, an \emph{interpretation} or \emph{model} of a formula $\phi$ is a valuation function mapping $\phi$ to the truth value `true'  or `1'. Each formula can therefore be associated with the set of interpretations or models under which its truth value is 1. If we define the frame of discernment of all the possible interpretations, each formula $\phi$ is associated with the subset $A(\phi)$ of this frame which collects all its interpretations.\\
If the available evidence allows to define a belief function on this frame of possible interpretations, to each formula $A(\phi) \subseteq \Theta$ is then naturally assigned a degree of belief $b(A(\phi))$ between 0 and 1 \cite{Saffiotti_abelief-function,haenni05isipta}, measuring the total amount of evidence supporting the proposition `$\phi$ is true'.

Saffiotti \cite{Saffiotti92} built a hybrid logic for representing uncertain logic called \emph{belief function logic} by attaching degrees of belief and degrees of doubt to classical first-order logic, and elaborating original angles on the role of Dempster's rule. The issue was studied by other authors as well. In \cite{benferhat95tech} and \cite{benferhat95belief}, Benferhat \emph{et al}., for instance, defined a semantics based on $\epsilon$-belief assignments where values committed to focal elements are either close to 0 or close to 1. Andersen and Hooker \cite{andersen96linear} proved probabilistic logic and Dempster-Shafer theory to be instances of a certain type of linear programming model, with exponentially many variables (see also the work of Hunter \cite{hunter87versus}). In particular it is worth mentioning the work of Resconi, Harmanec \emph{et al}. \cite{resconi93integration,harmanec94modal,harmanec94qualitative,resconi96interpretations,harmanec96modal}, who proposed `propositional modal logic' as a unifying framework for various uncertainty theories, including fuzzy, possibility and evidential theory, and proposed an interpretation of belief measures on infinite sets. In a series of papers by Tsiporkova \emph{et al} \cite{tsiporkova99evidence,tsiporkova99dempster} Harmanec's modal logic interpretation was further developed, and a modal logic interpretation of Dempster's rule was proposed. The links of DS theory with penalty logic were explored by Dupin \emph{et al} \cite{dupin94penalty}.

    % incidence calculus
To conclude this survey, \emph{incidence calculus} \cite{bundy85incidence,liu98method} is a probabilistic logic for dealing with uncertainty in intelligent systems. Incidences are assigned to formulae: they are the logic conditions under which a formula is true. Probabilities are assigned to incidences, and the probability of a formula is computed from the sets of incidences assigned to it. In \cite{liu98method} Liu, Bundy \emph{et al}. did propose a method for discovering incidences that can be used to calculate mass functions for belief functions.

\section{Applications} \label{sec:applications}

The number of applications of the theory of evidence to a variety of fields of engineering, computer science, and business has been steadily growing in the past decades \cite{smets99practical,shenoy94using,boston00signal,ip91exchange,soh98seaice,besserer93multiple,simpson90application,ferrari89coupling} - we will therefore not attempt to provide a comprehensive review of the matter here.

\subsection{Machine learning} \label{sec:applications-machine-learning}

Machine learning, including clustering, classification \cite{foucher99multiscale} and decision making, is a natural field of application for evidential reasoning \cite{bergsten97applying,quost06ipmu1,Burger06,Aran09,kessentini10ipmu}. A lot of work in this fields has been done by Thierry Denoeux and his co-workers \cite{mas09}. Recent efforts to generalise the maximum entropy classification framework were made by the Author \cite{Cuzzolin2018maxent}.

Already in the Nineties Denoeux and Zouhal \cite{denoeux95knearest} proposed a $k$-nearest neighbor classifier based on the theory of evidence, in which each neighbor of a test sample was considered as an item of evidence supporting hypotheses about the class membership of the test measure. The pieces of evidence provided by the $k$ nearest neighbors were then pooled as usual by Dempster's sum. The problem of tuning the parameters of the classification rule was solved by minimizing an error function \cite{zouhal98evidence}. Le-Hegarat, Bloch \emph{et al}. also worked on unsupervised classification in a multisource remote sensing environment \cite{lehegarat97application} in the framework of the theory of evidence, as it allows to consider {unions of classes}. Binaghi et al \cite{binaghi99fuzzy} defined an empirical learning strategy for the automatic generation of Dempster-Shafer classification rules from a set of training data. Fixsen \emph{et al}. described a modified rule of combination with foundations in the theory of random sets and proved this `modified Dempster-Shafer' (MDS) \cite{fixen95modified,fixsen97modified} approach's relation to Smets' pignistic probabilities. MDS was then used to build an original classification algorithm.
Elouedi et al. \cite{elouedi00classification,elouedi00decision} adapted the classical `decision tree' technique to the presence of uncertainty on the class value, uncertainty represented by a belief function.

Several papers have been written on the application of the theory of evidence to neural network classifiers (see for instance \cite{denoeux95neural,wang98majority,loonis95multi}). In \cite{loonis95multi}, for instance, Loonis \emph{et al}. compared a multi-classifier neural network fusion scheme with the straightforward application of Dempster's rule in a pattern recognition context \cite{denoeux97analysis,ng98equalisation}. Original work has been conducted by J. Schubert, who much contributed to studying the {clustering} problem in an evidential context \cite{schubert98neural,schubert98fast,schubert99fast,schubert99simultaneous}. In his approach, $2^n-1$ pieces of evidence were clustered into $n$ clusters by minimizing a `metaconflict' function. He found neural structures more effective and much faster than optimization methods for larger problems.\\ Building on his work on clustering of nonspecific evidence, Schubert \cite{schubert97creating} developed a classification method based on the comparison with prototypes representing clusters, instead of going for a full clustering of all the evidence. The resulting computational complexity is $\mathcal{O}(M\cdot N)$, where $M$ is the maximum number of subsets and $N$ the number of prototypes chosen for each subset.

Since they both are suitable to solve classification problems \cite{wilkinson90evidential}, neural networks and belief functions are sometimes
integrated to yield more robust systems \cite{wang91neural,mohiddin94evidential}. Giacinto \emph{et al}., on their side, ran a comparison of neural networks and belief-based approaches to pattern recognition in the context of earthquake
risk evaluation \cite{giacinto97application}.

\subsection{Computer vision and pattern recognition} \label{sec:applications-vision}

Computer vision applications are still rare \cite{wesley86cv,pinz96active,boshra99accommodating,li88evidential}, although of late there seems to be a growing interest of vision scientists for approximate reasoning and related techniques \cite{guironnet06eusipco,Burger08,Kes09}.\\ Andr\'e Ayoun and Philippe Smets (\cite{ayoun01data}) used the Transferable Belief Model to quantify the conflict among sources of information in order to solve the data association problem (see Chapter \ref{cha:total} for our approach to the problem), and applied this method to the detection of sub-marines. F. Martinerie \emph{et al}. \cite{martinerie92dataassociation} proposed a solution to target tracking \cite{bogler87} from distributed sensors by modeling the evolution of a target as a Markovian process, and combining the hidden Markov model formalism with evidential reasoning in the fusion phase.

To our knowledge only a few attempts have been made to apply the theory of evidence to \emph{recognition} problems \cite{lohmann91evidential}. Ip and Chiu \cite{ip94facial} adopted DS theory to deal with uncertainties on the features used to interpret facial gestures. In an interesting work published on Computing (\cite{borotschnig98comparison}), Borotschnig, Paletta \emph{et al}. compared probabilistic, possibilistic and evidential fusion schemes for active object recognition (in which, based on tentative object hypotheses, active steps are decided until the
classification is sufficiently unambiguous), using parametric eigenspaces as representation. The probabilistic approach seemed to outperform the other, probably due to the reliability of the produced likelihoods.\\ In another paper appeared on Pattern Recognition Letters, Printz, Borotschnig \emph{et al}. \cite{pinz96active} perfected this \emph{active fusion} framework for image interpretation.

In \cite{ng98equalisation} Ng and Singh applied the data equalization technique to the output node of individual classifiers in a multi-classifier system for pattern recognition, combining outputs by using a particular kind of support function (see Definition \ref{def:support}).

Some work has been done in the \emph{segmentation} field, too. In \cite{vasseur99perceptual} Vasseur, Pegard \emph{et al}. proposed a two-stage framework to solve the segmentation task on both indoor and outdoor scenes. In the second stage, in particular, a Dempster-Shafer-style fusion technique was used to detects object in the scene by forming groups of primitive segments (perceptual organization). Similarly, B. Besserer \emph{et al}. \cite{besserer93multiple} exploited multiple sources of evidence from segmented images to discriminate among possible object classes, using Dempster's rule to update beliefs in classes.

Among evidential applications to medical imaging and diagnostics \cite{smets78theory,smets79medical,smets98application,chen93medical,liu93datafusion,deutsch91knowledge}, I. Bloch used some key features of the theory, such as its representation of ignorance and conflict computation, for the purpose of classifying multi-modality medical images \cite{bloch96aspects}. Chen \emph{et al}. used multivariate belief functions to identify anatomical structures from x-ray data (\cite{chen92spatial}).

\subsection{Sensor fusion} \label{sec:applications-fusion}

Sensor fusion applications are more common \cite{aran2009sequential,mascle98introduction,reece97qualitative}, since Dempster's rule fits naturally in an information integration schemes \cite{smets00fusion}. An and Moon, for instance, \cite{an93structure} implemented an evidential framework for representing and integrating geophysical and geological information from remote sensors. Filippidis \cite{filippidis99fuzzy} compared fuzzy and evidential reasoning in surveillance tasks (deriving actions from identity attributes, such as `friend' or `foe'), illustrating the superior performance of belief calculus. Hong \cite{hong92recursive} designed an interesting recursive algorithm for information fusion using belief functions.

Target tracking is a typical problem whose solution relies on sensor fusion. Buede \cite{buede97target} proposed a comparison between Bayesian and evidential reasoning by implementing the same target identification problem involving multiple levels of abstraction (type, class and nature). He argued from the algorithms' convergence rates the superiority of the classical approach. A similar comparison was conducted in \cite{leung00bayesian}, using real-life as well as simulated radar data.

\subsection{Robotics and autonomous navigation} \label{sec:applications-robotics}

Decision problems are very common in autonomous navigation and path planning \cite{wesley93autonomous,xia97driven,golshani96dynamic,abel88lattice}. For example, robot {localization} techniques usually exploit different types of sensors to estimate the current position of the robot on a map. The inverse problem, called `map building', consists in inferring the structure of an unknown environment from sensor data. For instance, Pagac \emph{et al}. \cite{Pagac98} examined the problem of constructing and maintaining a map (namely a simple 2D occupancy grid) of an autonomous vehicle environment and used Dempster's rule to fuse sensor readings. In a related work Gambino \emph{et al}. \cite{gambino97tbm} adopted Smets' Transferable Belief Model for sensor fusion and compared the results to those of a straightforward application of Dempster's rule. Murphy \cite{murphy98dempster}, instead, used the evidential `weight of conflict' (see Chapter \ref{cha:toe}, Section \ref{sec:weight-of-conflict}) to measure the consensus among different sensors, and attempted to integrate abstract and logical information.

\subsection{Other applications} \label{sec:applications-others}

Another field of information technology which is seeing an increasing number of application of the theory of evidence is database management \cite{lim94resolving}: in particular, data mining \cite{bergsten97applying} and concept-oriented databases \cite{dubitzky99towards}. Mc Lean \emph{et al}. showed how to represent incomplete data frequently present in databases \cite{mcclean00background} by means of mass functions, and how to integrate distributed databases \cite{mcclean97evidence} using the evidential sensor fusion scheme.\\
It is also worth citing the work of Webster \emph{et al}. \cite{websterii99vadidation} on an entropy criterion based on the theory of evidence for the validation of expert systems \cite{Biswas89,iancu97prosum,guan90rule} performance. In \cite{xu96some}, some strategies for explanations \cite{Strat87} for belief-based reasoning in the context of expert systems were suggested. 

Several applications to control theory and the theory of dynamical systems have been brought forward in recent years \cite{ramasso2007forward}. Climate change \cite{haduong06climate} is a promising testbed for theories of uncertainty as well.

Finally, economics has always experimented with mathematical models in an attempt to describe the amazing complexity of the systems it needs to describe. To cite a few, applications of evidential reasoning to project management \cite{shipley99project}, exchange rate forecasting \cite{ip91exchange} and monetary unit sampling \cite{gillett00monetary} have been proposed.

\part{Advances}

\chapter{A geometric approach to belief calculus} \label{cha:geo}

\begin{center}
\includegraphics[width = 0.45 \textwidth]{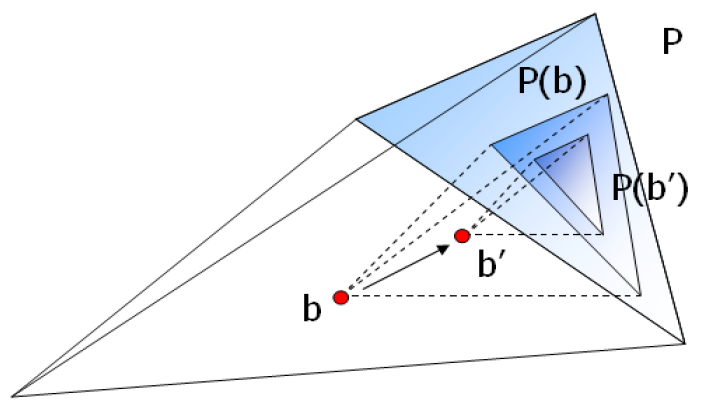}
\end{center}
\vspace{5mm}

When one tries and apply the theory of evidence to classical computer vision problems, a number of important issues arise. \emph{Object tracking} \cite{Cuzzolin99}, for instance, consists in estimating at each time instant the current configuration or `pose' of a moving object from a sequence of images of the latter. Image features can be represented as belief functions and combined to produce an estimate $\hat{q}(t)\in\tilde{\mathcal{Q}}$ of the object's pose, where $\tilde{\mathcal{Q}}$ is a finite approximation of the pose space $\mathcal{Q}$ of the object collected in a training stage (compare Chapter \ref{cha:pose}).\\ Deriving a pointwise estimate from the belief function emerging from the combination is desirable to provide an expected pose estimate - this can be done, for example, by finding the `best' probabilistic approximation of the current belief estimate and computing the corresponding expected pose. This requires a notion of `distance' between belief functions, or between a belief function and a probability distribution.

In \emph{data association} \cite{cuzzolin2000e}, a problem described in detail in Chapter \ref{cha:total}, the correspondence between moving points appearing in consecutive images of a sequence is sought. Whenever these points belong to an articulated body whose topological model is known, the rigid motion constraint acting on each link of the body can be used to obtain the desired correspondence. Since the latter can only be expressed in a conditional way, the notion of combining {conditional belief functions} in a filtering-like process emerges. Conditional belief functions can again be defined in a geometric fashion, as those objects which minimise an appropriate distance \cite{diaz06fusion,shi10distance,khatibi10new,jousselme10brest} between the original belief function and the  `conditional simplex' associated with the conditioning event $A$ (an approach developed in \cite{cuzzolin10brest,cuzzolin-geometric-conditioning,cuzzolin11isipta-conditional}).

From a more general point of view, the notion of representing uncertainty measures such as belief functions \cite{wang91geometrical} and probability distributions as points of a certain space \cite{black96examination,black97geometric,cuzzolin02fsdk,cuzzolin04smcb,monney91planar} can be appealing, as it provides a picture in which different forms of uncertainty descriptions are unified in a single geometric framework. Distances can there be measured, approximations sought, and decompositions easily calculated.\\ In this Chapter we conduct therefore a geometric analysis of the basis concepts of the theory of evidence, such as basic probability assignments and Dempster's rule, laying the foundations for such a geometric approach to uncertainty theory.

\subsection*{Chapter Outline}

A central role is played by the notion of \emph{belief space} $\mathcal{B}$, introduced in Section \ref{sec:belspace}, as the space of all the belief functions one can define on a given frame of discernment. \\ In Section \ref{sec:simplex} we characterize the relation between the focal elements of a belief function and the \emph{convex closure} operator in the belief space. In particular, we show that every belief function can be uniquely decomposed as a convex combination of `basis' or `categorical' belief functions, giving $\mathcal{B}$ the form of a \emph{simplex}, i.e., the convex closure of a set of affinely independent points.\\ In Section \ref{sec:bundle}, instead, the Moebius inversion lemma (\ref{eq:moebius}) is exploited to investigate the symmetries of the belief space. With the aid of some combinatorial results, a \emph{recursive bundle structure} of $\mathcal{B}$ is proved and an interpretation of its components (bases and fibers) in term of important classes of belief functions is provided.\\ In Section \ref{sec:global-geometry-dempster} the global behavior of Dempster's rule of combination within this geometric framework is analysed, by proving that the orthogonal sum \emph{commutes} with the convex closure operator. This allows us to give a geometric description of the set of belief functions combinable with a given b.f. $b$, and the set of belief functions that can be generated from $b$ by combination with additional evidence (its \emph{conditional subspace}).\\ In Section \ref{sec:pointwise-geometry-dempster}, instead, the \emph{pointwise} geometry of Dempster's rule is briefly studied and a geometric algorithm for Dempster's combination of two belief functions, based on the notion of `focus' of a conditional subspace, is outlined.

Finally (Section \ref{sec:applications-geometric-approach}) some of the many potential applications of the geometric approach to belief theory are discussed. In particular, we consider the computation of the canonical decomposition of a separable belief function (Section \ref{sec:canonical-decomposition}), the study of two different order relations associated with belief functions (Section \ref{sec:order-relations}) and the search for a probabilistic approximation of belief functions based on its behavior with respect to Dempster's rule of combination (Section \ref{sec:approx}).

\section{The space of belief functions} \label{sec:belspace}

Consider a frame of discernment $\Theta$ and introduce in the Euclidean space $\mathbb{R}^{|2^{\Theta|}}$ an orthonormal reference frame $\{ \vec{x}_i\}_{i=1,...,|2^{\Theta}|}$.
\begin{definition} \label{def:belief-space}
The \emph{belief space} associated with $\Theta$ is the set $\mathcal{B}_\Theta$ of vectors $\vec{v}$ of $\mathbb{R}^{|2^{\Theta|}}$ such that there exists a belief function $b:2^\Theta \rightarrow [0,1]$ whose belief values correspond to the components of $\vec{v}$, for an appropriate ordering of the subsets of $\Theta$.
\end{definition}
In the following we will drop the dependency on the underlying frame $\Theta$, and denote the belief space by $\mathcal{B}$.

\subsection{Limit simplex} \label{sec:limit-simplex}

The have a first idea of the shape of the belief space it can be useful to start understanding the geometric properties of Bayesian belief functions.

\begin{lemma}\label{lem:baysum}
Whenever $p:2^\Theta \rightarrow [0,1]$ is a Bayesian belief function defined on a frame $\Theta$, and $B$ is an arbitrary subset of $\Theta$, we have that:
\[
\sum_{A \subseteq B} p(A) = 2^{|B|-1} p(B).
\]
\end{lemma}
\begin{proof} The sum can be rewritten as $\sum_{\theta\in B}k_{\theta}p(\theta)$ where $k_{\theta}$ is the number of subsets $A$ of $B$ containing $\theta$. But $k_{\theta}=2^{|B|-1}$ for each singleton, so that:
\[
\sum_{A \subseteq B} p(A) = 2^{|B|-1} \sum_{\theta\in B} p(\theta) = 2^{|B|-1} p(B).
\]
\end{proof}

As a consequence, all Bayesian belief functions are constrained to belong to a well-determined region of the belief space.

\begin{corollary}
The set $\mathcal{P}$ of all the Bayesian belief functions which can be defined on a frame of discernment $\Theta$ is a subset of the following $|\Theta|-1$-dimensional region
\begin{equation} \label{eq:limit-simplex}
\mathcal{L} = \bigg \{ b: 2^\Theta \rightarrow [0,1] \in \mathcal{B} \;s.t. \; \sum_{A \subseteq \Theta} b(A) = 2^{|\Theta|-1} \bigg \}
\end{equation}
of the belief space $\mathcal{B}$, which we call the \emph{limit simplex}\footnote{As it can be proved that $\mathcal{L}$ is indeed a simplex, i.e., the convex closure of a number of affinely independent points (http://www.cis.upenn.edu/~cis610/geombchap2.pdf).}.
\end{corollary}

\begin{theorem} \label{the:dominated}
Given a frame of discernment $\Theta$, the corresponding belief space $\mathcal{B}$ is a subset of the region of $\mathbb{R}^{|2^{\Theta|}}$ `dominated' by the limit simplex $\mathcal{L}$:
\[
\sum_{A \subseteq \Theta} b(A) \leq 2^{|\Theta|-1},
\]
where the equality holds iff $b$ is Bayesian.
\end{theorem}
\begin{proof}
The sum $\sum_{A \subseteq \Theta} b(A)$ can be written as
\[
\sum_{A \subseteq \Theta} b(A) = \sum_{i=1}^{f} a_i \cdot m(A_i)
\]
where $f$ is the number of focal elements of $b$ and $a_i$ is the number of subsets of $\Theta$ which include the $i$-th focal element $A_i$, namely: $a_i = | \{ B \subset \Theta \; s.t. \; B \supseteq A_i \}|$.\\ Obviously, $a_i=2^{|\Theta\setminus A|}\leq 2^{|\Theta|-1}$ and the equality holds iff $|A|=1$. Therefore:
\[
\begin{array}{l}
\displaystyle \sum_{A \subseteq \Theta} b(A) = \sum_{i=1}^{f} m(A_i) 2^{|\Theta\setminus A|} \leq 2^{|\Theta|-1} \sum_{i=1}^{f} m(A_i) = 2^{|\Theta|-1} \cdot 1 = 2^{|\Theta|-1},
\end{array}
\]
where the equality holds iff $|A_i|=1$ for every focal element of $b$, i.e., $b$ is Bayesian.
\end{proof}

It is important to point out that $\mathcal{P}$ does not, in general, sell out the limit simplex $\mathcal{L}$. Similarly, the belief space does not necessarily coincide with the entire region bounded by $\mathcal{L}$. 

\subsection{Consistent probabilities and $L_1$ norm}  \label{sec:consistent-probabilities-l1}

Another hint on the structure of $\mathcal{B}$ comes from a particular property of Bayesian belief functions with respect to the classical $L_1$ distance in the Cartesian space $\mathbb{R}^{|\Theta|}$. Let $\mathcal{C}_b$ denote the core of a belief function $b$, and introduce the following order relation:
\begin{equation} \label{eq:order-relation-1}
b \geq b' \hspace{5mm} \Leftrightarrow \hspace{5mm} b(A) \geq b'(A) \hspace{5mm} \forall A \subseteq \Theta.
\end{equation}

\begin{lemma} \label{lem:core}
If $b \geq b'$, then $\mathcal{C}_b \subseteq \mathcal{C}_{b'}$.
\end{lemma}
\begin{proof}
Trivially, since $b(A)\geq b'(A)$ for every $A \subseteq \Theta$, that holds for $\mathcal{C}_{b'}$ too, so that $b(\mathcal{C}_{b'}) = 1$. But then, $\mathcal{C}_b \subseteq \mathcal{C}_{b'}$.
\end{proof}

\begin{theorem} \label{the:upper}
If $b:2^{\Theta} \rightarrow [0,1]$ is a belief function defined on a frame $\Theta$, then
\[
\| b - p \|_{L_1} = \sum_{A \subseteq \Theta} | b(A) - p(A) | = const
\]
for every Bayesian belief function $p : 2^{\Theta} \rightarrow [0,1]$ dominating $b$ according to order relation (\ref{eq:order-relation-1}).
\end{theorem}
\begin{proof}
Lemma \ref{lem:core} guarantees that $\mathcal{C}_p \subseteq \mathcal{C}_{b}$, so that $p(A) - b(A) = 1 - 1 = 0$ for every $A \supseteq \mathcal{C}_{b}$. On the other hand, if $A \cap \mathcal{C}_{b} = \emptyset$ then $p(A) - b(A) = 0 - 0 = 0$. We are left with sets which amount to the union of a non-empty proper subset of $\mathcal{C}_{b}$ and an arbitrary subset of $\Theta \setminus \mathcal{C}_{b}$. Given $A \subseteq \mathcal{C}_{b}$ there exist $2^{| \Theta \setminus \mathcal{C}_{b} |}$ subsets of the above type which contain $A$. Therefore:
\[
\sum_{A \subseteq\Theta}| b(A) - p(A)| = 2^{|\Theta\setminus \mathcal{C}_{b}|} \bigg [ \sum_{A \subseteq \mathcal{C}_{b}} p(A) - \sum_{A \subseteq \mathcal{C}_{b}} b(A) \bigg].
\]
Finally, by Lemma \ref{lem:baysum} the latter is equal to:
\begin{equation} \label{eq:star}
f(b) \doteq 2^{|\Theta \setminus \mathcal{C}_{b}|} \bigg [2^{|\mathcal{C}_{b}| - 1} - 1 - \sum_{A \subseteq \mathcal{C}_{b} } b(A) \bigg].
\end{equation}
\end{proof}

The $L_1$ distance (\ref{eq:star}) between a belief function and any Bayesian belief function $p$ dominating it is not a function of $p$, and depends only on $b$. A probability distribution satisfying the hypothesis of Theorem \ref{the:upper} is said to be \emph{consistent} with $b$ \cite{kyburg87bayesian}. Ha \emph{et al.} \cite{Ha} proved that the set $\mathcal{P}[b]$ of probability measures consistent with a given belief function $b$ can be expressed (in the {probability} simplex $\mathcal{P}$) as the sum of the probability simplexes associated with its focal elements $A_i, \;i=1,...,k$, weighted by the corresponding masses:
\[
\mathcal{P}[b] = \sum_{i=1}^{k} m(A_i) conv(A_i)
\]
where $conv(A_i)$ is the convex closure of the probabilities $\{ p_{\theta} : \theta\in A_i \}$ assigning mass 1 to a single element $\theta$ of $A_i$. The analytical form of the set $\mathcal{P}[b]$ of consistent probabilities has been further studied in \cite{cuzzolin03isipta}.

\subsection{Exploiting the Moebius inversion lemma} \label{sec:moebius}

These preliminary results suggest that the belief space may have the form of a simplex. To proceed in our analysis we need to resort to the axioms of basic probability assignments (Definition \ref{def:bpa}).\\
Given a belief function $b$, the corresponding basic probability assignment can be found by applying the Moebius inversion lemma (\ref{eq:moebius}), which we recall here:
\begin{equation} 
m(A) = \sum_{B \subseteq A}(-1)^{|A \setminus B|} b(B).
\end{equation}
We can exploit it to determine whether a point $b \in \mathbb{R}^{|2^{\Theta|}}$ corresponds indeed to a belief function, by simply computing the related b.p.a. and checking whether the resulting $m$ meets the axioms b.p.a.s must obey.

The \emph{normalization} constraint $\sum_{A \subseteq \Theta} m(A)=1$ trivially translates into $\mathcal{B} \subseteq \{ b : b(\Theta) = 1 \}$. The \emph{positivity} condition is more interesting, for it implies an inequality which echoes the third axiom of belief functions (compare Definition \ref{def:bel1} or \cite{Shafer76}, page 5):
\begin{equation} \label{eq:5}
b(A) - \sum_{B \subseteq A, |B|=|A|-1} b(B) +  \cdots + (-1)^{|A\setminus B|} \sum_{|B|=k} b(B) + \cdots + (-1)^{|A|-1} \sum_{\theta \in \Theta} b(\{\theta\}) \geq 0 \hspace{10mm} \forall A \subseteq \Theta.
\end{equation}

\subsubsection{Example: ternary frame} \label{sec:ex}

Let us see how these constraints act on the belief space in the case of a ternary frame $\Theta = \{ \theta_1,\theta_2,\theta_3 \}$. After introducing the notation
\[
\begin{array}{l}
x = b(\{\theta_1\}), \; y = b(\{\theta_2\}), \; z = b(\{\theta_3\}), \; u = b(\{\theta_1,\theta_2\}), v=s(\{\theta_1,\theta_3\}), \; w = b(\{\theta_2,\theta_3\})
\end{array}
\]
the positivity constraint (\ref{eq:5}) can be rewritten as
\begin{equation} \label{eq:4}
\mathcal{B}: \left \{
\begin{array}{l}
x\geq 0,\;\;  u\geq (x+y)\\ \\ y\geq 0,\;\; v\geq (x+z)\\
\\ z\geq 0,\;\; w\geq (y+z)\\ \\ 1-(u+v+w)+(x+y+z)\geq 0.
\end{array}
\right.
\end{equation}
Note that $b(\Theta)$ is not needed as a coordinate, for it can be recovered by normalization.

By combining the last equation in (\ref{eq:4}) with the others, it follows that the belief space $\mathcal{B}$ is the set of points $[x,y,z,u,v,w]'$ of $\mathbb{R}^{6}$ such that:
\[
\begin{array}{l}
0\leq x+y+z\leq 1,\;\;0\leq u+v+w\leq 2.
\end{array}
\]
After defining $k \doteq x+y+z$, it necessary follows that points of $\mathcal{B}$ ougth to meet:
\[
u\geq (x+y),\hspace{5mm} v\geq (x+z), \hspace{5mm} w\geq (y+z), \hspace{5mm} 2k \leq u+v+w\leq 1+k.
\]

\subsection{Convexity of the belief space} \label{sec:convexity}

Now, all the positivity constraints of Equation (\ref{eq:5}) (which determine the shape of the belief space $\mathcal{B}$) are of the form:
\[
\sum_{i\in G_1}x_i\geq \sum_{j\in G_2}x_j
\]
where $G_1$ and $G_2$ are two disjoint sets of coordinates, as the above example and Equation (\ref{eq:4}) confirm. It immediately follows that:

\begin{theorem} \label{the:convex}
The belief space $\mathcal{B}$ is convex.
\end{theorem}
\begin{proof}
Let us consider two points of the belief space $b_0, b_1 \in \mathcal{B}$ (two belief functions) and prove that all the points $b_\alpha$ of the segment $b_0 + \alpha (b_1 - b_0), \;0 \leq \alpha \leq 1$, belong to $\mathcal{B}$. Since $b_0,\;b_1$ belong to $\mathcal{B}$:
\[
\sum_{i\in G_1}x^0_i\geq \sum_{j\in G_2}x^0_j, \hspace{5mm} \sum_{i\in G_1}x^1_i\geq \sum_{j\in G_2}x^1_j
\]
where $x^0_i, x^1_i$ are the $i$-th coordinates in $\mathbb{R}^{2^{|\Theta|}}$ of $b_0, b_1$, respectively. Hence, for every point $b_\alpha$ with coordinates $x^{\alpha}_i$ we have that:
\[
\begin{array}{l}
\displaystyle \sum_{i \in G_1}x^{\alpha}_i = \sum_{i \in G_1} [x^0_i + \alpha(x^1_i - x^0_i)] = \sum_{i \in G_1} x^0_i + \alpha \sum_{i\in G_1}(x^1_i-x^0_i) =(1-\alpha)\sum_{i\in G_1}x^0_i+\alpha \sum_{i\in G_1}x^1_i \geq \\ \displaystyle \geq (1-\alpha)\sum_{j\in G_2}x^0_j+\alpha \sum_{j\in G_2}x^1_j = \sum_{j\in G_2} [x^0_j+\alpha(x^1_j-x^0_j)] = \sum_{j\in G_2}x^{\alpha}_j,
\end{array}
\]
hence $b_\alpha$ meets the same constraints. Therefore, $\mathcal{B}$ is convex.
\end{proof}

\subsubsection{Belief functions and coherent lower probabilities}

It is well-known that belief functions are a special type of \emph{coherent lower probabilities} (see Chapter \ref{cha:state}, Section \ref{sec:lower-probabilities}), which in turn can be seen as a sub-class of \emph{lower previsions} (consult \cite{Walley91}, Section 5.13). Walley proved that coherent lower probabilities are closed under convex combination --- this implies that convex combinations of belief functions (completely monotone lower probabilities) are still coherent.\\ Theorem \ref{the:convex} is a stronger result, stating that they are also completely monotone.

\subsection{Symmetries of the belief space} \label{sec:symmetry}

In the ternary example \ref{sec:ex}, the system of equations (\ref{eq:4}) exhibits a natural symmetry which reflects the intuitive partition of the variables in two sets, each associated with subsets of $\Theta$ of the same cardinality, respectively $\{x,y,z\}\sim |A|=1$ and $\{u,v,w\}\sim |A|=2$.\\ It is easy to see that the symmetry group of $\mathcal{B}$ (i.e., the group of transformations which leave the belief space unchanged) is the permutation group $S_3$, acting onto $\{x,y,z\}\times\{u,v,w\}$ via the correspondence:
\[
x \leftrightarrow w,\;\;\; y \leftrightarrow v, \;\;\; z \leftrightarrow u.
\]
This observation can be extended to the general case of a finite $n$-dimensional frame $\Theta = \{ \theta_1, \cdots, \theta_n \}$. Let us adopt here for sake of simplicity the following notation:
\[
x_ix_j...x_k\doteq b(\{\theta_i,\theta_j,...,\theta_k\}).
\]
The symmetry of the belief space in the general case is described by the following logic expression:
\[
\begin{array}{l}
\displaystyle \bigvee_{1\leq i,j \leq n}\;\;\; \bigwedge_{k=1}^{n-1}
\bigwedge_{\begin{array}{c}\{i_1,...,i_{k-1}\} \subset\{1,...,n\} \setminus \{i,j\}\end{array}} \begin{array}{ccc} x_i x_{i_1}\cdots x_{i_{k-1}}&
\leftrightarrow & x_j x_{i_1}\cdots x_{i_{k-1}}, \end{array}
\end{array}
\]
where $\bigvee (\bigwedge)$ denotes the logical or (and), while $\leftrightarrow$ indicates the permutation of
pairs of coordinates.

To see this, let us rewrite the Moebius constraints using the above notation:
\[
\displaystyle x_{i_1}\cdots x_{i_k}\geq \sum_{l=1}^{k-1}(-1)^{k-l+1}\sum_{\{j_1,...,j_{l}\} \subset \{i_1,...,i_{k}\}} x_{j_1}\cdots x_{j_l}.
\]
Focussing on the right side of the equation, it is clear that only a permutation between coordinates associated with subsets of the {same cardinality} may leave the inequality inalterate.\\
Given the triangular form of the system of inequalities (the first group concerning variables of size 1, the second one variables of size 1 and 2, and so on), permutations of size-$k$ variables are bound to be induced by permutations of variables of smaller size. Hence, the symmetries of $\mathcal{B}$ are determined by permutations of singletons. Each such swap $x_i\leftrightarrow x_j$ determines in turn a number of permutations of the coordinates related to subsets containing $\theta_i$ and $\theta_j$.

The resulting symmetry $V_k$ induced by $x_i\leftrightarrow x_j$ for the $k$-th group of constraints is then:
\[
\begin{array}{l}
(x_i\leftrightarrow x_j)\wedge \cdots \wedge (x_ix_{i_1}\cdots x_{i_{k-1}}\leftrightarrow x_j x_{i_1}\cdots
x_{i_{k-1}}) \hspace{5mm} \forall \{i_1,...,i_{k-1}\}\subset\{1,...,n\} \setminus \{i,j\}.
\end{array}
\]
Since $V_k$ is obviously implied by $V_{k+1}$, and $V_n$ is always trivial (as a simple check confirms), the
overall symmetry induced by a permutation of singletons is determined by $V_{n-1}$, and by considering all the
possible permutations $x_i \leftrightarrow x_j$ we have as desired.

In other words, the symmetries of $\mathcal{B}$ are determined by the action of the permutation group $S_n$ on the collection of cardinality-1 variables, {and} the action of $S_n$ naturally induced on higher-size variables by set-theoretical membership:
\begin{equation} \label{eq:cross}
\begin{array}{cccc}
s \in S_n : & P_k(\Theta)& \rightarrow & P_k(\Theta)\\ & x_{i_1}\cdots x_{i_k} & \mapsto & s x_{i_1}\cdots s x_{i_k},
\end{array}
\end{equation}
where $P_k(\Theta)$ is the collection of the size-$k$ subsets of $\Theta$. 

It is not difficult to recognize in (\ref{eq:cross}) the symmetry properties of a \emph{simplex}, i.e., the convex closure of a collection $v_0,v_1,...,v_{k}$ of $k+1$  of affinely independent\footnote{The points $v_0,v_1,...,v_{k}$ are said to be affinely independent iff $v_1 - v_0,...,v_{k}-v_0$ are linearly independent. \label{foot:affinely-independent}} points (vertices) of $\mathbb{R}^m$.

\section{Simplicial form of the belief space} \label{sec:simplex}

Indeeed, $\mathcal{B}$ is a simplex, with as vertices the special belief functions which assign unitary mass to a single subset of the frame of discernment.\\ Let us call \emph{categorical} belief function focused on $A \subseteq \Theta$, and denote it by $b_A$, the unique belief function with b.p.a. $m_{b_A}(A) = 1$, $m_{b_A}(B) = 0$ for all $B \neq A$.

\begin{theorem} \label{the:convex-combination}
Every belief function\footnote{Here and in the rest of the Chapter we will denote both a belief function and the vector of $\mathbb{R}^{N-2}$ representing it by $b$. This should not lead to confusion.} $b \in \mathcal{B}$ can be {uniquely} expressed as a convex combination of all the categorical belief functions:
\begin{equation} \label{eq:convex-combination}
b = \sum_{\emptyset \neq A \subsetneq \Theta} m(A) b_A,
\end{equation}
with coefficients given by the basic probability assignment $m$.
\end{theorem}
\begin{proof}
Every belief function $b$ in $\mathcal{B}$ is represented by the vector:
\[
b = \bigg [ \sum_{B \subseteq A} m(B),\; \emptyset \neq A \subsetneq \Theta \bigg ]' = \sum_{\emptyset \neq A \subsetneq \Theta} m(A) \big [ \delta(B),\; \emptyset \neq B \subsetneq \Theta \big ]' \in \mathbb{R}^{N-2},
\]
where $N \doteq |2^{\Theta}|$ and $\delta(B) = 1$ iff $B \supseteq A$. As the vector $[ \delta(B),\; B \subseteq \Theta ]'$ is the vector of belief values associated with the categorical b.f. $b_A$, we have the thesis.
\end{proof}

This `convex decomposition' property can be easily generalized in the following way.

\begin{theorem} \label{the:decompo}
The set of all the belief functions with focal elements in a given collection $\mathcal{X} \subset 2^{2^(\Theta)}$ is closed and convex in $\mathcal{B}$, namely:
\[
\big \{ b: \mathcal{E}_b \subset \mathcal{X} \big \} = Cl(\{ b_A : A \in \mathcal{X} \}),
\]
where $Cl$ denotes the convex closure of a set of points of a Cartesian space:
\begin{equation} \label{eq:convex-closure}
Cl(b_1,...,b_k) = \bigg \{ b \in \mathcal{B} : b = \alpha_1 b_1 + \cdots + \alpha_k b_k, \;\;\; \sum_i \alpha_i = 1, \; \alpha_i\geq 0\; \forall i \bigg \}.
\end{equation}
\end{theorem}
\begin{proof}
By definition:
\[
\big \{ b : \mathcal{E}_b \subset \mathcal{X} \big \} = \bigg \{ b : b = \bigg [\sum_{B \subseteq A, B \in \mathcal{E}_b} m(B), \emptyset \neq A \subsetneq \Theta \bigg ]', \mathcal{E}_b \subset X \bigg \}.
\]
But 
\[
b = \bigg [\sum_{B \subseteq A, B \in \mathcal{E}_b} m(B), \emptyset \neq A \subsetneq \Theta \bigg ]' = \sum_{B \in \mathcal{E}_b} m(B) b_B = \sum_{B \in \mathcal{X}} m(B) b_B 
\]
after extending $m$ to the elements $B \in \mathcal{X} \setminus \mathcal{E}_b$, by enforcing $m(B)=0$ for those elements. Since $m$ is a basic probability assignment, $\sum_{B\in X}m(B)=1$ and the thesis follows.
\end{proof}

As a direct consequence,

\begin{corollary} \label{cor:belief-space}
The belief space $\mathcal{B}$ is the convex closure of all the categorical belief function, namely:
\begin{equation} \label{eq:belief-space}
\mathcal{B} = Cl(b_A, \; \forall \emptyset \neq A \subseteq \Theta).
\end{equation}
\end{corollary}
As it is easy to see that the vectors $\{b_A, \emptyset \neq A \subsetneq \Theta \}$ associated with all categorical belief functions (except the vacuous one) are linearly independent, the vectors $\{b_A - b_\Theta = b_A, \emptyset \neq A \subsetneq \Theta \}$ (since $b_\Theta = 0$ is the origin of $\mathbb{R}^{N-2}$) are also linearly independent, i.e., the vertices $\{b_A, \emptyset \neq A \subseteq \Theta \}$ of the belief space (\ref{eq:belief-space}) are affinely independent. Hence:
\begin{corollary}
$\mathcal{B}$ is a simplex.
\end{corollary}

\subsection{Simplicial structure on a binary frame}

\begin{figure}[ht!]
\begin{center}
\includegraphics[width = 0.45 \textwidth]{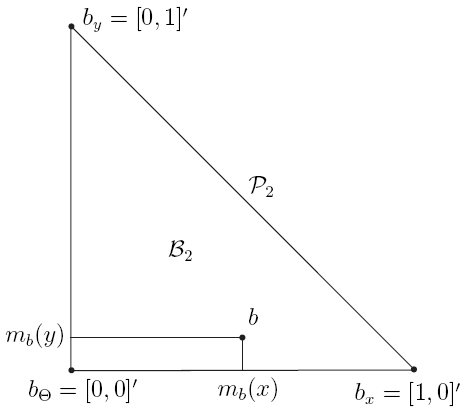}
\end{center}
\caption{\label{fig:b2} The belief space $\mathcal{B}_2$ for a binary frame is a triangle in $\mathbb{R}^2$ whose vertices are the categorical belief functions $b_x,b_y,b_\Theta$ focused on $\{x\},\{y\}$ and $\Theta$, respectively.}
\end{figure}

As an example let us consider a frame of discernment containing only two elements, $\Theta_2 = \{x,y\}$. In this very simple case each belief function $b:2^{\Theta_2}\rightarrow [0,1]$ is completely determined by its belief values $b(x),b(y)$, as $b(\Theta) = 1$ and $b(\emptyset)=0$ $\forall b$. 

We can therefore collect them in a vector of $\mathbb{R}^{N-2} = \mathbb{R}^2$ (since $N = 2^2 = 4$):
\begin{equation} \label{eq:b}
[b(x) = m(x), b(y) = m(y)]' \in \mathbb{R}^2.
\end{equation}
Since $m(x)\geq 0$, $m(y)\geq 0$, and $m(x) + m(y) \leq 1$ we can easily infer that the set $\mathcal{B}_2$ of all the possible belief functions on $\Theta_2$ can be depicted as the triangle in the Cartesian plane of Figure \ref{fig:b2}, whose vertices are the points:
\[
\begin{array}{ccc}
b_\Theta = [0,0]', & b_x = [1,0]', & b_y = [0,1]'
\end{array}
\]
(compare Equation (\ref{eq:belief-space})). These correspond (through Equation (\ref{eq:b})) to the `vacuous' belief function $b_\Theta$ ($m_{b_\Theta}(\Theta) = 1$), the categorical Bayesian b.f. $b_x$ with $m_{b_x}(x) = 1$, and the categorilca Bayesian b.f. $b_y$ with $m_{b_y}(y) = 1$, respectively.

Bayesian belief functions on $\Theta_2$ obey the constraint $m(x) + m(y) = 1$, and are therefore located on the segment $\mathcal{P}_2$ joining $b_x = [1,0]'$ and $b_y = [0,1]'$. Clearly the $L_1$ distance between $b$ and any Bayesian b.f. dominating it is constant and equal to $1 - m(x) - m(y)$ (see Theorem \ref{the:upper}).\\ The limit simplex (\ref{eq:limit}) is the region of set functions such that:
\[
b(\emptyset) + b(x) + b(y) + b(x,y) = 1 + b(x) + b(y) = 2,
\]
i.e. $b(x) + b(y) = 1$. Clearly $\mathcal{P}_2$ is a proper\footnote{The limit simplex is indeed the region of normalized sum functions (Section \ref{sec:normalized-sum-functions}) $\varsigma$ which meet the constraint $\sum_{x\in\Theta} m_\varsigma (x) = 1$} subset of the limit simplex (recall Section \ref{sec:limit-simplex}).

\section{The bundle structure of the belief space} \label{sec:bundle}

As the convexity results of Theorem \ref{the:convex} suggested us that the belief space may have the form of a simplex, the symmetry analysis of Section \ref{sec:symmetry} and the ternary example of Section \ref{sec:ex} advocate an interesting additional structure for $\mathcal{B}$.

Indeed, in the ternary example $\Theta = \{x,y,z\}$ we can note that the variables $d = [m_b(x),m_b(y),m_b(z)]'$ ($x,y,z$ in the notation of Section \ref{sec:ex}) associated with the masses of the singletons can move freely in the unitary simplex $\mathcal{D}$ (the `base space'), while the variables $m_b(\{x,y\}),m_b(\{x,z\}),m_b(\{y,z\})$ ($u,v,w$) associated with size-2 subsets are constrained to belong to a separate simplex (the `fiber') which is a function on the mass $d$ already assigned to subsets of smaller cardinality. We can express this fact by saying that there exists a projection $\pi : \mathcal{B} \rightarrow \mathcal{D}$ such that all the belief functions of a given fiber $\mathcal{F}(d)$ project onto the same point $d$ of the base space: $\mathcal{F}(d) = \{b: \pi[b] = d\}$ (see Figure \ref{fig:ternary-case}).

\begin{figure}[ht!]
\begin{center}
\includegraphics[width = 0.95 \textwidth]{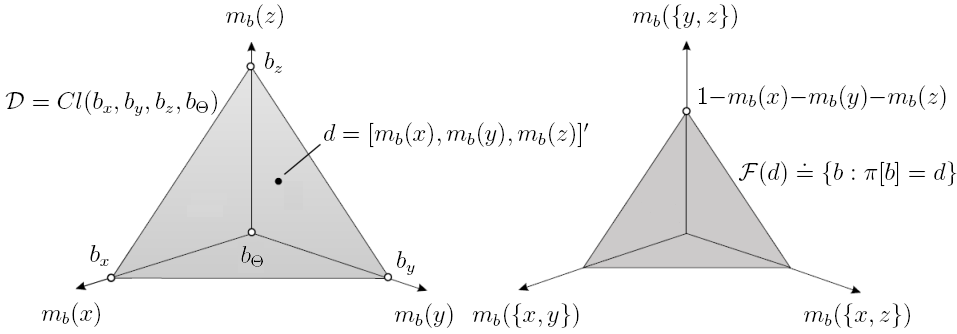}
\caption{Bundle decomposition of the belief space in the ternary case.}
\end{center}\vspace{-4mm} \label{fig:ternary-case}
\end{figure}

This decomposition is a hint of a general feature of the belief space: $\mathcal{B}$ can be \emph{recursively decomposed into bases and fibers}, each parameterized by sets of coordinates related to subsets of $\Theta$ with a same cardinality. Formally, the belief space has the structure of a \emph{fiber bundle} \cite{Novikov_russian}. 

\subsection{Fiber bundles} \label{sec:bundles}

A \emph{fiber bundle} \cite{Novikov_russian,Socolovsky94} is a generalization of the familiar Cartesian product, in which each point of the (total) space analyzed can be smoothly projected onto a \emph{base space} -- this projection determines a the decomposition of the total space into a collection of \emph{fibers} formed by points which all project onto the same element of the base.

\begin{definition}\label{def:smooth}
A smooth fiber bundle $\xi$ is a composed object $\{E, B, \pi, F, G, \mathcal{U}\}$, where
\begin{enumerate}
\item $E$ is an $s+r$-dimensional differentiable manifold called \emph{total space};
\item $B$ is an $r$-dimensional differentiable manifold called \emph{base space};
\item $F$ is an $s$-dimensional differentiable manifold called \emph{fiber};
\item $\pi:E \rightarrow B$ is a smooth application of full rank $r$ in each point of $B$, called
\emph{projection};
\item $G$ is the \emph{structure group};
\item the atlas $\mathcal{U}=\{(U_\alpha, \phi_\alpha)\}$ defines a \emph{bundle structure}; namely
\begin{itemize}
\item the base $B$ admits a covering with open sets $U_\alpha$ such that
\item $E_\alpha\doteq \pi^{-1}(U_\alpha)$ is equipped with \emph{smooth direct product coordinates}
\begin{equation}\label{eq:coordinates}
\begin{array}{cccc}
\phi_\alpha: & \pi^{-1}(U_\alpha) & \rightarrow & U_\alpha \times F\\ & e & \mapsto & (\phi'_\alpha(e),
\phi''_\alpha(e))
\end{array}
\end{equation}
satisfying two conditions:
\begin{itemize}
\item the coordinate component with values in the base space is \emph{compatible with the projection map}:
\begin{equation}\label{eq:uno}
\pi \circ \phi_\alpha^{-1}(x,f)=x
\end{equation}
or equivalently $\phi'_\alpha(e) = \pi(e)$;
\item the coordinate component with values on the fiber can be transformed, jumping from a coordinate chart into
another, by means of elements of the structure group. Formally the applications
\[
\begin{array}{cccc}
\lambda_{\alpha \beta}\doteq \phi_\beta \phi^{-1}_\alpha: & U_{\alpha \beta}\times F & \rightarrow & U_{\alpha
\beta}\times F\\ & (x,f) & \mapsto & (x,T^{\alpha \beta}(x)f)
\end{array}
\] called \emph{gluing functions} are implemented by means of transformations $T^{\alpha \beta}(x): F \rightarrow
F$ defined by applications from a domain $U_{\alpha \beta}$ to the structure group
\[T^{\alpha \beta}: U_{\alpha \beta} \rightarrow G\]
satisfying the following conditions
\begin{equation}\label{eq:due}
T^{\alpha \beta}=(T^{\beta \alpha})^{-1},\;\;\;\;\; T^{\alpha \beta} T^{\beta \gamma} T^{\gamma \alpha} =1.
\end{equation}
\end{itemize}
\end{itemize}
\end{enumerate}
\end{definition}
Intuitively, the base space is covered by a number of open neighborhoods $\{ U_\alpha \}$, which induce a similar
covering $\{E_\alpha = \pi^{-1}(U_\alpha)\}$ on the total space $E$. Points $e$ of each neighborhood $E_\alpha$ of the
total space admit coordinates separable into two parts: the first one $\phi'(e) = \pi(e)$ is the projection of $e$ onto
the base $B$, while the second part is its coordinate on the fiber $F$. Fiber coordinates are such that in the
intersection of two different charts $E_\alpha \cap E_\beta$ they can be transformed into each other by means of the
action of a group (the `structure group') $G$.

Note that in the following all the involved manifolds are linear spaces, so that each of them can be covered by a single chart. This makes the bundle structure trivial, i.e., the identity transformation. The reader can then safely ignore the gluing conditions on $\phi''_\alpha$.

\subsection{Combinatorial facts}

To prove the bundle decomposition of the belief space $\mathcal{B}$ we first need a simple combinatorial result.

\begin{lemma} \label{lem:combin2}
The following inequality holds:
\[
\sum_{|A| = i} b(A)\leq \displaystyle 1 + \sum_{m = 1}^{i-1} (-1)^{i - (m+1)}\binom{ n-(m+1)}{i-m} \sum_{|B| = m} b(B), 
\]
and the upper bound is reached whenever 
\[
\sum_{|A| = i} m_b(A) = 1 - \sum_{|A|<i} m_b(A).
\]
\end{lemma}
\begin{proof}
Since $\binom{n-m}{i-m}$ is the number of subsets of size $i$ containing a fixed set $B,\;|B|=m$ in a frame with $n$ elements, we can write:
\begin{equation} \label{eq:lemma-0}
\begin{array}{lll}
\displaystyle
\sum_{|A|=i} b(A) & = & \displaystyle \sum_{|A|=i} \sum_{B \subseteq A} m_b(B) = \sum_{m=1}^i \sum_{|B|=m}  \binom{n-m}{i-m} m_b(B) \\ & = & \displaystyle \sum_{|B|=i} m_b(B) + \sum_{m=1}^{i-1} \sum_{|B|=m} \binom{n-m}{i-m} m_b(B) \\ & \leq & \displaystyle 1 - \sum_{|B|<i} m_b(B) + \sum_{m=1}^{i-1} \sum_{|B|=m} \binom{n-m}{i-m} m_b(B),
\end{array}
\end{equation}
as $\displaystyle \sum_{|B|=i} m_b(B) = 1 - \sum_{|B|<i} m_b(B)$ by normalization. By M\"obius inversion (\ref{eq:moebius}):
\begin{equation} \label{eq:lemma-1}
\begin{array}{lll}
\displaystyle \sum_{|A|<i} m_b(A) = \sum_{|A|<i} \sum_{B \subseteq A}(-1)^{|A \setminus B|} b(B) = \sum_{|A| = m = 1}^{i - 1}\sum_{|B| = l = 1}^m (-1)^{m - l} \binom{n - l}{m - l} \sum_{|B| = l} b(B)
\end{array}
\end{equation}
for, again, $\binom{n-l}{m-l}$ is the number of subsets of size $m$ containing a fixed set $B,\;|B|=l$ in a frame with $n$ elements. The role of the indexes $m$ and $l$ can be exchanged, obtaining:
\begin{equation} \label{eq:lemma-11}
\sum_{|B|=l=1}^{i-1} m_b(B) = \sum_{|B|=l=1}^{i-1} \bigg[ \sum_{|B|=l} b(B) \cdot \sum_{m=l}^{i-1} (-1)^{m-l}\binom{n-l}{m-l} \bigg].
\end{equation}
Now, a well known combinatorial identity (\cite{Gould}, volume 3, Equation (1.9)) states that, for $i-(l+1)\geq 1$:
\begin{equation} \label{eq:lemma-2}
\sum_{m=l}^{i-1} (-1)^{m-l}\binom{n-l}{m-l}=(-1)^{i-(l+1)}\binom{n-(l+1)}{i-(l+1)}.
\end{equation}
By applying (\ref{eq:lemma-2}) to the last equality, (\ref{eq:lemma-1}) becomes:
\begin{equation}\label{eq:lemma-3}
\sum_{|B|=l=1}^{i-1} \bigg [ \sum_{|B|=l} b(B) \cdot (-1)^{i-(l+1)}\binom{n-(l+1)}{i-(l+1)} \bigg ].
\end{equation}
Similarly, by (\ref{eq:lemma-11}) we have:
\[
\begin{array}{lll}
\displaystyle \sum_{m=1}^{i-1} \sum_{|B|=m} \binom{n-m}{i-m} m_b(B) & = & \displaystyle \sum_{l = 1}^{i-1} \sum_{|B|=l} b(B) \cdot \sum_{m=l}^{i-1} (-1)^{m-l} \binom{n-l}{m-l} \binom{n-m}{i-m} \\ & = & \displaystyle \sum_{l = 1}^{i-1} \sum_{|B|=l} b(B) \cdot \sum_{m=l}^{i-1} (-1)^{m-l} \binom{i-l}{m-l} \binom{n-l}{i-l},
\end{array}
\]
as it is easy to verify that $\displaystyle \binom{n-l}{m-l} \binom{n-m}{i-m} = \binom{i-l}{m-l} \binom{n-l}{i-l}$.\\ By applying (\ref{eq:lemma-2}) again to the last equality we get:
\begin{equation}\label{eq:lemma-4}
\sum_{m=1}^{i-1} \sum_{|B|=m} \binom{n-m}{i-m} m_b(B) = \sum_{l = 1}^{i-1} \sum_{|B|=l} (-1)^{i-(l+1)} \binom{n-l}{i-l}.
\end{equation}
By replacing (\ref{eq:lemma-11}) and (\ref{eq:lemma-4}) in (\ref{eq:lemma-0}) we get the thesis.
\end{proof}

The bottom line of Lemma \ref{lem:combin2} is that, having assigned mass to events of size $1,...,i-1$ the upper
bound for $\sum_{|A| = i} b(A)$ is obtained by assigning all the remaining mass to the collection of size-$i$ subsets.

\subsection{Normalized sum functions} \label{sec:normalized-sum-functions}

The definition of fiber bundle (see Appendix) requires the involved spaces (bases and fibers) to be \emph{manifolds}. The ternary case suggests instead that the belief space decomposes into \emph{simplices}. The idea of recursively assigning mass to subsets of increasing size, however, does not necessarily require the mass itself to be positive. 

Indeed, each vector $v = [v_A, \emptyset \subsetneq A \subsetneq \Theta]' \in \mathbb{R}^{N-2}$ can be thought of as a set function 
\[
\varsigma:2^\Theta \setminus \emptyset \rightarrow {\mathbb{R}}\; s.t. \;\varsigma(A) = v_A, \; \varsigma(\Theta) = 1 - \sum_{A\neq \Theta} v_A.
\] 
By applying the M\"obius transformation (\ref{eq:moebius}) to such a $\varsigma$ we obtain another set function $m_{\varsigma}:2^\Theta\setminus \emptyset \rightarrow \mathbb{R}$ with $\varsigma(A)=\sum_{B\subseteq A} m_\varsigma(B)$ (as it is the case for belief functions, compare (\ref{eq:belvalue})). Contrarily to basic probability assignments, however, the M\"obius inverse $m_\varsigma$ of such a \emph{normalized sum function} (n.s.f.) \cite{cuzzolin04smcb,cuzzolin14annals} $\varsigma$ is not guaranteed to meet the non-negativity constraint: 
\[
m_\varsigma(A) \not\geq 0 \hspace{5mm} \forall A\subseteq\Theta.
\]
Geometrically, normalized sum functions correspond to arbitrary points of $\mathcal{S} = \mathbb{R}^{N-2}$. 

\subsection{Recursive convex fiber bundle structure} \label{sec:recursive-bundle-structure}

If we recursively assign mass to normalized sum functions, we obtain a classical fiber bundle structure for the space $\mathcal{S}$ of all n.s.f.s on $\Theta$, in which all the involved bases and fibers are `conventional' {linear} spaces. 

\begin{theorem} \label{the:bundle-nsf}
The space $\mathcal{S} = \mathbb{R}^{N-2}$ of all the sum functions $\varsigma$ with domain on a finite frame $\Theta$
of cardinality $|\Theta| = n$ has a recursive fiber bundle structure. Namely, there exists a sequence of smooth fiber bundles
\[
\begin{array}{ccc}
\xi_i = \Big \{ \mathcal{F}_\mathcal{S}^{(i-1)}, \mathcal{D}_\mathcal{S}^{(i)}, \mathcal{F}_\mathcal{S}^{(i)}, \pi_i \Big \}, & &
i = 1,..., n-1
\end{array}
\]
where $\mathcal{F}_\mathcal{S}^{(0)} = \mathcal{S} = \mathbb{R}^{N-2}$, the \emph{total space} $\mathcal{F}_\mathcal{S}^{(i-1)}$, the
\emph{base space} $\mathcal{D}_\mathcal{S}^{(i)}$ and the \emph{fiber} $\mathcal{F}_\mathcal{S}^{(i)}$ of the $i$-th
bundle level are linear subspaces of $\mathbb{R}^{N-2}$ of dimension $\sum_{k = i}^{n-1} \binom{n}{k}, \binom{n}{i},
\sum_{k = i+1}^{n-1} \binom{n}{k}$ respectively. 

Both $\mathcal{F}_\mathcal{S}^{(i-1)}$ and $\mathcal{D}_\mathcal{S}^{(i)}$ admit a \emph{global} coordinate chart. As
\[
\dim\mathcal{F}_\mathcal{S}^{(i-1)} = \sum_{k=i,...,n-1} \binom{n}{k} = \Big | \Big \{ A \subset \Theta : i \leq |A| < n \Big \} \Big |,
\]
each point $\varsigma^{i-1}$ of $\mathcal{F}_\mathcal{S}^{(i-1)}$ can be written as
\[
\varsigma^{i-1} = \Big [ \varsigma^{i-1}(A), A\subset\Theta, i\leq |A| < n \Big ]'
\]
and the \emph{smooth direct product coordinates}
(\ref{eq:coordinates}) at the $i$-th bundle level are
\[
\begin{array}{cc}
\phi'(\varsigma^{i-1}) = \Big \{ \varsigma^{i-1}(A),\; |A|=i \Big \}, & \phi''(\varsigma^{i-1}) = \Big \{ \varsigma^{i-1}(A),\; i<|A|<n \Big \}.
\end{array}
\]
The projection map $\pi_i$ of the i-th bundle level is a full-rank differentiable application
\[
\begin{array}{cccc}
\pi_{i}:& \mathcal{F}_\mathcal{S}^{(i-1)} & \rightarrow & \mathcal{D}_\mathcal{S}^{(i)}\\ & \varsigma^{i-1} & \mapsto &
\pi_i [\varsigma^{i-1}]
\end{array}
\]
whose expression in this coordinate chart is
\begin{equation}\label{eq:projection}
\pi_i [\varsigma^{i-1} ] = [ \varsigma^{i-1}(A), |A|= i]'.
\end{equation}
\end{theorem}

\begin{proof} (sketch)
The bottom line of the proof of Theorem \ref{the:bundle-nsf} is that {the mass associated with a sum function can be recursively assigned to subsets of increasing size}. The proof is done by induction \cite{cuzzolin14annals}.
\end{proof}

As the belief space is a simplex immersed in $\mathcal{S} = \mathbb{R}^{N-2}$, the fibers of $\mathbb{R}^{N-2}$ do intersect the space of belief functions too. The belief space $\mathcal{B}$ then inherits some sort of bundle structure from the Cartesian space in which it is immersed. 

\begin{theorem} \label{the:bundle-bf}
The belief space $\mathcal{B} \subset \mathcal{S} = \mathbb{R}^{N-2}$ inherits by intersection with the recursive
bundle structure of $\mathcal{S}$ a `convex'-bundle decomposition. Each $i$-th level `fiber' can be expressed as
\begin{equation}\label{eq:fiber}
\mathcal{F}_\mathcal{B}^{(i-1)}(d^1,...,d^{i-1}) = \Big \{ b \in \mathcal{B} : V_i \wedge \cdots \wedge V_{n-1}
(d^1,...,d^{i-1}) \Big \},
\end{equation}
where $V_i(d^1,...,d^{i-1})$ denotes the system of constraints
\begin{equation}\label{eq:vi}
V_i(d^1,...,d^{i-1}) : \left\{
\begin{array}{ll}
m_b(A) \geq 0 & \forall A \subseteq \Theta : |A|=i, \\ \displaystyle \sum_{|A|=i} m_b(A) \leq 1 - \sum_{|A|<i} m_b(A) &
\end{array}
\right.
\end{equation}
and depends on the mass assigned to lower size subsets $d^m = [ m_b(A), |A| = m ]'$, $m = 1,...,i-1$. 

The corresponding $i$-th level convex `base' $\mathcal{D}_\mathcal{B}^{(i)}(d^1,...,d^{i-1})$ can be expressed in terms of basic probability assignments as the collection of b.f.s $b \in \mathcal{F}^{(i-1)}(d^1,...,d^{i-1})$ such that
\begin{equation}\label{eq:base-m}
\left\{ \begin{array}{ll} m_b(A) = 0, & \forall A : i<|A|<n \\ \\ m_b(A) \geq 0, & \forall A : |A|=i \\ \\ \displaystyle \sum_{|A| = i} m_b(A) \leq 1 - \sum_{|A|<i} m_b(A). & \end{array} \right.
\end{equation}
\end{theorem}

\begin{proof} (sketch)
The proof is based on understanding the effect on $\mathcal{B}\subset\mathcal{S}$ of the bundle decomposition of the space of normalized sum functions $\mathcal{S} = \mathbb{R}^{N-2}$. This is done by applying the non-negativity $m_\varsigma\geq 0$ and normalization $\sum_A m_\varsigma(A) = 1$ constraints on the admissible values of the coordinates of points of $\mathcal{S}$, recursively to collections of subsets of the same size. For a full proof see \cite{cuzzolin14annals}.
\end{proof}

The intersections of the fibers of $\mathcal{S} =\mathbb{R}^{N-2}$ with the simplex $\mathcal{B}$ are themselves simplices: bases and fibers in the belief space case are polytopes, rather than linear spaces.

%\noindent Figure \ref{fig:bundle} summarizes our knowledge of the bundle structure of the belief space. $\mathcal{S}$ can be decomposed into a base $\mathcal{D}^{(1)}$ (the simplex $u=(x+y),\;v=(x+z),\;w=(y+z)$ in Example \ref{sec:ex}) whose points are glued to a fiber ($R_P$ in the ternary case) whose dimension reduces to zero at the upper border $\mathcal{P}^{(1)}$ of $\mathcal{D}^{(1)}$. This decomposition recursively applies to the fibers, for $i=1,...,n-1$.\\ It is interesting to point out that the elements of this decomposition have an intuitive meaning. For instance, $\mathcal{P}^{(1)}= \mathcal{P}$ is the set of the Bayesian belief functions, while $\mathcal{D}^{(1)}$ coincides to the collection of the \emph{discounted} probabilities (see \cite{Shafer76}).

\subsection{Bases and fibers as simplices} \label{sec:simplices}

We have seen that the  $(i-1)$-th level fiber $\mathcal{F}_\mathcal{B}^{(i-1)}(d^1,...,d^{i-1})$ of $\mathcal{B}$ admits a pseudo-bundle structure whose pseudo-base space is $\mathcal{D}_\mathcal{B}^{(i)}(d^1,\cdots,d^{i-1})$ given by Equation (\ref{eq:base-m}). Let us denote by $k = \sum_{|A|<i} m_A$ the total mass already assigned to lower size events, and call
\[
\begin{array}{l}
\displaystyle \mathcal{P}^{(i)}(d^1,...,d^{i-1}) \doteq \Big\{ b \in \mathcal{F}_\mathcal{B}^{(i-1)}(d^1,...,d^{i-1}) :
\sum_{|A|=i} m_b(A) = 1 - k \Big\} \\ \displaystyle \mathcal{O}^{(i)}(d^1,...,d^{i - 1}) \doteq \Big\{ b \in
\mathcal{F}_\mathcal{B}^{(i-1)}(d^1,...,d^{i-1}) : m_b(\Theta) = 1 - k \Big\}
\end{array}
\]
the collections of belief functions on the fiber $\mathcal{F}_\mathcal{B}^{(i-1)}(d^1,...,d^{i-1})$ assigning all the remaining basic probability $1 - k$ to subsets of size $i$ or to $\Theta$, respectively.\\ 

Each belief function $b \in \mathcal{F}_\mathcal{B}^{(i-1)}(d^1,...,d^{i-1})$ on such a fiber can be written as:
\[
b = k \sum_{|A|<i} m_{b_0}(A) b_A + (1 - k) \sum_{|A|\geq i} m_{b'}(A) b_A = k b_0 + (1-k) b'
\]
where $b_0 \in Cl(b_A : |A|<i)$ and $b' \in Cl(b_A : |A|\geq i)$. Note that $b_0$ is the same for all the b.f.s on the fiber, while the second component $b'$ is free to vary in $Cl(b_A: |A|\geq i)$. 

It can be proven that the following convex expressions for $\mathcal{F}_\mathcal{B}^{(i-1)},\mathcal{P}^{(i)}$ and $\mathcal{O}^{(i)}$ (neglecting for sake of simplicity the dependence on $d^1,...,d^{i-1}$ or, equivalently, on $b_0$) hold \cite{cuzzolin14annals}:
\begin{equation}\label{eq:elements}
\begin{array}{lll}
\mathcal{F}_\mathcal{B}^{(i-1)} & = & \Big\{ b = k b_0 + (1-k) b', b'\in Cl(b_A, |A|\geq i) \Big\} = k b_0 + (1-k) Cl(b_A, |A|\geq i), \\ \\ \mathcal{P}^{(i)} & = & k b_0 + (1 - k) Cl(b_A : |A| = i),\\ \\ \mathcal{O}^{(i)} & = & k b_0 + (1 - k) b_\Theta,
\end{array}
\end{equation}
and we can write:
\[
\mathcal{D}_\mathcal{B}^{(i)} = Cl(\mathcal{O}^{(i)}, \mathcal{P}^{(i)}).
\]
As a consequence, the elements of the convex bundle decomposition of $\mathcal{B}$ possess a natural meaning in terms of belief values. In particular, ${\mathcal{P}}^{(1)}= {\mathcal{P}}$ is the set of all the Bayesian belief functions, while ${\mathcal{D}}^{(1)}$ is the collection of all the \emph{discounted} probabilities \cite{Shafer76}, i.e., belief functions of the form $(1-\epsilon) p + \epsilon b_\Theta$, with $0 \leq \epsilon \leq 1$ and $p \in \mathcal{P}$. 

B.f.s assigning mass to events of cardinality smaller than a certain size $i$ are called in the literature \emph{$i$-additive belief functions} \cite{miranda04ipmu}. The set $\mathcal{P}^{(i)}$ (\ref{eq:elements}) is nothing but the collection of all $i$-additive b.f.s. The $i$-th level base of $\mathcal{B}$ can then be interpreted as the region of all `discounted' $i$-additive belief functions.

\section{Global geometry of Dempster's rule} \label{sec:global-geometry-dempster}

Once established the geometrical properties of belief functions as set functions, we take a step forward and analyze the behaviour
of the rule of combination in the framework of the belief space.

\subsection{Commutativity} \label{sec:commutativity}

In \cite{cuzzolin08smcc} we proved the following fundamental results on the relationship between Dempster's sum and the convex combination of belief functions as points of a Cartesian space.

\begin{proposition}\label{pro:concom}
Consider a belief function $b$ and a collection of b.f.s $\{ b_1,\cdots , b_n \}$ such that at least one of them is combinable with $b$. If $\sum_i \alpha_i = 1$, $\alpha_i\geq 0$ for all $i = 1,...,n$\footnote{Here $n$ is only an index, with nothing to do with the cardinality of the frame on which the belief functions are defined.} then
\[
b\oplus \sum_i \alpha_i b_i = \sum_i \beta_i (b\oplus b_i),
\]
where
\begin{equation}\label{eq:beta}
\beta_i = \frac{\alpha_i k(b,b_i)}{\sum_{j=1}^n \alpha_j k(b,b_j)}
\end{equation}
and $k(b,b_i)$ is the normalization factor for the sum $b \oplus b_i$: $k(b,b_i) \doteq \sum_{A \cap B \neq \emptyset} m_b(A) m_{b_i}(B)$.
\end{proposition}

Proposition \ref{pro:concom} can be used to prove that convex closure and Dempster's sum \emph{commute} \cite{cuzzolin08smcc}, i.e., the order of their action on a set of b.f.s can be swapped.

\begin{theorem}\label{the:commutativity}
$Cl$ and $\oplus$ commute in the belief space. Namely, if $b$ is combinable (in Dempster's sense) with $b_i,\;\forall i=1,...,n$ then:
\[
b \oplus Cl(b_1,\cdots,b_n) = Cl(b \oplus b_1,\cdots, b \oplus b_n).
\]
\end{theorem}
\begin{proof}

\emph{Sufficiency.}
We need to prove that if $b' \in b\oplus Cl(b_1,...,b_n)$ then $b' \in Cl(b\oplus b_{1},...,b\oplus b_{n})$. If $b'= b \oplus \sum_{i=1}^n \alpha_i b_i$, $\sum_i \alpha_i = 1$, $\alpha_i \geq 0$, then by Proposition \ref{pro:concom}:
\[
b' = \sum_{i=1}^n \beta_i b \oplus b_i \in Cl(b\oplus b_{1},...,b\oplus b_{n}),
\]
as $\beta_i$ given by Equation (\ref{eq:beta}) is such that $\sum_i \beta_i = 1$, $\beta_i \geq 0$ for all $i$.

\emph{Necessity.}
We have to show that if $b' \in Cl(b\oplus b_{1}, ..., b\oplus b_{n})$ then $b' \in b\oplus Cl(b_{1},...,b_{n})$.
An arbitrary element of $Cl(b\oplus b_{1}, ..., b\oplus b_{n})$ has the form:
\begin{equation} \label{eq:tmp4}
\sum_{i = 1}^n \alpha'_i b\oplus b_{i}
\end{equation}
for some set of coefficients $\alpha'_i$ such that $\sum_i \alpha'_i = 1$, $\alpha'_i \geq 0$.\\ On the other hand, elements of $b \oplus \Big( \sum_{i = 1}^n \alpha_i b_{i} \Big)$, $\sum_i \alpha =1$, $\alpha \geq 0$ for all $i$, have the form:
\[
\sum_{i = 1}^n \beta_i b\oplus b_{i}
\]
where $\beta_i$ is given by Equation (\ref{eq:beta}).

Hence, any belief function $b'$ of the form (\ref{eq:tmp4}) with coefficients $\{ \alpha'_i,\;i = 1,...,n\}$ belongs to the region $b \oplus Cl(b_{1} , \cdots , b_{n})$ iff we can find another collection of coefficients $\{ \alpha_i,\;i = 1,...,n\}$, $\sum_i \alpha_i = 1$, such that the following constraints are met:
\begin{equation}\label{eq:tmp5}
\alpha'_i = \beta_i = \frac{\alpha_i k_i}{\sum_i \alpha_i k_i} \hspace{5mm} \forall i = 1,...,n,
\end{equation}
where $k_i = k(b,b_i)$ for all $i$.

An admissible solution to the system of equations (\ref{eq:tmp5}) is $\tilde{\alpha}_i \doteq \beta_i / k_i$, as $\forall i$ $\beta_i = \beta_i /\sum_i \beta_i = \beta_i$ (since the $\beta_i$s sum to one), and system (\ref{eq:tmp5}) is satisfied up to the normalization constraint.\\ We can further normalize the solution by setting:
\[
\alpha_i = \tilde{\alpha_i}/\sum_{j} \tilde{\alpha_{j}} = \frac{\beta_i}{k_i \sum_{j} (\frac{\beta_{j}}{k_{j}})},
\] 
for which (\ref{eq:tmp5}) is still met.
\end{proof}

\subsection{Conditional subspaces} \label{sec:conditional-subspaces}

The fact that orthogonal sum and convex closure commute is a powerful tool, for it provides us with a simple language in which to express the geometric interpretations of the notions of combinability and conditioning.

\begin{definition} \label{def:conditional-subspace}
The \emph{conditional subspace} $\langle b \rangle$ associated with a belief function $b$ is the set of all the belief functions \emph{conditioned by $b$}, namely
\begin{equation}
\langle b \rangle \doteq \big\{ b\oplus b', \forall b'\in{\mathcal{B}}\;s.t.\;\exists\;b\oplus b'
\big\}.
\end{equation}
\end{definition}
In other words, the conditional subspace $\langle b \rangle$ is the possible `future' of the imprecise knowledge state encoded by a belief function $b$, under the assumption that evidence combination follows Dempster's rule.

Since belief functions are not necessarily combinable, we first need to understand the geometry of the notion of combinability.
\begin{definition} \label{def:nc}
The \emph{non-combinable region} $NC(b)$ associated with a belief function $b$ is the collection of all the b.f.s which are not combinable with $b$, namely:
\[
NC(b)\doteq \{b':\nexists b'\oplus b\} = \{ b' : k(b,b') = 0 \}.
\]
\end{definition}
The results of Section \ref{sec:simplex} once again allow us to understand the shape of this set. As a matter of fact the non-combinable region $NC(b)$ of $b$ is also a simplex, whose vertices are the categorical belief functions related to subsets disjoint from the core $\mathcal{C}_b$ of $b$ (the union of its f.e.s) \cite{cuzzolin08smcc}.
\begin{proposition}
$NC(b) = Cl(b_A, A \cap \mathcal{C}_b = \emptyset)$.
\end{proposition}
Clearly, as the vertices of a simplex are affinely independent (see Footnote \ref{foot:affinely-independent}), the dimension of the linear space generated by $NC(b)$ is $2^{|\Theta\setminus {\mathcal{C}}_b|}-2$. Using Definition \ref{def:nc} we can write:
\[
\langle b \rangle = b\oplus ({\mathcal{B}}\setminus NC(b)) = b\oplus\{b': {\mathcal{C}}_{b'}\cap {\mathcal{C}}_b \neq \emptyset\},
\]
where $\setminus$ denotes, as usual, the set-theoretic difference $A\setminus B = A\cap \overline{B}$.

Unfortunately, ${\mathcal{B}}\setminus NC(b)$ does not satisfy Theorem \ref{the:decompo}: for a b.f. $b'$ to be compatible with $b$ it suffices for it to have \emph{one} focal element intersecting the core ${\mathcal{C}}_b$, not necessarily all of them. Geometrically, this means that $\mathcal{B}\setminus NC(b)$ \emph{is not a simplex}. Hence, we cannot apply the commutativity results of Section \ref{sec:commutativity} directly to $\mathcal{B}\setminus NC(b)$ to find the shape of the conditional subspace.

Nevertheless, $\langle b \rangle$ can still be expressed as a Dempster's sum of $b$ and a simplex.

\begin{definition} \label{def:compatible-simplex}
The \emph{compatible simplex} $C(b)$ associated with a belief function $b$ is the collection of all
b.f.s whose focal elements are in the core of $b$:
\[
C(b)\doteq \{b' : {{\mathcal{C}}}_{b'} \subseteq {{\mathcal{C}}}_b\} = \{ b' : \mathcal{E}_{b'}
\subseteq 2^{\mathcal{C}_b} \}.
\]
\end{definition}

From from Theorem \ref{the:decompo} it follows that
\begin{corollary}$C(b) = Cl(b_A: A\subseteq {{\mathcal{C}}}_b)$.
\end{corollary}
The compatible simplex $C(b)$ is only a \emph{proper} subset of the collection of b.f.s combinable with $b$, ${\mathcal{B}}\setminus NC(b)$ --- nevertheless, \emph{it contains all the relevant information}. As a matter of fact:
\begin{theorem} \label{the:compa}
The conditional subspace $\langle b \rangle$ associated with a belief function $b$ coincides with the orthogonal sum of $b$ and the related compatible simplex $C(b)$, namely:
\[
\langle b \rangle = b \oplus C(b).
\]
\end{theorem}
\begin{proof}
Let us denote by $\mathcal{E}_{b} = \{A_i,i\}$ and $\mathcal{E}_{b'} = \{B_j,j\}$ the lists of focal elements of $b$ and $b'$, respectively. By definition $A_i = A_i \cap \mathcal{C}_b$ so that $B_j \cap A_i = B_j \cap (A_i \cap \mathcal{C}_b) = (B_j\cap\mathcal{C}_b)\cap A_i$. Once defined a new belief function $b^{''}$ with focal elements $\{ B'_k,\;k = 1,...,m \} \doteq \{B_j \cap \mathcal{C}_b,\; j = 1,...,|\mathcal{E}_{b'}| \}$ (note that $m\leq |\mathcal{E}_{b'}|$ since some intersections may coincide) and basic probability assignment
\[
m_{b^{''}}(B'_k) = \sum_{j : B_j \cap \mathcal{C}_b = B'_k} m_{b'}(B_j)
\]
we have that $b \oplus b' = b \oplus b^{''}$.
\end{proof}

We are now ready to understand the convex geometry of conditional subspaces. From Theorems \ref{the:decompo} and \ref{the:compa} it follows that:
\begin{corollary}\label{cor:cond}
The conditional subspace $\langle b \rangle$ associated with a belief function $b$ is the convex closure of the orthogonal sums involving $b$ and all the categorical belief functions compatible with it. Namely:
\[
\langle b \rangle = Cl(b\oplus b_A, \forall A\subseteq {{\mathcal{C}}}_b).
\]
\end{corollary}
Note that, since $b \oplus b_{{\mathcal{C}}_b} = b$ (where $b_{{\mathcal{C}}_b}$ is the categorical belief function focused on the core of $b$), $b$ is always one of the vertices of $\langle b \rangle$. Furthermore, $\langle b \rangle \subseteq C(b)$, since the core of a belief function $b$ is a monotone function on the partially ordered set $(\mathcal{B},\geq_{\oplus})$, namely \cite{Shafer76} $\mathcal{C}_{b\oplus b'} = \mathcal{C}_b \cap \mathcal{C}_{b'} \subseteq \mathcal{C}_b$. 

Furthermore:
\begin{equation}\label{eq:dim}
\dim (\langle b \rangle) = 2^{|{\mathcal{C}}_b|} - 2,
\end{equation}
as the dimension of the linear space generated by $\langle b \rangle$ is simply the cardinality of $C(b)$ (note that $\emptyset$ is not included) minus 1. We can observe that, in general:
\[
\dim(NC(b)) + \dim(\langle b \rangle) \neq \dim(\mathcal{B}).
\]

\section{Pointwise geometry of Dempster's rule} \label{sec:pointwise-geometry-dempster}

Corollary \ref{cor:cond} depicts, in a sense, the \emph{global} behavior of the rule of combination in the belief space, as it describes the form of the collection of all the possible outcomes of the combination of new evidence with a given belief function. We still do not know understand the \emph{pointwise} geometric behavior of $\oplus$, i.e., how the location (in the belief space $\mathcal{B}$) of a Dempster's sum $b_1 \oplus b_2$ is related to that of the belief functions $b_1,b_2$ to combine. In this Section we will analyze the simple case of a binary frame $\Theta = \{x,y\}$, and recall the recent general results on the matter \cite{cuzzolin04smcb}.

\subsection{Binary case} \label{sec:example-binary-dempster}

Given two belief functions $b_1 = [m_1(x), m_1(y)]'$ and $b_2 = [m_2(x), m_2(y)]'$ defined on a binary frame $\Theta = \{ x,y \}$, it is straightforward to derive the point $b_1 \oplus b_2 = [m(x), m(y)]'$ of the belief space $\mathcal{B}_2$ which corresponds to their orthogonal sum:

\begin{equation} \label{eq:dempster2}
\begin{array}{l}
\displaystyle m(x) = 1 - \frac{(1 - m_1(x))(1 - m_2(x))}{1 - m_1(x) m_2(y) - m_1(y) m_2(x) }\\ \\ \displaystyle m(y) = 1 - \frac{(1 - m_1(y))(1 - m_2(y))}{1 - m_1(x) m_2(y) - m_1(y) m_2(x)}.
\end{array}
\end{equation}

Let us fix the first belief function $b_1$, and analyze the behaviour of $b_1 \oplus b_2$ as a function of of $b_2$ (or, equivalently, of the two variables $m_2(x), m_2(y)$). If we assume $m_2(x)$ constant in Equation (\ref{eq:dempster2}), the combination $b_1 \oplus b_2 = [m(x), m(y)]'$ describes a line segment in the belief space. Analogously, if we keep $m_2(y)$ constant the combination describes a different segment.\\ These facts are illustrated in Figure \ref{fig:dempster}.

\begin{figure}[ht!]
\includegraphics[width = 0.65 \textwidth]{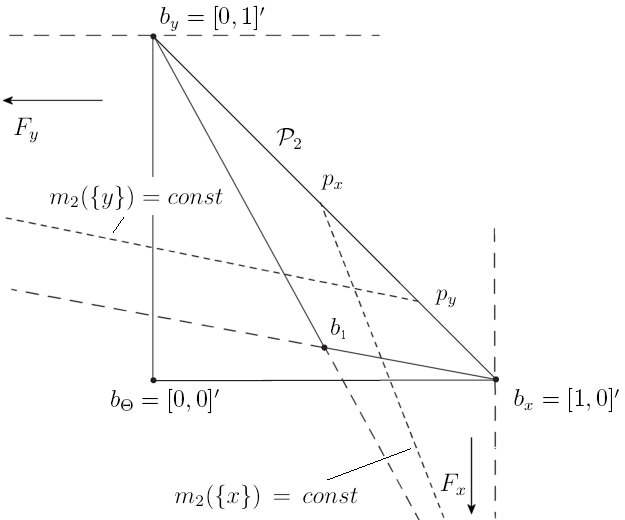}
\caption{\label{fig:dempster} Pointwise geometrical representation of Dempster's rule in the binary belief space $\mathcal{B}_2$.}
\end{figure}

As we know from Section \ref{sec:global-geometry-dempster}, the region of all orthogonal sums involving $b_1$ is a triangle whose sides are the probability simplex $\mathcal{P}$, the locus $\{ b = b_1 \oplus b_2 \; : \; m_2(\{ x \}) = 0 \}$ and its dual $\{ b = b_1 \oplus b_2 \; : \; m_2(\{y\}) = 0 \}$. In other words, the collection of all the belief functions obtainable from $b_1$ by Dempster's combination of additional evidence is:
\[
\langle b_1 \rangle = Cl(\{ b_1, b_1 \oplus b_{x}, b_1 \oplus b_{y} \}) = Cl(\{ b_1, b_{x} = [1,0]', b_{y} = [0,1]' \})
\]
(where $Cl$ as usual denotes the convex closure operator (\ref{eq:convex-closure})), confirming Corollary \ref{cor:cond}.

\subsection{Foci of a conditional subspace in the binary case} \label{sec:foci}

Indeed, collections of Dempster's combinations of the form $b_1 \oplus b_2$ with $m_2(\{ x \}) = const$ are intersections with the conditional subspace $\langle b_1 \rangle$ of lines all passing through a `focal' point $F_{x}$ outside the belief space . Dually, sums $b_1\oplus b_2$ with $m_2(\{ y \}) = const$ lie on a set of lines all passing through a twin point $F_{y}$ (see Figure \ref{fig:dempster}).\\ The coordinates of $F_{x}$ are obtained by simply intersecting the line associated with $m_2(\{x\}) = 0$ with that associated with $m_2(\{x\}) = 1$ -- dually for $F_{y}$. We get:

\begin{equation}\label{eq:foci}
\begin{array}{cc}
\displaystyle F_{x} = \bigg [ 1, - \frac{m_{1}(\Theta)}{m_{1}(\{x\})} \bigg]' & \displaystyle F_{y} = \bigg[ -
\frac{m_{1}(\Theta)}{m_{1}(\{y\})}, 1 \bigg ] '.
\end{array}
\end{equation}
We call these points \emph{foci} of the conditional subspace generated by $b_1$ (see Figure \ref{fig:dempster} again). By Equation (\ref{eq:foci}) it follows that:

\begin{proposition}
$\lim_{b_1 \rightarrow b_\Theta} F_{x/y} =  \infty$, $\lim_{b_1 \rightarrow p, \; p \in{\mathcal{P}}} F_{x/y} = b_{x/y}$.
\end{proposition}

where $\infty$ denotes the point at infinity of the Cartesian plane.

In other words, the `categorical' probabilities $b_{x}$, $b_{y}$ can be interpreted as the foci of the probability simplex $\mathcal{P}$, seen as the conditional subspace generated by any probability measure $p$\footnote{The case in which $b_1 = p \in \mathcal{P}$ is a Bayesian belief function is a singular one, as pointed out in \cite{cuzzolin04smcb}.}.

\begin{figure}[ht!]
\includegraphics[width = 0.55 \textwidth]{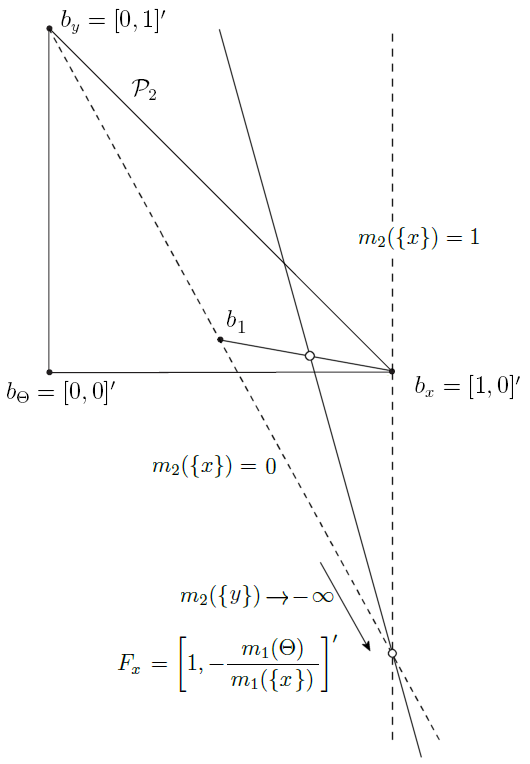}
\caption{\label{fig:focus-binary} The $x$ focus of a conditional subspace $\langle b_1 \rangle$ in the binary belief space $\mathcal{B}_2$ (when $m_1(\Theta) \neq 0$). The white circle in correspondence of $F_{x}$ indicates that the latter is a missing point for each of the lines representing images of constant mass loci.}
\end{figure}

\subsection{Probabilistic `coordinates' of conditional belief functions} \label{sec:probabilistic-coordinates}

From Figure \ref{fig:dempster} we can appreciate that to each point $b_2 \in \langle b_1 \rangle$ of the conditional subspace generated by a belief function $b_1$ (at least in the binary case considered so far) is uniquely attached a pair of probabilities, the intersections of the lines $l_x = \overline{F_{x} b_2}$ and $l_y = \overline{F_{y} b_2}$ with the probability simplex $\mathcal{P}$:
\begin{equation}
\begin{array}{cc}
p_x \doteq l_x \cap {\mathcal{P}} & p_y \doteq l_y \cap {\mathcal{P}}.
\end{array}
\end{equation}

\begin{definition}
We call $p_x$ and $p_y$ the \emph{probabilistic coordinates} of the conditional belief function $b_2|_{b_1} \in \langle b_1 \rangle$.
\end{definition}

In the binary case we can calculate their analytical form by simply exploiting their definition.\\ We obtain for $p_y$:      

\begin{equation} \label{eq:procord}
\begin{array}{ccc}
p_y & = & \displaystyle \bigg [ \frac{(1 - m_2(x) - m_2(y))(1 - m_1(y))(1 - m_2(y))}{(1 - m_2(x))(1 - m_2(y) - m_1(y) m_2(x)) - m_2(y) (1 - m_1(y))(1 - m_2(y))}, \\ \\ & & \displaystyle  \frac{m_1(y) (1 - m_2(x) - m_2(y))(1 - m_2(x))}{(1 - m_2(x)) (1 - m_2(y) - m_1(y) m_2(x)) - m_2(y) (1 - m_1(y)) (1 - m_2(y))} \bigg ]'.
\end{array}
\end{equation}
Similar expressions can be derived for $p_x$. 

Probabilistic coordinates have some remarkable properties. For instance:
\begin{proposition}
If $b_1 \in \mathcal{P}$ then $p_x = p_y = b_1 \oplus b_2 \in \mathcal{P}$ for all $b_2 \in \langle b_1 \rangle$.
\end{proposition}
namely, when the conditioning b.f. is Bayesian the probabilistic coordinates of every conditional b.f. coincide (it suffices to look at Figure \ref{fig:dempster}). 

\begin{figure}[ht!] 
\includegraphics[width = 0.65 \textwidth]{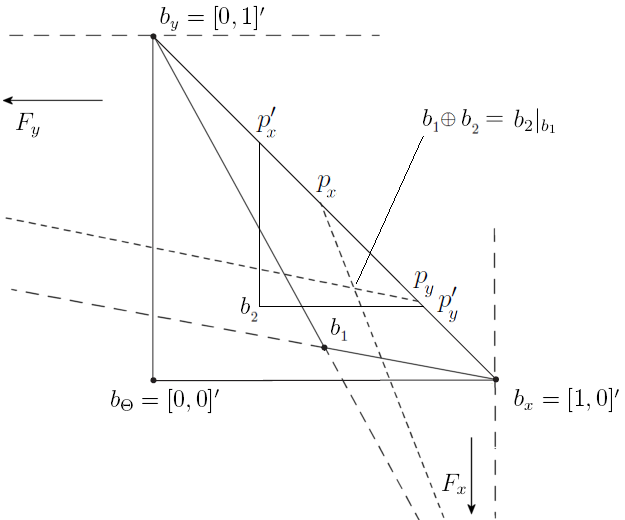}
\caption{Graphical construction of Dempster's orthogonal sum in $\mathcal{B}_2$. \label{fig:dempalg}}
\end{figure}

\subsection{Geometric construction of Dempster's rule}

Most importantly, probabilistic coordinates allow us to formulate a geometrical construction for the orthogonal sum of a pair of belief functions $b_1, b_2 \in \mathcal{B}_2$ defined on a binary frame (Figure \ref{fig:dempalg}). \vspace{3mm}

\textbf{Algorithm (binary case).} \vspace{3mm}

\begin{enumerate}
\item Compute the foci $F_{x}$, $F_{y}$ of the conditional subspace $\langle b_1 \rangle$;
\item project $b_2$ onto $\mathcal{P}$ along the orthogonal directions, obtaining two Bayesian belief functions $p'_x$ and $p'_y$;
\item combine $b_1$ with $p'_x$ and $p'_y$ to obtain the probabilistic coordinates $p_x$ and $p_y$ of $b_2$;
\item finally, draw the lines $\overline{p_x F_x}$ and $\overline{p_y F_y}$: their intersection is the desired orthogonal sum $b_1 \oplus
b_2$.
\end{enumerate}

This geometric algorithm's main feature is that the convex structure of conditional subspaces can be used to `decompose' the belief function $b_2$ to combine into a pair of probabilities (Bayesian b.f.s). Such probabilities are then combined with the first operand (a quite simpler operation than the orthogonal sum of two proper belief functions) and used in a simple intersection of lines to generate the desired Dempster's combination. A generalization of the above algorithm would therefore be rather appealing, as it would allow a significant reduction of the computational complexity of Dempster's rule.

Indeed, such a general geometric construction exists, and has been proven in \cite{cuzzolin04smcb}. In the general case, each conditional subspace $\langle b \rangle$ possesses a separate focus $\mathcal{F}_A(b)$ for each subset $A$ of the core $\mathcal{C}_b$ of the generating belief function $b$. \vspace{3mm}

\textbf{Algorithm (general case).} \vspace{3mm}

\begin{enumerate}
\item All the foci $\{ \mathcal{F}_A(b), A \subseteq \mathcal{C}_b \}$ of the subspace $\langle b \rangle$ conditioned by the first belief function $b$ are calculated as the affine subspaces\footnote{The affine space generated by a set of points $v_1,...,v_m$ is the set $\{\alpha_1 v_1 + ... + \alpha_m v_m, \; \sum_{i=1}^m \alpha_i = 1\}$. \label{foot:affine-space}} generated by their \emph{focal points} (see \cite{cuzzolin04smcb}, Corollary 4):
\begin{equation} \label{eq:focal-points}
\varsigma_B = \frac{1}{1 - pl_b(B)} b + \frac{pl_b(B)}{pl_b(B) - 1} b \oplus b_B, \hspace{5mm} B \subseteq \mathcal{C}_b, B \neq A;
\end{equation}
\item given the second belief function to combine $b'$, an additional point $b \oplus m_{b'}(A) b_A$ for each event $A \subseteq \mathcal{C}_b$ is detected (these correspond to $b'$s `probabilistic coordinates' in the binary case);
\item for each $A \subseteq \mathcal{C}_b$, each pair focus + additional point selects an affine subspace of normalized sum functions (see Section \ref{sec:normalized-sum-functions}), namely that generated by the points:
\[
b \oplus [m_{b'}(A) b_A + (1 - m_{b'}(A)) b_B] \hspace{3mm} \forall B \subset \mathcal{C}_b, B \neq A;
\]
\item all such affine subspaces are intersected, yielding the desired combination $b \oplus b'$.
\end{enumerate}
The pointwise behavior of the rule of combination depends, in conclusion, on the notion of `constant mass locus' \cite{cuzzolin04smcb}.

It is interesting to note that the focal points (\ref{eq:focal-points}) have to be computed just once as trivial functions of the upper probabilities (plausibilities) $pl_b(B)$. In fact, each focus is nothing more than a particular selection of $2^{|\mathcal{C}_b|} - 3$ focal points among a collection of $2^{|\mathcal{C}_b|} - 2$. Furthermore, the computation of each focal point $\varsigma_B$ involves a single application of Bayes' conditioning rather then general Dempster's sum, avoiding time-consuming multiplications of probability assignments.

\section{Applications of the geometric approach} \label{sec:applications-geometric-approach}

\subsection{Towards a geometric canonical decomposition} \label{sec:canonical-decomposition}

The graphical representation introduced above can be used to find the canonical decomposition of a generic belief function $b \in \mathcal{B}_2$ defined on a binary frame of discernment $\Theta = \{x,y\}$ (since any b.f. in $\mathcal{B}_2$ is separable, but $b_x$ and $b_y$). Let us call $\mathcal{CO}_x$ and $\mathcal{CO}_y$ the sets of simple support functions focussing on $\{x\}$ and $\{y\}$, respectively\footnote{The notation comes from the fact that simple support functions coincide with consonant belief functions on a binary frame, see Chapter \ref{cha:toe}, Section \ref{sec:consonant}.}.

\begin{theorem} \label{the:canonical}
For all $b \in \mathcal{B}_2$ there exist two {uniquely determined} simple belief functions $e_x \in \mathcal{CO}_x$ and $e_y \in \mathcal{CO}_y$ such that:
\[
b = e_x \oplus e_y.
\]
The two simple support functions are geometrically defined as the intersections:
\begin{equation} \label{eq:canonical}
e_x = \overline{b_y b} \cap \mathcal{CO}_x,\hspace{5mm} e_y = \overline{b_x b} \cap \mathcal{CO}_y.
\end{equation}
where $\overline{b_x b}$ denotes the line passing through $b_x$ and $b$.
\end{theorem}
\begin{proof}
The proof is illustrated in Figure \ref{fig:canonical}. The ordinate axis is mapped by Dempster's combination with the first simple component $e_x \oplus (.)$ to $\overline{b_y b}$. As $e_x = e_x \oplus b_\Theta$ and $b_\Theta$ belongs to the ordinate, $e_x$ must lie on the line $\overline{b_y b}$, while belonging to $\mathcal{CO}_x$ by definition. The thesis trivially follows.

Analogously, the abscissa is mapped by Dempster's combination with the second simple component $e_y \oplus (.)$ to $\overline{b_x b}$. As $e_y = e_y \oplus b_\Theta$ and $b_\Theta$ belongs to the abscissa, $e_y$ must lie on the line $\overline{b_x b}$ (while also belonging to $\mathcal{CO}_y$).
\end{proof}

\begin{figure}[ht!] 
\includegraphics[width = 0.45 \textwidth]{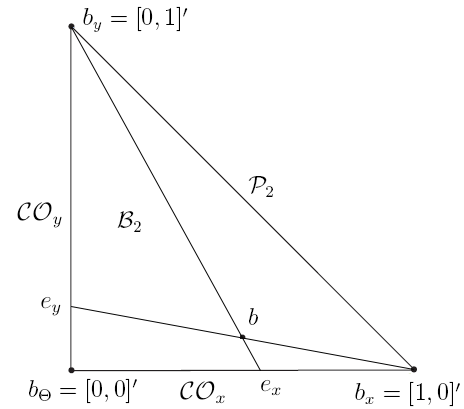}
\caption{Geometric canonical decomposition of a belief function in $\mathcal{B}_2$. \label{fig:canonical}}
\end{figure}

Proposition \ref{pro:canonical} suggests the possibility of exploiting our knowledge of the geometry of conditional subspaces to generalize Theorem \ref{the:canonical} to arbitrary belief spaces.\\ Indeed, Equation (\ref{eq:canonical}) can be expressed in the following way:
\[
e_{x/y} = Cl(b, b \oplus b_{x/y})\cap Cl(b_\Theta, b_{x/y}) = Cl(b,b_{x/y}) \cap Cl(b_\Theta,b_{x/y}).
\]
This shows that the geometric language we introduced in this Chapter, based on the two operators of convex closure and orthogonal sum, may be powerful enough to provide a general solution to the canonical decomposition problem (see Chapter \ref{cha:state}), alternative to both Smets' \cite{smets95canonical} and Kramosil's \cite{kramosil97measure}).

\subsection{Order relations} \label{sec:order-relations}

One of the problems which motivate the introduction of a geometrical representation of belief measures is the search for a rigorous computation of consonant and probabilistic approximations of belief functions. In order to tackle this problem we need to introduce two sets of coordinates naturally associated with the belief space, which we call here \emph{simple} and \emph{belief} coordinates.

\subsubsection{Partially ordered sets and lattices} \label{sec:posets}

\begin{definition} \label{def:pos}
A \emph{partially ordered set} or \emph{poset} consists of a set $P$ together with a binary relation $\leq$ over $P$ which is reflexive, antisymmetric and transitive. Namely, for all $p,q$ and $r$ in $P$ the following properties hold:
\begin{itemize}
\item
$p \leq p$ (reflexivity);
\item
if $p \leq q$ and $q \leq p$ then $p = q$ (antisymmetry);
\item
if $p \leq q$ and $q \leq r$ then $p \leq r$ (transitivity).
\end{itemize}
\end{definition}
We also write $q \geq p$ whenever $p \leq q$. A \emph{chain} of a poset is a collection of consecutive elements: two elements $p,q$ are consecutive if $p\geq q$ (or $q\geq p$) and $\not\exists r$ s.t. $p\geq r\geq q$ ($q\geq r\geq p$). The \emph{length} of a chain is the number of consecutive elements which form it. A poset is said to have \emph{finite length} if the length of all its chains is bounded. An \emph{interval} $I[p,q]$ in $P$ is the following subset of $P$: $\{ r \in L : p \leq r \leq q \}$.

In a poset the dual notions of \emph{least upper bound} and \emph{greatest lower bound} of a pair of elements can be introduced.
\begin{definition} \label{def:inf}
Given two elements $p , q \in P$ of a poset $P$ their \emph{least upper bound} $\sup_P (p,q)$ is the smallest element of
$P$ that is bigger than both $p$ and $q$. Namely, $\sup_P (p,q) \geq p,q $ and 
\[ 
\exists r \; s.t. \; r \leq \sup_P (p,q),\; r \geq p,q \hspace{5mm} \Rightarrow \hspace{5mm} r = \sup_P (p,q).
\]
\end{definition}
\begin{definition} \label{def:sup}
Given two elements $p, q \in P$ of a poset $P$ their \emph{greatest lower bound} $\inf_P (p,q)$ is the biggest element of $P$ that is smaller than both $p$ and $q$. Namely, $\inf_P (p,q) \leq p,q$ and 
\[
\exists r \; s.t. \; r \geq \inf_P (p,q), \; r \leq p,q \hspace{5mm} \Rightarrow \hspace{5mm} r = \inf_P (p,q).
\]
\end{definition}
The standard notations for greatest lower bound and least upper bound are $\inf(p,q) = p \wedge q$ and $\sup(p,q) = p \vee q$, respectively. By induction $\sup$ and $\inf$ can be defined for arbitrary finite collections, too. However, any collection of elements of a poset does not admit $\inf$ and/or $\sup$, in general.

\begin{definition} \label{def:lattice}
A \emph{lattice} is a poset in which any arbitrary finite collection of elements admits both $\inf$ and $\sup$. The latter meet the following properties:
\begin{enumerate}
\item associativity: $p \vee (q\vee r) = (p\vee q) \vee r,\;$ $p \wedge (q\wedge r) = (p\wedge q) \wedge r$;
\item commutativity: $p \vee q=q \vee p,\;$ $p \wedge q=q \wedge p$;
\item idempotence: $p\vee p=p,\;$ $p \wedge p=p$;
\item $(p\vee q)\wedge p=p,\;$ $(p\wedge q)\vee p=p$.
\end{enumerate}
\end{definition}

\subsubsection{The order relation $\geq_+$}

We have seen that the following relation:
\begin{equation} \label{eq:order+}
b \geq_+ b' \equiv b(A) \geq b'(A) \hspace{5mm} \forall A\subseteq \Theta,
\end{equation}
known as `weak inclusion', plays an important role in our geometric framework. 
Indeed:
\begin{proposition}
The belief space endowed with the weak inclusion relation (\ref{eq:order+}), $(\mathcal{B},\geq_+)$, is a partially ordered set.
\end{proposition}
It is interesting to note that:
\begin{proposition}
If $b \geq_+ b'$ then $\mathcal{C}_b \subseteq \mathcal{C}_{b'}$, i.e., the core is a monotone function on the poset $(\mathcal{B},\geq_+)$.
\end{proposition}
\begin{proof}
If $b \geq_+ b'$ then $b(\mathcal{C}_{b'}) \geq b'(\mathcal{C}_{b'}) = 1$ so that $b(\mathcal{C}_{b'})=1$, i.e., $\mathcal{C}_{b'} \supseteq \mathcal{C}_{b}$.
\end{proof}
The inverse condition does not hold. We can prove, however, that:
\begin{proposition}
If ${\mathcal{C}}_b \subseteq {\mathcal{C}}_{b'}$ then $b(A) \geq b'(A)$ $\forall A \supseteq \mathcal{C}_b.$
\end{proposition}
Unfortunately $(\mathcal{B},\geq_+)$ \emph{is not a lattice} (compare Definition \ref{def:lattice}), i.e., there exist finite collections $F$ of belief functions in $\mathcal{B}$ which have no common upper bound, namely: $\nexists u\in{\mathcal{B}}$ s.t. $u \geq_+ f \;\forall f\in F$. For instance, any finite set of probabilities $\{p_1,...,p_k\}$ form such a collection (observe Figure \ref{fig:order-relations}-left). The vacuous belief function $b_{\Theta}:\;m_{b_{\Theta}}(\Theta) = 1$, on the other hand, is a lower bound for every arbitrary subset of the belief space.

\begin{figure}[ht!]
\begin{center}
\begin{tabular}{cc}
\includegraphics[width = 0.45 \textwidth]{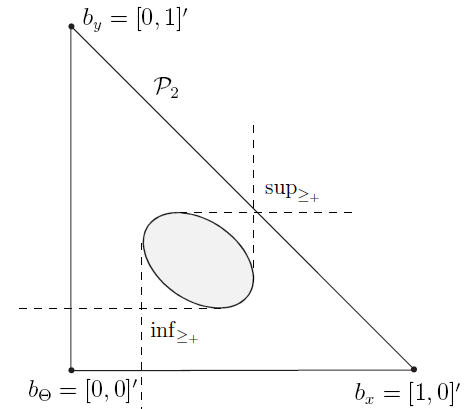} & \includegraphics[width = 0.52 \textwidth]{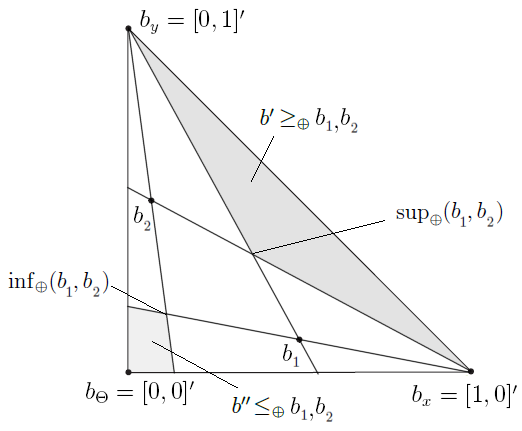}
\end{tabular}
\caption{\label{fig:order-relations} Left: geometrical representation of $\inf_{\geq_+}$ and $\sup_{\geq_+}$ in the binary belief space $\mathcal{B}_2$. Right: least upper bound and greatest lower bound for the order relation $\geq_\oplus$ in $\mathcal{B}_2$.}
\end{center}
\end{figure}

\subsubsection{Bayesian and consonant belief functions}

Bayesian and consonant belief functions behave in a opposite ways with respect to order relation (\ref{eq:order+}). For instance, Bayesian b.f.s are greater than any belief function which does not assign any mass to the singleton elements of the frame. Furthermore, Bayesian belief functions are upper elements in the partially ordered set $(\mathcal{B},\geq_+)$.
\begin{proposition}
If $b \geq_+ p$, and $p$ is Bayesian, then $b = p$.
\end{proposition}
\begin{proof}
Since by hypothesis $m_b(\{x\}) = b(\{x\}) \geq p(x)$ $\forall x \in\Theta$, and by normalization $\sum_x m_b(\{x\}) = 1 = \sum_x p(x)$, it must follow that $m_b (\{x\}) = p(x)$ $\forall x$.
\end{proof}
As a consequence:
\begin{corollary}
Probability measures are upper elements in any interval of belief functions, i.e., if $I = [b,b']$ is an interval with
respect to the order relation (\ref{eq:order+}), and $\exists p\in I$ Bayesian, then $[p,b'] \subseteq \mathcal{P}$.
\end{corollary}

On the other hand, we prove that:
\begin{theorem}
If $b' \geq_+ b$ and $\cap_{A_i\in\mathcal{E}_b} A_i = \emptyset$ (i.e., the intersection of all the focal elements of $b$ is void) then $b'$ is not consonant.
\end{theorem}
\begin{proof} 
Suppose that $b'$ is consonant. Hence $A_1$, the innermost focal element of $b'$, must be a subset of all the focal elements of $b$. But since $\cap_{A_i\in\mathcal{E}_b} A_i = \emptyset$ this is impossible.
\end{proof}

\subsubsection{Order relation $\geq_{\oplus}$}

The operation of orthogonal sum $\oplus$ is internal for any belief space (setting aside the combinability problem). It therefore generates an order relation of its own. Let us briefly see its properties, as we have done for $\geq_+$.

\begin{definition} \label{def:commutative-monoid}
A set $\mathcal{M}$ endowed with an operation $\cdot$ is called a \emph{commutative monoid} whenever for every $p,q,r \in \mathcal{M}$ the following properties hold:
\begin{enumerate}
\item 
$p \cdot (q \cdot r) = (p \cdot q) \cdot r$ (associativity);
\item
$p \cdot q =q \cdot p$ (commutativity);
\item 
$\exists 1 \in \mathcal{M}$ such that $p \cdot 1 = p$ $\forall p \in \mathcal{M}$ (existence of unit).
\end{enumerate}
\end{definition}
The monoid is said to be `with annihilator' if, in addition, $\exists 0 \in \mathcal{M}$ such that $ p \cdot 0 = 0$ $\forall p \in \mathcal{M}$.

\begin{theorem} \label{the:commutative-monoid}
The belief space endowed with Dempster's sum, $(\mathcal{B},\oplus)$, is a commutative monoid.
\end{theorem}
\begin{proof}
\emph{Commutativity}: $b \oplus b' = b' \oplus b$ by definition. \emph{Associativity}: $b \oplus (b' \oplus b'') = (b \oplus b') \oplus b''$ by the associativity of set-theoretical intersection.\\ \emph{Unity}: $b \oplus {b}_\Theta = b$ for all $b$, by direct application of Dempster's rule.
\end{proof}

Dempster's sum has a `preferential' direction, so there is no opposite element $b^{-1}$ such that $b \oplus b^{-1} = b_\Theta$ for any non-vacuous belief function: evidence combination cannot be reversed.\\
It is interesting to note that the only `annihilators' of $\oplus$ are the categorical probabilities $\{ b_x, x \in \Theta \}$:
$b \oplus b_x = b_x$ for all $x \in \Theta$ such that the combination exists.

\begin{definition} \label{the:order-dempster}
We say that $b$ `is conditioned by' $b'$ and write $b \geq_{\oplus} b'$ if $b$ belongs to the subspace of
$\mathcal{B}$ conditioned by $b'$. Namely:
\[
\begin{array}{ccccc}
b \geq_{\oplus} b' & \equiv & b \in \langle b' \rangle & \equiv & \exists b'' \in \mathcal{S} \; s.t.\; b = b' \oplus b''.
\end{array}
\]
\end{definition}

\begin{proposition}
$\geq_{\oplus}$ is an order relation.
\end{proposition}
\begin{proof} 
Monoids are inherently associated with an order relation (see Chapter \ref{cha:alg}, Section \ref{sec:dual-order-relations}). $\geq_{\oplus}$ is clearly the order relation induced by the monoid $(\mathcal{B},\oplus)$ of Theorem \ref{the:commutative-monoid}.
\end{proof}

Is $(\mathcal{B},\geq_{\oplus})$ also a lattice, namely, do every finite collection of belief functions admit a greater lower bound $\inf$ and a smaller upper bound $\sup$? The analysis of the binary case $\mathcal{B}_2$ suggests that only pairs of \emph{separable} support functions (Chapter \ref{cha:toe}, Section \ref{sec:separable}) whose cores are \emph{not disjoint} admit $\inf_\oplus$ and $\sup_\oplus$. Their analytic expressions can be easily calculated for $\mathcal{B}_2$, and their geometric locations are shown in Figure \ref{fig:order-relations}-right.

The latter also suggests a significant relationship between canonical decomposition and Dempster's rule-induced ordering.
\begin{proposition} \label{pro:dempster-canonical}
If $b,b' : 2^\Theta \rightarrow [0,1]$ are two belief functions defined on a binary frame $\Theta = \{ x,y \}$ and  $b = e_x^b \oplus e_y^b$, $b' = e_x^{b'} \oplus e_y^{b'}$ are the unique canonical decompositions of $b$ and $b'$, respectively, we have that:
\[
\begin{array}{l}
\inf_\oplus(b,b') = \inf (e_x^b , e_x^{b'}) \oplus \inf (e_y^b, e_y^{b'})\\ \\ \sup_\oplus (b,b') = \sup (e_x^b, e_x^{b'}) \oplus \sup (e_y^b, e_y^{b'}),
\end{array}
\]
where $\inf$ and $\sup$ on the right hand side of the equations denote the standard greatest lower bound and least upper bound on real numbers.

Namely, $\inf$ and $\sup$ commute with canonical decomposition.
\end{proposition}
Furthermore:
\begin{proposition} \label{pro:geq}
$b \geq_{\oplus} b'$, with $b = e_x^b \oplus e_y^b$ and $b' = e_x^{b'} \oplus e_y^{b'}$, if and only if $e_x^{b'} \geq e_x^b$ and $e_y^{b'} \geq e_y^b$.
\end{proposition}

In other words, $\geq_{\oplus}$ {is induced by the usual order relation over real numbers, when applied to the canonical components} of the two belief functions.\\ Bayesian belief functions play the role of upper bounds also under the order relation induced by the rule of combination.

\begin{proposition}
Probabilities are upper elements of any interval of belief functions with respect to $\geq_\oplus$. Namely, if $I=[b,b']$ is an interval with respect to the order relation $\geq_\oplus$, and $p\in I$ is a Bayesian belief function, then $b'=p$.
\end{proposition}

\subsection{A Dempster's rule-based probabilistic approximation} \label{sec:approx}

To conclude, let us investigate the possibility of exploiting our geometric approach to belief calculus in order to approximate, according to a criterion to be established, a given belief function with a finite probability (or Bayesian b.f.).\\ Indeed, much work has already been done on both probabilistic \cite{voorbraak89efficient,smets88beliefversus} and possibilistic \cite{dubois90,baroni04ipmu,cuzzolin09ecsqaru-outer,cuzzolin11isipta-consonant,cuzzolin14tfs} approximation of belief functions. The reader is referred to Chapter \ref{cha:state}, Section \ref{sec:transformation} for a more complete review of the topic. 

Nevertheless, we explore in this Section a different angle on the problem provided by the geometric framework introduced here.

\subsubsection{Approximation criteria}

Suppose first that the desired approximation is the Bayesian belief function which minimizes a certain distance from the original b.f., measured in the belief space $\mathcal{B}$. Such an approximation should meet a number of sensible criteria. In particular, the desired transformation should be such that:
\begin{itemize}
\item the result does \emph{not} depend on the choice of a specific distance function in the belief space;
\item the outcome is a \emph{single} pointwise approximation, rather than a whole set of approximations;
\item its rationale is consistent with the main principles of the theory of belief functions.
\end{itemize}
In Section \ref{sec:order-relations} we have learned that not every belief function has canonical coordinates (in particular, non-separable ones do not). Hence, no distance function based on canonical coordinates is suitable to induce such a `sensible' probabilistic approximation.

Let us then focus on belief coordinates $\{ b(A), \emptyset \neq A \subset \Theta \}$. The issue of what specific distance function based on them we should be choosing arises.\\ A (limited) number of options are provided by the usual $L_p$ norms:
\begin{equation} \label{eq:dist}
\|b - p\|_{L_1} = \sum_{A\subset \Theta} | b(A) - p(A) |, \;\; \| b - p \|_{L_2} = \sqrt{ \sum_{A \subset \Theta} | b(A) - p(A) |^2}, \;\; \| b - p \|_{L_{\infty}} = \sup_{A \subset \Theta} | b(A) - p(A) |.
\end{equation}
Theorem \ref{the:upper}, however, states that every belief function $b$ is related to a whole subset of Bayesian belief functions at the same (minimum) $L_1$ distance from it. Clearly, $L_1$ does not satisfy our criteria for a probabilistic approximation.

\subsubsection{External behavior and approximation criterion}

On the other side, there seems to be no justification for the choice of any the above distances. The \emph{raison d'etre} of the theory of evidence is the rule of combination: a belief function is useful only when fused with others in an automated reasoning process. Consequently, from the principles of evidential reasoning it follows that  \emph{a good approximation, when combined with any other belief function, has to produce results `similar' to what obtained by combining the original function}. Now, how `similarity' between the result of evidence combination is measured remains to be decided. 

Analytically, such a criterion translates into looking for approximations of the form:
\begin{equation} \label{eq:approx}
\hat{b} = \arg \min_{b' \in \mathcal{C}l } \int_{t \in C(b)} dist(b \oplus t, b' \oplus t) dt,
\end{equation}
\noindent where $dist$ is a distance function in Cartesian coordinates (for instance, an $L_p$ norm (\ref{eq:dist})), and $\mathcal{C}l$ is the class of belief functions where the approximation must belong.

\subsubsection{The desired approximation in ${\mathcal{B}}_2$}

At least in the binary frame, such an approximation can indeed be computed. Let us focus, in particular, on the target class of Bayesian belief functions: $\mathcal{C}l = \mathcal{P}$. Firstly, intuition suggests a slight simplification of expression (\ref{eq:approx}).\vspace{3mm}

\textbf{Conjecture.} When computing the probabilistic approximation of a given belief function $b$, it suffices to measure the integral distance (\ref{eq:approx}) only on the collection of \emph{Bayesian} belief functions compatible with $b$, $\mathcal{P}\cap C(b)$, namely:
\begin{equation} \label{eq:approx1}
\hat{p} = \arg \min_{p \in \mathcal{P}} \int_{t \in \mathcal{P} \cap C(b)} dist(b \oplus t, p \oplus t) dt.
\end{equation}
As we know, on a binary frame the compatible subset coincides with the whole belief space $\mathcal{B}_2$ for every belief function distinct from $b_x$ and $b_y$, so that $\mathcal{P} \cap C(b) = \mathcal{P}$ $\forall b \in\mathcal{B}_2$. 

The outcome of approximation criterion (\ref{eq:approx1}) turns out to be rather interesting.

\begin{theorem} \label{the:prob-approx-in-b2}
For every belief function $b \in \mathcal{B}_2$ defined on a binary frame the solution of the optimization problem (\ref{eq:approx1}) is unique, and corresponds to the normalized plausibility of singletons (\ref{eq:relplaus}) {regardless the choice of the distance function} $L_{p}$ $\forall p$.
\end{theorem}
\begin{proof}
Using the analytic expression (\ref{eq:dempster2}) of Dempster's rule in $\mathcal{B}_2$, and adopting the notations $b = [a_1,a_2]'$,  $p = [\pi,1 - \pi ]$ and $t = [\tau, 1 - \tau]'$ for the involved belief and probability functions, we get:
\[
b \oplus t - p \oplus t = \frac{\tau (1-\tau)}{1 - p t} \cdot \frac{\pi(1-a_1+1-a_2)-(1-a_2)}{(1-a_1)+\tau(a_1-a_2)}\cdot [1,-1]' ,
\]
since both $b \oplus t$ and $p \oplus t$ belong to $\mathcal{P}$ (the line $a_1+a_2 = 1$), so that their difference is proportional to the vector $[1,-1]'$. This implies:
\[
\| b \oplus t - p\oplus t\|_{L_p}^p = 2 \cdot \left | \frac{\tau(1-\tau)}{1-\pi \tau}\cdot \frac{\pi(1-a_1+1-a_2)-(1-a_2)}{(1-a_1) + \tau(a_1 -a_2)} \right|^p.
\]

The solution of the desired approximation problem (\ref{eq:approx1}) becomes therefore:
\[
\begin{array}{l}
\displaystyle \int_{0}^1 \| b \oplus t - p\oplus t\|_{L_p}^p d\tau = \\ \hspace{10mm} \displaystyle = 2 \cdot
|\pi(1-a_1+1-a_2)-(1-a_2)|^p \cdot \int_{0}^1 \left | \frac{\tau(1 - \tau)}{(1 - \pi \tau)[(1-a_1)+\tau(a_1-a_2)]} \right |^p d\tau \\ \\ \hspace{10mm} = 2 \cdot |\pi(1-a_1+1-a_2)-(1-a_2)|^p\cdot I(\pi) = F(\pi)\cdot I(\pi).
\end{array}
\]
Of the two factors involved, $I(\pi)\neq 0$ for every $\pi$, since its argument is strictly positive for $\pi\in(0,1)$. The other factor $F(\pi)$, instead, is nil whenever $\pi(1-a_1+1-a_2)-(1-a_2)=0$, i.e., when:
\[
\pi=\frac{1-a_2}{1-a_1+1-a_2}.
\]
The sought approximation is therefore:
\[
\hat{p} = \left[ \frac{1-a_2}{1-a_1+1-a_2},\frac{1-a_1}{1-a_1+1-a_2} \right]'.
\]
\end{proof}

Going back to Chapter \ref{cha:toe}, Equation (\ref{eq:upper}), we can recognize the Bayesian belief function $\tilde{pl}_b$ obtained by normalizing the plausibility values of the singleton elements of the binary frame, or `relative plausibility of singletons' \cite{voorbraak89efficient,Cobb03isf,cuzzolin05hawaii,cuzzolin06ipmu,cuzzolin10amai}. It is interesting to note that, also:
\[
\tilde{pl}_b = \arg \min_{p \in\mathcal{P}} \| b - p \|^2,
\]
i.e., \emph{the normalized plausibility of singletons is also the unique probability that minimizes the standard quadratic distance from the original belief function in the Euclidean space}.

Theorem \ref{the:prob-approx-in-b2} suggests that the optimal approximation, according to criterion (\ref{eq:approx}), could be computed in closed form in the general case as well, at least in the case of probability transformations. In any case the proposed criterion has a general scope and rests on intuitive principles at the foundation of the theory of evidence. It has therefore the potential to bring order to the matter of transformations of belief functions if further developed, as we intend to do in the near future.

\section{Conclusive comments} \label{sec:comments}

The geometric framework introduced in this Chapter is still in its early days \cite{cuzzolin01thesis,cuzzolin01space,cuzzolin2008geometric}, although some interesting results have already been achieved. We now have an overall picture of the behavior of belief functions as geometrical objects, but many questions still need to be addressed \cite{cuzzolin03isipta}.\\ As far as our approximation criterion is concerned, our preliminary results appear to confirm the soundness of our criteria. Simple maths in the {consonant} approximation case confirm the independence of the outcome from the chosen distance function, and its link to what obtained by minimizing a standard quadratic distance.

The lack of an evidential counterpart of the notion of random process is perhaps one of the major drawbacks of the theory of evidence (as we mentioned in the Introduction), preventing a wider application of belief calculus to engineering problems. The knowledge of the geometrical form of conditional subspaces could indeed be useful to predict the behavior of \emph{series of belief functions}:
\[
\lim_{n\rightarrow \infty} (b_1 \oplus \cdots \oplus b_n)
\]
and their asymptotic properties.\\ On the other side, the geometric description of conditional subspaces promises to be a suitable tool for the solution of problems such as canonical decomposition and the search for a geometric construction of Dempster's rule, providing as well a quantitative measure of the \emph{distance from separability} of an arbitrary belief function. Recently, the Author has worked towards extending this geometric analysis of other combination operators \cite{Cuzzolin2018geo}.

It seems safe to argue that, although the geometric interpretation of belief functions was originally motivated by the approximation problem, its potential applications are rather more far-reaching \cite{rota97book,ha98geometric}.

\chapter{Algebraic structure of the families of compatible frames} \label{cha:alg}

\begin{center}
\includegraphics[width = 0.6 \textwidth]{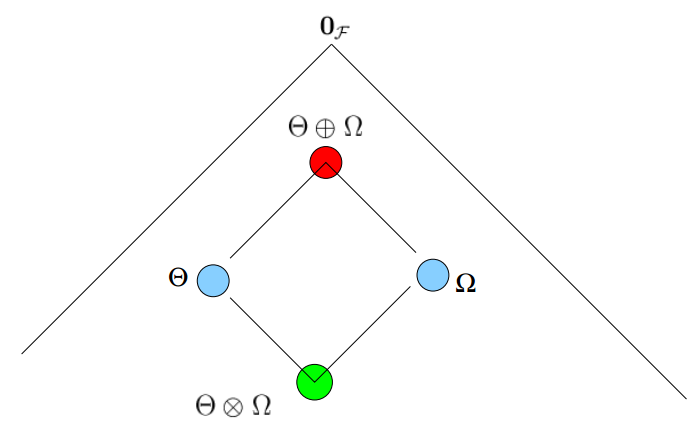}
\end{center}
\vspace{5mm}

The introduction of belief functions as the most suitable mathematical descriptions of empirical evidence or subjective states of belief, and the related mechanisms for their combination in a belief revision process is the most important contribution of the theory of evidence. Another major pillar of evidential reasoning is the formalization of the idea of structured collection of representations of the external world, encoded by the notion of `{family of frames}' (see Chapter \ref{cha:toe}, Section \ref{sec:families-of-frames}).\\ Indeed, in sensor fusion applications \cite{hunter06fusion} and computer vision \cite{Cuzzolin99}, among others, a variety of measurements of different nature need  typically to be extracted in order to make inferences on the problem at hand -- think for instance of different types of images features. If we choose to represent such measurements as belief functions, the latter turn out to be inherently defined on distinct frames of a family. Nevertheless, they need to be combined in order to reach a decision or to provide an estimate of the internal state of a system (e.g., the pose of an articulated body \cite{cuzzolin13fusion}).

Unfortunately, the combination of belief functions defined on distinct, compatible frames is not always possible. As we show in the following (Theorem \ref{the:7}), this is guaranteed at all times only when their domains are `{independent}' in the sense of Definition \ref{def:indep}. The existence of collections of belief functions which are not combinable under Dempster's rule is an issue often referred to by the term \emph{conflict problem} \cite{dubois92various,wierman01measuring,lefevre02if,cattaneo03isipta,josang03strategies,1163941}. A naive solution to the conflict problem consists in building a graph whose nodes are associated with the belief functions to combine, and recursively detecting the most coherent set of belief functions as a maximal clique of this graph \cite{Cuzzolin99}. This, however, is a rather ad-hoc solution which suffers from a high computational cost. In addition, no clear criteria for choosing a specific `maximal' collection of belief functions rather than a different one are provided by the theory.

\subsection*{Scope of the Chapter}

In this Chapter we lay the foundations for a rigorous algebraic analysis of the conflict problem, by studying the algebraic structure of the families of compatible frames \cite{kb95} as mathematical objects obeying a small number of axioms, namely those proposed by Shafer in \cite{Shafer76}. The intuition comes from a striking resemblance between the notion of independence of frames (see Definition \ref{def:indep})
\[
\begin{array}{cc}
\rho_1 (A_1) \cap \cdots \cap \rho_n (A_n) \neq \emptyset, & \forall A_i\subset\Theta_i
\end{array}
\]
and the familiar linear independence relation for collections of subspaces $\{V_i\}$ of a vector space $V$:
\[
\begin{array}{cc}
v_1 + \cdots + v_n \neq 0, & \forall v_i\in V_i.
\end{array}
\]
As we will see in Chapters \ref{cha:alg} and \ref{cha:independence}, this is more than a simple analogy: it is the symptom of a deeper similarity of these structures at an algebraic level.

Given a collection of arbitrary elements of a vector space a well-known procedure (the \emph{Gram-Schmidt algorithm}) can be applied to generate a new collection of independent vectors, spanning the same subspace. A similar procedure acting on a set of belief functions defined on arbitrary elements of a family of frames,  and able to generate a second collection of (combinable) belief functions defined on independent frames with Dempster's sum `equivalent' (in a sense to be defined) to that of the original set of b.f.s, would provide a definitive solution to the conflict problem. 

In Chapter \ref{cha:alg} we prepare the mathematical ground for this ambitious goal, by studying the monoidal and lattice structures of families of compatible frames of discernment. We distinguish finite from general families of frames, describe the monoidal properties of compatible collections of both frames and refinings, and introduce the internal operation of `{maximal coarsening}', which in turns induces in a family of frames the structures of Birkhoff, upper semimodular and lower semimodular lattice \cite{Szasz,Jacobson}.\\
Both vector subspaces and families of frames share the structure of \emph{Birkhoff lattice} \cite{Stern} (Corollary \ref{cor:family}). A formal linear independence relation can be introduced on the atoms of a Birkhoff lattice (the elements covering its initial element $\mathbf{0}$), which form therefore a \emph{matroid} \cite{Whitney35,harary69,Oxley,cuzzolin08isaim-matroid}, the algebraic structure which constitutes the classical formalization of the notion of independence. Unfortunately, this linear independence relation cannot be uniquely extended to arbitrary elements of the lattice, nor the resulting relations make the lattice itself a matroid. 

In Chapter \ref{cha:independence} we will investigate the relation between Shafer's classical definition of independence of frames and these various extensions of matroidal independence to compatible frames, as elements of a Birkhoff lattice, and draw some conclusions on the conjectured algebraic solution to the conflict problem \cite{cuzzolin14algebraic}.

\subsection*{Related Work} \label{sec:previous-work}

To out knowledge not much work has been done along this line of research. In \cite{Shafer87b} an analysis of the collections of partitions of a given frame in the context of a hierarchical representation of belief could be found. A more extensive discussion of the algebraic properties of the families of frames appeared in \cite{Kohlas95}. There, Chapter 7 was devoted to the lattice-theoretical interpretation of families of frames (compare Section \ref{sec:lattice-structure} of this Chapter) and the meaning of the concept of independence. Chapter 8, instead, explored the consequences of the application of suitable constraints to the structure of the family and developed in more detail the properties of {Markov trees}.

\subsection*{Chapter Outline} \label{sec:chapter-6-outline}

The Chapter is structured as follows. We start from Shafer's definition of a family of compatible frames, and look for a `{constructive}' set of axioms (Section \ref{sec:axiom-analysis}). Assuming a finite knowledge of the problem at hand (a realistic assumption in real-world applications), the latter allows us to build the subfamily of compatible frames of discernment generated by any given frame.

In Section \ref{sec:monoidal} we focus on these  finite subfamilies of frames, and show that the minimal refinement operator $\otimes$ induces on them a structure of commutative monoid with annihilator (Theorem \ref{the:monfin}). The collection of refinings of a finite subfamily of frames is also a monoid, as one can build an isomorphism between frames and refinings (Equation (\ref{eq:corr})). More importantly, both the set of frames $(\mathcal{F},\otimes)$ and the set of refinings $(\mathcal{R},\otimes)$ of a \emph{general} family of compatible frames of discernment admit the algebraic structure of commutative monoid (Section \ref{sec:general}), with finite subfamilies as their submonoids.

As the internal operation of a monoid induces an order relation, we are led to a lattice-theoretic interpretation of families of frames (Section \ref{sec:lattice-structure}). As a matter of fact, in Section \ref{sec:semimodular} we prove that the collection of sets of a family of compatible frames is a {Birkhoff lattice} (with minimal refinement as least upper bound, and the dual operation of `{maximal coarsening}' as greatest lower bound) of locally finite length and with a smallest element. Families of frames are also both upper semimodular and lower semimodular lattices, with respect to a dual pair of order relations.

\section{Axiom analysis} \label{sec:axiom-analysis}

Let us recall, for sake of simplicity, the axioms which define a family of frames $\{ \mathcal{F}, \mathcal{R} \}$ (Chapter \ref{cha:toe}, Definition \ref{def:1}).

\begin{axiom}
\emph{Composition of refinings}: if $\rho_1 : 2^{\Theta_1} \rightarrow 2^{\Theta_2}$ and
$\rho_2:2^{\Theta_2}\rightarrow2^{\Theta_3}$ are in $\mathcal{R}$, then $\rho_1 \circ \rho_2$ is in $\mathcal{R}$.
\end{axiom}
\begin{axiom} \emph{Identity of coarsenings}: if $\rho_1:2^{\Theta_1} \rightarrow 2^{\Omega}$and $\rho_2 : 2^{\Theta_2} \rightarrow2^{\Omega}$ are in $\mathcal{R}$ and
$\forall\;\theta_1\in\Theta_1\;\exists\;\theta_2\in\Theta_2 \;such\; that\;
\rho_1(\{\theta_1\}) = \rho_2(\{\theta_2\})$ then $\Theta_1=\Theta_2$ and $\rho_1 = \rho_2$.
\end{axiom}
\begin{axiom} 
\emph{Identity of refinings}: if $\rho_1:2^{\Theta}\rightarrow2^{\Omega}$ and $\rho_2:2^{\Theta}\rightarrow2^{\Omega}$ are in $\mathcal{R}$, then $\rho_1 = \rho_2$.
\end{axiom}
\begin{axiom} 
\emph{Existence of coarsenings}: if $\Omega\in\mathcal{F}$ and $A_1,...,A_n$ is a disjoint partition of $\Omega$, then there is a coarsening $\Omega'$ of $\Omega$ in $\mathcal{F}$ corresponding to this partition, i.e., $\forall A_i$ there exists an element of $\Omega'$ whose image
under the appropriate refining is $A_i$.
\end{axiom}
\begin{axiom} \emph{Existence of refinings}: if $\theta\in\Theta\in\mathcal{F}$ and $n \in \mathbb{N}$ then there exists a refining $\rho : 2^{\Theta} \rightarrow2^{\Omega}$ in $\mathcal{R}$ and $\Omega\in{\mathcal{F}}$ such that $\rho(\{\theta\})$ has $n$ elements.
\end{axiom}
\begin{axiom} \emph{Existence of common refinements}: every pair of elements in $\mathcal{F}$ has a common refinement in $\mathcal{F}$.
\end{axiom}

Consider now an arbitrary finite set $S$ and assess the results of applying the above axioms. We first need to apply Axiom $A4$, obtaining the collection of all the possible partitions of $S$ and the refinings between each of them and $S$ itself. By applying $A4$ again to the sets so generated we obtain all the refinings between them: no other set is added to the collection. Axioms $A2$ and $A3$ guarantee the uniqueness of the maps and sets so generated. Observe that rule $A1$ is \emph{in this situation} redundant, for it does not add any new refining.\\ Besides, it is clear at a first glance that rule $A6$ states an existence condition but is not `{constructive}', i.e., it does not allow us to generate new frames from a given initial collection.

Let us therefore introduce a new axiom:
\begin{axiom}
\emph{Existence of the minimal refinement}: every pair of elements of $\mathcal{F}$ (compatible frames) have a minimal refinement in $\mathcal{F}$, i.e., a frame satisfying the conditions of Proposition \emph{\ref{the:minimal}}.
\end{axiom}
\noindent and build a new set of axioms by replacing $A6$ with $A7$.\\ Let us call $A_{1..6}$ and $A_{1..5,7}$ these two formulations.

\begin{theorem} \label{the:equivalence-of-axioms}
$A_{1..6}$ and $A_{1..5,7}$ are equivalent formulations of the notion of family of frames.
\end{theorem}
\begin{proof} 
It is necessary and sufficient to prove that (i) Axiom $A7$ can be obtained by using the set of axioms $A1,...,A6$ and (ii) that $A6$ can be obtained from $A1,...,A5,A7$.\\ (i) See \cite{Shafer76} or Proposition \ref{the:minimal}. (ii) Each common refinement of a given pair of frames $\Theta_1,\Theta_2$ can be obtained by arbitrarily refining $\Theta_1\otimes \Theta_2$ by means of Axiom $A5$. In fact their minimal refinement is obviously itself a refinement, so that $A7\Rightarrow A6$.
\end{proof}

\subsection{Family generated by a set} \label{sec:family-generated-by-a-set}

If we assume that our knowledge of the phenomenon is \emph{finite and static}, Axiom $A5$ of the definition of families of compatible frames (Definition \ref{def:1}) cannot be used. According to the  notation established above we will call $A_{1..4,7}$ the set of axioms corresponding to such finite knowledge case.
\begin{definition}
We define the \emph{subfamily generated by a collection of sets} $\Theta_1,...,\Theta_n$ by means of a set of axioms $\mathcal{A}$ as the smallest collection of frames $\langle \Theta_1,\cdots,\Theta_n \rangle_{\mathcal{A}}$ which includes $\Theta_1,...,\Theta_n$ and is closed under the application of the axioms in $\mathcal{A}$.
\end{definition}
\begin{lemma} \label{lem:auxx}
The minimal refinement of two coarsenings $\Theta_1$, $\Theta_2$ of a frame $\Theta$ is still a coarsening of $\Theta$.
\end{lemma}
\begin{proof} 
By hypothesis $\Theta$ is a common refinement of $\Theta_1$ and $\Theta_2$. Since the minimal refinement is a coarsening of every other refinement the thesis follows.
\end{proof}
Lemma \ref{lem:auxx} allow us to prove that:
\begin{theorem}
The subfamily of compatible frames generated by the application of the restricted set of rules $A_{1..4,7}$ to a basis frame $\Theta$ is the collection of all the disjoint partitions of $\Theta$ along with the associated refinings.
\end{theorem}
Note that this is not necessarily true when Axiom $A6$ is employed.

\section{Monoidal structure of families of frames} \label{sec:monoidal}

Let us introduce in a family of compatible frames the internal operation
\begin{equation}
\begin{array}{llll} \otimes : & \mathcal{F}\times \cdots \times
\mathcal{F} & \longrightarrow & \mathcal{F}\\ & \{\Theta_1,...,\Theta_n\}& \mapsto & \otimes_i \Theta_i
\end{array}
\end{equation}
mapping a collection of frames to their minimal refinement. The above operation is well defined, for Axiom $A7$ ensures the existence of $\otimes_i \Theta_i$ and the results which follow guarantee its associativity and commutativity.

\subsection{Finite families as commutative monoids} \label{sec:finite-families}

Let us first consider finite subfamilies of frames of the form $\{ \mathcal{F}', \mathcal{R}' \} = \langle \Omega \rangle_{A_{1..4,7}}$ for some $\Omega \in \mathcal{F}$.

\begin{theorem} \label{the:monfin}
A finite family of frames of discernment $(\mathcal{F}',\otimes)$ is a \emph{finite commutative monoid with annihilator} (recall Chapter \ref{cha:geo}, Definition \ref{def:commutative-monoid}) with respect to the internal operation of minimal refinement.
\end{theorem}

\begin{proof} 
\emph{Associativity and commutativity.} Going back to Theorem \ref{the:minimal}, associativity and commutativity follow from the analogous properties of set-theoretic intersection.

\emph{Unit.} Let us prove that there exists a \emph{unique} frame in $\mathcal{F}'$ of cardinality 1. As $\Theta\in\mathcal{F}'$, due to Axiom $A4$ (existence of coarsenings) there exists a coarsening $1_{\Theta}$ of $\Theta$, together with the refining $\rho_{1_{\Theta}\Theta} : 2^{1_{\Theta}} \longrightarrow 2^{\Theta}$. But then, by Axiom $A1$ there exists another refining $1_{\rho}\in \mathcal{R}'$ such that 
\[
1_{\rho}:2^{1_{\Theta}}\longrightarrow 2^{\Omega}.
\]
Now, consider a pair of elements of $\mathcal{F}'$, say $\Theta_1,\;\Theta_2$. The above procedure yields two pairs set-refining $(1_{\Theta_i},1_{\rho_i})$ with $1_{\rho_i}:2^{1_{\Theta_i}} \rightarrow 2^{\Omega}$. If we call $1_{\theta_i}$ the single element of $1_{\Theta_i}$ we have that:
\[
1_{\rho_i}(\{1_{\theta_i}\}) = \Omega \hspace{5mm} \forall i=1,2.
\]
By Axiom $A2$ (`identity of coarsenings') the uniqueness of the unit frame follows: $1_{\rho_1} = 1_{\rho_2}$, $1_{\Theta_1}=1_{\Theta_2}$. 

\emph{Annihilator.} If $\Theta\in\mathcal{F}'$ then obviously $\Theta$ is a coarsening of $\Omega$. Therefore, their minimal refinement coincides with $\Omega$ itself.
\end{proof}

\subsubsection{Isomorphism frames-refinings}

A family of frames can be dually viewed as a set of refining maps with attached the associated domains and codomains (perhaps the most correct approach, for it takes frames into account  automatically). The following correspondence can be established in a finite family of frames $\langle \Omega \rangle_{A_{1..4,7}}$:
\begin{equation} \label{eq:corr}
\Theta \longleftrightarrow \rho^{\Theta}_\Omega : 2^{\Theta} \rightarrow 2^\Omega,
\end{equation}
which associates each frame of the family with the corresponding unique refining to the base set $\Omega$. As a consequence, the minimal refinement operator induces a composition of refinings as follows.

\begin{definition} \label{def:composition-of-refinings}
Given two refinings $\rho_1 : 2^{\Theta_1} \rightarrow 2^{\Omega}$ and $\rho_2 : 2^{\Theta_2} \rightarrow 2^{\Omega}$, their composition induced by the operation of minimal refinement is the unique (by Axiom $A3$) refining from $\Theta_1 \otimes \Theta_2$ to $\Omega$:
\begin{equation} \label{eq:compo}
\rho_1 \otimes \rho_2 : 2^{\Theta_1 \otimes \Theta_2} \rightarrow 2^\Omega. 
\end{equation}
\end{definition}

\begin{theorem} \label{the:monoid-of-refinings}
The subcollection of refinings of a finite family of frames $\langle \Omega \rangle_{A_{1..4,7}}$ with codomain $\Omega$, namely:
\[
\langle \Omega \rangle_{A_{1..4,7}}^{\rho^{(.)}_\Omega} \doteq \Big \{ \rho^{\Theta}_\Omega,\;\Theta \in \langle \Omega \rangle_{A_{1..4,7}} \Big \}
\]
is a finite commutative monoid with annihilator with respect to operation (\ref{eq:compo}).
\end{theorem}
\begin{proof} 
Associativity and commutativity follow from the analogous properties of the minimal refinement operator. The unit element is $1_{\rho} : 2^{1_\Omega} \rightarrow 2^\Omega$, as
\[
Dom(1_{\rho} \otimes \rho) = 1_\Omega \otimes Dom(\rho) = Dom(\rho)
\]
which implies $1_{\rho} \otimes \rho = \rho$ by Axiom $A2$. On the other side, if we consider the unique refining $0_{\rho} : 2^\Omega \rightarrow 2^\Omega$ from $\Omega$ onto itself (which exists by Axiom $A4$ with $n=1$) we have:
\[
Dom(0_{\rho} \otimes \rho) = \Omega \otimes Dom(\rho) = \Omega,
\]
so that $0_{\rho}$ is an annihilator for $\langle \Omega \rangle_{A_{1..4,7}}^{\rho^{(.)}_\Omega}$.
\end{proof}

From Theorem \ref{the:monoid-of-refinings}'s proof it follows that:

\begin{corollary}
Given a finite family of compatible frames, the map (\ref{eq:corr}) is an isomorphism between commutative monoids.
\end{corollary}

It is interesting to note that the existence of both the unit element and the annihilator in a finite family of frames are consequences of the following result.

\begin{proposition} \label{pro:1}
$(\rho^{\Theta}_\Omega \circ \rho^{\Theta'}_{\Theta}) \otimes \rho^{\Theta}_\Omega = \rho^{\Theta}_\Omega \hspace{5mm} \forall \Theta, \Theta' \in \langle \Omega \rangle_{A_{1..4,7}}$.
\end{proposition}
\begin{proof} 
As $Cod(\rho^{\Theta}_\Omega \circ \rho^{\Theta'}_{\Theta}) = \Omega$, the mapping $\rho^{\Theta}_\Omega \circ \rho^{\Theta'}_{\Theta}$ is a refining to the base set $\Omega$ of the finite family of frames $\langle \Omega \rangle_{A_{1..4,7}}$, and the composition of refinings $\otimes$ (\ref{eq:compo}) can be applied. After noting that $Dom(\rho^{\Theta}_\Omega \circ \rho^{\Theta'}_{\Theta}) = \Theta'$ is a coarsening of $Dom(\rho^{\Theta}_\Omega)=\Theta$ we get:
\[
\begin{array}{l}
Dom((\rho^{\Theta}_\Omega \circ \rho^{\Theta'}_{\Theta}) \otimes \rho^{\Theta}_\Omega) =Dom(\rho^{\Theta}_\Omega \circ \rho^{\Theta'}_{\Theta}) \otimes Dom(\rho^{\Theta}_\Omega) = Dom(\rho^{\Theta}_\Omega) = \Theta. 
\end{array}
\] 
By Axiom $A3$ the thesis follows.
\end{proof}

As a matter of fact, if $\Theta' = 1_\Omega$ then $\rho^{\Theta}_\Omega \circ \rho^{\Theta'}_{\Theta} = \rho^{\Theta}_\Omega \circ \rho^{1_\Omega}_{\Theta} = 1_{\rho}$ and we get $1_{\rho} \otimes \rho^{\Theta}_\Omega = \rho^{\Theta}_\Omega$. On the other hand, whenever $\Theta = \Omega$ we have $\rho^{\Theta}_\Omega = \rho^\Omega_\Omega = 0_{\rho}$ and the annihilation property $\rho^{\Theta'}_\Omega \otimes 0_{\rho} = 0_{\rho}$ holds.

\subsubsection{Generators of finite families of frames}

Given a monoid $\mathcal{M}$, the submonoid $\langle S \rangle$ generated by its subset $S \subset \mathcal{M}$ is defined as the intersection of all the submonoids of $\mathcal{M}$ containing $S$.

\begin{definition} \label{def:generators}
The set of \emph{generators} of a monoid $\mathcal{M}$ is a finite subset $S$ of $\mathcal{M}$ whose generated submonoid coincides with $\mathcal{M}$: $\langle S \rangle = \mathcal{M}$.
\end{definition}

\begin{theorem} \label{the:generators-of-finite-families}
The set of generators of a finite family of frames $\langle \Omega \rangle_{A_{1..4,7}}$, seen as a finite commutative monoid with respect to the internal operation $\otimes$ of minimal refinement, is the collection of all its binary frames. The set of generators of a finite family of refinings $\langle \Omega \rangle_{A_{1..4,7}}^{\rho^{(.)}_\Omega}$ is the collection of refinings from all the binary partitions of $\Omega$ to $\Omega$ itself:
\[
\langle \Omega \rangle_{A_{1..4,7}} = \Big \langle \Big \{ \rho_{ij} : 2^{\Theta_{ij}} \rightarrow 2^\Omega \Big | \; |\Theta_{ij}| = 2, \; \rho_{ij} \in \mathcal{R} \Big \} \Big \rangle.
\]
\end{theorem}
\begin{proof} 
We need to prove that all possible partitions of $\Omega$ can be obtained as the minimal refinement of a number of binary partitions. Consider a generic partition $\Pi = \{ \Omega^1, \cdots , \Omega^n \}$ of $\Omega$, and define the following associated partitions:
\[
\begin{array}{lll}
\Pi_1 & = & \big \{ \Omega^1, \Omega^2 \cup \cdots \cup \Omega^n \big \} \doteq \{ A_1, B_1 \} \\ \\ \Pi_2 & = & \big \{ \Omega^1 \cup \Omega^2, \Omega^3 \cup \cdots \cup \Omega^n \big \} \doteq \{ A_2, B_2 \} \\ \\ & & \cdots \\ \\ \Pi_{n-1} & = & \big \{ \Omega^1 \cup \cdots \cup \Omega^{n-1}, \Omega^n \big \} \doteq \{ A_{n-1}, B_{n-1} \}.
\end{array}
\]
It is not difficult to see that every arbitrary intersection of elements of $\Pi_1, \cdots , \Pi_{n-1}$ is an element of the $n$-ary partition $\Pi$. Indeed,
\[
A_i \cap B_k = \emptyset \; \forall k\geq i, \hspace{5mm} A_i\cap A_k = A_i \; \forall k\geq i, \hspace{5mm} B_i\cap B_k = B_k \; \forall k \geq i
\]
so that:
\[
\begin{array}{cc}
\bigcap_i A_i = A_1 = \Omega^1, & \bigcap_i B_i = B_{n-1} = \Omega^n.
\end{array}
\]
If both $A$s and $B$s are present in the intersection, the result is $\emptyset$ whenever there exists a pair $B_l,A_m$ with $l \geq m$. Consequently, the only non-empty mixed intersections in the class $\{ X_1 \cap \cdots \cap X_n \}$ are of the following kind:
\[
\begin{array}{c}
B_1 \cap \cdots \cap B_k \cap A_{k+1} \cap \cdots \cap A_{n-1} = (\Omega^{k+1} \cup \cdots \cup \Omega^n) \cap (\Omega^{1} \cup \cdots \cup \Omega^{k+1}) = \Omega^{k+1},
\end{array}
\]
with $k+1$ ranging from 2 to $n-1$. This meets the fundamental condition for $\Pi$ to be the minimal refinement of $\Pi_1,...,\Pi_{n-1}$: note that the choice of the binary frames is not unique.

The second part of the thesis, concerning the set of generators of finite families of refinings, comes directly from the existence of  isomorphism (\ref{eq:corr}).
\end{proof}

\subsection{General families as commutative monoids} \label{sec:general-families}

We can ask whether a general family of frames also possesses the algebraic structure of monoid. The answer is positive.

\begin{theorem} \label{the:mongen} 
A family of compatible frames $\mathcal{F}$ is an infinite commutative monoid without annihilator.
\end{theorem}

\begin{proof} 
The proof of Theorem \ref{the:monfin} holds for the first two points (associativity and commutativity).\\
\emph{Existence of unit}. Suppose there exist two frames $\mathbf{1} = \{1\}$ and $\mathbf{1}'=\{1'\}$ of cardinality 1. By Axiom $A6$ they have a common refinement $\Theta$, with $\rho_\mathbf{1} : 2^{\mathbf{1}} \rightarrow 2^{\Theta}$ and $\rho_{\mathbf{1}'} : 2^{\mathbf{1}'} \rightarrow 2^{\Theta}$ refinings. But then $\rho_\mathbf{1}(\{ 1 \}) = \Theta = \rho_{\mathbf{1}'}(\{ 1' \})$, and by Axiom $A2$ $\mathbf{1} = \mathbf{1}'$. \\
Now, for every frame $\Theta' \in \mathcal{F}$ Axiom $A4$ ensures that there exists a partition of $\Theta'$ with only one element, $\mathbf{1}_{\Theta'}$. From the above argument it follows that $\mathbf{1}_{\Theta'} = \mathbf{1}$. In conclusion, there is only one monic frame in $\mathcal{F}$, and since this frame is a refinement of every other frame, it is the unit element with respect to $\otimes$. 

\emph{Annihilator}. Suppose a frame $0_{\Theta}$ exists such that $0_{\Theta} \otimes \Theta = 0_{\Theta}$ for each $\Theta$. Then we would have, given a refinement $\Theta'$ of $0_{\Theta}$ (obtained via Axiom $A5$), that $0_{\Theta} \otimes \Theta' = \Theta'$, which is a contradiction.
\end{proof}

In a general family of frames, whatever basis frame $\Omega$ you choose there exist refinings with codomain distinct from $\Omega$. It is therefore impossible to establish a 1-1 correspondence between frames and refinings. Nevertheless, after noticing that for any two refinings $\rho_1$ and $\rho_2$ their codomains $\Omega_1$, $\Omega_2$ always have a common refinement (which we denote by $\Omega$), we can write:
\begin{equation}
\rho_1\otimes \rho_2 \doteq \rho'_1 \otimes \rho'_2 \label{eq:otimes}
\end{equation}
where, calling $w_1$ ($w_2$) the refining map between $\Omega_1$ ($\Omega_2$) and $\Omega$:
\[
\rho'_1 = w_1 \circ \rho_1,\hspace{5mm} \rho'_2 = w_2 \circ \rho_2
\]
and the $\otimes$ sign on the right side of Equation (\ref{eq:otimes}) stands for the composition of refinings in the finite family $\langle\Omega\rangle_{A_{1..4,7}}$ generated by $\Omega$. In this way the composition of refinings $\otimes$ is again well-defined, even in general families of frames.

In fact, a more general definition can be provided as follows.

\begin{definition}
Given a family of frames $( \mathcal{F}, \mathcal{R})$ the composition of two refinings $\rho_1 : 2^{\Theta_1} \longrightarrow 2^{\Omega_1}$ and $\rho_2 : 2^{\Theta_2} \longrightarrow 2^{\Omega_2}$, $\rho_1, \rho_2 \in \mathcal{R}$, is defined as:
\begin{equation} \label{eq:otimes1}
\rho_1 \otimes \rho_2 : 2^{\Theta_1 \otimes \Theta_2} \longrightarrow 2^{\Omega_1 \otimes \Omega_2}.
\end{equation}
\end{definition}

This operation is well-defined, for the correspondence
\begin{equation} \label{eq:corr1}
(Dom(\rho),Cod(\rho)) \longleftrightarrow \rho 
\end{equation}
(whose existence is guaranteed by Axiom $A3$) is a bijection.

\begin{theorem} 
The set of refinings $\mathcal{R}$ of a general family of frames is a commutative monoid with respect to the internal operation \emph{(\ref{eq:otimes1})}.
\end{theorem}

\begin{proof} 
Obviously $\otimes$ is commutative and associative because of the commutativity and associativity of the operation of minimal refinement of frames.\\ As for the unit element, it suffices to note that such a refining $1_{\rho} : 2^\Theta \rightarrow 2^\Omega$ has to be such that:
\[
\begin{array}{ccc}
\Theta\otimes \Theta_1=\Theta_1\;\forall \Theta_1\in \mathcal{F} & \wedge & \Omega\otimes \Omega_1 = \Omega_1 \;\forall \Omega_1\in \mathcal{F}.
\end{array}
\] 
This implies $\Theta = \Omega = \mathbf{1}$, so that $1_{\rho} : 2^{\mathbf{1}} \rightarrow 2^{\mathbf{1}}$ and $1_{\rho}$ is simply the identity map on the unit frame $\mathbf{1}$.
\end{proof}

\begin{corollary}
$\mathcal{R}$ is a submonoid of the product monoid $(\mathcal{F},\otimes) \times (\mathcal{F},\otimes)$, via the mapping $(\ref{eq:corr1})$.
\end{corollary}

\subsection{Monoidal structure of the family}

To complete our picture of the algebraic structures arising from the notion of family of frames, we need to specify some missing relations.

Clearly, $(\langle \Omega \rangle_{A_{1..4,7} }^{\rho},\otimes)$ (where $\langle \Omega \rangle_{A_{1..4,7}}^{\rho}$ is the collection of \emph{all} the refinings of the finite family with base frame $\Omega$) is a monoid, too, and:

\begin{proposition}
$(\langle \Omega \rangle_{A_{1..4,7}}^{\rho_\Omega^{(.)}},\otimes)$ is a submonoid of $(\langle
\Omega \rangle_{A_{1..4,7}}^{\rho}, \otimes)$.
\end{proposition}

\begin{proof} 
Obviously $\langle \Omega \rangle_{A_{1..4,7}}^{\rho_\Omega^{(.)}} \subset \langle \Omega \rangle_{A_{1..4,7}}^{\rho}$ in a set-theoretical sense. We only have to prove that the internal operation of the first monoid is inherited from that of the second one.\\ Given $\rho_1:2^{\Theta_1} \rightarrow 2^{\Omega}$ and $\rho_2 : 2^{\Theta_2} \rightarrow 2^{\Omega}$, since $\Omega \otimes \Omega = \Omega$, we have that:
\[
\rho_1 \otimes_{\langle \Omega \rangle^{\rho}} \rho_2 : 2^{\Theta_1 \otimes \Theta_2} \rightarrow 2^{\Omega \otimes \Omega} = 2^\Omega \hspace{5mm} \equiv \hspace{5mm} \rho_1 \otimes_{\langle \Omega \rangle^{\rho_\Omega^{(.)}}} \rho_2,
\]
where on the right hand side we have the composition of refinings of Definition \ref{def:composition-of-refinings}.
\end{proof}

Clearly, monoids associated with finite families of frames are submonoids of those associated with general families.

\begin{proposition}
$(\langle \Omega \rangle_{A_{1..4,7}}^{\rho}, \otimes)$ is a submonoid of $(\mathcal{R},\otimes)$.
\end{proposition}

\begin{proof} 
It suffices to prove that $\langle \Omega \rangle_{A_{1..4,7}}^{\rho}$ is closed with respect to the composition operator (\ref{eq:otimes1}). But then, given two maps $\rho_1,\rho_2$ whose domains and codomains are both coarsening of $\Omega$, $Dom(\rho_1\otimes \rho_2) = Dom(\rho_1) \otimes Dom(\rho_1)$ and $Cod(\rho_1 \otimes \rho_2) = Cod(\rho_1) \otimes Cod(\rho_1)$ are still coarsenings of $\Omega$.
\end{proof}

\begin{proposition}
$(\langle \Omega \rangle_{A_{1..4,7}}^{\Theta}, \otimes)$ is a submonoid of $(\mathcal{F},\otimes)$.
\end{proposition}

\begin{proof} 
Trivial, for the finite family $(\langle \Omega \rangle_{A_{1..4,7}}^{\Theta}, \otimes)$ is strictly included in the complete one $(\mathcal{F},\otimes)$ and $\langle \Omega \rangle_{A_{1..4,7}}^{\Theta}$ is closed with respect to $\otimes$, i.e., if $\Theta_1$ and $\Theta_2$ are coarsenings of $\Omega$ then their minimal refinement is still a coarsening of $\Omega$.
\end{proof}

The various relationships between monoidal structures associated with a family of compatible
frames $({\mathcal{F}},{\mathcal{R}})$ are summarized in the following diagram.

%\begin{footnotesize}
\[
\begin{array}{ccccc}
(\langle \Omega \rangle_{A_{1..4,7}}^{\rho_\Omega^{(.)}}, \otimes) & \subset & (\langle
\Omega \rangle_{A_{1..4,7}}^{\rho}, \otimes) & \subset & (\mathcal{R}, \otimes)\\ \\ \wr & & \wr & &  \\ \\ (\langle \Omega \rangle_{A_{1..4,7}}^{\Theta},\otimes) & \subset &
\begin{array}{c}
(\langle \Omega \rangle_{A_{1..4,7}}^{\Theta}, \otimes)\\ \times \\ (\langle \Omega \rangle_{A_{1..4,7}}^{\Theta}, \otimes) \end{array} & & \cap \\ \\ = & & & & \\ (\langle
\Omega \rangle_{A_{1..4,7}}^{\Theta}, \otimes) & \subset & (\mathcal{F}, \otimes) & \subset &
\begin{array}{c}
(\mathcal{F}, \otimes)\\ \times \\ (\mathcal{F},\otimes) \end{array}
\end{array}
\]
%\end{footnotesize}

\section{Lattice structure of families of frames} \label{sec:lattice-structure}

\subsection{Two dual order relations} \label{sec:dual-order-relations}

It is well-known (see \cite{Jacobson}, page 456) that the internal operation $\cdot$ of a monoid $\mathcal{M}$ induces an order relation (see Definition \ref{def:pos}) $|$ on the elements of $\mathcal{M}$. Namely:
\[
a|b \hspace{5mm} \equiv \hspace{5mm} \exists \;c\;s.t.\; b=a\cdot c.
\]
For monoids of compatible frames this monoid-induced order relation reads as:
\begin{equation} \label{eq:10}
\begin{array}{ccccc}
\Theta_2 \geq^* \Theta_1 & \equiv & \exists\;\Theta_3\;s.t.\;\Theta_2 = \Theta_1 \otimes \Theta_3 & \equiv & \Theta_1\otimes \Theta_2 =\Theta_2,
\end{array}
\end{equation}
i.e., \emph{$\Theta_2$ is a refinement of $\Theta_1$}. Since both finite and general families of frames are monoids:
\begin{proposition}
Both $\langle \Omega \rangle_{A_{1..4,7}}$ and $\mathcal{F}$ are {partially ordered sets} with respect to the order relation (\ref{eq:10}).
\end{proposition}

Indeed, in a family of compatible frames one can define two distinct order relations on pairs of frames, both associated with the notion of refining (Chapter \ref{cha:toe}, Section \ref{sec:families-of-frames}):
\begin{equation}\label{eq:order-coa}
\Theta_1 \leq^* \Theta_2 \Leftrightarrow \exists \rho : 2^{\Theta_1} \rightarrow 2^{\Theta_2} \; refining
\end{equation}
(the same as (\ref{eq:10})), or
\begin{equation}\label{eq:order-ref}
\Theta_1 \leq \Theta_2 \Leftrightarrow \exists \rho : 2^{\Theta_2} \rightarrow 2^{\Theta_1} \; refining
\end{equation}
i.e., $\Theta_1$ is a refinement of $\Theta_2$. Relation (\ref{eq:order-ref}) is clearly the inverse of (\ref{eq:order-coa}). It makes sense to distinguish them explicitly as they generate two distinct algebraic structures, in turn associated with {different extensions of the notion of matroidal independence}, as we will see in Chapter \ref{cha:independence}.

Immediately:

\begin{theorem} \label{the:sup}
In a family of frames $\mathcal{F}$ seen as a poset with order relation (\ref{eq:order-coa}) the $\sup$ of a finite collection $\Theta_1,\cdots,\Theta_n$ of frames coincides with their minimal refinement, namely:
\[
\sup_{(\mathcal{F},\leq^*)}(\Theta_1, \cdots ,\Theta_n) = \Theta_1\otimes \cdots \otimes \Theta_n.
\]
\end{theorem}

\begin{proof} 
Of course $\Theta_1\otimes \cdots \otimes \Theta_n \geq^* \Theta_i\;\forall i=1,...,n$ for there exists a
refining between each $\Theta_i$ and the minimal refinement. Now, if there exists another frame $\Omega$ greater than each $\Theta_i$ then $\Omega$ is a common refinement for $\Theta_1,...,\Theta_n$, hence it is a refinement of the minimal refinement -- namely, $\Omega \geq^* \Theta_1\otimes \cdots \otimes \Theta_n$ according to order relation (\ref{eq:order-coa}).
\end{proof}

At a first glance is not clear what $\inf_{(\mathcal{F},\leq^*)} \{ \Theta_1,...,\Theta_n\}$, instead, should represent.

\subsection{Common and maximal coarsening} \label{sec:common-coarsening}

Let us then introduce a new operation acting on finite collections of frames.

\begin{definition} \label{def:common-coarsening}
A \emph{common coarsening} of two frames $\Theta_1,\Theta_2$ is a set $\Omega$ such that $\exists \rho_1 : 2^{\Omega} \rightarrow 2^{\Theta_1}$ and $\rho_2 : 2^{\Omega} \rightarrow 2^{\Theta_2}$ refinings, i.e., $\Omega$ is a coarsening of both $\Theta_1$ and $\Theta_2$.
\end{definition}

\begin{theorem}\label{the:comcoa} 
If $\Theta_1,\Theta_2\in\mathcal{F}$ are elements of a family of compatible frames then they possess a common coarsening.
\end{theorem}
\begin{proof} 
From the proof of Theorem \ref{the:mongen} it follows that  $\Theta_1,\Theta_2$ have at least the unit frame $\mathbf{1}$ as a common coarsening.
\end{proof}

As is the case for common refinements, among the many common coarsenings of a collection $\Theta_1,...,\Theta_n$ of frames there exists a unique one characterized by being the most refined in the group.

\begin{theorem} \label{the:maximal}
Given any collection $\Theta_1,...,\Theta_n$ of elements of a family of compatible frames $\mathcal{F}$ there exists a unique element $\Omega \in \mathcal{F}$ such that:\\
\begin{enumerate}
\item $\forall i$ there exists a refining $\rho_i : 2^\Omega \rightarrow 2^{\Theta_i}$ from $\Theta_i$ to $\Omega$;\\
\item  $\forall \omega \in \Omega$ $\not \exists A_1 \subseteq \rho_1(\{ \omega \}), \cdots , A_n \subseteq \rho_n (\{ \omega \})$ s.t. $\eta_1 (A_1) = \cdots = \eta_n(A_n)$,\\
\end{enumerate}
where $\eta_i:2^{\Theta_i} \rightarrow 2^{\Theta_1 \otimes \cdots \otimes \Theta_n}$, i.e., no subsets of the images in the various $\Theta_i$ of the same element of $\Omega$ are mapped to the same subset of the minimal refinement.
\end{theorem}

We first need an intermediate Lemma.

\begin{lemma} \label{lem:aux}
Suppose $\Theta_1 \otimes \cdots \otimes \Theta_n$ is the minimal refinement of $\Theta_1, \cdots,\Theta_n$, with refinings $\eta_i : 2^{\Theta_i} \rightarrow 2^{\Theta_1 \otimes \cdots \otimes \Theta_n}$. Suppose also that there exist $X_1 \subseteq \Theta_1,..., X_n \subseteq \Theta_n$ with $\eta_1(X_1) = \cdots =\eta_n(X_n)$ such that:
\[
\nexists A_1\subseteq X_1,...,A_n\subseteq X_n \; s.t. \; \eta_1(A_1)=...=\eta_n(A_n)
\]
and $A_j \neq X_j$ for some $j \in [1,...,n]$.

Then, for every common coarsening $\Omega$ of $\Theta_1, \cdots ,\Theta_n$ with refinings $\rho_i : 2^{\Omega} \rightarrow 2^{\Theta_i}$, there exists $\omega \in \Omega$ such that $X_i \subseteq \rho_i (\{\omega\})$ for all $i = 1,...,n$.
\end{lemma}

\begin{proof}
Let us assume that such an element $\omega$ does not exist, and that some $X_i$ is covered instead by a non-singleton subset $\{ \omega_1, \cdots , \omega_k \} \subset \Omega$ of the common coarsening: 
\[
X_i \subseteq \rho_i (\{\omega_1, \cdots , \omega_k\}).
\]
Clearly, for each of its elements $\omega_j$:
\[
\eta_i (\rho_i (\omega_j) \cap X_i) = \eta_i ( \rho_i (\omega_j) ) \cap \eta_i (X_i).
\]
By definition of common coarsening, on the other hand:
\[
\eta_1 (\rho_1 (\omega_j)) = ... = \eta_n( \rho_n (\omega_j)).
\]
Therefore, since $\eta_1 (X_1) = \cdots = \eta_n (X_n)$ by hypothesis, we have that:
\[
\eta_1 (\rho_1 (\omega_j) \cap X_1 ) = \cdots = \eta_n (\rho_n ( \omega_j ) \cap X_n)
\]
with $A_i \doteq ( \rho_i ( \omega_j ) \cap X_i ) \subsetneq X_i$, which goes against what assumed.
\end{proof}

Now we can tackle the proof of Theorem \ref{the:maximal}.

\begin{proof} \emph{Existence.} The proof is constructive.\\ Let us take an arbitrary coarsening $\mathcal{L}$ of $\Theta_1,...,\Theta_n$ (which exists by Theorem \ref{the:comcoa}) and check for every $l\in\mathcal{L}$ whether there exists a collection of subsets $\{A_i\subset \rho_i(\{l\}),\;i=1,...,n\}$ such that $\eta_1(A_1)=...=\eta_n(A_n)$. If the answer is negative $\mathcal{L} = \Omega$, and we have the desired frame.\\ Otherwise we can build a new common coarsening ${\mathcal{L}}'$ of $\Theta_1,...,\Theta_n$ by simply splitting $\{l\}$ into a pair $\{l_1,l_2\}$ such that for all $i$:
\[
\rho'_i(\{l_1\}) = A_i,\hspace{5mm} \rho'_i (\{l_2\})=B_i,
\]
where $B_i \doteq \rho_i(\{l\})\setminus A_i$. This splitting does exist, for we can note that if $\rho_j (\{ l \}) \setminus A_j \neq \emptyset$ for some $j \in [1,...,n]$ then $\rho_i (\{l\}) \setminus A_i \neq \emptyset$ for all $i$. 

This splitting procedure can be repeated until there are no subsets $\{A_i\}$ satisfying condition (2). The procedure terminates, since the number of possible bisections of the images $\rho_i (\{l\})$ of $l$ in the various frames $\Theta_i$ is finite. More precisely, the maximum number of splitting steps is:
\[
\lceil \log_2 \max_{l\in{\mathcal{L}}} \min_{i=1,...,n} |\rho_i(\{l\})| \rceil.
\]

\emph{Uniqueness.} Suppose $\Omega'$ is another common coarsening satisfying condition (2), with refinings $\rho'_i : 2^{\Omega'} \rightarrow 2^{\Theta_i}$, distinct from $\Omega$. If we define $X_i\doteq \rho_i(\{\omega\})$ whenever $\omega\in\Omega$, by Lemma \ref{lem:aux} there exists $\omega' \in\Omega'$ such that $\rho_i (\{ \omega \}) \subset \rho'_i (\{ \omega' \})$. But then condition (2) implies that $\rho_i (\{ \omega \}) = \rho'_i (\{ \omega' \})$ for every pair $\omega,\omega'$, so that $\Omega = \Omega'$.
\end{proof}

\begin{definition}
We call this unique frame the \emph{maximal coarsening} of $\Theta_1,...,\Theta_n$, and denote it by $\Theta_1\oplus...\oplus\Theta_n$.
\end{definition}

\subsection{Maximal coarsening as greatest lower bound}

\begin{theorem}\label{the:inf}
If $\Omega$ is a common coarsening of a finite set of compatible frames $\Theta_1,...,\Theta_n$ then $\Omega$ is a coarsening of their maximal coarsening too, namely there exists a refining $\rho : 2^\Omega \rightarrow 2^{\Theta_1\oplus\cdots\oplus\Theta_n}$.
\end{theorem}
\begin{proof} 
Consider a different common coarsening $\Omega'$ of $\Theta_1,...,\Theta_n$, with refinings
$\rho'_i : 2^{\Omega'} \rightarrow 2^{\Theta_i}$. If it meets condition (2) of Theorem \ref{the:maximal} then, because of the uniqueness of the maximal coarsening $\Theta_1 \oplus \cdots \oplus \Theta_n$, the frame $\Omega'$ coincides with the latter. Otherwise, the
splitting procedure of the proof of Theorem \ref{the:maximal} can be applied to generate such a frame. Again, uniqueness guarantees that the outcome is indeed the maximal coarsening, and by construction it is a refinement of $\Omega'$.
\end{proof}

In other words,

\begin{corollary} \label{cor:maximal-coarsening-inf}
The maximal coarsening $\Theta_1 \oplus \cdots \oplus \Theta_n$ of a collection of compatible frames $\Theta_1,...,\Theta_n$ is the greatest lower bound ($\inf$) of $\Theta_1,...,\Theta_n$, seen as elements of the poset $(\mathcal{F},\leq^*)$ associated with order relation (\ref{eq:order-coa}).
\end{corollary}

\subsection{The dual lattices of frames} 

Recalling the definition of lattice (Chapter \ref{cha:geo}, Definition \ref{def:lattice}), Proposition \ref{the:minimal} and Theorems \ref{the:sup}, \ref{the:maximal} and \ref{the:inf} have a straightforward consequence on the algebraic structure of families of frames \cite{cuzzolin05amai}.

\begin{corollary} \label{cor:lattice-of-frames}
Both $(\mathcal{F},\leq)$ and $(\mathcal{F},\leq^*)$ where $\mathcal{F}$ is the collection of all sets of a family of compatible frames of discernment are lattices, where:
\[
\begin{array}{ccc}
\bigwedge_i \Theta_i = \bigotimes_i \Theta_i, \; \bigvee_i \Theta_i = \bigoplus_i \Theta_i, &  & \bigwedge^*_i \Theta_i = \bigoplus_i \Theta_i, \; \bigvee^*_i \Theta_i = \bigotimes_i \Theta_i.
\end{array}
\]
\end{corollary}

An infinite lattice $L$ is said \emph{complete} if any arbitrary collection (even not finite) of points in $L$ admits both $\sup$ and $\inf$. In a complete lattice $L$ there exist an \emph{initial element} $\mathbf{0} \equiv \wedge L$ and a \emph{final element} $\mathbf{1} \equiv \vee L$.\\ For example, the power set $2^\Theta$ of a set $\Theta$ is a complete lattice, with $\mathbf{0} = \emptyset$ and $\mathbf{1} = \{\Theta\}$. The \emph{height} $h(x)$ or `dimension' of an element $x$ in a lattice $L$ is the length of the maximal chain from $\mathbf{0}$ to $x$. For the power set $2^\Theta$, the height of a subset $A\in 2^\Theta$ is simply its cardinality $|A|$.

Now, a family of frames $\mathcal{F}$ lacks the attribute of completeness: the axioms which define the notion of family of compatible frames do not guarantee the existence of a minimal refinement (or a maximal coarsening) for an infinite (if only countable) collection of sets. 

\section{Semimodular structure of families of frames} \label{sec:semimodular}

More can be said about the class of lattices a family of compatible frames belongs to. A special such class arises from \emph{projective geometries} \cite{Rosenbaum}, i.e., collections $L(V)$ of all the subspaces of any vector space $V$.  

\begin{definition} \label{def:modular-lattice}
A lattice $L$ is called \emph{modular} iff whenever $a,b,c \in L$, if $a\geq b$, $a\wedge c = b\wedge c$ and $a\vee c= b\vee c$ then $a=b$.
\end{definition}

Modular lattices, as many authors have shown, are related to abstract independence. Therefore, the vector independence analogy illustrated in the Introduction to this Chapter would suggest a link between families of frames and the class of modular lattices. This is not entirely true.

\subsection{Upper and lower semimodularity, Birkhoff property}

Let us introduce a few necessary algebraic notions. Given two elements $x,y \in L$ of a lattice with order relation $\geq$, we say that $x$ `covers' $y$ (denoted by $x \succ y $), whenever $x$ is greater than $y$ ($x \geq y$) and they admit no intermediate element, namely: $\nexists z \in L$ such that $x \geq z \geq y$.

\begin{definition} \label{def:sm}
A lattice $L$ is \emph{upper semimodular} \cite{dilworth44} if for each pair $x,y$ of elements of $L$, $x \succ x\wedge y$ implies $x\vee y \succ y$. A lattice $L$ is \emph{lower semimodular} if for each pair $x,y$ of elements of $L$, $x\vee y \succ y$ implies $x\succ x\wedge y$.
\end{definition}
If $L$ is upper semimodular with respect to an order relation $\leq$, than the corresponding dual lattice with order relation $\leq^*$ is lower semimodular, as:
\begin{equation} \label{eq:duality}
x\succ x\wedge y \vdash x\vee y \succ y \hspace{5mm} \Rightarrow \hspace{5mm} x \vee^* y \succ^* x
\vdash y \succ^* x \wedge^* y.
\end{equation}
For lattices of finite length, upper and lower semimodularity together imply modularity. In this sense semimodularity is indeed ``one half" of modularity.

Another related class is that of `Birkhoff' lattices.

\begin{definition} \label{def:birhoff-lattice}
A lattice $L$ is called \emph{Birkhoff} \cite{Szasz} iff whenever $a\vee b\succ a,b$ then $a,b\succ a\wedge b$. 
\end{definition}

As for modular lattices, if a lattice if both upper and lower semimodular, then it is also Birkhoff. Nevertheless the two concepts remain distinct.

\subsection{The Birkhoff lattice of frames} \label{sec:birkhoff-lattice}

Indeed, finite families of frames endowed with order relation (\ref{eq:order-ref}) are Birkhoff.

\begin{theorem}\label{the:finite}
A finite family of compatible frames $(\langle \Theta \rangle_{A_{1,..,4,7}},\leq)$ generated by a base set $\Theta$, endowed with order relation (\ref{eq:order-ref}), is a {complete Birkhoff lattice of finite length}.
\end{theorem}

\begin{proof}\footnote{For an alternative proof based on the equivalence of $(\langle \Theta \rangle_{A_{1,..,4,7}} , \leq)$ to the equivalence (partition) lattice $\Pi(\Theta)$ see \cite{Szasz}, where it is proven that $(\langle \Theta \rangle_{A_{1,..,4,7}},\leq)$ is also \emph{relatively complemented.}}

\begin{itemize}
\item
The family $(\langle \Theta \rangle_{A_{1,..,4,7}},\leq)$ is \emph{complete}. Indeed, every finite lattice is complete for it does not contain any infinite collection of elements. 
\item
The family $(\langle \Theta \rangle_{A_{1,..,4,7}}, \leq)$ is \emph{Birkhoff} (Definition \ref{def:birhoff-lattice}). Consider two elements of the family $\Theta_1$, $\Theta_2$, and assume that their maximal coarsening indeed covers both frames: $\Theta_1\oplus \Theta_2 = \Theta_1 \vee \Theta_2 \succ \Theta_1,\Theta_2$. Then $\Theta_1\oplus \Theta_2$ must have cardinality:
\[
|\Theta_1\oplus \Theta_2| = |\Theta_1| - 1 = |\Theta_2| - 1,
\]
so that the associated refinings $\rho_1 : 2^{\Theta_1 \oplus \Theta_2} \rightarrow 2^{\Theta_1}$ and $\rho_2 : 2^{\Theta_1 \oplus \Theta_2} \rightarrow 2^{\Theta_2}$ leave unchanged each element of $\Theta_1 \oplus \Theta_2$ but one, replaced by two new elements.

Now, $\Theta_1$ and $\Theta_2$ also represent partitions of their minimal refinement $\Theta_1 \otimes \Theta_2$. By construction these partitions coincide in all but the elements obtained by refining the above two elements, as shown by Figure \ref{fig:part}.

\begin{figure}[ht!]
\begin{center}
\includegraphics[width = 0.25 \textwidth]{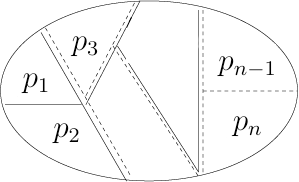}
\caption{\label{fig:part} Examples of partitions such as $\Pi_1$ (solid line) and $\Pi_2$ (dashed line) of Equation (\ref{eq:pi}).}
\end{center}
\end{figure}

Analytically:
\begin{equation} \label{eq:pi}
\Pi_1 = \big \{ p_1 \cup p_2, p_3, ..., p_n \big \}, \hspace{5mm} \Pi_2 = \big \{ p_1, ..., p_{n-2}, p_{n-1} \cup p_{n} \big \},
\end{equation}
having denoted the elements of the minimal refinement $\Theta_1 \otimes \Theta_2$ by $\{p_1,...,p_n\}$. The cardinality of the latter is then equal to $|\Theta_1 \otimes \Theta_2 | = |\Theta_1 | + 1 = |\Theta_2| + 1$. Clearly then $\Theta_1 \otimes \Theta_2$ is covered by both frames, for there cannot exist a frame with (integer) cardinality between $|\Theta_1| = |\Theta_2|$ and $|\Theta_1 \otimes \Theta_2| = |\Theta_1| + 1$.
\end{itemize}
\end{proof}

\begin{figure}[ht!]
\begin{center}
\includegraphics[width = 0.6 \textwidth]{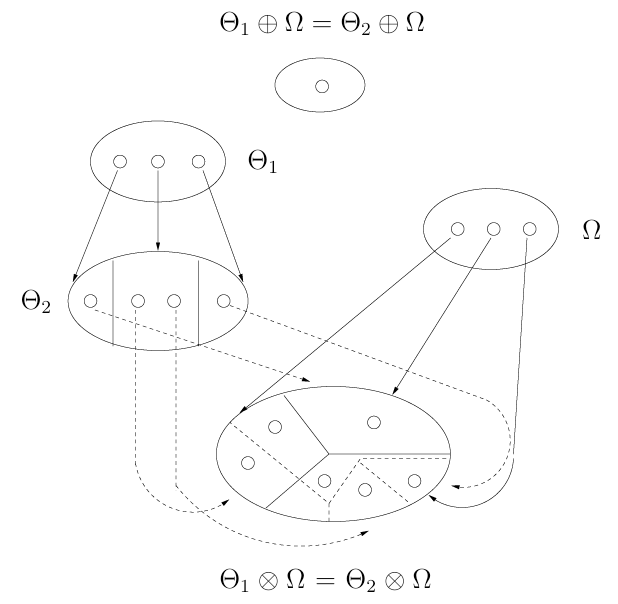}
\caption{\label{fig:nonmodular} Non-modularity of finite families of frames: a counterexample.}
\end{center}
\end{figure}

However, finite lattices of frames (regardless what order relation, $\leq$ or $\leq^*$, we pick) are \emph{not} modular: Figure \ref{fig:nonmodular} shows a simple counterexample in which the two frames on the left $\Theta_1$ and $\Theta_2$, linked by a refining (so that both $\Theta_2 \geq^* \Theta_1$ and $\Theta_1 \geq \Theta_2$), have the same minimal refinement $\Theta_1 \otimes \Omega = \Theta_2 \otimes \Omega$ and the same maximal coarsening $\Theta_1 \oplus \Omega = \Theta_2 \oplus \Omega$ (the unit frame) with the frame on the right $\Omega$.

On the other hand, the proof of Theorem \ref{the:finite} supports the (local) Birkhoff property of \emph{general} families of frames as well, within the sublattice $[\Theta_1 \oplus \cdots \oplus \Theta_n , \Theta_1 \otimes \cdots \otimes\Theta_n]$. Recall that a poset is said to have \emph{locally finite length} is each of its intervals, considered as posets, have finite length.

\begin{corollary} \label{cor:family}
The collection of sets $\mathcal{F}$ of a family of compatible frames is a \emph{locally Birkhoff lattice bounded below}, i.e., a Birhoff lattice of locally finite length with initial element.
\end{corollary}
\begin{proof}
It remains to point out that, by Theorem \ref{the:mongen}, every arbitrary collection of frames in $\mathcal{F}$ possesses a common coarsening $\mathbf{1}$, which plays the role of initial element of the lattice.
\end{proof}

\subsection{The upper and lower semimodular lattices of frames}

We can go a step further, and prove a stonger result: families of frames are both upper and lower semimodular lattices with respect to the dual order relations (\ref{eq:order-coa}) and (\ref{eq:order-ref}), respectively.

\begin{theorem} \label{the:semi}
$(\mathcal{F},\leq)$ is an upper semimodular lattice; $(\mathcal{F},\leq^*)$ is a lower semimodular
lattice.
\end{theorem}
\begin{figure}[ht!]
\begin{center}
\includegraphics[width = 0.75 \textwidth]{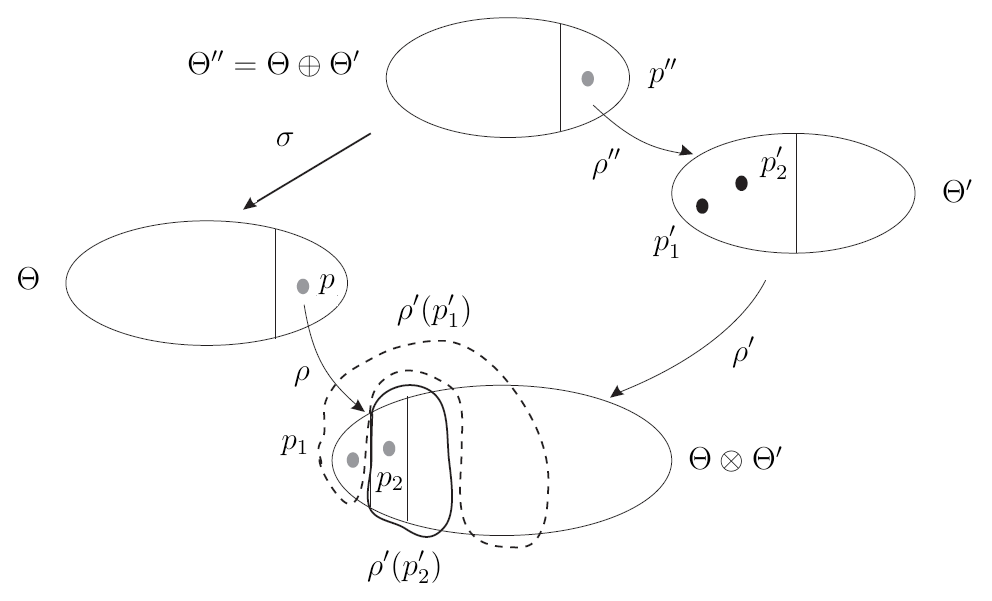}
\end{center}
\caption{Proof of the upper semimodularity of $(\mathcal{F},\leq)$. \label{fig:proof}}
\end{figure}
\begin{proof}
We just need to prove the upper semimodularity with respect to $\leq$ (\ref{eq:order-ref}).\\ Consider two compatible frames
$\Theta,\Theta'$, and suppose that $\Theta$ covers their minimal refinement $\Theta\otimes\Theta'$
(their inf with respect to $\leq$). The proof articulates into the following steps (see Figure
\ref{fig:proof}):
\begin{itemize}
\item
as $\Theta$ covers $\Theta\otimes\Theta'$ we have that $|\Theta| = |\Theta\otimes\Theta'| - 1$;
\item
this means that there exists a single element $p\in\Theta$ which is refined into a pair of
elements $\{p_1,p_2\}$ of $\Theta\otimes\Theta'$, while all other elements of $\Theta$ are left
unchanged: $\{p_1,p_2\} = \rho(p)$, where $\rho : 2^\Theta \rightarrow 2^{\Theta \otimes \Theta'}$;
\item
this in turn implies that $p_1,p_2$ each belong to the image of a different element of $\Theta'$
(otherwise $\Theta$ would itself be a refinement of $\Theta'$, and we would have
$\Theta\otimes\Theta' = \Theta$):
\[
p_1\in \rho'(p_1'), \hspace{10mm} p_2\in \rho'(p_2'),
\]
where $\rho'$ is the refining from $\Theta'$ to $\Theta \otimes \Theta'$;
\item
now, if we merge $p'_1,p'_2$ we obviously have a coarsening $\Theta''$ of $\Theta'$:
\[
\{p'_1,p'_2\} = \rho''(p''),
\]
with refining $\rho'' : 2^{\Theta''} \rightarrow 2^{\Theta'}$;
\item
but $\Theta''$ is a coarsening of $\Theta$, too, as we can build the refining $\sigma: 2^{\Theta''} \rightarrow 2^\Theta$ such that:
\[
\sigma(q) \doteq \rho'(\rho''(q)).
\]
\item
Indeed $\rho'(\rho''(q))$ is a subset of $\Theta$ $\forall q \in\Theta''$, as:
\begin{itemize}
\item
when $q = p''$ we can define:
\[
\sigma(p'') = \{ p \} \cup \big ( \rho'(p'_1) \setminus \{p_1\} \big ) \cup \big ( \rho'(p'_2) \setminus \{p_2\} \big ),
\]
as both $(\rho'(p'_1)\setminus \{p_1\})$ and $(\rho'(p'_2)\setminus \{p_2\})$ are elements of $\Theta$ that are not refined through $\rho$ when moving from $\Theta$ to $\Theta \otimes \Theta'$;
\item
when $q \neq p''$, $\rho'(\rho''(q)) \subset \Theta \otimes \Theta'$ is also a subset of $\Theta$, as all the elements of $\Theta$ but $p$ are left unchanged by $\rho$.
\end{itemize}
\item
as $|\Theta''| = |\Theta'|-1$ we have that $\Theta''$ is the maximal coarsening of $\Theta,\Theta'$: $\Theta'' =
\Theta \oplus \Theta'$;
\item
hence $\Theta \oplus \Theta'$ (the sup of $\Theta,\Theta'$ in $(\mathcal{F},\leq)$) covers $\Theta'$, and the lattice is upper semimodular (compare Definition \ref{def:sm}).
\end{itemize}
The lower semimodularity with respect to $\leq^*$ is then a consequence of (\ref{eq:duality}).
\end{proof}

In the following we will here focus on {finite} families of frames. More precisely, given a set of compatible frames $\{ \Theta_1,...,\Theta_n \}$ we will consider the set $P(\Theta)$ of all partitions of their minimal refinement $\Theta = \Theta_1\otimes \cdots \otimes \Theta_n$. As the independence condition (Definition \ref{def:indep}) involves only partitions of $\Theta_1\otimes \cdots \otimes \Theta_n$, we can conduct our analysis there. 

We will denote by $L^*(\Theta) \doteq (P(\Theta),\leq^*)$, $L(\Theta)\doteq (P(\Theta),\leq)$ the two lattices associated with the set $P(\Theta)$ of partitions of $\Theta$, endowed with order relations (\ref{eq:order-coa}), (\ref{eq:order-ref}) respectively.

\subsubsection{Example: the partition lattice $P_4$} \label{sec:partition-lattice}

Consider for example the partition lattice associated with a frame of size 4: $\Theta =
\{1,2,3,4\}$, depicted in Figure \ref{fig:p4}, with order relation $\leq^*$.
\begin{figure}[ht!]
\begin{center}
\begin{tabular}{c}
\includegraphics[width = 0.65 \textwidth]{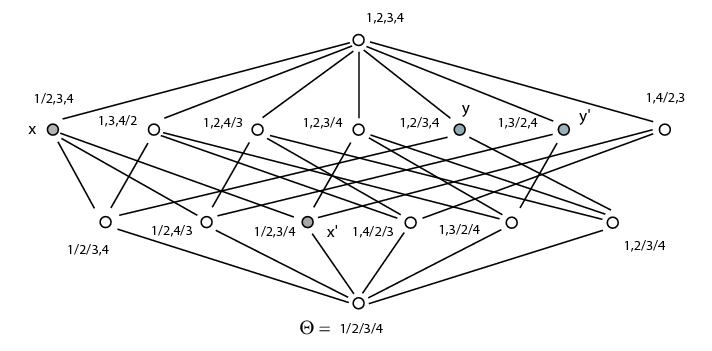}
\end{tabular}
\end{center}
\caption{The partition (lower) semimodular lattice $L^*(\Theta)$ for a frame $\Theta$ of size 4. Partitions $A_1,...,A_k$ of $\Theta$ are denoted by $A_1/.../A_k$. Partitions with the same number of elements are arranged on the same level. An edge between two nodes indicates that the bottom partition covers the top one. \label{fig:p4}}
\end{figure}
Each edge indicates here that the bottom partition covers $\succ$ the top one.

To understand how inf and sup work in the frame lattice, pick the following partitions:
\[
x = \{ 1/2,3,4 \}, \hspace{5mm} x' = \{ 1/2,3/4\}.
\]
According to the diagram the partition $x\vee^* x'$ which refines both and has smallest size is $\Theta = \{ 1/2/3/4\}$ itself. Their $\inf$ $x\wedge^* x'$ is $x$, as $x'$ is a refinement of $x$. If we pick instead the pair of partitions $y = \{ 1,2/3/4 \}$ and $y' = \{1,3/2,4\}$, we can notice that both $y,y'$ cover their $\inf$ $y \wedge^* y' = \{1,2,3,4\}$ but in turn their $\sup$ $y \vee^* y' = \Theta = \{1/2/3/4 \}$ does not cover them. Therefore, $(P(\Theta),\leq^*)$ is not upper semimodular (while it is \emph{lower} semimodular).

\chapter{Algebra of independence and conflict} \label{cha:independence}

\begin{center}
\includegraphics[width = 0.9 \textwidth]{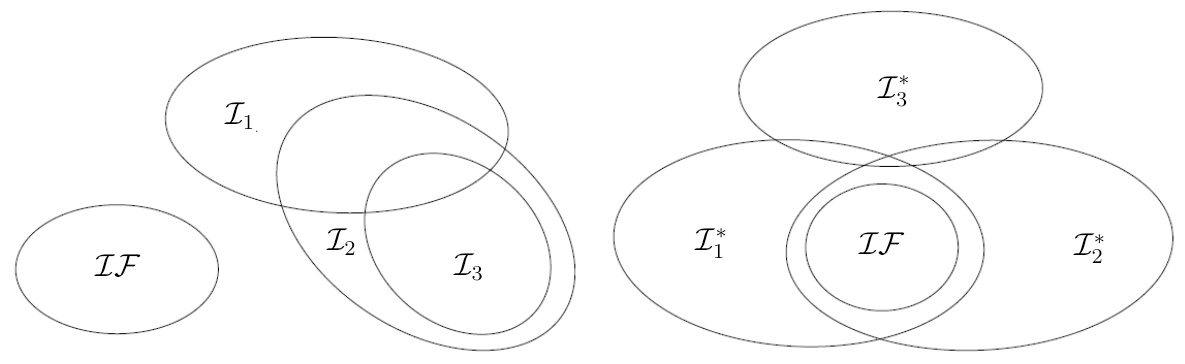}
\end{center}
\vspace{5mm}

% toe, independence of frames
As we recalled in Chapter \ref{cha:toe}, the theory of evidence was born as a contribution towards a mathematically rigorous description of subjective probability. In subjective probability, different observers (or `experts') of the same phenomenon possess in general different notions of what the decision space is. Different sensors may capture distinct aspects of the phenomenon to describe. Mathematically, this translates into admitting the existence of several distinct representations of this decision space at different levels of refinement. This idea is embodied in the theory of evidence by the notion of {family of frames}, which we introduced in Chapter \ref{cha:toe}, and whose algebraic properties we studied in Chapter \ref{cha:alg}. 

As we mentioned in our introduction to Chapter \ref{cha:alg}, the evidence gathered on distinct frames needs to be combined on a common frame, typically their minimal refinement.  Unfortunately, evidence fusion (at least under Dempster's orthogonal sum \cite{Dempster67,Dempster68a,Dempster69}) is guaranteed to take place in all cases if and only if the involved frames are independent \cite{Shafer76,cuzzolin05amai} as dictated by Definition \ref{def:indep}, which follows from the notion of independence of Boolean sub-algebras \cite{Sikorski}. We will denote in the following by $\mathcal{IF}$ the independence relation introduced in Definition \ref{def:indep}. As Dempster's sum assumes the conditional independence of the underlying probabilities generating belief functions through multi-valued mappings \cite{Dempster67,Dempster68a,Dempster69}, it is not surprising to realize that combinability (in Dempster's approach) and independence of frames (in Shafer's formulation of the theory of evidence) are strictly intertwined.

\subsection*{Scope of the Chapter}

In our Introduction to Chapter \ref{cha:alg} we outlined a proposal for dealing with possibly conflicting belief functions defined on different compatible frames in an algebraic setting, based on building a new collection of combinable b.f.s via a pseudo Gram-Schmidt algorithm. To investigate this possibility, we analysed the algebraic structure of families of frames and showed that they form upper and lower semimodular lattices, depending on which order relation we pick (Section \ref{sec:semimodular}). 

Now, some of the elements of a semimodular lattice possess interesting properties. Recall that if $L$ is a lattice bounded below then its \emph{atoms} are the elements of $L$ covering its initial element $\mathbf{0}$, namely:
\[ A = \big \{ a \in L \big | a \succ \mathbf{0} \big \}. \]
As a matter of fact the atoms of a semimodular lattice form a matroid \cite{harary69,Szasz}, so that a formal independence relation can be defined on them \cite{cuzzolin01bcc}. The latter can be generalised to arbitrary elements of the lattice, but the result is not univocal.

In this Chapter we take a further step forward and investigate the relation between Shafer's definition of independence of frames and these various extensions of matroidal independence to compatible frames as elements of a lattice, in order to draw some conclusions on the conjectured algebraic solution to the conflict problem. We study relationships and differences between the different forms of lattice-theoretical independence, and understand whether $\mathcal{IF}$ can be reduced to one of them. 

As a result, this Chapter poses the notion of independence of frames in a wider context by highlighting its relation with classical independence in modern algebra. Although $\mathcal{IF}$ turns out not to be a cryptomorphic form of matroidal independence, it does exhibit correlations with several extensions of matroidal independence to lattices, stressing the need for a more general, comprehensive definition of this widespread and important notion \cite{yaghlane00independence}.

\subsection*{Related Work}

Evidence combination has been widely studied \cite{zadeh86simple,yager87new} in different mathematical frameworks \cite{dubois92various} -- a comprehensive review would be impossible here. In particular, work has been done on the issue of merging conflicting evidence \cite{deutsch90study,josang03strategies,lefevre02if,wierman01measuring}, specially in critical situations in which the latter is derived from dependent sources \cite{cattaneo03isipta}. Campos and de Souza \cite{campos05nlp} have presented a method for fusing highly conflicting evidence which overcomes well known counterintuitive results. Liu \cite{liu06analyzing} has formally defined when two basic belief assignments are in conflict by means of quantitative measures of both the mass of the combined belief assigned to the emptyset before normalization, and the distance between betting commitments of beliefs. Murphy \cite{murphy00combining}, on her side, has studied a related problem: the failure to balance multiple evidence. The notion of conflicting evidence is well known in the context of sensor fusion \cite{carlson05tech}: the matter has been recently surveyed by Sentz and Ferson \cite{sentz02tech}. 

In opposition, not much work has been done on the properties of the families of compatible frames and their link with evidence combination. In \cite{Shafer87b} an analysis of the collections of all the partitions of a given frame in the context of the hierarchical representation of belief can nevertheless be found, while in \cite{Kohlas95} both the lattice-theoretical interpretation of families of frames and the meaning of the concept of independence are discussed. In \cite{cuzzolin05amai} these themes are reconsidered: the structure of Birkhoff lattice of a family of frames is proven, and the relation between Dempster's combination and independence of frames highlighted. Chapter \ref{cha:alg} is largely the result of the line of research first explored in \cite{cuzzolin05amai}.

\subsection*{Chapter Outline}

We start by characterizing the relationship between conflicting belief functions and independence of frames (Section \ref{sec:if}), proving that Dempster's combination is guaranteed {if and only if the underlying frames are independent} (Theorem \ref{the:7}). In Section \ref{sec:analogy} the classical notion of independence on matroids is recalled. Even though families of frames endowed with Shafer's independence $\mathcal{IF}$ \emph{do not} form a matroid, the form both upper and lower semimodular lattices, on which atomical matroidal independence can be extended to arbitrary elements. In Section \ref{sec:if-semimodular} we identify three diffferent extensions of  matroidal independence to arbitrary elements of the frame lattice. We discuss their interpretation, and thoroughly analyze their links with Shafer's independence of frames. Finally, in Section \ref{sec:discussion-frames} we recap what we learned about the relationship between independence of frames and the various algebraic definitions of independence, and outline the steps of a future investigation of the conjectured algebraic solution to the conflict problem.

% theory of evidence and independence of frames
\section{Independence of frames and Dempster's combination} \label{sec:if}

Although not equivalent to independence of sources in the original formulation of Dempster's combination, independence of frames is strictly intertwined with combinability.\\ In particular, the combination of belief functions defined on distinct frames of a family is guaranteed \emph{only for trivially interacting feature spaces}.

\begin{theorem}\label{the:7}
Let $\Theta_1, ..., \Theta_n$ a set of compatible FODs. Then the following conditions are equivalent:
\begin{enumerate}
\item all possible collections of belief functions $b_1, ..., b_n$ defined over $\Theta_1, ..., \Theta_n$, respectively, are combinable over the latter's minimal refinement $\Theta_1\otimes...\otimes\Theta_n;$
\item $\Theta_1, ..., \Theta_n$ are \emph{independent} ($\mathcal{IF}$);
\item there exists a 1-1 correspondence 
\[
\Theta_1\otimes\cdots\otimes\Theta_n \leftrightarrow \Theta_1\times\cdots\times\Theta_n,
\]
i.e., the minimal refinement of $\Theta_1, ..., \Theta_n$ is simply their Cartesian product;
\item $|\Theta_1\otimes\cdots\otimes\Theta_n|=\prod_{i=1}^n |\Theta_i|.$
\end{enumerate}
\end{theorem}
\begin{proof} $(1) \Rightarrow (2)$. We know that if $b_1,...,b_n$ are combinable then $b_i,b_j$ must be combinable $\forall
i,j=1,...,n$. Hence $\rho_i (\mathcal{C}_i) \cap \rho_j(\mathcal{C}_j)\neq \emptyset\;\forall i,j$, where
$\mathcal{C}_i$ denotes the core of $b_i$ and $\rho_i$ the refining linking $\Theta_i$ to the minimal refinement $\Theta_1\otimes...\otimes\Theta_n$. As $b_i,b_j$ can be chosen arbitrarily, their cores $\mathcal{C}_i,\mathcal{C}_j$ can be
any pair of subsets of $\Theta_i,\Theta_j$ respectively. Consequently, the previous condition can be rewritten as:
\[
\begin{array}{cc}
\rho_i(A_i) \cap \rho_j(A_j) \neq \emptyset & \forall A_i\subseteq \Theta_i, A_j\subseteq \Theta_j.
\end{array}
\]

$(2) \Rightarrow (1)$. It suffices to pick $A_i=\mathcal{C}_i$ $\forall i=1,...,n$.

$(2) \Rightarrow (3)$. We first note
that:
\[
\bigcap_i \rho_i(\theta_i^k) = \bigcap_i \rho_i(\theta_i^l) \Leftrightarrow \theta_i^k = \theta_i^l \hspace{5mm}\forall
i=1,...,n.
\] 
Indeed, if $\theta_i^k\neq \theta_i^l$ then $\rho_i (\{\theta_i^k\}) \neq \rho_i(\{\theta_i^l\})$, by definition of refining. But then:
\[
\bigcap_i \rho_i(\{\theta_i^k\}) \neq \bigcap_i \rho_i(\{\theta_i^l\}).
\] 
Therefore, the number of such intersections coincides with the number $|\Theta_1| \times \cdots \times |\Theta_n|$ of $n$-tuples of elements picked each from one of the frames.

$(3) \Rightarrow (2)$. By Proposition \ref{the:minimal} each element $\theta$ of the minimal refinement $\Theta_1\otimes \cdots \otimes \Theta_n$ corresponds to a subset of the form $\bigcap_i \rho_i(\{\theta_i^k\})$. Since by hypothesis there are $|\Theta_1|\times \cdots \times |\Theta_n|$ such elements, there exist an equal number of subsets of the above form. But this is possible only if they all are non-empty: hence, $\Theta_1,...,\Theta_n$ are independent.

$(3) \Rightarrow (4)$. Obvious.

$(4) \Rightarrow (3)$. Once again $\Theta_1\otimes \cdots \otimes \Theta_n = \big \{ \bigcap_i \rho_i(\{\theta_i\}) \; \forall \theta_i\in\Theta_i \big \}$. Hence, if its cardinality is $|\Theta_1|\times \cdots \times |\Theta_n|$ then the intersections which form its elements must all be non-empty. Each of them can be labeled by $(\theta_1,...,\theta_n)$.
\end{proof}

Indeed, as we recalled in Chapter \ref{cha:toe}, Section \ref{sec:weight-of-conflict}, any given set of belief functions is characterized by a \emph{level of conflict} $\mathcal{K}$ --- if $\mathcal{K}=\infty$ they are not combinable. A basic property of the level of conflict is the following. Given $n+1$ belief functions $b_1,...,b_{n+1}$\footnote{E.g. the projections on their minimal refinement of belief functions encoding measurements inherently defined on different compatible domains, see Chapter \ref{cha:pose} for an application to object pose estimation.}:
\[
\mathcal{K}(b_1, \cdots , b_{n+1}) = \mathcal{K}(b_1, ... ,b_n) + {\mathcal{K}}(b_1\oplus ... \oplus b_n, b_{n+1}),
\]
so that if ${\mathcal{K}}(b_i,b_j)=+\infty$ then ${\mathcal{K}}(b_i,b_j,b_k) = +\infty$ $\forall k$. 

This suggests a bottom-up technique \cite{Cuzzolin99,Cuzzolin2000,cuzzolin05isipta}. First the level of conflict is computed for each pair of belief functions $(b_i,b_j)$, $i,j=1,...,n$. Then a suitable threshold is established and a `{conflict graph}' is built in which each node represents a belief function, while an edge indicates a (pairwise) conflict level below the set threshold. Finally, the subsets of combinabile b.f.s of size $d+1$ are {recorsively} computed from those of size $d$, eventually leading to the detection of the most coherent set of features. 

This approach, however, suffers from a high computational cost when large groups of belief functions are found to be compatible.

\section{An algebraic study of independence of frames} \label{sec:analogy}

Is there an alternative to the computationally expensive conflict graph technique? Theorem \ref{the:7} suggests that belief functions never conflict when the domains on which they are defined are independent, according to Definition \ref{def:indep}.

As outlined in the introduction to Chapter \ref{cha:alg} in \cite{cuzzolin05amai}, starting from an analogy between independence of frames and linear independence, we conjectured a possible algebraic solution to the conflict problem based on a mechanism similar to the classical Gram-Schmidt algorithm for the orthogonalization of vectors. Indeed, the independence condition (\ref{eq:2}) closely resembles the condition under which a collection of vector subspaces has maximal span:
\begin{equation}\label{eq:analogy}
\begin{array}{ccc}
v_1 + \cdots + v_n \neq \vec{0}, \; \forall \vec{0} \neq v_i \in V_i & \equiv & span\{ V_1,...,V_n \} = V_1 \times \cdots \times V_n \\ \\ \rho_1(A_1) \cap \cdots \cap \rho_n(A_n) \neq \emptyset, \; \forall \emptyset \neq A_i \subseteq \Theta_i & \hspace{3mm} \equiv \hspace{3mm} & \Theta_1 \otimes \cdots \otimes \Theta_n = \Theta_1 \times \cdots \times \Theta_n,
\end{array}
\end{equation}
where $\vec{0}$ is the common origin of the vector spaces $\{ V_1,...,V_n \}$.\\
Let us call $\{ V_1,...,V_n \}$ `independent' iff each collection of non-null representative vectors $\{ v_i\in V_i, i=1,..n \}$, each member of a different subspace, are linearly independent. It follows that while a collection of compatible frames $\{ \Theta_1,...,\Theta_n \}$ are $\mathcal{IF}$ iff each selection of representative subsets $A_i\in 2^{\Theta_i}$ have non-empty intersection, a collection of vectors subspaces $\{ V_1,...,V_n \}$ are independent iff for each choice of non-null vectors $v_i\in V_i$ their sum is non-zero.

The collection of all subspaces of a vector space (or \emph{projective geometry} \cite{Rosenbaum}) forms a {modular lattice} (see Chapter \ref{cha:alg}, Definition \ref{def:modular-lattice}). As we have seen in Chapter \ref{cha:alg}, instead, families of compatible frames are semimodular lattices, hinting at a possible explanation of this analogy. Here we move on to analyze the notion of independence of frames (and its relationships with other definitions of independence in other fields of modern algebra) from an algebraic point of view.

\subsection{Matroids} \label{sec:matroids}

The paradigm of abstract independence in modern algebra is represented by the notion of \emph{matroid}, introduced by Whitney in the 1930s \cite{Whitney35}. He and other authors, among which van der Waerden \cite{vanderwaerden37}, Mac Lane \cite{maclane38}, and Teichmuller \cite{teichmuller36}, recognized at the time that several apparently different notions of dependence \cite{harary69,Rosenbaum} in algebra (such as circuits in graphs, flats in affine geometries) have many properties in common with that of linear dependence of vectors.
    % definition of matroid
\begin{definition}\label{def:matroid}
A \emph{matroid} $M = (E,\mathcal{I}) $ is a pair formed by a \emph{ground set} $E$ and a collection of \emph{independent sets} $\mathcal{I}\subseteq 2^E$, which obey the following axioms:
\begin{enumerate}
\item
$\emptyset \in \mathcal{I}$;
\item
if $I\in\mathcal{I}$ and $I'\subseteq I$ then $I'\in\mathcal{I}$;
\item
if $I_1$ and $I_2$ are in $\mathcal{I}$, and $|I_1|<|I_2|$, then there is an element $e$ of $I_2 \setminus I_1$ such that $I_1 \cup e \in \mathcal{I}$.
\end{enumerate}
\end{definition}
Condition (3) is called \emph{augmentation} axiom, and is the foundation of the notion of matroidal independence, as it can be proved that a number of domain-specific independence relations can be reduced to the augmentation property. The name was coined by Whitney because of a fundamental class of matroids which arise from the the collections of linearly independent (in the ordinary sense) sets of columns of a matrix, called `{vector matroid}' \cite{Whitney35}. 

\subsection{Families of frames are not matroids}

Unfortunately,

\begin{theorem}
A family of compatible frames $\mathcal{F}$ endowed with Shafer's independence $\mathcal{IF}$ is \emph{not} a matroid.
\end{theorem}

\begin{proof}
In fact, $\mathcal{IF}$ does not meet the augmentation axiom (3) of Definition \ref{def:matroid}. Consider two independent compatible frames $I = \{ \Theta_1,\Theta_2 \}$. If we pick another arbitrary frame $\Theta_3$ of the family, the collection $I' = \{ \Theta_3  \}$ is trivially $\mathcal{IF}$. Suppose $\Theta_3\neq \Theta_1,\Theta_2$. Then, since $|I|>|I'|$, by
augmentation we can form a new pair of independent frames by adding any of $\Theta_1,\Theta_2$ to $\Theta_3$. But it is easy to find a counterexample, for instance by picking as $\Theta_3$ the common coarsening of $\Theta_1$ and $\Theta_2$ (compare the remark after Definition \ref{def:indep}).
\end{proof}

Matroidal independence, though, generalizes to `sister' relations in other algebraic structures, in particular semimodular and `{geometric}' lattices \cite{Stern}. Although families of frames are \emph{not} matroids, they do form (upper and lower) semimodular lattices (Chapter \ref{cha:alg}, Section \ref{sec:semimodular}). As a consequence, $\mathcal{IF}$ inherits interesting relations with some extensions of matroidal independence to semimodular lattices \cite{Birkhoff35}, as we are going to see in Section \ref{sec:if-semimodular}. Indeed, $\mathcal{IF}$ is opposed to matroidal independence (Section \ref{sec:i-lower}).

% relation between IF,LI1,2,3 in the lattice of frames
\section{Independence on lattices versus independence of frames}\label{sec:if-semimodular}

\subsection{Atom matroid of a semimodular lattice}\label{sec:atom-li}

Consider again the usual example of linear independence of vectors. By definition $\{ v_1,...,v_n \}$ are \emph{linearly independent} iff $\displaystyle \sum_i \alpha_i v_i = \vec{0}$ implies $\alpha_i = 0 \; \forall i.$

This classical definition can be given several equivalent formulations:
\begin{equation}\label{eq:i-vec}
\begin{array}{lll}
\mathcal{I}_1: \hspace{5mm} & v_j \not \subset span(v_i,i\neq j) & \forall j = 1,...,n; \\
\mathcal{I}_2: \hspace{5mm} & v_j \cap span(v_1,...,v_{j-1}) = \vec{0} \hspace{5mm} & \forall j = 2,...,n; \\
\mathcal{I}_3: \hspace{5mm} & \dim(span(v_1,...,v_n)) = n. &
\end{array}
\end{equation}
\begin{figure}[ht!]
\includegraphics[width = 0.65 \textwidth]{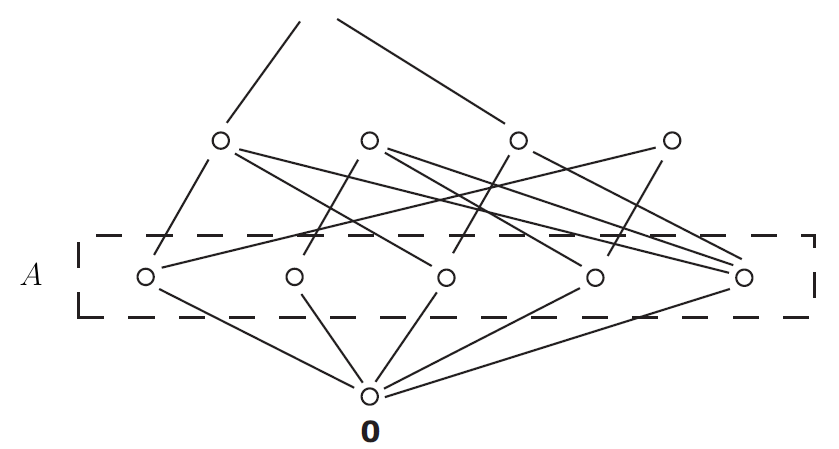}
\caption{A lattice can be represented as a (\emph{Hasse}) diagram in which covering relations are drawn as undirected edges. The atoms $A$ of a lattice which initial element $\mathbf{0}$ (\emph{bounded below}) are the elements covering $\mathbf{0}$.\label{fig:atoms}}
\end{figure}
Remember that the one-dimensional subspaces of a vector space $V$ are the atoms of the lattice $L(V)$ of all the linear subspaces of $V$, for which $span = \vee$, $\cap = \wedge$, $dim = h$ and $\mathbf{0} = \vec{0}$. Following this intuition, we can extend the relations (\ref{eq:i-vec}) to collections of arbitrary (non necessarily atomic) non-zero elements of an arbitrary semimodular lattice with initial element, as follows.

\begin{definition} \label{def:i}
The following relations on the elements of a semimodular lattice with initial element $\mathbf{0}$ can be defined:
\begin{enumerate}
\item
$\{ l_1,...,l_n \}$ are $\mathcal{I}_1$ if $\displaystyle l_j \not \leq \bigvee_{i \neq j} l_i$ (or, equivalently, $\displaystyle l_j \wedge \bigvee_{i \neq j} l_i \neq l_j$) for all $j = 1,...,n$;
\item
$\{l_1,...,l_n\}$ are $\mathcal{I}_2$ if $\displaystyle l_j \wedge \bigvee_{i < j} l_i = \mathbf{0}$ for all $j = 2,...,n$;
\item
$\{ l_1,...,l_n \}$ are $\mathcal{I}_3$ if $\displaystyle h \bigg(\bigvee_{i} l_i \bigg) = \sum_i h(l_i)$. 
\end{enumerate}
\end{definition}
These relations have been studied by several authors in the past. Our goal here is to understand their relation with independence of frames in the semimodular lattice of frames. Graphical interpretations of $\mathcal{I}_1, \mathcal{I}_2$ and $\mathcal{I}_3$ in terms of Hasse diagrams are given in Figure \ref{fig:graphical}.

\begin{figure}[ht!]
\includegraphics[width = 1 \textwidth]{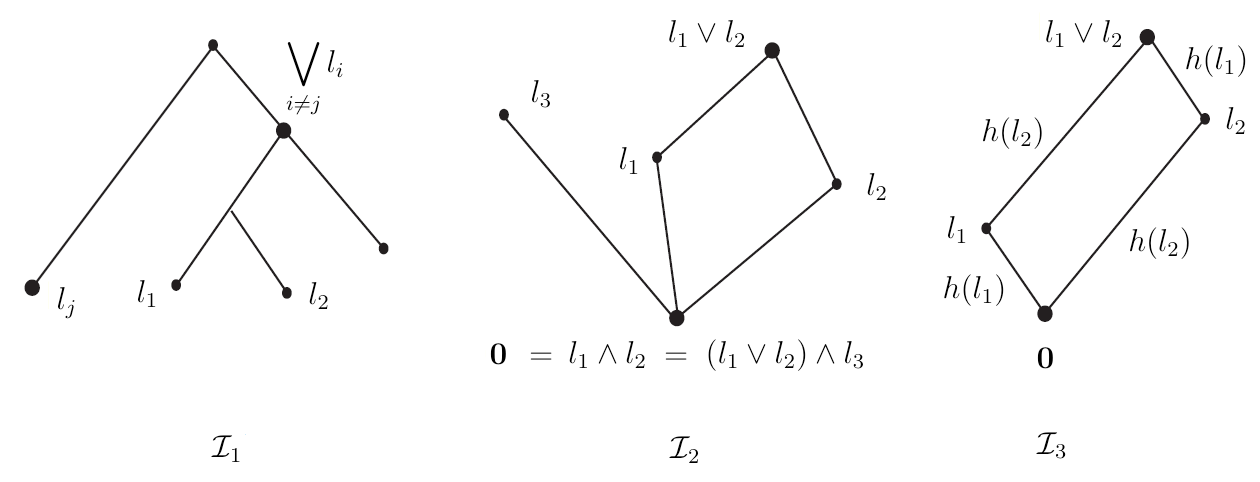}
\caption{Graphical interpretation of the relations introduced in Definition \ref{def:i}. \label{fig:graphical}}
\end{figure}

When applied to arbitrary elements of a lattice $\mathcal{I}_1$, $\mathcal{I}_2$, $\mathcal{I}_3$ are distinct, and none of them generates a matroid. However, when defined on the \emph{atoms} of an upper semimodular lattice with initial element they do coincide, and form a matroid \cite{Szasz}.

\begin{proposition} \label{pro:szasz}
The restrictions of the above relations to the set of the atoms $A$ of an upper semimodular lattice $L$ with initial element coincide, namely $\mathcal{I}_1 = \mathcal{I}_2 = \mathcal{I}_3 = \mathcal{I}$ on $A$, and $(A,\mathcal{I})$ is a matroid.
\end{proposition}

As the partition lattice (see Chapter \ref{cha:alg}, Section \ref{sec:partition-lattice}) has both an upper $L(\Theta)$ and lower $L^*(\Theta)$ semimodular form, we can introduce there two dual forms $\mathcal{I}_1$, $\mathcal{I}_2$, $\mathcal{I}_3$ and $\mathcal{I}^*_1$, $\mathcal{I}^*_2$, $\mathcal{I}^*_3$ of the above relations, respectively associated with $L(\Theta)$ and $L^*(\Theta)$. All these relations constitute valid {extensions of matroidal independence to all the elements of a semimodular lattice}. In the remainder of this Chapter we will investigate their relations with Shafer's independence of frames.

\subsection{Subspace lattice} \label{sec:subspace-lattice}

It can be interesting to see first how these candidate independence relations behave in the lattice of all vector subspaces $L(V)$ (see Equation (\ref{eq:i-vec})).

\begin{theorem} \label{the:li}
In the subspace lattice $L(V)$ relation ${\mathcal{I}}_2$ implies $\mathcal{I}_1$.
\end{theorem}

\begin{proof}
Let us consider a number of linear subspaces $V_1,...,V_n$ of a vector space $V$ which meet ${\mathcal{I}}_2$, namely:
\[
V_2\cap V_1 = \vec{0}, \hspace{3mm} V_3\cap span(V_1,V_2) = \vec{0}, \hspace{3mm} \cdots \hspace{3mm}, V_n \cap span(V_1,...,V_{n-1}) = \vec{0}.
\]
Suppose then that there exists a subspace $V_k$ such that:
\[
V_k\subset span(V_1, \cdots ,V_{k-1}, V_{k+1}, \cdots , V_n),
\] 
i.e., the collection $V_1,...,V_n$ is not $\mathcal{I}_1$. By hypothesis $V_k\cap span(V_1,...,V_{k-1}) = \vec{0}$, so that the last condition implies $V_k \subset span(V_{k+1},...,V_n)$, which is in turn equivalent to:
\[
\exists \; l \in [k+1,...,n] \; s.t. \; V_k \cap V_l \neq  \vec{0}.
\]
But again by hypothesis $V_l \cap span(V_1,...,V_{l-1}) = \vec{0}$ which implies $V_l\cap V_k = \vec{0}$, since $k<l$. Therefore, we have a contradiction.
\end{proof}

By looking at the proof of Theorem 70, page 152 of \cite{Szasz}, restated as follows:

\begin{proposition}
If a finite set of atoms of a semimodular lattice bounded below is $\mathcal{I}_2$, then it is $\mathcal{I}_1$.
\end{proposition}

we can note that it is based on the assumption that $\mathcal{I}_1$ is a linear independence relation \emph{among atoms}, in particular that $\mathcal{D}_1 = \overline{\mathcal{I}_1}$ satisfies the augmentation axiom (see Definition \ref{def:matroid}). In $L(V)$, on the other hand, this does not hold. Nevertheless, Theorem \ref{the:li} overcomes this difficulty by providing a proof of the implication between the two candidate independence relations.

\subsection{Boolean and lattice independence in the upper semimodular lattice $L(\Theta)$}

Let us then consider the candidate independence relations on upper semimodular form of the partition lattice, and investigate their relationships with independence of frames ($\mathcal{IF}$).

\subsubsection{Forms and semantics of extended matroidal independence}

In $L(\Theta)$ the relations introduced in Definition \ref{def:i} assume the forms:
\begin{equation}\label{eq:i1sm}
\begin{array}{lllll}
\{ \Theta_1,...,\Theta_n \} \in \mathcal{I}_1 & & \Leftrightarrow & & \displaystyle \Theta_j \otimes
\bigoplus_{i\neq j} \Theta_i \neq \Theta_j \hspace{5mm} \forall \; j = 1,...,n,
\end{array}
\end{equation}
\begin{equation}\label{eq:i2sm}
\begin{array}{lllll}
\{ \Theta_1,...,\Theta_n \} \in \mathcal{I}_2 & & \Leftrightarrow & & \displaystyle \Theta_j \otimes
\bigoplus_{i < j} \Theta_i = \Theta \hspace{5mm} \forall \; j = 2,...,n,
\end{array}
\end{equation}
\begin{equation}\label{eq:i3sm}
\begin{array}{lllll}
\{ \Theta_1,...,\Theta_n \} \in \mathcal{I}_3 & & \Leftrightarrow & & \displaystyle |\Theta| -
\Big| \bigoplus_{i =1}^n \Theta_i \Big| = \sum_{i=1}^n (|\Theta| - |\Theta_i|),
\end{array}
\end{equation}
as in the lattice $L(\Theta)$ we have $\Theta_i \wedge \Theta_j = \Theta_i \otimes \Theta_j$,
$\Theta_i \vee \Theta_j = \Theta_i \oplus \Theta_j$, $h(\Theta_i) = |\Theta| - |\Theta_i|$, and $\mathbf{0}=\Theta$.\\
They read as follows: $\{ \Theta_1,...,\Theta_n \}$ are $\mathcal{I}_1$ iff no frame $\Theta_j$ is a
refinement of the maximal coarsening of all the others. They are $\mathcal{I}_2$ iff $\forall \; j
= 2,...,n$ $\Theta_j$ does not have a non-trivial common refinement with the maximal coarsening of
all its predecessors. 

The interpretation of $\mathcal{I}_3$ is perhaps more interesting, for $\mathcal{I}_3$ is equivalent to say that the coarsening that generates $| \bigoplus_{i =1}^n \Theta_i |$ can be broken up into $n$ steps of the same length of the coarsenings that generate each of the frames $\Theta_i$ starting from $\Theta$. Namely: first $\Theta_1$ is obtained from $\Theta$ by merging $|\Theta| - |\Theta_1|$ elements, then $|\Theta| - |\Theta_2|$ elements of this new frame are merged, and so on until we get $| \bigoplus_{i =1}^n \Theta_i |$. We will return on this when discussing the dual relation on the lower semimodular lattice $L^*(\Theta)$.

To study the logical implications between these lattice-theoretic relations and independence of frames, and between themselves, we first need a useful lemma. Let $\mathbf{0}_\mathcal{F}$ denote the unique unitary frame of a family $\mathcal{F}$ of compatible frames of discernment (see Chapter \ref{cha:alg}).
\begin{lemma} \label{lem:1}
$\{ \Theta_1,...,\Theta_n \} \in \mathcal{IF}$, $n>1$ $\vdash$ $\bigoplus_{i=1}^n \Theta_i =
\mathbf{0}_\mathcal{F}$.
\end{lemma}
\emph{Proof}. We prove Lemma \ref{lem:1} by induction. For $n=2$, let us suppose that $\{ \Theta_1,\Theta_2 \}$ are $\mathcal{IF}$. Then $\rho_1(A_1) \cap \rho_2(A_2) \neq \emptyset$ $\forall A_1 \subseteq \Theta_1, A_2 \subseteq \Theta_2$, $A_1,A_2\neq\emptyset$ ($\rho_i$ denotes as usual the refining from $\Theta_i$ to $\Theta_1\otimes \Theta_2$). Suppose by absurd that their common coarsening contains more than a single element, $\Theta_1\oplus \Theta_2 = \{a,b\}$. But then
\[
\rho_1(\rho^1(a)) \cap \rho_2(\rho^2(b)) = \emptyset
\]
(where $\rho^i$ denotes the refining between $\Theta_1\oplus \Theta_2$ and $\Theta_i$), going against the hypothesis.\\ \emph{Induction step}. Suppose that the thesis is true for $n-1$. We know that $\{ \Theta_1,...,\Theta_n \} \in \mathcal{IF}$ implies $\{\Theta_i, i\neq j\} \in \mathcal{IF}$. By inductive hypothesis, the latter implies:
\[
\begin{array}{ccc}
\displaystyle \bigoplus_{i\neq j} \Theta_i = \mathbf{0}_\mathcal{F} & & \forall j=1,...,n.
\end{array}
\]
Then, since $\mathbf{0}_\mathcal{F}$ is a coarsening of $\Theta_j$ $\forall$ $j$,
$\displaystyle
\Theta_j \oplus \bigoplus_{i\neq j} \Theta_i = \Theta_j \oplus \mathbf{0}_\mathcal{F} =
\mathbf{0}_\mathcal{F}$. $\Box$

\subsubsection{Pairs of frames}

Let us consider first the special case of collections of just two frames.
For $n=2$ the three relations $\mathcal{I}_1,\mathcal{I}_2,\mathcal{I}_3$ read respectively as:
\begin{equation}\label{eq:in2}
\Theta_1 \otimes \Theta_2 \neq \Theta_1,\Theta_2, \hspace{5mm} \Theta_1 \otimes \Theta_2 = \Theta,
\hspace{5mm} |\Theta| + |\Theta_1 \oplus \Theta_2| = |\Theta_1| + |\Theta_2|.
\end{equation}
It is interesting to remark that $\{ \Theta_1,\Theta_2 \} \in \mathcal{I}_1$ implies $\Theta_1,\Theta_2\neq \Theta$. 

\begin{theorem} \label{the:n2}
The following relationships between the various form of independence, when applied to pairs of frames $\Theta_1,\Theta_2$ considered as elements of $L(\Theta)$, hold:
\begin{enumerate}
\item
$\{ \Theta_1,\Theta_2 \} \in \mathcal{IF} \vdash \{ \Theta_1,\Theta_2 \} \in \mathcal{I}_1$ if $\Theta_1,\Theta_2 \neq \mathbf{0}_F$;
\item
$\{ \Theta_1,\Theta_2 \} \in \mathcal{I}_1 \not \vdash  \{ \Theta_1,\Theta_2 \} \in \mathcal{IF}$;
\item
$\{ \Theta_1,\Theta_2 \} \in \mathcal{I}_2 \vdash \{ \Theta_1,\Theta_2 \} \in \mathcal{I}_1$ iff $\Theta_1,\Theta_2\neq\Theta$;
\item
$\{ \Theta_1,\Theta_2 \} \in \mathcal{I}_3 \vdash \{ \Theta_1,\Theta_2 \} \in \mathcal{I}_1$ iff $\Theta_1,\Theta_2\neq\Theta$;
\item
$\{ \Theta_1,\Theta_2 \} \in \mathcal{IF} \wedge \mathcal{I}_3$ iff $\Theta_i = \mathbf{0}_F$ and $\Theta_j = \Theta$,
\end{enumerate}
where $\mathbf{0}_F$ is the unique unitary frame of $L(\Theta)$.
\end{theorem}
\begin{proof}
Let us consider the conjectured properties in the given order.
\begin{enumerate}
\item
If $\{ \Theta_1,\Theta_2 \}$ are $\mathcal{IF}$ then $\Theta_1$ is not a refinement of $\Theta_2$, and vice-versa, unless one of them is $\mathbf{0}_\mathcal{F}$. But then they are $\mathcal{I}_1$ ($\Theta_1\otimes\Theta_2 \neq \Theta_1,\Theta_2$).
\item
We can give a counterexample (see Figure \ref{fig:th4})
\begin{figure}[ht!]
\includegraphics[width = 0.65 \textwidth]{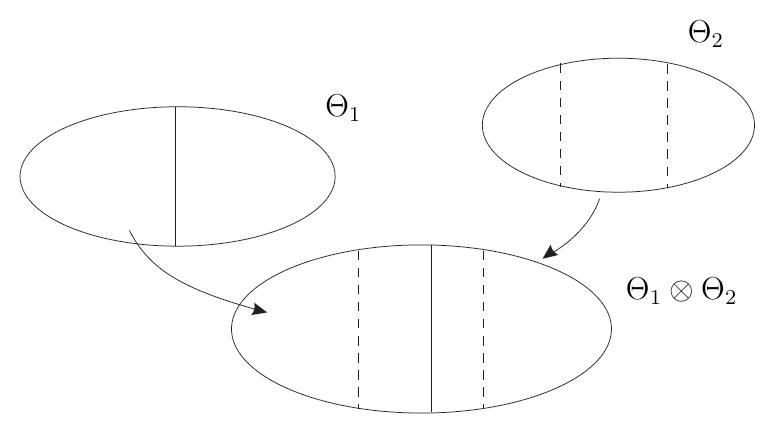}
\caption{Counterexample for the conjecture $\mathcal{I}_1 \vdash \mathcal{IF}$ of Theorem \ref{the:n2}. \label{fig:th4}}
\end{figure}
in which $\{ \Theta_1,\Theta_2 \}$ are $\mathcal{I}_1$ (as none of them is a refinement of the other one) but their minimal refinement $\Theta_1\otimes\Theta_2$ has cardinality $4 \neq |\Theta_1|\cdot |\Theta_2| = 6$ (hence they are not $\mathcal{IF}$).
\item
Trivial.
\item
$\mathcal{I}_3 \vdash \mathcal{I}_1$ is equivalent to $\neg \mathcal{I}_1 \vdash \neg \mathcal{I}_3$. But $\{ \Theta_1,\Theta_2 \} \in \neg \mathcal{I}_1$ reads as $\Theta_1 \otimes \Theta_2 = \Theta_i$, which is in turn equivalent to $\Theta_1 \oplus \Theta_2 = \Theta_j$. I.E.,
\[
\{ \Theta_1,\Theta_2 \} \in \mathcal{I}_3 \equiv |\Theta| + |\Theta_j| = |\Theta_i| + |\Theta_j| \equiv |\Theta| = |\Theta_i|.
\]
But then $\{ \Theta_1,\Theta_2 \} \in\mathcal{I}_3 \vdash \{ \Theta_1,\Theta_2 \} \in \mathcal{I}_1$ iff $\Theta_1,\Theta_2 \neq \Theta$.
\item
As $\{ \Theta_1,\Theta_2 \}$ are $\mathcal{IF}$, by Lemma \ref{lem:1} $|\Theta_1\oplus \Theta_2| = 1$, so that $\{ \Theta_1,\Theta_2 \} \in \mathcal{I}_3$ is equivalent to 
\begin{equation} \label{eq:diamond}
|\Theta| + 1 = |\Theta_1| + |\Theta_2|. 
\end{equation}
Now, by definition:
\[
|\Theta| \geq | \Theta_1 \otimes \Theta_2| = |\Theta_1| |\Theta_2|
\]
(the last passage holding as those frames are $\mathcal{IF}$).\\ Therefore $\{ \Theta_1,\Theta_2 \} \in\mathcal{IF}$ and $\{ \Theta_1,\Theta_2 \} \in \mathcal{I}_3$ together imply:
\[
|\Theta_1| + |\Theta_2| = |\Theta| + 1 \geq |\Theta_1| |\Theta_2| + 1 
\]
which is equivalent to
\[
|\Theta_1| -1 \geq |\Theta_1| |\Theta_2| - |\Theta_2| = |\Theta_2| (|\Theta_1| - 1) \equiv |\Theta_2| \leq 1.
\]
The latter holds iff $\Theta_2 = \mathbf{0}_\mathcal{F}$, which in turn implies that $|\Theta_2| = 1$. By (\ref{eq:diamond}) we have $|\Theta_1| = |\Theta|$, i.e., $\Theta_1 = \Theta$.
\end{enumerate}
\end{proof}

In the `singular' case $\Theta_1 = \mathbf{0}_\mathcal{F}, \Theta_2 = \Theta$, by Equation (\ref{eq:in2}) the pair $\{ \mathbf{0}_\mathcal{F}, \Theta \}$ is both $\mathcal{I}_2$ and $\mathcal{I}_3$, but not $\mathcal{I}_1$. Besides, two frames can be both $\mathcal{I}_2$ and $\mathcal{IF}$ without being singular in the above sense. The pair of frames $\{ y,y' \}$ in Chapter \ref{cha:alg}, Figure \ref{fig:p4} provides such an example, as $y\otimes y' = \Theta$ ($\mathcal{I}_2$) and $\{ y,y' \}$ are $\mathcal{IF}$ (easy to check).

As it well known that \cite{Szasz} on an upper semimodular lattice (such as $L(\Theta)$)
\begin{proposition} \label{pro:szasz2}
$\mathcal{I}_3 \vdash \mathcal{I}_2$.
\end{proposition}
the overall picture formed by the different lattice-extended matroidal independence relations for \emph{pairs} of frames (excluding singular cases) is as in Figure \ref{fig:relations-upper-n2}. Independence of frames and the strictest form $\mathcal{I}_3$ of extended matroidal independence are mutual exclusive, and are both stronger than the weakest form $\mathcal{I}_1$. Some of those features are retained by the general case too.

\begin{figure}[ht!]
\centering
\includegraphics[width = 0.5 \textwidth]{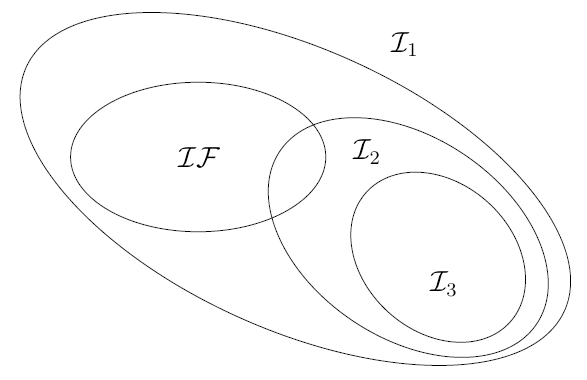}
\caption{Relations between independence of frames $\mathcal{IF}$ and the different extensions $\mathcal{I}_1, \mathcal{I}_2$ and $\mathcal{I}_3$ of matroidal independence to pairs of frames as elements of the upper semimodular lattice $L(\Theta)$ (from Theorem \ref{the:n2}). \label{fig:relations-upper-n2}}
\end{figure}

\subsubsection{General case, $n>2$}

The situation is somehow different in the general case of a collection of $n$ frames. $\mathcal{IF}$ and $\mathcal{I}_1$, in particular, turn out to be incompatible.
\begin{theorem} \label{the:upper-ifi1}
If $\{ \Theta_1,...,\Theta_n \} \in \mathcal{IF}$, $n>2$ then $\{ \Theta_1,...,\Theta_n \} \in \neg \mathcal{I}_1$.
\end{theorem}
\begin{proof}
If $\{ \Theta_1,...,\Theta_n \}$ are $\mathcal{IF}$ then any collection formed by some of those frames is
$\mathcal{IF}$ (otherwise we could find empty intersections in $\Theta_1\otimes \cdots \otimes
\Theta_n$). But then, by Lemma \ref{lem:1}:
\[
\bigoplus_{i\in L\subset\{1,...,n\}} \Theta_i = \mathbf{0}_\mathcal{F}
\]
for all subsets $L$ of $\{1,...,n\}$ with at least 2 elements: $|L|>1$.\\
Thus, as $L = \big \{ i\neq j, i \in \{1,...,n\} \big \}$ has cardinality $n-1>1$ (as $n>2$) we have that
$\bigoplus_{i\neq j} \Theta_i = \mathbf{0}_\mathcal{F}$ for all $j\in\{1,...,n\}$. Therefore:
\[
\Theta_j \otimes \bigoplus_{i\neq j} \Theta_i = \Theta_j \otimes \mathbf{0}_\mathcal{F} = \Theta_j
\hspace{5mm} \forall j \in \{1,...,n\},
\]
and $\{ \Theta_1,...,\Theta_n \}$ are not $\mathcal{I}_1$.
\end{proof}

Indeed, $\mathcal{IF}$ is incompatible with $\mathcal{I}_2$ as well.

\begin{theorem} \label{the:upper-ifi2}
If $\{ \Theta_1,...,\Theta_n \} \in \mathcal{IF}$, $n>2$ then $\{ \Theta_1,...,\Theta_n \} \in \neg\mathcal{I}_2$.
\end{theorem}
\begin{proof}
If $\{ \Theta_1,...,\Theta_n \} \in \mathcal{IF}$ then $\{ \Theta_1,...,\Theta_{k-1} \} \in \mathcal{IF}$ for all $k = 3,...,n$. But by Lemma \ref{lem:1} this implies $\bigoplus_{i<k} \Theta_i = \mathbf{0}_\mathcal{F}$, so that:
\[
\Theta_k \otimes \bigoplus_{i<k} \Theta_i = \Theta_k \hspace{5mm} \forall k>2.
\]
Now, $\{ \Theta_1,...,\Theta_n \} \in \mathcal{IF}$ with $n>2$ implies $\Theta_k \neq
\Theta$ $\forall k$. The latter holds because, as $n>2$, there is at least one frame $\Theta_i$ in the collection
$\Theta_1,...,\Theta_n$ distinct from $\mathbf{0}_\mathcal{F}$, and clearly $\{ \Theta_i,\Theta\}$
are not $\mathcal{IF}$ (as $\Theta_i$ is a non-trivial coarsening of $\Theta$). Hence:
\[
\Theta_k \otimes \bigoplus_{i<k} \Theta_i \neq \Theta \hspace{5mm} \forall k>2,
\]
which is, in fact, a much stronger condition than $\neg\mathcal{I}_2$.
\end{proof}
A special case is that in which one of the frames is $\Theta$ itself. By Definitions (\ref{eq:i1sm}) and (\ref{eq:i2sm}) of $\mathcal{I}_1$ and $\mathcal{I}_2$, if $\exists j: \Theta_j = \Theta$ then $\{ \Theta_1,...,\Theta_n \} \in \mathcal{I}_2$ $\vdash$ $\{ \Theta_1,...,\Theta_n \} \in \neg\mathcal{I}_1$. From Proposition \ref{pro:szasz2} it follows that
\begin{corollary} \label{cor:upper-ifi3}
If $\{ \Theta_1,...,\Theta_n \} \in \mathcal{IF}$, $n>2$ then $\{ \Theta_1,...,\Theta_n \} \in \neg\mathcal{I}_3$.
\end{corollary}
Theorems \ref{the:n2} and \ref{the:upper-ifi2} and Corollary \ref{cor:upper-ifi3} considered together imply that $\mathcal{IF}$ and $\mathcal{I}_3$ are incompatible in all significant cases.
\begin{corollary} \label{cor:upper-ifi3-final}
If $\{ \Theta_1,...,\Theta_n \}$ are $\mathcal{IF}$ then they are not $\mathcal{I}_3$, unless $n=2$, $\Theta_1 = \mathbf{0}_\mathcal{F}$ and $\Theta_2 = \Theta$.
\end{corollary}

\subsection{Boolean and lattice independence in the lower semimodular lattice $L^*(\Theta)$} \label{sec:i-lower}

\subsubsection{Forms and semantics of extended matroidal independence} \label{sec:semantics}

Analogously, the extended matroidal independence relations associated with the lower semimodular lattice $L^*(\Theta)$
read as:
\begin{equation} \label{eq:i1sm*}
\{ \Theta_1,...,\Theta_n \} \in \mathcal{I}^*_1 \hspace{3mm} \Leftrightarrow \hspace{3mm} \displaystyle \Theta_j \oplus
\bigotimes_{i\neq j} \Theta_i \neq \Theta_j \hspace{5mm} \forall \; j = 1,...,n,
\end{equation}
\begin{equation} \label{eq:i2sm*}
\{ \Theta_1,...,\Theta_n \} \in \mathcal{I}^*_2 \hspace{3mm} \Leftrightarrow \hspace{3mm} \displaystyle \Theta_j \oplus
\bigotimes_{i = 1}^{j - 1} \Theta_i = \mathbf{0}_\mathcal{F} \hspace{5mm} \forall \;j = 2,...,n,
\end{equation}
\begin{equation} \label{eq:i3sm*}
\{ \Theta_1,...,\Theta_n \} \in \mathcal{I}^*_3 \hspace{3mm} \Leftrightarrow \hspace{3mm} \displaystyle \Big| \bigotimes_{i
=1}^n \Theta_i \Big| - 1 = \sum_{i=1}^n (|\Theta_i| - 1),
\end{equation}
as $\Theta_i \wedge^* \Theta_j = \Theta_i \oplus \Theta_j$, $\Theta_i \vee^* \Theta_j = \Theta_i \otimes \Theta_j$, $h^*(\Theta_i) = |\Theta_i| - 1$, and $\mathbf{0} = \mathbf{0}_\mathcal{F}$.

As in the upper semimodular case, these relations have quite interesting semantics. The frames $\{ \Theta_1,...,\Theta_n \}$ are $\mathcal{I}^*_1$ iff none of them is a coarsening of the minimal refinement of all the others. In other words, \emph{there is no proper subset of $\{\Theta_1,...,\Theta_n\}$ which has still $\Theta_1 \otimes ... \otimes \Theta_n$ as common refinement}.\\
They are $\mathcal{I}^*_2$ iff $\forall j>1$ $\Theta_j$ does not have a non-trivial common coarsening with the minimal refinement of its predecessors.\\ Finally, the third form $\mathcal{I}^*_3$ of extended matroidal independence relation can be naturally interpreted in terms of probability spaces. As the dimension of the polytope of probability measures definable on a domain of size $k$ is $k - 1$, $\Theta_1,...,\Theta_n$ are $\mathcal{I}^*_3$ \emph{iff the dimension of the probability polytope for the minimal refinement is the sum of the dimensions of the polytopes associated with the individual frames}:
\begin{equation} \label{eq:dimp}
\{\Theta_1,...,\Theta_n\} \in \mathcal{I}^*_3 \equiv dim \mathcal{P}_{\bigotimes_{i =1}^n \Theta_i} = \sum_i \dim \mathcal{P}_{\Theta_i}.
\end{equation}
From this remark the following analogy between independence of frames and $\mathcal{I}_3$ follows. While the equivalent condition for $\mathcal{IF}$
\begin{equation} \label{eq:crosses}
\Theta_1 \otimes \cdots \otimes \Theta_n = \Theta_1 \times \cdots \times \Theta_n
\end{equation}
states that the minimal refinement is the Cartesian product of the individual frames, Equation (\ref{eq:dimp}) affirms that under $\mathcal{I}^*_3$ {the probability simplex of the minimal refinement is the Cartesian product of the individual probability simplices}. We will consider their relationship in more detail in Section \ref{sec:discussion-frames}.

\subsubsection{General case} \label{sec:general}

\begin{theorem} \label{the:lower-ifi1*}
If $\{\Theta_1,...,\Theta_n\} \in \mathcal{IF}$ and $\Theta_j \neq \mathbf{0}_\mathcal{F}$ $\forall j = 1,...,n$, then $\{\Theta_1,...,\Theta_n\} \in \mathcal{I}^*_1$.
\end{theorem}
\begin{proof}
Let us suppose that $\{\Theta_1,...,\Theta_n\}$ are $\mathcal{IF}$ but not $\mathcal{I}^*_1$, i.e., $\exists j: \Theta_j$ coarsening of $\bigotimes_{i\neq j} \Theta_i$ (and therefore $\Theta_1 \otimes \cdots \otimes \Theta_n = \bigotimes_{i\neq j} \Theta_i$). 

We need to prove that $\exists A_1\subset \Theta_1, ..., A_n\subset\Theta_n$ s.t.:
\[
\rho_1(A_1) \cap \cdots \cap \rho_n(A_n) = \emptyset,
\]
where $\rho_i$ denotes the refining from $\Theta_i$ to $\Theta_1\otimes \cdots \otimes \Theta_n$.\\ Since $\Theta_j$ is a coarsening of $\bigotimes_{i\neq j} \Theta_i$ then there exists a partition $\Pi_j$ of $\bigotimes_{i\neq j} \Theta_i$ associated with $\Theta_j$, and a refining $\rho$ from $\Theta_j$ to $\bigotimes_{i\neq j} \Theta_i$.\\ As $\{\Theta_i,i\neq j\}$ are $\mathcal{IF}$, for all $\theta\in \bigotimes_{i\neq j} \Theta_i$ there exist $\theta_i \in \Theta_i$, $i\neq j$ s.t.
\[
\{ \theta \} = \bigcap_{i\neq j} \rho_i(\theta_i),
\]
where $\rho_i$ is the refining from $\Theta_i$ to $\bigotimes_{i\neq j} \Theta_i$ (remember that $\Theta_1 \otimes \cdots \otimes \Theta_n = \bigotimes_{i\neq j} \Theta_i$). 

Now, $\theta$ belongs to a certain element $A$ of the partition $\Pi_j$. By hypothesis ($\Theta_j \neq \mathbf{0}_\mathcal{F}$ $\forall j$) $\Pi_j$ contains at least two elements. But then we can choose an element $\{ \theta_j \} = \rho^{-1}(B)$ of $\Theta_j$ which is refined to a different element $B$ of the disjoint partition $\Pi_j$. In that case we obviously get:
\[
\rho_j(\theta_j) \cap \bigcap_{i\neq j} \rho_i(\theta_i) = \emptyset,
\]
which implies that $\{\Theta_i,i=1,...,n\} \in \neg \mathcal{IF}$ against the hypothesis.
\end{proof}
Does $\mathcal{IF}$ imply $\mathcal{I}^*_1$ even when $\exists \Theta_i = \mathbf{0}_\mathcal{F}$? The answer is negative. $\{\Theta_1,...,\Theta_n\} \in \neg \mathcal{I}^*_1$ means that $\exists i$ s.t. $\Theta_j$ is a coarsening of $\bigotimes_{i\neq j} \Theta_i$. But if $\Theta_i = \mathbf{0}_\mathcal{F}$ then $\Theta_i$ \emph{is} a coarsening of $\bigotimes_{i\neq j} \Theta_i$.\\
The reverse implication does not hold: $\mathcal{IF}$ and $\mathcal{I}_1$ are distinct.
\begin{theorem}
$\{\Theta_1,...,\Theta_n\} \in \mathcal{I}^*_1$ $\nvdash$ $\{\Theta_1,...,\Theta_n\} \in \mathcal{IF}$.
\end{theorem}
\begin{proof}
We need a simple counterexample. Consider two frames $\Theta_1$ and $\Theta_2$ in which $\Theta_1$ is not a coarsening of $\Theta_2$ ($\Theta_1,\Theta_2$ are $\mathcal{I}^*_1$). Then $\Theta_1,\Theta_2 \neq \Theta_1\otimes \Theta_2$ but it easy to find an example (see Figure \ref{fig:fig1}) in which $\Theta_1,\Theta_2$ are not $\mathcal{IF}$.
\end{proof}
\begin{figure}[ht!]
\begin{center}
\includegraphics[width = 0.55 \textwidth]{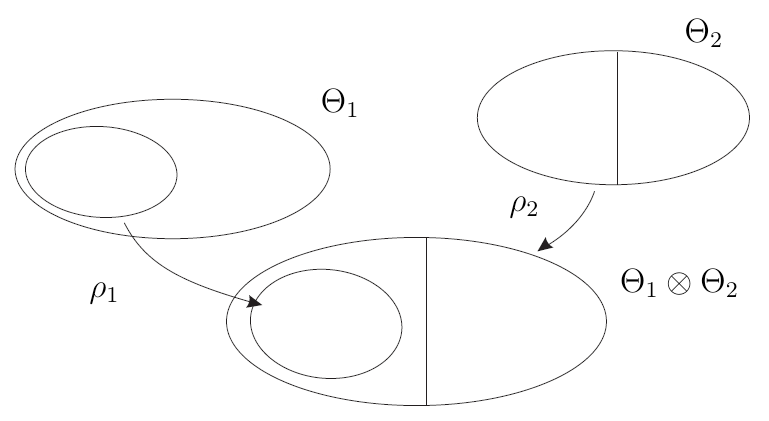}
\end{center}
\caption{A counterexample to $\mathcal{I}^*_1 \vdash \mathcal{IF}$. \label{fig:fig1}}
\end{figure}
Besides, as in the upper semimodular case, $\mathcal{I}^*_2$ does not imply $\mathcal{I}^*_1$.
\begin{theorem}
$\{\Theta_1,...,\Theta_n\} \in \mathcal{I}^*_2$ $\nvdash$ $\{\Theta_1,...,\Theta_n\} \in \mathcal{I}^*_1$.
\end{theorem}
\begin{proof} Figure \ref{fig:fig2} shows a counterexample to the conjecture $\mathcal{I}^*_2
\vdash \mathcal{I}^*_1$. Given $\Theta_1 \otimes \cdots \otimes \Theta_{j-1}$ and $\Theta_j$, one
possible choice of $\Theta_{j+1}$ s.t. $\Theta_1,...,\Theta_{j+1}$ are $\mathcal{I}^*_2$ but not
$\mathcal{I}^*_1$ is shown.
\end{proof}
\begin{figure}[ht!]
\begin{center}
\includegraphics[width = 0.9 \textwidth]{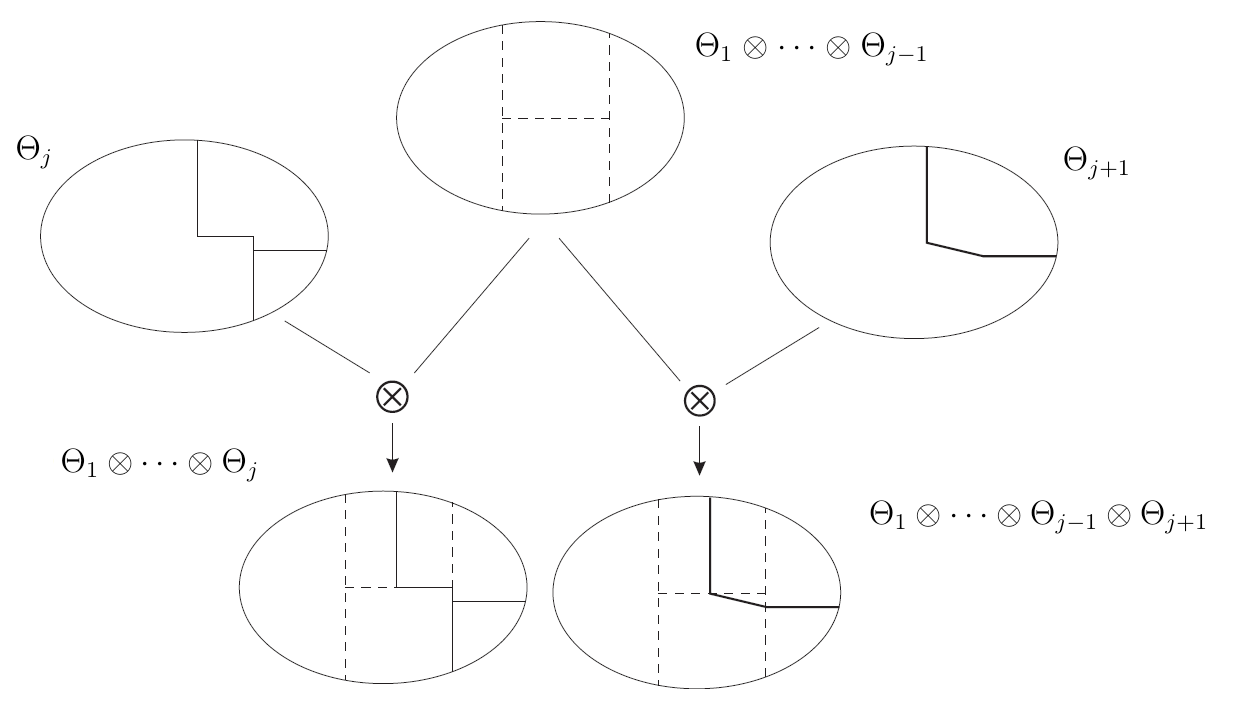}
\end{center}
\caption{A counterexample to $\mathcal{I}^*_2 \vdash \mathcal{I}^*_1$. \label{fig:fig2}}
\end{figure}
$\mathcal{IF}$ is a stronger condition than $\mathcal{I}^*_2$ as well.
\begin{theorem} \label{the:lower-ifi2*}
$\{\Theta_1,...,\Theta_n\} \in \mathcal{IF}$ $\vdash$ $\{\Theta_1,...,\Theta_n\} \in \mathcal{I}^*_2$.
\end{theorem}
\begin{proof}
We first need to show that $\{\Theta_1,...,\Theta_n\}$ are $\mathcal{IF}$ iff $\forall j=1,...,n$ the pair $\{ \Theta_j, \otimes_{i\neq j} \Theta_i \}$ is $\mathcal{IF}$. As a matter of fact (\ref{eq:crosses}) can be written as:
\[
\Theta_j \otimes \bigotimes_{i\neq j} \Theta_i = \Theta_j \times \Big ( \times_{i\neq j} \Theta_i \Big)
\equiv \Big \{\Theta_j,\bigotimes_{i\neq j} \Theta_i \Big \} \in \mathcal{IF}.
\]
But then by Lemma \ref{lem:1} we get as desired.
\end{proof}
It follows from Theorems \ref{the:lower-ifi1*} and \ref{the:lower-ifi2*} that, unless some frame is unitary,
\begin{corollary}
$\{\Theta_1,...,\Theta_n\} \in \mathcal{IF} \vdash \{\Theta_1,...,\Theta_n\} \in \mathcal{I}^*_1 \wedge \mathcal{I}^*_2$.
\end{corollary}
i.e., \emph{independence of frames is a more demanding requirement than both the first two forms of lattice-theoretic independence}.\\ Note that the converse is false. Think of a pair of frames ($n=2$), for which
\[
\begin{array}{ccc}
\Theta_1 \oplus \Theta_2 \neq \Theta_1,\Theta_2 \; (\{ \Theta_1,\Theta_2\} \in \mathcal{I}^*_1), & \hspace{5mm} & \Theta_1 \oplus \Theta_2
= \mathbf{0}_\mathcal{F} \; (\{ \Theta_1,\Theta_2\} \in \mathcal{I}^*_2).
\end{array}
\]
Such conditions are met, for instance, by the counterexample of Figure \ref{fig:fig1} (in which the two frames are not $\mathcal{IF}$).

\begin{theorem} \label{the:if-i3*}
If a collection $\{ \Theta_1,...,\Theta_n \}$ of compatible frames is $\mathcal{IF}$ then it is not $\mathcal{I}_3$, unless $n=2$ and one of the frames is the trivial partition.
\end{theorem}
\begin{proof}
According to Equation (\ref{eq:crosses}), $\{ \Theta_1,...,\Theta_n \}$ are $\mathcal{IF}$ iff $|\otimes \Theta_i | = \prod_i |\Theta_i|$, while according to (\ref{eq:i3sm*}) they are $\mathcal{I}$ iff $|\Theta_1\otimes \cdots
\otimes \Theta_n| - 1 = \sum_i (|\Theta_i|-1)$. Those conditions are both met iff
\[
\sum_i |\Theta_i| - \prod_i |\Theta_i| = n-1
\]
which happens only if $n=2$ and either $\Theta_1 = \mathbf{0}_\mathcal{F}$ or $\Theta_1 = \mathbf{0}_\mathcal{F}$.
\end{proof}
Instead of being algebraically related notions, independence of frames and matroidicity work against each other. As the former derives from independence of Boolean subalgebras of a Boolean algebra \cite{Sikorski}, this is likely to have interesting wider implications on the relationship between independence in those two fields of mathematics.

% summary, comment, diagram
\section{Perspectives} \label{sec:discussion-frames}

\subsection{On abstract independence}

Figure \ref{fig:relations} illustrates what we have learned in this Chapter about the relations between independence of frames and the various extensions of matroidal independence to semimodular lattices, in both the upper (left) and lower (right) semimodular lattice of frames. 
Only the general case of a collection of more than two non-atomic frames is shown for sake of simplicity: special cases ($\Theta_i=\mathbf{0}_\mathcal{F}$ for $L^*(\Theta)$, $\Theta_i = \Theta$ for $L(\Theta)$) are also neglected.

\begin{figure}[ht!]
\centering
\begin{tabular}{cc}
\includegraphics[width = 0.9 \textwidth]{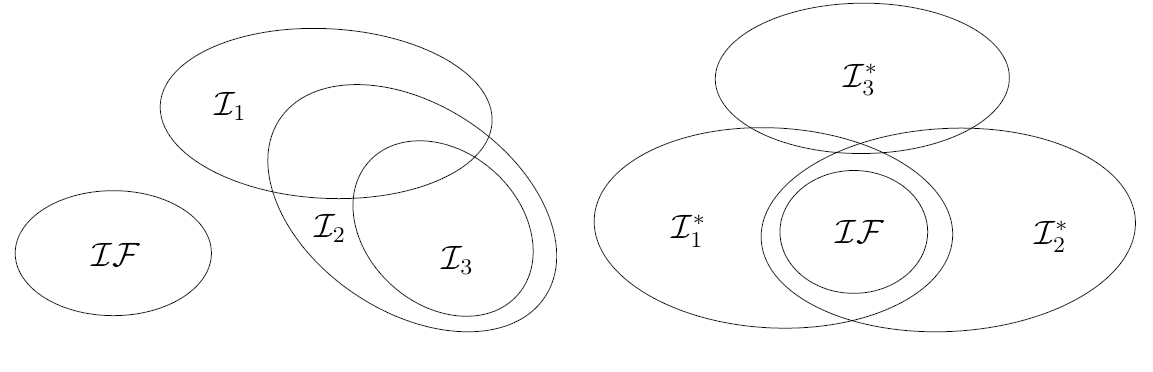}
\end{tabular}
\caption{Left: Relations between independence of frames $\mathcal{IF}$ and the various extended forms of matroidal independence on the upper semimodular lattice $L(\Theta)$. Right: Relations on the lower semimodular lattice $L^*(\Theta)$. \label{fig:relations}}
\end{figure}

% summary for L(\Theta)
In the upper semimodular case, minding the special case in which one of the frames is $\Theta$ itself, independence of frames $\mathcal{IF}$ is mutually exclusive with all lattice-theoretic relations $\mathcal{I}_1,\mathcal{I}_2,\mathcal{I}_3$ (Theorems \ref{the:upper-ifi1}, \ref{the:upper-ifi2} and Corollary \ref{cor:upper-ifi3}) unless we consider two non-atomic frames, for which $\mathcal{IF}$ implies	 $\mathcal{I}_1$ (Theorem \ref{the:n2}). In fact they are \emph{the negation} of each other in the case of atoms of $L(\Theta)$ (frames of size $n-1$), when $\mathcal{I} = \mathcal{I}_1 = \mathcal{I}_2 = \mathcal{I}_3$ is trivially true for all frames, while $\mathcal{IF}$ is never met. The exact relation between $\mathcal{I}_1$ and $\mathcal{I}_2$, $\mathcal{I}_3$ is not yet understood, but we know that the latter implies the former when dealing with pairs.\\
% summary for L*(\Theta)
In the lower semimodular case $\mathcal{IF}$ is \emph{a stronger condition} than both $\mathcal{I}^*_1$ and $\mathcal{I}^*_2$ (Theorems \ref{the:lower-ifi1*}, \ref{the:lower-ifi2*}). On the other side, notwithstanding the analogy expressed by Equation (\ref{eq:dimp}), $\mathcal{IF}$ is mutually exclusive with the third independence relation even in its lower semimodular incarnation.

% common features
Some common features do emerge: the first two forms of lattice independence are always trivially met by atoms of the related lattice. Moreover, independence of frames and the third form of lattice independence are mutually exclusive in both cases.

% lower sem case more insteresting
The lower semimodular case is clearly the most interesting. Indeed, on $L(\Theta)$ independence of frames and lattice-theoretic independence are basically unrelated (see Figure \ref{fig:relations}-left). Their lower semimodular counterparts, instead, although distinct from  $\mathcal{IF}$, have meaningful links with it. The knowledge of which collections of frames are $\mathcal{I}^*_1$, $\mathcal{I}^*_2$ and $\mathcal{I}^*_3$ tells us much about collections of $\mathcal{IF}$ frames, as the latter are necessarily in:
\[
\mathcal{I}^*_1 \cap \mathcal{I}^*_2 \cap \neg \mathcal{I}^*_3.
\]
We know that $\mathcal{IF}$ is \emph{strictly} included in $\mathcal{I}^*_1 \cap \mathcal{I}^*_2$ (Section \ref{sec:general}), but the possibility that independence of frames may indeed coincide with $\mathcal{I}^*_1 \cap \mathcal{I}^*_2 \cap \neg \mathcal{I}^*_3$  still needs to be explored.

\subsection{On the conflict problem: towards a pseudo Gram-Schmidt procedure?} \label{sec:gram-schmidt}

The problem of conflicting belief functions, so important for sensor fusion applications, is inherently related to the notion of independence of frames (Theorem \ref{the:7}). This in turn display similarities with the notion of independence of vector subspaces, which have suggested a possible algebraic solution to the conflict problem. These similarities can be recapped as in the following Table.
\[
\begin{array}{ccc}
\sum_i v_i \neq \vec{0} & \Longleftrightarrow &  v_1+ ... + v_n \neq \vec{0}, \;\forall \vec{0} \neq v_i\in V_i \\ & & \Big \Updownarrow\\
\bigcap_i V_i = \vec{0} & \Longleftrightarrow & span\{V_1,...,V_n\} = \mathbb{R}^{d_1}\times\cdots\times \mathbb{R}^{d_n}  \\
\\ \bigoplus_i \Theta_i = \mathbf{0}_{\mathcal{F}} & \Longleftrightarrow & \Theta_1\otimes \cdots \otimes \Theta_n =
\Theta_1\times \cdots \times \Theta_n\\ & & \Big \Updownarrow\\ \bigcap A_i \neq \bigwedge & \Longleftrightarrow &
\rho_1(A_1) \cap \cdots \cap \rho_n(A_n)\neq \emptyset, \; \forall \emptyset \neq A_i \subseteq \Theta_i,
\end{array}
\]
where $\bigcap A_i \neq \bigwedge$ (with $\bigwedge$ the initial element of a Boolean algebra) is the independence condition for Boolean sub-algebras \cite{Sikorski} (Equation \ref{eq:independence-boolean}).

In Chapter \ref{cha:alg} we have seen that families of frames form upper semimodular, lower semimodular, and Birkhoff lattices, but not modular lattices (unlike projective geometries). Here the analogy breaks down for, while the atoms of a Birkhoff lattice do form a matroid (therefore admitting the notion of independence), this matroid cannot be trivially extended to arbitrary elements of the lattice. As a consequence a true independence relation cannot be defined for frames of a family, although various extensions (as we have seen in detail) can be defined and do display relationships with Boolean independence of frames.

% gram-schmidt
Given a collection of arbitrary elements of a vector space, the well known {Gram-Schmidt algorithm} is able to generate another collection of independent vectors spanning the same subspace. The main ingredients of the algorithm are the notion of linear independence of vectors and a mechanisms for projecting vectors onto the linear subspace generated by other vectors.\\ We can then imagine a `pseudo Gram-Schmidt' procedure resting on the algebraic structure of Birkhoff lattice of commutative monoids (shared by both compatible frames and linear subspaces), and the associated independence relation. This algorithm, starting from a set of belief functions $b_i : 2^{\Theta_i} \rightarrow [0,1]$ defined over a finite collection of FODs $\Theta_1,\cdots,\Theta_n$, would create a new collection of \emph{independent} frames of the same family:
\[
\begin{array}{ccc}
\Theta_1,...,\Theta_n\in{\mathcal{F}} & \longrightarrow & \Theta'_1,...,\Theta'_m\in{\mathcal{F}},
\end{array}
\]
with $m \neq n$ in general, and the same minimal refinement:
\[
\Theta_1\otimes \cdots \otimes \Theta_n = \Theta'_1\otimes \cdots \otimes \Theta'_m.
\]
Once projected the $n$ original b.f.s $b_1,...,b_n$ onto the new set of frames (which could be done as the new frames would belong to the same family) we would achieve a set of \emph{surely combinable} belief functions $b'_1,...,b'_m$, \emph{equivalent}, in some sense, to the previous one. 

The search for a formal definition of the equivalence of possibly non-combinable collections of belief functions is the most intriguing element of this proposal: it is reasonable to conjecture that Dempster's combination will have to be involved.

\section{Conclusive comments}

In the last two Chapters we have given a rather exhaustive description of families of compatible frames in terms of the algebraic structures they form: Boolean sub-algebras (as in their original definition), monoids, upper and lower semimodular lattices. Many of those structures come with a characteristic form of `independence', not necessarily derived from classical matroidal independence.\\
We compared them with Shafer's notion of independence of frames, with the final goal of pursuing an algebraic interpretation of independence in the theory of evidence. Although $\mathcal{IF}$ cannot be straighforwardly explained in terms of classical matroidal independence, it does possess interesting relations with the latter's extended forms on semimodular lattices. 

Independence of frames is actually opposed to matroidal independence (Theorem \ref{the:if-i3*}). Although this can be seen as a negative result in the perspective of finding an algebraic solution to the problem of merging conflicting belief functions on non-independent frames (Chapter \ref{cha:alg}, Section \ref{sec:gram-schmidt}), we now understand much better where independence of frames stands from an algebraic point of view. New lines of research have opened as a result, e.g. concerning a possible explanation of independence of frames as independence of flats in a geometric lattice \cite{cuzzolin08isaim-matroid}. We believe that the prosecution of this study may in the future shed more light on both the nature of independence of sources in the theory of subjective probability, and the relationship between matroidal and Boolean independence in discrete mathematics, pointing out the necessity for a more general, comprehensive definition of this very important notion.

As a last remark, the implications of our algebraic description of families of frames go beyond a potential algebraic solution to the problem of conflicting evidence. Many concepts of the theory of evidence are inherently connected to the structure of the underlying domains. For example, the notion of \emph{support function} (Definition \ref{def:support}) rests on that of refining, and may quite possibly be reformulated using the algebraic language developed here. Its analysis in the context of the lattice structure of $\mathcal{F}$ (Corollary \ref{cor:lattice-of-frames}) could eventually lead to a novel solution to the {canonical decomposition} problem (Section \ref{sec:canonical-decomposition}), alternative to Smets' and Kramosil's (see Chapter \ref{cha:geo}).

\part{Visions}

\chapter{Data association and the total belief theorem} \label{cha:total}

\begin{center}
\includegraphics[width = 0.45 \textwidth]{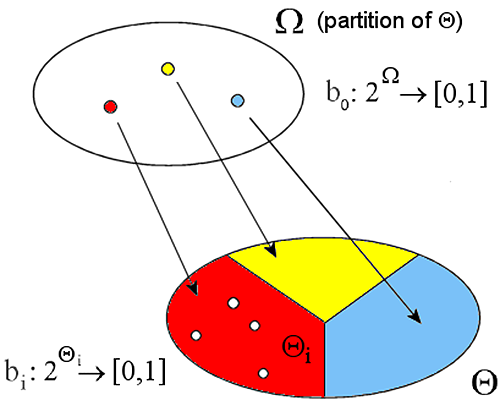}
\end{center}
\vspace{5mm}

\emph{Data association} \cite{DaveyDefence,Karlsson01,smets2004kalman,ristic06if,ayoun01data} is one of the more intensively studied computer vision applications for its important role in the implementation of automated defense systems, and its connections to the field of {structure from motion}, i.e., the reconstruction of a rigid scene from a sequence of images.\\ In data association a number of feature points moving in the 3D space are tracked by one or more cameras, appearing in an image sequence as unlabeled dots (i.e., the correspondences between points in two consecutive frames are not known) \cite{martinerie92dataassociation}. A typical example is provided by a set of markers set at fixed positions on a moving articulated body - in order to reconstruct the trajectory of each `target' of the cloud (each marker placed on the underlying body) we need to associate feature points belonging to pairs of consecutive images,  $I(k)$ and $I(k+1)$.

A popular approach (the \emph{joint probabilistic data association} or JPDA filter, \cite{HagerCVPR98,HagerPAMI2001,Shalom88}) rests on designing a number of Kalman filters (each associated with a feature point), whose aim is to predict the future position of the target, in order to generate the most likely labeling for the cloud of points in the next image. Unfortunately, this method suffers from a number of drawbacks: for instance, when
several feature points converge to the same, small region (\emph{coalescence}, \cite{Bloem95}) the algorithm cannot tell them apart anymore. A number of techniques have been proposed to overcome this sort of problems, in particular the \emph{condensation} algorithm \cite{Blake96} due to Michael Isard and Andrew Blake.

\subsection*{Scope of the Chapter}

In the first part of this Chapter we prospect an evidential solution to the \emph{model-based} data association task, in which feature points are the images of fixed locations on an articulated body whose topological model is \emph{known}. The bottom line of the approach is to express the prior, logical information carried by the body model in term of belief functions on a suitable frame of discernment. This piece of evidence can then be combined with the probabilistic information carried by the usual set of Kalman filters, each associated with a target as in the classical JPDA approach. A tracking process can then be set up, in which the current belief estimate of all model-to-features associations is continually refined by means of the new, incoming evidence.

As we shall see in the following, a rigid motion constraint can be derived from each link in the topological model of the moving body. This constraint, however, can be expressed in a \emph{conditional} way only -- in order to test the rigidity of the motion of two observed feature points at time $k$ we need to know the correct association between points of the model and feature points at time $k - 1$. Hence, the task of combining \emph{conditional} belief functions arises. 

Unfortunately, as we have seen in Chapter \ref{cha:toe}, the theory is not currently equipped with any result analogous to the total probability theorem of classical probability theory. In the second part of the Chapter we will therefore provide a formal statement of the problem, which consists in combining (conditional) belief functions defined over disjoint subsets of a frame of discernment, while simultaneously constraining the resulting total belief function to meet a prior condition represented as a belief function over a coarsening of the original frame. As there are several admissible solutions to this problem, that of minimum size is sought. 

This `total belief' setting is shown in this Chapter to be equivalent to building a square linear system with positive solution, whose columns are associated with the focal elements of the candidate total belief function. We introduce a class of linear transformations of the columns of candidate solution systems, and show that such candidate solutions form the nodes of a graph whose edges are transformations of the above class. We claim that there is always at least one path leading to a valid total belief function whatever starting point in the graph we pick.

We do not provide a full solution to the total belief problem here, but we focus on the restricted case in which the a-priori belief function only has disjoint focal elements.

\section{The data association problem} \label{sec:data-association}

Let us first formally define the {data association} problem.\\ Given a sequence of images $\{I(k),k\}$, each containing a number of feature points $\{z_i(k)\}$ which are projections of 3D locations in the real world, we want to find the correspondences $z_i(k)\longleftrightarrow z_j(k+1)$ between feature points of two consecutive images that correspond to the same material point. The task is complicated by the fact that sometimes one or more material points are \emph{occluded}, i.e., they do not appear in the current image. In other cases \emph{false features} may be generated by defects of the vision system. Overall, the number of feature points at each time instant is in general variable, and not all visible feature points are images of actual material points.

\subsection{Joint probabilistic data association} \label{sec:jpda}

In the \emph{joint probabilistic data association} framework (see \cite{Shalom88} for a more detailed illustration) each feature point is tracked by an individual Kalman filter \cite{Jung97,smets2004kalman}, whose purpose is to generate a prediction of the latter's future position. These predictions are then used to estimate the most probable feature labeling in the following image.

Let us assume that each feature measurement $z(k+1)$ at time $k+1$, conditioned by the past observations $Z^k={[Z(j)]}_{j=1..k}$ (where $Z(j)={\{z_i(j)\}}_{i=1..m_j}$ is the set of measurements at time $j$), has a normal distribution with mean $\hat{z}(k|k+1)$ and variance $S(k+1)$.\\ For each target $\gamma$ we call \emph{validation region} 
\[
\tilde{V}_{k+1}(\gamma) = \big \{ z: V'(k+1)S^{-1}(k+1)V(k+1) \leq \gamma \big \} 
\]
the zone outside which it is improbabile to find measurements associated with target $\gamma$. Given our Gaussian assumption, $\tilde{V}_{k+1}(\gamma)$ is an ellipse centered in the mean prediction $\hat{z}(k|k+1)$. Assuming that a classical linear model for the measurements is adopted,
\begin{equation}\label{eq:a}
\left\{
\begin{array}{l}
x(k+1) = F(k)x(k) + v(k) \\ z(k) = H(k)x(k) + w(k),
\end{array}
\right.
\end{equation}
we can distinguish a number of different filtering approaches:
\begin{itemize}
\item \emph{nearest neighbor}: the measurement which is the closest to the prediction in the validation region is used to update the filter;
\item \emph{splitting}: multiple hypotheses are formulated by considering all the measurements in the validation region;
\item \emph{PDA filter}: the probability of it being the correct association is computed for each measurement at the current time instant $k$;
\item \emph{optimal Bayesian filter}: the same probability is calculated for \emph{entire series} of measurements, rather than at the current time only.
\end{itemize}
The assumption behind the PDA filter is that the state $x(k)$ of system (\ref{eq:a}) is also normally distributed around
the current prediction $\hat{x}(k|k-1)$ with variance $P(k|k-1)$. 

To get an estimation equation we need to compute $\beta_i(k) = P[\theta_i(k)|Z^k]$ for $i=0 ... m_k$, where $\theta_i,\theta_0$ represent the following hypotheses: $\theta_i(k)=\{z_i(k)$ true$\}$, $\theta_0(k)=\{$all the measurements are false$\}$. Then the update equation for the state becomes:
\[
\hat{x}(k|k) = E[x(k)|Z^k] = \sum_i \hat{x}_i(k|k)\beta_i(k).
\]
The equations for state and output predictions, instead, come from the standard Kalman filter formulation \cite{Moore95}:
\[
\left\{
\begin{array}{l}
\hat{x}(k+1|k)=F(k)\hat{x}(k|k)+G(k)u(k)\\ \hat{z}(k+1|k)=H(k+1)\hat{x}(k+1|k).
\end{array}
\right.
\]
A simple variant of the PDA filter is the \emph{joint probabilistic data association} (JPDA) filter, which focuses on the \emph{joint association} event:
\[
\bar{\theta} = \bigcap_{j=1..m_k}\theta_{jt_j},\hspace{5mm} j=1,...,m_k,\; t=0,...,T,
\]
where $\theta_{jt_j}$ is the event `measurement $j$ is associated with target $t$'. A \emph{validation matrix} is then
defined as $\Omega=[\omega_{jt}]$, where $\omega_{jt}=1$ when $z_j$ is found within the ellipse associated to
target $t$. An \emph{admissible event} is a matrix of the same kind $\hat{\Omega}=[\hat{\omega}_{jt}]$, subject to
the following constraints:
\[
\sum_t \hat{\omega}_{jt}(\theta)=1,\;\;\;\delta_t(\theta)=\sum_j\hat{\omega_{jt}}(\theta)\leq 1.
\]
The sum $\delta_t(\theta)$ of column $t$'s components of $\hat{\Omega}$ is called target $t$'s \emph{detection indicator}. The filter's equations are obtained as in the single target's case.

\subsection{Model-based data association} \label{sec:model-based-data-association}

As anticipated in the Introduction, JPDA suffers from a number of drawbacks. When several features converge to the same region of space (a phenomenon called \emph{coalescence} \cite{Bloem95}), for instance, the algorithm cannot tell them apart anymore. The `{condensation}' algorithm is one technique that has been proposed to address this issue. Here we are bring forward a radically different solution to a slightly different problem: \emph{model-based} data association. 

Namely, we assume that the targets represent fixed positions on an articulated body, and that we know which pairs of markers are connected by a rigid link. Clearly, this information is equivalent to the knowledge of a \emph{topological model} of the articulated body, in the form of an undirected graph whose edges represent the rigid motion constraints coupling pairs of targets (see Figure \ref{fig:topological-model}). We can then exploit this \emph{a-priori} information to solve the association task in those critical situations in which several target points fall within the validation region of a single Kalman filter. This knowledge can be expressed as a set of {logical constraints} on the admissible relative positions of the markers, and consequently on those of the feature points.

\begin{figure}[ht!]
\begin{center}
\begin{tabular}{c}
\includegraphics[width = 0.8 \textwidth]{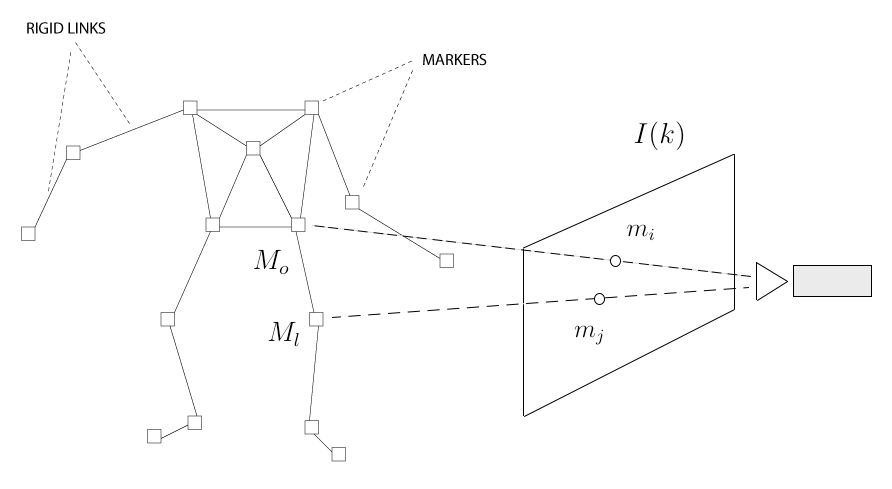}
\end{tabular}
\caption{Topological model of a human body: adjacency relations between pairs of markers are shown are undirected edges. \label{fig:topological-model}}
\end{center}
\end{figure}

A suitable environment in which to combine logical and probabilistic pieces of evidence, exploiting the information carried by the model, is naturally provided by evidential reasoning.\\ What exactly can be inferred from a topological model of the body? We can identify, for instance:
\begin{itemize}
\item a \emph{prediction} constraint, encoding the likelihood of a measurement in the current image being associated with
a measurement of the past image;
\item an \emph{occlusion} constraint, expressing the chance of a given marker $M_l$ of the model being occluded in the
current image;
\item a \emph{metric} constraint, representing what we know about the lengths of the various rigid links, lengths that can be learned from the history of past associations;
\item a \emph{topological} or \emph{rigid motion} constraint acting on pairs of markers linked by an edge in the topological model (e.g. $M_o$ and $M_l$ in Figure \ref{fig:topological-model}).
\end{itemize}
Now, all these constraints can be expressed as belief functions over a suitable frame of discernment. For instance, the metric one can be implemented by checking which pairs of feature points $(m_i,m_j),\;i,j=1,...,n(k)$ (where $n(k)$ is the number of observed features at time $k$) are at the same distance as any given pair $(M_o,M_l),\;o,l=1,...,N$ of model points, within a certain tolerance. 

For each rigid link $M_o - M_l$ the metric constraint can then be represented as a belief function $b$ on the frame of discernment:
\[
\Theta_{M_{o l}}^k = \big \{ (m_i,m_j),\;i,j=1,...,n(k) \big \}
\]
of all the pairs of observed feature points at time $k$, with basic probability assignment:
\begin{equation} \label{eq:metric}
m_b(A) = \left\{\begin{array}{ll} 1-p & A = \big \{ (m_i,m_j):\|m_i-m_j\|\simeq \|M_o - M_l \| \big \}\\ \\ p & A = \Theta_{M_{o l}}^k.
\end{array}
\right.
\end{equation}
Here $p$ is the probability of occlusion of at least one of $M_o$ and $M_l$, estimated by means of a statistic analysis of past model-to-feature associations. 

The likelihood values $\beta_i(k)$ generated by a battery of classical Kalman filters can be encoded in the same way: quantitative and logical pieces of information may all be combined in the framework of belief calculus.

\subsection{The family of the past-present associations frames}

\begin{figure}[ht!]
\begin{center}
\begin{tabular}{c}
\includegraphics[width = 0.85 \textwidth]{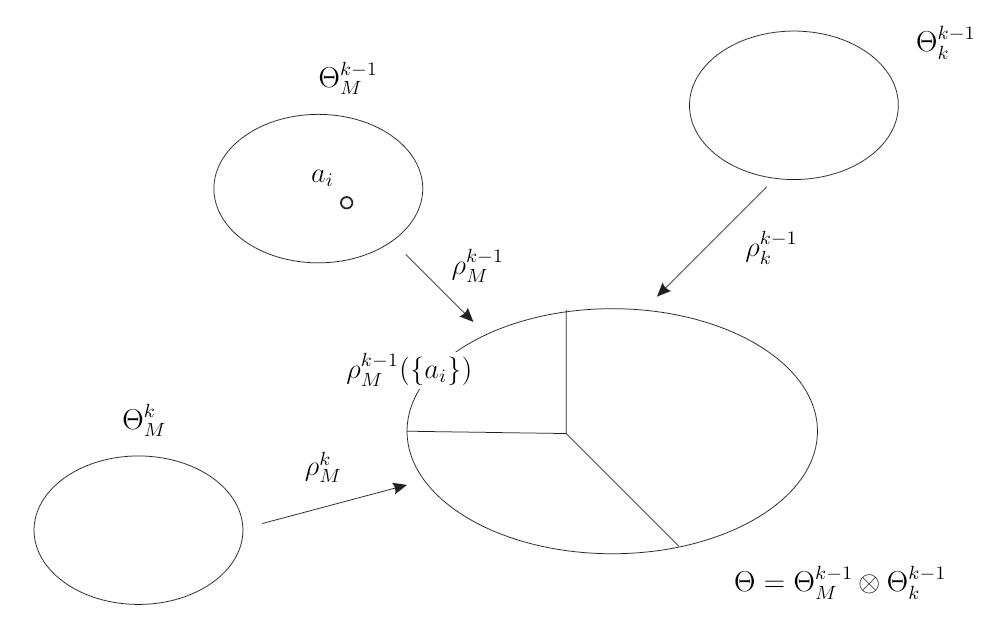}
\end{tabular}
\caption{\label{fig:totalbda} The family of past and present association frames. All the constraints of the model-based association problem are combined over the common refinement $\Theta$ and then re-projected onto the current association frame to yield a belief estimate of the current feature-to-model association.}
\end{center}
\end{figure}

By observing the nature of the constraints introduced above, we can note that the information carried by predictions of filters and occlusions inherently concerns \emph{associations between feature points belonging to consecutive images}, rather than points of the model. In fact, they are independent from any assumptions on the model of the underlying articulated body. Some constraints can be expressed \emph{instantaneously} in the frame of the \emph{current} feature-to-model associations: the metric constraint is a natural example. Some others, however, depend on the {model-to-measurement association} estimated at the previous step. This is the case for belief functions encoding information on the motion of the articulated body, which are expression of topological and rigid motion constraints. 

These three classes of belief functions are defined over distinct frames, elements of a family of compatible frames, representing \emph{past model-to-feature} associations
\[
\Theta_{M}^{k-1}\doteq \Big \{ m_i(k-1) \leftrightarrow M_j,\;\forall i=1,...,n(k-1)\;\forall j=1,...,M \Big \},
\]
\emph{feature-to-feature} associations
\[
\Theta_{k}^{k-1}\doteq \Big \{m_i(k-1) \leftrightarrow m_j(k),\;\forall i=1,...,n(k-1)\;\forall j=1,...,n(k) \Big \},
\]
and \emph{current model-to-feature} associations
\[
\Theta_{M}^{k}\doteq \Big \{ m_i(k) \leftrightarrow M_j,\;\forall i=1,...,n(k)\;\forall j=1,...,M \Big \},
\]
respectively. Note that the individual metric constraint (\ref{eq:metric}) is defined on a frame ($\Theta_{M_{o l}}^k$) which is a
coarsening of $\Theta_{M}^{k}$. 

As all these frames form a compatible family, all the available evidence can then be combined on their \emph{minimal refinement} (see Theorem \ref{the:minimal}), the \emph{combined association} frame $\Theta_{M}^{k-1} \otimes \Theta_{k}^{k-1}$ (Figure \ref{fig:totalbda}). Projecting the resulting belief function back onto the current association frame $\Theta_{M}^{k}$ produces the current best estimate.

A serious complication comes from the fact that, as we said, constraints of the third type can be expressed in a \emph{conditional} way only (i.e., given an estimate of the feature-to-model association at time $k-1$). Consequently, the computation of a belief estimate of the current feature-to-model association requires \emph{combining a set of conditional belief functions} \cite{Slobodova97,slobodova94conditional,kohlas88b}, induced by the conditional constraint on the combined association frame $\Theta = \Theta_{M}^{k-1} \otimes \Theta_{k}^{k-1}$. 

More precisely, the rigid motion constraint generates an entire set of belief functions $b_i : 2^{\rho_M^{k-1}(\{ a_i\})} \rightarrow [0,1]$, each defined over an element $\rho_M^{k-1}(\{a_i\})$ of the disjoint partition of $\Theta = \Theta_{M}^{k-1} \otimes \Theta_{k}^{k-1}$ induced on the combined frame by its coarsening $\Theta_{M}^{k-1}$ (see Figure \ref{fig:totalbda} again). Here $a_i \in \Theta_{M}^{k-1}$ is the $i$-th possible association at time $k-1$. 

In order for us to obtain a belief estimate, these conditional belief functions must be reduced to a single \emph{total belief function}, that is eventually pooled with those generated by all the other constraints.

\section{The total belief theorem} \label{sec:total-belief}

Let us now abstract from the data association problem and state the conditions an overall, total belief function $b$ must obey, given a set of conditional functions $b_i : 2^{\Pi_i} \rightarrow [0,1]$ over $N$ of the elements $\Pi_i$ of the partition $\Pi=\{\Pi_1,...,\Pi_{|\Omega|}\}$ of a frame $\Theta$ induced by a coarsening $\Omega$.\vspace{3mm}

\begin{enumerate}
\item \emph{A-priori constraint}: the restriction (\ref{eq:restriction}) on the coarsening $\Omega$ of the frame $\Theta$ of the candidate total belief function $b$ must coincide with a given \emph{a-priori} b.f. $b_0:2^\Omega \rightarrow [0,1]$.
\end{enumerate}\vspace{3mm}

In the data association problem, in particular, the \emph{a-priori} constraint is represented by the belief function encoding the estimate of the past feature-to-model association $M\leftrightarrow m(k-1)$, defined over $\Theta_{k}^{k-1}$ (Figure \ref{fig:totalbda}). It ensures that the total function is compatible with the last available estimate.\vspace{3mm}

\begin{enumerate} \addtocounter{enumi}{1}
\item \emph{Conditional constraint}: the belief function $b \oplus b_{\Pi_i}$ obtained by conditioning the total belief function $b$ with respect to each element $\Pi_i$ of the partition $\Pi$ must coincide with the corresponding given conditional belief function $b_i$:
\[
b \oplus b_{\Pi_i} = b_i \hspace{5mm} \forall i=1,...,N
\]
\end{enumerate}\vspace{3mm}
\noindent where $m_{\Pi_i} : 2^\Theta \rightarrow [0,1]$ is such that:
\begin{equation} \label{eq:pi-i}
m_{\Pi_i}(A)=\left\{\begin{array}{ll} 1 & A=\Pi_i\\ \\ 0 & A \subseteq \Theta, A\neq \Pi_i.
\end{array} \right.
\end{equation}

\subsection{Formulation} 

We can then formulate the generalization of the total probability theorem to the theory of belief functions -- the \emph{total belief theorem} -- as follows (Figure \ref{fig2}).

\begin{theorem} \label{the:total-belief}
\cite{Zhou2017uai} Suppose $\Theta$ and $\Omega$ are two frames of discernment, and $\rho:2^{\Omega}\rightarrow2^{\Theta}$ the unique refining between them. Let $b_0$ be a belief function defined over $\Omega = \{ \omega_1,...,\omega_{|\Omega|} \}$. Suppose there exists a collection of belief functions ${b_i} : 2^{\Pi_i} \rightarrow [0,1]$, where $\Pi = \{ \Pi_1,...,\Pi_{|\Omega|} \}$, $\Pi_i = \rho(\{\omega_i\})$, is the partition of $\Theta$ induced by its coarsening $\Omega$. 

Then, there exists a belief function $b : 2^{\Theta} \rightarrow [0,1]$ such that:
\begin{enumerate}
\item 
$b_0$ is the restriction of $b$ to $\Omega$, $b_0 = b|_{\Omega}$ (Equation (\ref{eq:restriction}), Chapter \ref{cha:toe});\\
\item 
$b \oplus b_{\Pi_i} = b_i$ $\forall i=1,...,|\Omega|$, where $ b_{\Pi_i}$ is the categorical belief function with b.p.a. (\ref{eq:pi-i});
\end{enumerate}
\end{theorem}

\begin{figure}[ht!]
\begin{center}
\includegraphics[width = 0.65 \textwidth]{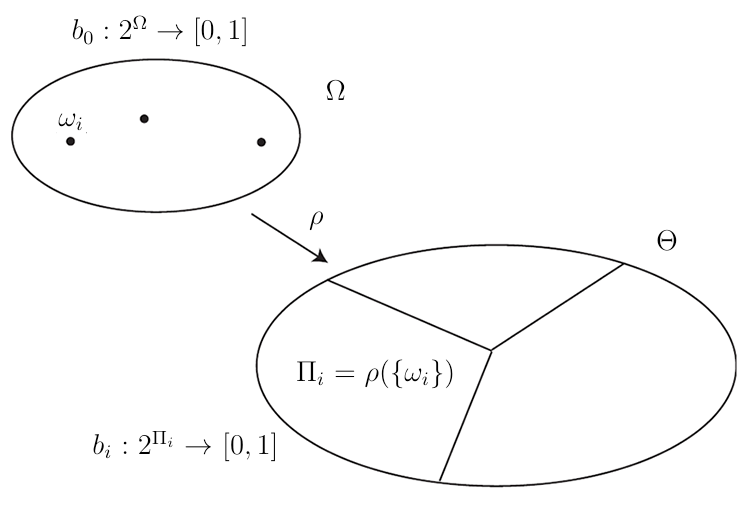} 
\caption{Pictorial representation of the total belief theorem hypotheses.\label{fig2}}
\end{center}
\end{figure}

\subsection{Effect of the a-priori constraint}

The \emph{a-priori} constraint induces an interesting condition on the focal elements of the total function $b$.

\begin{lemma} \label{lem:a-priori}
Let $\rho$ be the refining between $\Omega$ and $\Theta$, and denote by $e_{(.)}$ an arbitrary focal element of a valid total belief function $b$. Then, the inner reduction (\ref{eq:inner-reduction}) $\bar{\rho}(e_{(.)})$ of $e_{(.)}$ is a focal element of the \emph{a-priori} belief function $b_0$.\\ In other words, there exists a focal element $E_k \in \mathcal{E}_{b_0}$ of $b_0$ such that $e_{(.)}$ is a subset of $\rho(E_k)$ and all the projections $\rho(\omega)$, $\omega \in E_k$ of singleton elements of $E_k$ have non-empty intersections with $e_{(.)}$:
\[
\begin{array}{l}
\forall e_{(.)}\in{\mathcal{E}}_b \hspace{3mm} \exists E_k\in\mathcal{E}_{b_0} \hspace{3mm} s.t. \hspace{3mm} e_{(.)} \subset \rho({{\mathcal{E}}_k}) \hspace{3mm} \bigwedge \hspace{3mm} e_{(.)} \cap\rho(\omega) \neq \emptyset\;\forall \omega \in E_k.
\end{array}
\]
\end{lemma}
The proof is rather straightforward, and can be found in \cite{Shafer76}. 

\subsection{Effect of conditional constraints and structure of the total focal elements}

Conditional constraints (2), on the other hand, provide the structure all the focal elements of the candidate total belief function $b$ must adhere to.\\ Let us denote by $e_{(.)}^k$ any focal element of $b$ which is a subset of $\rho({E}_k)$, where $E_k$ is again an arbitrary focal element of the a-priori b.f. $b_0$.

\begin{lemma} \label{lem:conditional}
Each focal element $e_{(.)}^k$ of a total belief function $b$ is the union of \emph{exactly one} focal element of each of the conditional belief functions whose domain $\Pi_i$ is a subset of $\rho(E_k)$, where $E_k$ is the smallest focal element of the a-priori belief function $b_0$ s.t. $e_{(.)}^k \subset \rho(E_k)$. Namely:
\begin{equation} \label{eq:structure}
\displaystyle e_{(.)}^k = \bigcup_{i : \Pi_i \subset \rho(E_k)} e_i^{j_i}
\end{equation}
where $e_i^{j_i}\in{\mathcal{E}}_{b_i}$ $\forall i.$
\end{lemma}
\begin{proof}
Since $b \oplus b_{\Pi_i} = b_i$, where $m_{\Pi_i}(\Pi_i) = 1$, by Dempster's rule it necessarily follows that:
\[
e^k_{(.)}\cap \Pi_i = e_i^{j_i}
\]
for some focal element $e_i^{j_i}$ of $b_i$. Furthermore, $e^k_{(.)} \cap \Pi_i$ must be non-empty for all $i$, for if there existed an integer $l$ such that $e^k_{(.)}\cap \Pi_l=\emptyset$ for some $\Pi_l$ we would have:
\[
\bar{\rho}(e^k_{(.)}) \subsetneq E_k,
\]
contradicting the assumption that $E_k$ is the smallest focal elements of $b_0$ whose image contains $e^k_{(.)}$.
\end{proof}

Note that (\ref{eq:structure}) is a union of \emph{disjoint} elements.\\ Lemma \ref{lem:conditional} is very important, for it describes the general structure of focal elements of a total belief function $b$. As represented in Figure \ref{fig:total-focal-elements}, each f.e. of $b$ can be seen as an `elastic band' covering a single focal element for each conditional b.f.s $b_i$.

\begin{figure}[ht!]
\begin{center}
\includegraphics[width = 0.7 \textwidth]{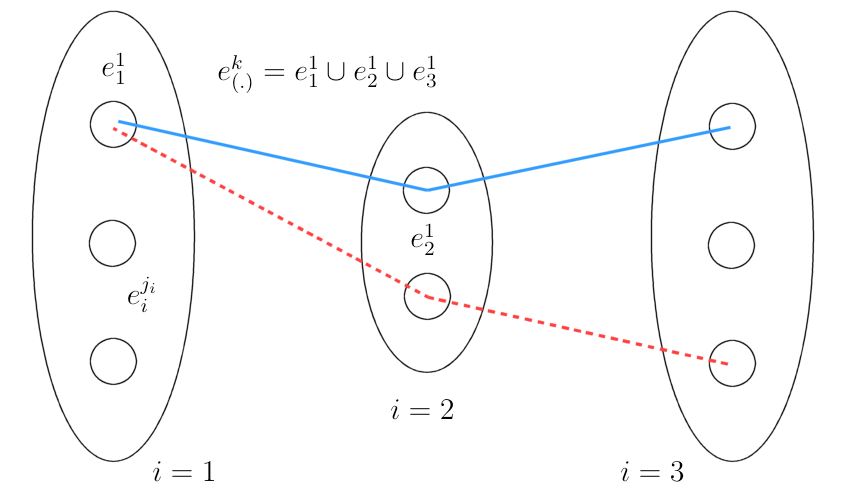} 
\caption{Pictorial representation of the structure of the focal elements of a total belief function $b$ lying in the image $\rho(E_k) = \Pi_1 \cup \Pi_2 \cup \Pi_3$ of a focal element of $b_0$ of cardinality 3. For each $i=1,2,3$ an admissible focal element $e^k_{(.)}$ of $b$ must be such that $e^k_{(.)} \cap \Pi_i = e_i^{j_i}$ for some $j_i$. Set-theoretical relationships between focal elements of the individual conditional b.f.s $b_i$ are irrelevant, and are not represented in this diagram. \label{fig:total-focal-elements}}
\end{center}
\end{figure}

It also determines a constraint on the minimum number of focal elements lying within the image of each focal element of the prior $b_0$ a total belief function $b$ must possess.

\begin{lemma} \label{lem:number}
Let $\rho : 2^\Omega \rightarrow 2^\Theta$ and let $E_k$ be a focal element of $b_0$.\\
The minimum number of focal elements $e^k_{(.)}$ of the total belief function $b$ which are subsets of $\rho({{E}}_k)$ is:
\[
n = \sum_{i=1,...,|{E}_k|}(n_i-1),
\]
where $n_i = |\mathcal{E}_{b_i}|$ is the number of focal elements of the $i$-th conditional belief function.
\end{lemma}
\begin{proof}
Let us call $e_i^j$ the $j$-th focal element of $b_i$. Since $b \oplus \Pi_i = b_i$ for all $i$, we have that, by definition of Dempster's rule:
\begin{equation} \label{eq:lemma-number}
m_{b_i} (e_i^j) = \frac{\sum_{e^k_{(.)} \cap \Pi_i = e_i^j} m_b (e^k_{(.)}) \cdot 1 }{Z},
\end{equation}
where $Z$ is just a normalization factor. The minumum number of focal elements of $b$ inside $\rho(E_k)$ is then equal to the number of contraints of the form (\ref{eq:lemma-number}) imposed by all the conditional b.f.s with domain within $\rho(E_k)$. For each $i$ Dempster's sum (\ref{eq:lemma-number})  enforces $n_i -1$ constraints, for the $n_i$-th is a linear combination of the other due to the normalization constraint acting on the focal elements of $b_i$.\\ The thesis easily follows.
\end{proof}
Note that if $E_k$ is the only focal element of the prior b.f. $b_0$, the usual normalization constraint (this time acting on the focal element \emph{of $b$}) needs to be added, setting the minimum number of focal elements of $b$ to: 
\[
n = \sum_{i=1,...,|{E}_k|}(n_i-1) + 1.
\]

\section{The restricted total belief theorem} \label{sec:restricted-total-belief}

If we enforce the \emph{a-priori} function $b_0$ to have only \emph{disjoint} focal elements (i.e., $b_0$ to be the vacuous extension of a Bayesian function defined on some coarsening of $\Omega$), we have what we call the \emph{restricted total belief theorem}. This is the case, for instance, of the data association problem illustrated above. There, the prior b.f. $b_0$ is usually a simple support function whose core contains only a few disjoint focal elements.\\ In this special case it suffices to solve the $K = |\mathcal{E}_{b_0}|$ sub-problems obtained by considering each focal element $E_k$ of $b_0$ \emph{separately}, and then combine the resulting partial solutions by simply weighing the resulting basic probability assignments using the a-priori mass $m_{b_0}(E_k)$, to obtain a fully normalized total belief function. 

%\begin{theorem} \label{the:minimality}
%The above described procedure delivers a solution to the total belief problem which is optimal, i.e., associated with the minimum number of focal elements, if and only if the \emph{a-priori} belief function $b_0$ has no intersecting focal elements ($b_0$ is the vacuous extension of a Bayesian belief function).
%\end{theorem}
%\begin{proof}
%The number of constraints of the general total belief problem (see Figure \ref{fig2}) is
%\begin{equation} \label{eq:ng}
%n_g = \sum_{i=1}^N (n_i-1)+\mathcal{K}.
%\end{equation}
%On the other hand, if we combine the $\mathcal{K}$ subproblems related to each $E_k\in{\mathcal{E}}_{b_0}$ we obtain a number of constraints equal to:
%\begin{equation} \label{eq:nr}
%n_r = \sum_{k=1}^{\mathcal{K}} \bigg ( \sum_{i=1}^{N_k}(n_i-1) + 1 \bigg) = \sum_{i=1}^{N}\mu_i(n_i-1)+\mathcal{K},
%\end{equation}
%where $\mu_i\geq 1$ is the \emph{multiplicity} of $\sigma_i$, i.e. the number of focal elements ${{E}}_k$ of $n_0$ that contain $\sigma_i$. By comparing (\ref{eq:ng}) and (\ref{eq:nr}) we can observe that $n_r \geq n_g$, and that:
%\[ n_r=n_g\;\Leftrightarrow \mu_i=1\;\forall\;i=1,...,N, \]
%i.e, the focal elements $E_k$ of the a-priori b.f. $b_0$ are disjoint.
%\end{proof}

As we will see in the following, for each individual focal element of $b_0$ the task of finding a suitable solution to the total belief problem translates into a linear algebra problem.

\subsection{A simple case study} \label{sec:case-study}

%Let us focus on the sub-problem associated with a single focal element of $b_0$, and illustrate in a simple case study the features of the restricted total belief problem.

Suppose that the considered focal element $E$ of $b_0$ has cardinality three, so that its image $\rho(E)$ covers three partitions $\Pi_1$, $\Pi_2$ and $\Pi_3$ of $\Omega$. Suppose also that: the conditional belief function $b_1$ defined on $\Pi_1$ has two focal elements $e_1^1$ and $e_1^2$; the conditional b.f. $b_2$ defined on $\Pi_2$ has a single focal element $e_2^1$; $b_3$ on $\Pi_3$ has two focal elements, $e_3^1$ and $e_3^2$ (see Figure \ref{fig:case-study}).

\begin{figure}[ht!]
\begin{center}
\includegraphics[width = 0.6 \textwidth]{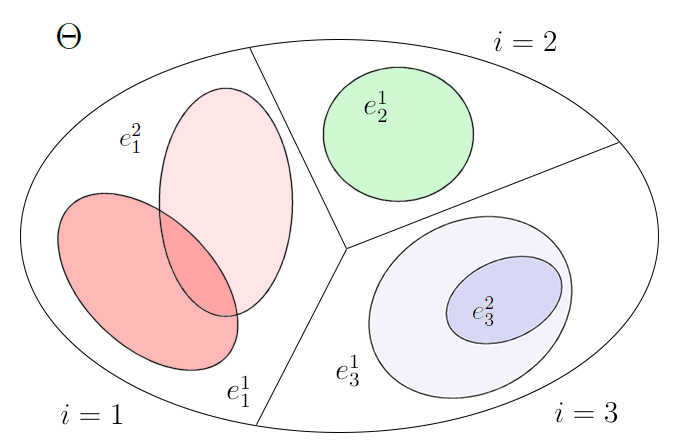} 
\caption{The conditional belief functions considered in our case study. Once again the set-theoretical relations between their focal elements are immaterial to the study's solution. \label{fig:case-study}}
\end{center}
\end{figure}

Clearly in this example $n_1 = 2$, $n_2 = 1$ and $n_3 = 2$. The number of possible focal elements which satisfy the structure proven in Lemma \ref{lem:conditional} is therefore $n_{max} = n_1 \times n_2 \times n_3 = 4$. They are listed as follows:
\begin{equation} \label{eq:fe-case-study}
\begin{array}{lllllll}
e_1 & = & e_1^1 & \cup & e_2^1 & \cup & e_3^1; \\
e_2 & = & e_1^1 & \cup & e_2^1 & \cup & e_3^2; \\
e_3 & = & e_1^2 & \cup & e_2^1 & \cup & e_3^1; \\
e_4 & = & e_1^2 & \cup & e_2^1 & \cup & e_3^2. 
\end{array}
\end{equation}
In order to meet the conditional constraints of the total belief problem, the total belief function $b$ with focal elements $\{ e_1, e_2, e_3, e_4\}$ ought to satisfy the following equalities:
\begin{equation} \label{eq:conditional-constraints-case-study}
\left \{ \begin{array}{l} b \oplus b_{\Pi_1} = b_1; \\ b \oplus b_{\Pi_2} = b_2; \\ b \oplus b_{\Pi_3} = b_3,  \end{array} \right .
\end{equation}
which translate into the following three sets of constraints:
\[
\begin{array}{ccc}
\left \{ \begin{array}{l} m_1(e_1^1) = m(e_1) + m(e_2) \\ m_1(e_1^2) = m(e_3) + m(e_4) \end{array} \right .
&
\left \{ \begin{array}{l} m_2(e_2^1) = \sum_i m(e_i) = 1 \end{array} \right .
&
\left \{ \begin{array}{l} m_3(e_3^1) = m(e_1) + m(e_3) \\ m_3(e_3^2) = m(e_2) + m(e_4), \end{array} \right .
\end{array}
\]
where $m_i$ denotes the b.p.a. of $b_i$ and $m$ that of the candidate total function $b$.

Now, the last constraint in each set is a direct consequence of (or is equal to) the normalization equality $m(e_1) + m(e_2) + m(e_3) + m(e_4) = 1$. Therefore, the conditional constraints (\ref{eq:conditional-constraints-case-study}) amount in the end to the following linear system:
\begin{equation} \label{eq:system-case-study}
\left \{ \begin{array}{l} m(e_1) + m(e_2) = m_1(e_1^1)  \\ m(e_1) + m(e_3) = m_3(e_3^1) \\ m(e_1) + m(e_2) + m(e_3) + m(e_4) = 1. \end{array} \right .
\end{equation}
As the latter is underdetermined, we have an infinite plurality of solutions in the vector:
\[
\vec{x} = [m(e_1), m(e_2), m(e_3), m(e_4) ]', 
\]
which form an entire linear variety of normalized sum functions (see Chapter \ref{cha:geo}, Section \ref{sec:normalized-sum-functions}). Possibly, some of these solutions will have all positive components, i.e., they will correspond to admissible belief functions. In particular, we are interested in solutions with the \emph{minimal} number of focal element, in this case $n = 3$. Note that this confirms the result of Lemma \ref{lem:number}, as $1 + \sum_i (n_1 - 1) = 1 + 1 + 0 +1 = 3$ (taking into account the normalization constraint).

System (\ref{eq:system-case-study}) can be written as $A \vec{x} = \vec{b}$, where $\vec{b} = [m_1(e_1^1), m_3(e_3^1), 1]'$ and:
\[
A = \left [ \begin{array}{cccc} 1 & 1 & 0 & 0 \\ 1 & 0 & 1 & 0 \\ 1 & 1 & 1 & 1   \end{array} \right ].
\]
The matrix has full row rank 3, as rows are all linearly independent.
By selecting any three columns from $A$, then, we obtain a linear system with a unique solution. There are $\binom{4}{3} = 4$ possible column selections, which yield the following matrices for the resulting four linear systems:
\[
\begin{array}{cccc}
\left [ \begin{array}{ccc} 1 & 1 & 0 \\ 1 & 0 & 1 \\ 1 & 1 & 1 \end{array} \right ],
&
\left [ \begin{array}{ccc} 1 & 1 & 0 \\ 1 & 0 & 0 \\ 1 & 1 & 1 \end{array} \right ],
&
\left [ \begin{array}{ccc} 1 & 0 & 0 \\ 1 & 1 & 0 \\ 1 & 1 & 1 \end{array} \right ],
&
\left [ \begin{array}{ccc} 1 & 0 & 0 \\ 0 & 1 & 0 \\ 1 & 1 & 1 \end{array} \right ],
\end{array} 
\]
whose solutions are, respectively:
\[
\begin{array}{cc}
\left \{ 
\begin{array}{l}
m(e_1) = m_1(e_1^1) + m_3(e_3^1) - 1 \\
m(e_2) = 1 - m_3(e_3^1) \\
m(e_3) = 1 - m_1(a_1^1),
\end{array}
\right .
&
\left \{ 
\begin{array}{l}
m(e_1) = m_3(e_3^1) \\
m(e_2) = m_1(e_1^1) - m_3(e_3^1)  \\
m(e_4) = 1 - m_1(e_1^1) ,
\end{array}
\right .
\\ \\
\left \{ 
\begin{array}{l}
m(e_1) = m_1(e_1^1) \\
m(e_3) = m_3(e_3^1) - m_1(e_1^1) \\
m(e_4) = 1 - m_3(e_3^1),
\end{array}
\right .
&
\left \{ 
\begin{array}{l}
m(e_2) = m_1(e_1^1) \\
m(e_3) = m_3(e_3^1) \\
m(e_4) = 1 - m_1(e_1^1) - m_3(e_3^1).
\end{array}
\right .
\end{array} 
\]
We can notice a number of facts:
\begin{enumerate}
\item
minimal solutions can have negative components, i.e, amount to normalized sum function rather than proper belief functions;
\item
nevertheless, there always exists a solution with all positive components, i.e., a proper total belief function.
\end{enumerate}
As for 1), looking at the first candidate minimal solution we can notice that the first component $m(e_1) = m_1(e_1^1) + m_3(e_3^1) - 1$ is not guaranteed to be non-negative: therefore, it will yield an admissible belief function only if $m_1(e_1^1) + m_3(e_3^1) < 1$.
However (Point 2)), we can notice that in the latter case, the fourth candidate minimal solution is admissible, as $m(e_4) = 1 - m_1(e_1^1) - m_3(e_3^1) > 0$.

Similarly, whenever the second solution is non-admissible ($m_1(e_1^1) - m_3(e_3^1) < 0$) the third one is ($m_3(e_3^1) - m_1(e_1^1) > 0$). In conclusion, no matter what the actual b.p.a.s of $b_1$, $b_2$ and $b_3$ are, there always exists an admissible total b.f.

In the following we will work towards proving that this is the case in the general setting as well.

\subsection{Candidate minimal solution systems} \label{sec:candidate-solution-systems}

In the general case, let $N$ be the number of singleton elements of a given focal element $E$ of $b_0$ (the number of partition elements of $\Theta$ covered by $\rho(E)$). From the proof of Lemma \ref{lem:number}, an in particular by Equation (\ref{eq:lemma-number}), a candidate solution to the restricted total belief problem (more precisely, to the subproblem associated with $E$) is the solution to a linear system with $n_{min} = \sum_{i=1,...,N}(n_i - 1) + 1$ equations and $n_{max} = \prod_i n_i$ unknowns:
\begin{equation} \label{eq:candidate-solution}
A \vec{x} = \vec{b},
\end{equation}
where each column of $A$ is associated with an admissible (i.e., meeting the structure of Lemma \ref{lem:conditional}) focal element $e_j$ of the candidate total belief function, $\vec{x} = [m_b (e_1), \cdots, m_b(e_n)]$ and $n = n_{min}$ is the number of equalities generated by the $N$ conditional constraints.

Each solution system has the form:
\begin{equation} \label{eq:solution-system}
\left \{ 
\begin{array}{ll}
\displaystyle \sum_{e_j \cap \Pi_i  = e_i^{j_i}} m(e_j) = m_i(e_i^{j_i}) & \forall i=1,...,N, \; \forall j_i = 1,..., n_i -1
\\
\displaystyle \sum_{j} m_b (e_j) = 1.
\end{array} 
\right .
\end{equation}
where, again, $e_i^{j_i} \subset \Pi_i$ denotes the $j_i$-th focal element of $b_i$.

Since it is straightforward to prove that
\begin{lemma} \label{lem:full-rank}
The rows of the solution system (\ref{eq:solution-system}) are linearly independent.
\end{lemma}
\begin{proof} 
It suffices to point out that each new constraint (row of $A$) involves candidate focal elements that are not involved in the constraints preceeding it (for if $e_j \cap \Pi_i  = e_i^{j_i}$ then obviously $e_j \cap \Pi_i  \neq e_i^{k}$ for all $k<j_i$).
\end{proof}
any system of equation obtained by selecting $n_{min}$ columns from $A$ has a unique solution.

A \emph{minimal} solution to the restricted total belief problem (\ref{eq:solution-system}) (i.e., a solution with the minimum number of focal elements) is then uniquely determined by the solution of a system of equations obtained by selecting $n_{min}$ columns from the $n_{max}$ columns of $A$. Additionally, we need to look for minimal solutions which are admissible belief functions, i.e., we have to identify a selection of $n_{min}$ columns from $A$ such that the resulting square linear system has a solution with all positive components.

\subsection{Transformable columns} \label{sec:transformable-columns}

Consider an arbitrary minimal candidate solution system, obtained by choosing $n_{min}$ elements from the set of columns of $A$. As we have seen in our case study, some columns may potentially correspond to \emph{negative} components of the solution.\\ Nevertheless, the particular form of the square linear systems involved suggests a way to reach an admissible solution by applying a series of linear transformations (or, equivalently, a series of column substitutions) which may eventually lead to a solution whose components are all positive.

Namely, each row of the solution system (\ref{eq:solution-system}) enforces the sum of the masses of the selected focal elements of $b$ to be positive (as $ m_i(e_i^{j_i}) > 0$ for all $i, j_i$). Therefore, whenever $m(e_k) < 0$ for some component $k$ of the solution there must exist for all $i$ at least another focal element $e_{l_i}$ of the total belief function which coincides with $e_k$ over the related partition element $\Pi_i$:
\[
e_k \cap \Pi_i = e_{l_i} \cap \Pi_i.
\]
In other words, looking at the $A$ matrix of the candidate minimal solution system, whenever a column possesses a `1' in a certain row there is at least one (but possibly more) other column with a `1' in the same row.

We say that $e_k$ is a \emph{transformable} column, and call such columns $e_{l_i}$ the `companions' of $e_k$.

\subsubsection{Case study}

Going back to the case study of Section \ref{sec:case-study}, the four possible total focal elements can be represented as `elastic bands' (see Figure \ref{fig:total-focal-elements}) as in the following diagram:

\begin{figure}[ht!]
\begin{center}
\includegraphics[width = 0.65 \textwidth]{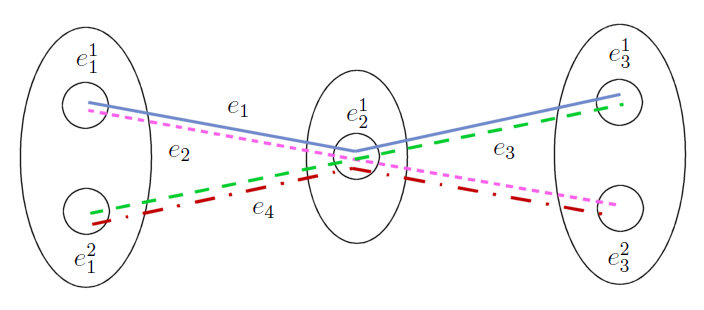}
\caption{\label{fig:solution-case-study} Graphical representation of the four possible focal elements (\ref{eq:fe-case-study}) of the case study of Section \ref{sec:case-study}.}
\end{center}
\end{figure}
Graphically, the `companions' of a focal element $e_k$ on a partition element $i$ are those $e_j$s which cover on that partition the same node (focal element of $b_i$). For instance, the companions of $e_1$ are $e_2$ (on $\Pi_1$), $e_2,$ $e_3$ and $e_4$ (on $\Pi_2$) and $e_3$ (on $\Pi_3$). Clearly a focal element can be a companion for another on several partitions of $\Theta$ -- we will discuss this point later.

Now, if we select $e_1$, $e_2$ and $e_4$ (first, second and fourth column of the complete solution system) to form a candidate minimal solution system, we can notice that $e_1$ is not `covered' for all $i$ (namely on $\Pi_3$), and therefore cannot correspond to a negative solution component (is not `transformable'). The same is true for $e_4$. The only transformable column is $e_2$.

As we suggested above, we can actually replace the transformable column $e_2$ with another one (namely $e_3$) by applying the following linear transformation:
\[
e_2 \mapsto e'_2 = - e_2 + (e_1 + e_4) = \left [ \begin{array}{c} -1 +1 +0 \\ 0 +1 +0 \\ -1 +1 +1 \end{array} \right ] = \left [ \begin{array}{c} 0 \\ 1 \\ 1 \end{array} \right ] = e_3,
\]
in which the column to be replaced is subtracted, while its companions are added in to yield another column corresponding to an admissible focal element. We will study the effect of such a transformation on the solution of a candidate minimal system in the following.

\subsection{A class of linear transformations}

\begin{definition} \label{def:column-transformation}
We define a class $\mathcal{T}$ of transformations acting on transformable columns $e$ of a candidate minimal solution system via the following formal sum:
\begin{equation} \label{eq:column-transformation}
e \mapsto e' = -e + \sum_{i\in \mathcal{C}} e_i -\sum_{j\in\mathcal{S}} e_j
\end{equation}
where $\mathcal{C}$, $|\mathcal{C}|<N$ is a covering set of companions of $e$ (i.e.,  every component of $e$ is covered by at least one of them), and a number of \emph{selection} columns $\mathcal{S}$, $|\mathcal{S}|=|\mathcal{C}|-2$, are employed to compensate the side effect of $\mathcal{C}$ to yield an admissible column (i.e., a candidate focal element meeting the structure of Lemma \ref{lem:conditional}).
\end{definition}
We call the elements of $\mathcal{T}$ \emph{column substitutions}.

A sequence of column substitutions induces a discrete path in the solution space: the values of the solution components associated with each column vary, and in a predictable way. If we denote by $s<0$ the (negative) solution component associated with the old column $e$:
\begin{enumerate}
\item the new column $e'$ has as solution component $-s>0$;
\item the solution component associated with each companion column decreases by $|s|$;
\item the solution component associated with each selection increases by $|s|$;
\item all other columns retain the old values of their solution components.
\end{enumerate}
The proof is a direct consequence of the linear nature of the transformation (\ref{eq:column-transformation}).

Clearly, if we choose to substitute the column with the most negative solution component, the overall effect is that: the most
negative component is changed into a positive one; components associated with selection columns become more positive (or less negative); as for companion columns, while some of them may end up being assigned negative solution components, in absolute value these will be smaller than $|s|$ (since their initial value was positive). Hence:

\begin{theorem} \label{the:column-transformation}
Column substitutions of the class $\mathcal{T}$ reduce the absolute value of the most negative solution component.
\end{theorem}

\subsection{Sketch of an existence proof} \label{sec:sketch-proof}

We can think of using Theorem \ref{the:column-transformation} to prove that there always exists a selection of columns of $A$ (focal elements of the total belief function) such that the resulting square linear system has a positive vector as a solution. This can be done in a constructive way, by applying a transformation of the type (\ref{eq:column-transformation}) recursively to the column associated with the most negative component, to obtain a path in the solution space which eventually lead to the desired solution.

The following sketch of an existence proof for the restricted total belief theorem exploits the effects on solution components of colum substitutions of type $\mathcal{T}$:
\begin{enumerate}
\item at each column substitution the most negative solution component {decreases} by Theorem \ref{the:column-transformation};
\item if we keep substituting the most negative variable we keep obtaining \emph{distinct} linear systems, for at each step the transformed column is assigned a positive solution component and therefore, if we follow the proposed procedure, \emph{cannot be changed back to a negative one} by applying transformations of class $\mathcal{T}$;
\item this implies that there can be no cycles in the associated path in the solution space;
\item the number $\binom{n_{max}}{n_{min}}$ of solution systems is obviously finite, hence the procedure must terminate.
\end{enumerate}

Incidentally, the (Euclidean) \emph{length} $\| \vec{x} - \vec{x}' = \mathcal{T}(\vec{x}) \|_2$ of a transition in the solution space of the total belief problem is, trivially:
\[
\sqrt{\sum_{i=1}^n  ( \vec{x}_i - \vec{x}'_i )^2} = \sqrt{4 s^2 + \sum_{i \in \mathcal{C}} s^2 + \sum_{i \in \mathcal{S}} s^2
} = \sqrt{s^2 (4 + |\mathcal{C}| + |\mathcal{S}|)} = \sqrt{2} |s| \sqrt{|\mathcal{S}| + 3},
\]
where $s$ is the solution component related to the substituted column. Simple counterexamples show that the shortest path to an admissible system \emph{is not necessarily composed by longest (greedy) steps}. This means that algorithms based on greedy choices or dynamic programming cannot work, for the problem does not seem to meet the `{optimal substructure}' property.

If every transformable column (possessing companions on every partition $\Pi_i$ of $\Theta$) was $\mathcal{T}$-transformable the procedure could not terminate with a minimal solution system with negative solution components, for in that case they would have one or more  companions on each partition $\Pi_i$.\\ Unfortunately, counterexamples show that there are `transformable' columns (associated with negative solution components) which do not admit a transformation of the type (\ref{eq:column-transformation}). Although they do have companions on every partition $\Pi_i$, such counterexamples do not admit a complete collection of `selection' columns.

\subsection{Solution graphs and types of candidate solutions} \label{sec:solution-graphs}

To better understand the complexity of the problem, and in particular address the issue with the number of admissible solutions to the (restricted) total belief problem, it is useful to define an {adjacency} relation between solution systems. 

We say that a candidate minimal solution system $\sigma$ is adjacent to another system $\tau$ if $\tau$ can be obtained from $\sigma$ by substituting a column by means of a transformation of the form (\ref{eq:column-transformation}) (and vice-versa)\footnote{Note that column substitutions of the form  (\ref{eq:column-transformation}) are reversible: what we claimed above is that if we keep replacing columns with the most negative solution component we can never go back to systems we have already visited.}.\\ This allows us to rearrange the candidate minimal solution systems related to a problem of a given size $\{n_i,\;i=1,...,N\}$ into a \emph{solution graph}.

\begin{figure}[ht!]
\begin{center}
\includegraphics[width = 0.5 \textwidth]{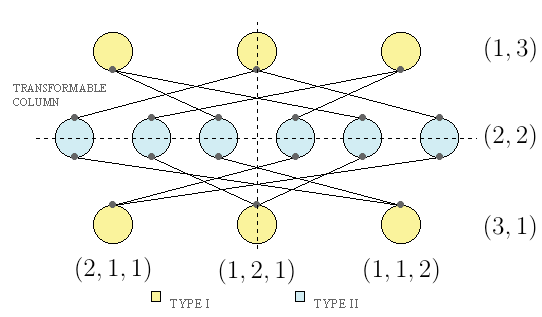} 
 \caption{\label{fig:3x2} The solution graph associated with the restricted total belief problem with $N=2$, $n_1=3$ and $n_2=2$.}
\end{center}
\end{figure}

\subsubsection{Examples of solution graphs}

It can be educational to see some significant examples of solution graphs, in order to infer their properties and general structure.

Figure \ref{fig:3x2} shows the solution graph formed by all the candidate solution systems for the problem of size $N=2$, $n_1 = 3$, $n_2 = 2$. The number of possible minimal solution systems is: 
\[
\binom{n_{max} = n_1 \cdot n_2 = 6}{n_{min} = (n_1 - 1) + (n_2 - 1) +1 = 4} = 12. 
\]
The twelve candidate solution systems can be arranged into a matrix whose rows and columns are labeled respectively with the counts $(c_1, ... , c_{n_i})$ of focal elements of the associate candidate total function (columns of the solution system) containing each focal element of $b_i$. For instance, the label $(2,1,1)$ indicates that, of the $(n_1 - 1) + (n_2 - 1) +1 = 2 +1 +1 = 4$ focal elements of the minimal total b.f. $b$ generated by solution systems in that entry of the matrix, 2 cover $e_1^1$, one covers $e_1^2$ and one $e_1^3$. 

We can also observe that the candidate solution systems for this problem can be arranged in two classes according to the number of transformable columns they possess and the number of admissible $\mathcal{T}$ transformations (edges) for each transformable column. Type II systems (in blue, central row)  possess two transformable columns, each admitting only one column substitution of type $\mathcal{T}$; type I systems, instead (in yellow, top and bottom rows) only have one transformable column which can be substituted, however, in two different ways.

In the perspective of proving the existence of a solution to the restricted total belief problem, it is interesting to note that the graph of Figure \ref{fig:3x2} can be rearranged to form a {chain} of solution systems: its edges form a single, closed loop. This implies that we can reach any solution system starting from any other: starting from an initial non-admissible solution we can reach an admissible one via column substitutions of the proposed type\footnote{Incidentally, such a chain is composed by `rings' whose central node is a type I system connected to a pair of type II systems. Two consecutive rings are linked by a type II system.}.

\begin{figure}[ht!]
\begin{center}
\includegraphics[width = 0.6 \textwidth]{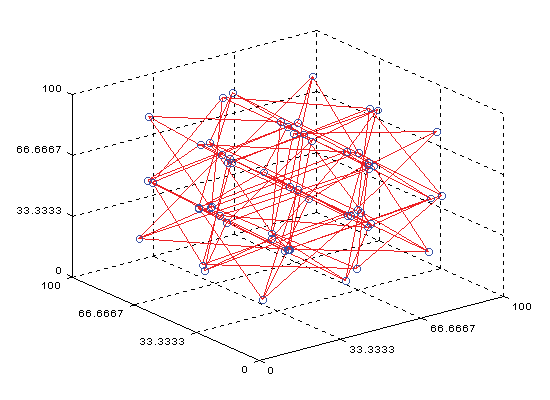} 
\caption{\label{fig:2x2x2} The solution graph associated with the restricted total belief problem of size $N=3$, $n_1 = n_2 = n_3 = 2$.}
\end{center}
\end{figure}
Figure \ref{fig:2x2x2} shows a more complex example, in which the overall symmetry of its solution graph emerges.

\subsection{Graph symmetries, types of solution systems and number of admissible solutions}

Once again, we can produce counterexamples in which the same entry does contain systems of different type. Hence, the `type' of a solution system (how many transformable columns it has, and in how many ways they can be substituted) must be induced by some other global property of the solution graph. Let 
\begin{equation} \label{eq:group}
G = S_{n_1} \times \cdots \times S_{n_N}
\end{equation}
be the group of {permutations of focal elements} of all the conditional belief functions $b_i$. The group is trivially the product of the permutation groups $S_{n_i}$ acting on the collections of focal elements of each individual conditional belief function $b_i$. As it alters the ordering of the focal elements of each $b_i$, $G$ acts on a solution system by moving it to a different location of the solution graph. \\ Given a solution system $\sigma$, the \emph{orbit} induced by the action of $G$ on $\sigma$ is the set of all solution systems (nodes of the graph) obtained by some permutation of the focal elements within at least some of the partition elements $\Pi_i$.

The following conjecture originates from the study of the structure of several significant solution graphs. \vspace{3mm}

\textbf{Conjecture.} The {orbits} of $G$ coincide with the types of solution systems.\vspace{3mm}

The conjecture is quite reasonable, for the behavior of a solution system in terms of transformable columns depends only on the cardinality of the collections of focal elements containing each focal element of $b_i$ for each $i=1,...,N$. It does not depend on which specific e.f. is assigned to which collection. From group theory we known that the orbits of $G$ are disjoint, forming therefore a partition of the set of nodes of the graph, as they should if they indeed represented types of solution system.

The number of orbits of (\ref{eq:group}) could be related to the number of admissible minimal solutions to the restricted total belief problem. We will pursue this line of research in the near future.

%\subsection{Some intriguing interpretations} \label{sec:interpretations}
%\subsubsection{In terms of positive linear systems}
%\subsubsection{In terms of Grassman manifolds}

\section{Conclusive comments} \label{sec:conclusive-comments}

The notion of transformable column seems to point in the right direction. The algorithm of Section \ref{sec:sketch-proof} can be interpreted as the proof of existence of an optimal path within a graph: having chosen an arbitrary node $\sigma$ of the graph, there exists {at least} one path to a different node which corresponds to a system with positive solution (an admissible total belief function).

Unfortunately, we still do not have a complete constructive proof of the restricted total belief theorem, for wanting of a more general class of linear transformations $\mathcal{T}'$ applicable to \emph{any} column with negative solution component. After detecting such a class of transformations an investigation of the properties of the associated optimal paths will be in place, together with a global analysis of the structure of solution graphs and their mutual relationships. \\
The structure of solution graphs for a number of significant special cases suggests that each graph contain a number of `copies' of solution graphs related to lower size problems. For instance, the graph for the problem $N=2$, $n_1 = 4$, $n_2 = 2$ is composed by 32 nodes (candidate minimal solution systems) arranged in 8 chains, each isomorphic to the graph of Figure \ref{fig:3x2} associated with the problem $N =2$, $n_1 = 3$, $n_2=2$. Each system of the larger graph is covered by 3 of these chains. \\ Such inclusion relationships between graphs of problems of different size are potentially extremely useful in the perspective of addressing the other major missing element of the restricted total belief problem -- the computation of the number of admissible solutions (proper minimal-size total belief functions).\\
Our conjecture about the relationship between the action of the group of permutations $G$ and the global symmetry of the solution graph also needs to be investigated.

Finally, while we know the minimal number of focal elements of the total function in the restricted case (in which the a-priori belief function $b_0$ has disjoint focal elements), we still do not understand what the similar bound should look like in the general case of an arbitrary prior, or whether the presence of intersecting focal elements in $b_0$ does at all influence the structure of the focal elements of the sought total function (in other words, how does Lemma \ref{lem:conditional} generalize to the case of arbitrary prior belief functions). 

These challenges are open for the whole belief functions community to be taken on. As for us we will keep working towards their solution, and investigate the fascinating relationships between the total belief problem and transversal matroids \cite{Oxley}, on one hand, and positive linear systems \cite{Farina}, on the other, in an ongoing attempt to bridge the mathematics of uncertainty with combinatorics, algebra and geometry.

\chapter{Belief Modeling Regression} \label{cha:pose}

\begin{center}
\includegraphics[width = \textwidth]{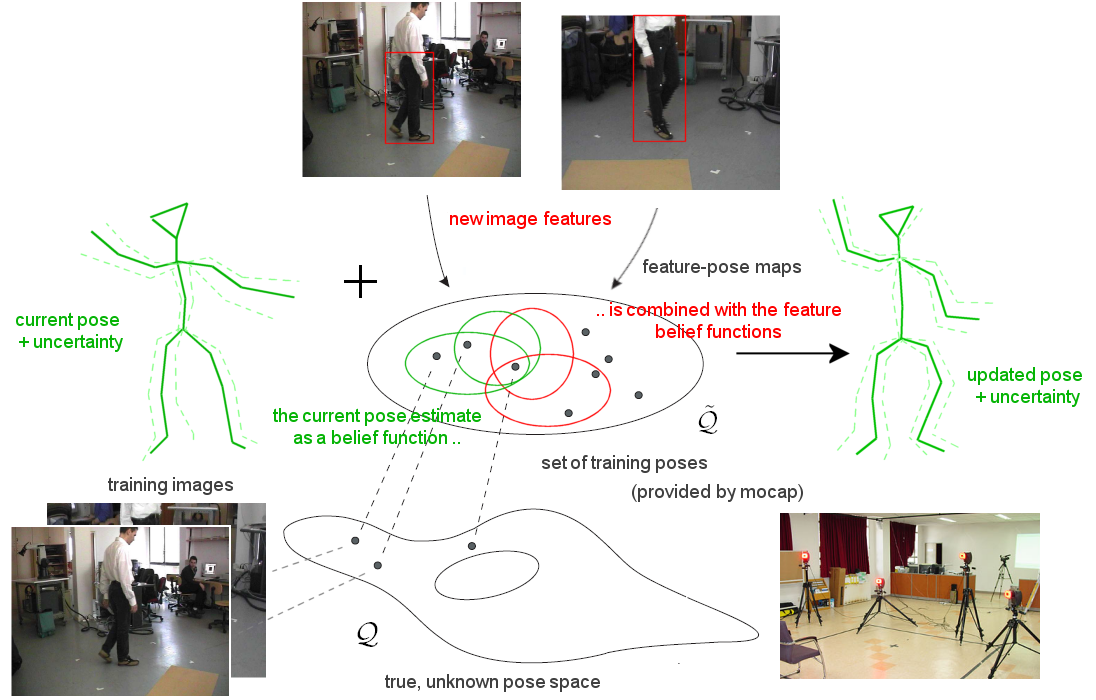}
\end{center}
\vspace{5mm}

\emph{Pose estimation} is another well studied problem in computer vision. Given an image sequence capturing the motion and evolution of an object of interest, the problem consists in estimating the position and orientation of the object at each time instant, along with its internal configuration or \emph{pose}. Such estimation is typically based on two pillars: the extraction of salient measurements or \emph{features} from the available images and, when present, a \emph{model} of the structure and kinematics of the moving body. Pose estimation is, among others, a fundamental ingredient of {motion capture}, i.e., the reconstruction of the motion of a person throughout a video sequence, usually for animation purposes in the movie industry or for medical analysis of posture and gait. Other major applications include human-computer interaction, image retrieval on the internet, robotics (Figure \ref{fig:pose-applications}).

\subsection*{Related Work}

    % state of the art: make distinction between three approaches clear
Current methodologies for pose estimation can roughly be classified into `model-based', `learning-based' and `example-based` approaches. The former \cite{Deutscher00,Sidenbladh00a} presuppose an explicitly known parametric body model: pose recovery is typically achieved by matching the pose variables to a forward rendered model based on the extracted features. Initialization is often difficult, and the pose optimization process can be subject to local minima \cite{Sminchisescu03kinematic}. In contrast, \emph{learning-based} approaches \cite{agarwal06pami,bb71611,946721,1068941} exploit the fact that typical (human) motions involve a far smaller set of poses than the kinematically possible ones, and learn a model that directly recovers pose estimates from observable image quantities. Such methods \cite{1099953,bb71628,bb33622,niyogi96afgr} are appealing and generally faster, due to the lower dimensionality of the models employed, and typically provide a better predictive performance when the training set is comprehensive. On the other hand, they sometimes require heavy training to produce a decent predictive model, and the resulting description can lack generalization power.

\begin{figure}[ht!]
\begin{center}
\includegraphics[width = 0.9 \textwidth]{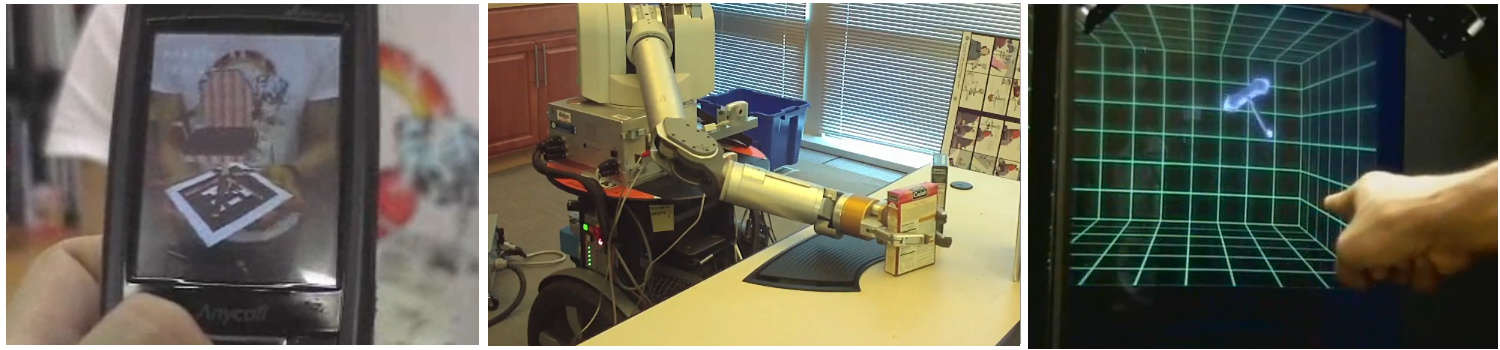}
\end{center}
\caption{Some applications of pose estimation: human-machine interaction, robotics, virtual reality. \label{fig:pose-applications}} 
\end{figure}
    % features and issues of example-based estimation
\emph{Example-based} methods, which explicitly store a set of training examples whose 3D poses are known, estimate pose by searching for training image(s) similar to the given input image and interpolating from their poses \cite{946721,Athitsos04cvpr}. They can then be used to initialize model-based methods in a `smart' way, as in the  monitoring of an automobile driver's head movements provided in \cite{niyogi96afgr}. No prior analytic structure of the pose space is incorporated in the estimation process, although the training data itself do amount to a rough approximation of the configuration space. 

Most of these methods share a common architecture. Vectors of feature measurements (such as moments of silhouette images \cite{rosales00human-motion}, multi-scale edge direction histograms \cite{dalal05cvpr}, distribution of shape contexts \cite{agarwal06pami}, and Harr-like features \cite{viola01cvpr}) are extracted from each individual image. Indeed, the integration of multiple cues is crucial in pose estimation to increase both the resolution/accuracy of the estimation and its robustness \cite{Hel-Or95,Darrell98,Moeslund00,Sminchisescu01,Sidenbladh03}.
    % feature-pose maps
Then, the likely pose of the object is predicted by feeding this feature vector to \emph{a map from the features space to the pose space}, which is learned from a training set of examples or a model whose parameters are learned from the training data, and whose purpose is to (globally or locally) represent the relationship between image and pose. This mapping, albeit unknown, is bound to be (in general) one-to-many: more than one object configuration can generate the same feature observation, because of occlusions, self-occlusions and the ambiguities induced by the perspective image projection model.

Since only limited information is provided to us in the training session, only an approximation of the true feature-pose mapping can be learned. The accuracy of the estimation depends on the forcibly limited size and distribution of the available examples, which are expensive and time-consuming to collect. This has suggested in the past to consider a more constrained, activity-based setting to constrain the search space of possible poses.\\ In \cite{rosales00human-motion}, for instance, an inverse mapping between image silhouette moments and 2D joint configurations is learned, for each cluster obtained by fitting a Gaussian mixture to 2D joint configurations via the EM algorithm. In \cite{agarwal06pami} a Relevant Vector Machine (RVM) is used to infer human pose from a silhouette's shape descriptor, while more recently an extension to mixtures of RVMs has been proposed by Thayananthan et al. \cite{Thayananthan06eccv}. In \cite{mori06pami}, a number of exemplar 2D views of the human body are stored; the locations of the body joints are manually marked and labeled. The input image is then matched via `shape context matching' to each stored view, and the locations of the body joints in the matched exemplar view are transferred to the test image. Other approaches include Local Weighted Regression \cite{946721}, BoostMap \cite{Athitsos04cvpr}, Bayesian Mixture of Experts \cite{1068941} and Gaussian Process Regression (GPR) \cite{rasmussen06gpr}.

    % issues
The accuracy of example-based approaches critically depends on the amount and representativeness of the training data. Queries can be potentially computationally expensive, and need to be performed quickly and accurately \cite{946721,Athitsos04cvpr}. In addition, example-based approaches often have problems when working in high-dimensional configuration spaces, as it is difficult to collect enough examples to densely cover them. 

\subsection*{Scope of the Chapter}

In this Chapter we describe a Belief Modeling Regression (BMR) \cite{cuzzolin13fusion,gong2018tfs} framework for example-based pose estimation based on the theory of evidence. Our framework uses the finite amount of evidence provided in a training session to build, given a new feature value, a belief function on the set of training poses. In this context we favour the interpretation of belief functions as convex sets of probability distributions (credal sets, see Chapter \ref{cha:state}, Section \ref{sec:credal-sets}), according to which a belief function on the pose space is equivalent to a {set of linear constraints on the actual conditional pose distribution (given the features)}. Regression is made possible by learning during training a refining (in the evidence-theoretical sense of Definition \ref{def:refining}) between an approximation of each feature space, obtained via Expectation-Maximization, and the set of training poses. At test time each feature value, encoded as a set of likelihoods, translates into a belief function on the set of training poses. This determines a convex sets of distributions there, which in turn generates an interval of pose estimates.

Multiple features are necessary to obtain decent accuracy in terms of pose estimation. All single-feature refinings are collected in an `evidential model' of the object: the information they carry is fused before estimating the object's pose in the belief framework, allowing a limited resolution for the individual features to translate into a relatively high estimation accuracy (in a similar way to tree-based classifiers \cite{Meynet08} or boosting approaches, in which weak features are combined to form a strong classifier). The size of the resulting convex set of probabilities reflects the amount of training information available: the larger and more densely distributed within the pose space the training set is, the narrower the resulting credal set. Both a {pointwise estimate} of the current pose and a {measure of its accuracy} \cite{melkonyan06ijar} can then be obtained. In alternative, a separate pose estimate can be computed for each vertex of the credal set, in a robust statistical fashion \cite{berger90jspim,seidenfeld93dilation}.

As we show in the last part of the Chapter, an evidential model essentially provides a constraint on the family of admissible feature-to-pose maps, in terms of smooth upper and lower bounds. All mappings (even discontinuous, or 1-many) within those smooth bounds are possible under the model. The width of this space of mappings reflects the uncertainty induced by the size and distribution of the available training set.

    % outline of the paper
\subsection*{Chapter Outline} 

The Chapter is structured as follows.\\ First (Section \ref{sec:problem}) the scenario and assumptions of the problem are laid down. The learning of an `evidential model' of the body from the learning data, based on approximations of the unknown feature-to-pose maps, is described in Section \ref{sec:model-learning}. In Section \ref{sec:estimation} the special class of Dirichlet belief functions is proposed to model the uncertainty due to the scarcity of the training data. From the belief estimate which results from their conjunctive combination either a pointwise estimate or a set of extremal estimates of the pose can be extracted. The computational complexity of learning and estimation algorithms is also analyzed.\\ In Section \ref{sec:assessing} model assessment criteria based on the theory of families of compatible frames are discussed.\\ Section \ref{sec:results} illustrates the performance of Belief Modeling Regression in an application to human pose recovery, showing how BMR outperforms our implementation of both Relevant Vector Machine and Gaussian Process Regression. Section \ref{sec:discussion} discusses motivation and advantages of the proposed approach in comparison with other competitors, and analyzes approaches alternative to Dirichlet modeling for belief function inference. Finally, Section \ref{sec:tracking} outlines an extension of Belief Modeling Regression to fully-fledged tracking, in which temporal consistency is achieved via the total belief theorem (extensively considered in Chapter \ref{cha:total}).

    % problem statement
\section{Scenario} \label{sec:problem}

We consider the following scenario:
\begin{itemize}
\item
the available evidence comes in the form of a training set of images containing sample poses of an \emph{unspecified} object;
\item
we only know that the latter's configuration can be described by a vector $q \in \mathcal{Q} \subset \mathbb{R}^D$ in a pose space $\mathcal{Q}$ which is a subset of $\mathbb{R}^D$;
\item
a source of ground truth exists which provides for each training image $I_k$ the configuration $q_k$ of the object portrayed in the image;
\item
the location of the object within each training image is known, in the form of a bounding box containing the object of interest.
\end{itemize}
In a training session the object explores its range of possible configurations, and a set of poses is collected to form a finite approximation $\tilde{\mathcal{Q}}$ of the parameter space:
\begin{equation}\label{eq:qtilde}
\tilde{\mathcal{Q}} \doteq \Big \{ q_k,k=1,...,T \Big \}.
\end{equation}
At the same time a number $N$ of distinct features are extracted from the available image(s), within the available bounding box:
\begin{equation}\label{eq:fea}
\begin{array}{cc}
\tilde{\mathcal{Y}} \doteq \Big \{ y_i(k), k = 1,...,T \Big \},& i = 1,...,N.
\end{array}
\end{equation}

In order to collect $\tilde{\mathcal{Q}}$ we need a source of ground truth to tell us what pose the object is in at each instant $k$ of the training session. One option is to use a motion capture system, as it is done in \cite{rosales00human-motion} for the human body tracking problem. After applying a number of reflective markers in fixed positions of the moving object, the system is able to provide by triangulation the 3D locations of the markers throughout the training motion. Since we do not know the parameter space of the object, it is reasonable to use as body pose vector the collection of all marker's 3D locations. 

Based on this evidence, at test time:
\begin{itemize}
\item
a supervised localization algorithm (trained in the training stage using the annotation provided in terms of bounding boxes, e.g. \cite{felzenszwalb-2010}) is employed to locate the object within each test image: image features are only extracted from within the resulting bounding box;
\item
such features are exploited to produce an estimate of the object's configuration, together with a measure of how reliable this estimate is.
\end{itemize}

\section{Learning evidential models} \label{sec:model-learning}

\subsection{Building feature-pose maps} \label{sec:approximate-feature-spaces}

Consider an image feature $y$, whose values live in a feature space $\mathcal{Y}$, and let us denote by $\rho^* : \mathcal{Y} \rightarrow 2^{\mathcal{Q}}$ the unknown mapping linking the feature space $\mathcal{Y}$ to the collection $2^{\mathcal{Q}} = \{ Q \subseteq \mathcal{Q} \}$ of sets of object poses. We seek to learn from the training data an approximation $\tilde{\rho}$ of this unknown mapping, which is applicable to any feature value, and ideally produces only admissible object configurations. {In fact, as evidence is limited, we can only constrain $\tilde{\rho}$ to have output in the space $\mathbb{R}^D$ the true pose space $\mathcal{Q}$ is embedded into.}\\
We propose to obtain such an approximation by applying \emph{EM clustering} \cite{moore-veryfast} to the training data (\ref{eq:qtilde}), (\ref{eq:fea}), individually for each feature component.

Consider the $N$ sequences of feature values $\{y_i(k), k=1,...,T\}$, $i=1,...,N$, acquired during training. EM clustering can be applied to them to obtain a Mixture of Gaussians (MoG)
\begin{equation}\label{eq:mog}
\Big \{ \Gamma_i^j, j=1,...,n_i \Big \}, \;\;\; \Gamma_i^j \sim \mathcal{N}(\mu_i^j,\Sigma_i^j)
\end{equation}
with $n_i$ Gaussian components, separately for each feature space (the range $\mathcal{Y}_i\subset \mathbb{R}^{d_i}$ of the unknown feature function $y_i:\mathcal{I}\rightarrow \mathcal{Y}_i$ acting on the set of all images $\mathcal{I}$). MoG models are often employed in bottom-up pose estimation\footnote{For instance, in \cite{1068941} several `expert' predictions are combined in a Gaussian mixture model. In \cite{Rosales00} conditional distributions are also assumed to be Gaussian mixtures.}, as their parameters can be speedily estimated via the EM algorithm \cite{moore-veryfast}.

Here we use the learnt MoG (\ref{eq:mog}) to build a particle-based discrete approximation of the unknown feature pose mapping. The former induces an implicit partition
\begin{equation} \label{eq:approximate-feature-space}
\Theta_i \doteq \Big \{ \mathcal{Y}_i^1, \cdots , \mathcal{Y}_i^{n_i} \Big \}
\end{equation}
of the $i$-th feature range, where $\mathcal{Y}_i^j = \big\{ y\in \mathcal{Y}_i \; s.t. \; \Gamma_i^j(y) > \Gamma_i^l(y) \; \forall l\neq j \big \}$ is the region of $\mathcal{Y}_i$ in which the $j$-th Gaussian component dominates all the others (Figure \ref{fig:implicit}-right). We call (\ref{eq:approximate-feature-space}) the $i$-th {`approximate' feature space}. The purpose here, however, is to model the feature-pose relation in an efficient way rather than to approximate the actual feature space.
\begin{figure}[ht!]
\begin{center}
\vspace{-0mm}
\begin{tabular}{cc}
\includegraphics[width = 0.47\textwidth]{compat.png} & \includegraphics[width = 0.53\textwidth]{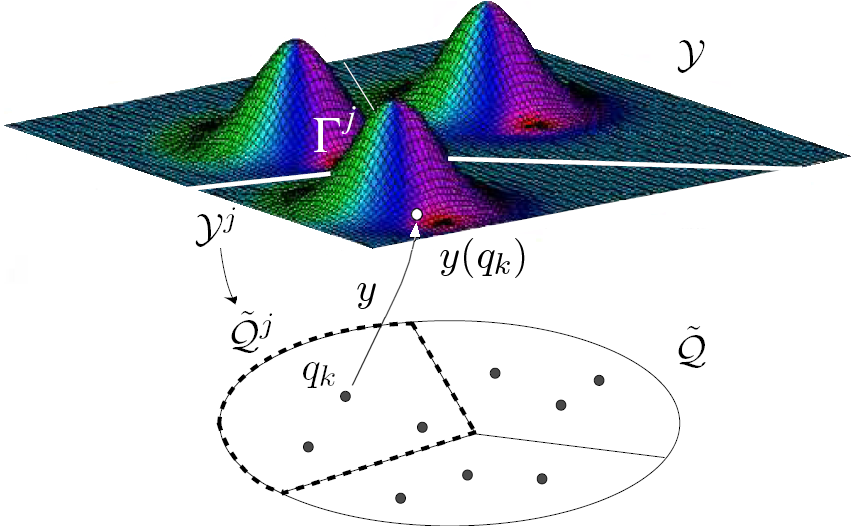}
\end{tabular}
\vspace{-0mm}
\end{center}
\caption{Left: a probability measure $P$ on $\Omega$ induces a belief function $b$ on $\Theta$ through a multi-valued mapping $\rho$. Right: a Mixture of Gaussian learned via EM from the training features defines an implicit partition on the set of training poses $\tilde{\mathcal{Q}}$. A set of Gaussian densities $\{ \Gamma^j, j=1,...,n \}$ on the range $\mathcal{Y}$ of a feature function $y$ define a partition of $\mathcal{Y}$ into disjoint regions $\{ \mathcal{Y}^j \}$. Each of those regions $\mathcal{Y}^j$ is in correspondence with the set $\tilde{\mathcal{Q}}^j$ of sample poses $q_k$ whose feature value $y(q_k)$ falls inside $\mathcal{Y}^j$.\label{fig:implicit}} \vspace{-0mm}
\end{figure}

In virtue of the fact that features are computed during training in synchronous with the true poses provided by the source of ground truth, each element $\mathcal{Y}_i^j$ of the approximate feature space is associated with the set of training poses $q_k \in \mathcal{Q}_k$ whose $i$-th feature value falls in $\mathcal{Y}_i^j$ (Figure \ref{fig:implicit}-right again):
\begin{equation}\label{eq:rhoi}
{\rho}_i : \mathcal{Y}_i^j \mapsto \tilde{\mathcal{Q}}_i^j \doteq \Big \{ q_k\in\tilde{\mathcal{Q}} :
y_i(k)\in \mathcal{Y}_i^j \Big \}.
\end{equation}
Applying EM clustering separately to each training feature sequence (\ref{eq:fea}) yields therefore both $N$ approximate feature spaces $\Theta_i = \{ \mathcal{Y}_i^1, \cdots , \mathcal{Y}_i^{n_i} \}$, $i=1,...,N$, and $N$ maps (\ref{eq:rhoi}) from each of them to the approximate pose space (the set of training poses) $\tilde{\mathcal{Q}}$.
The learned feature-pose maps (\ref{eq:rhoi}) amount to constraints on the unknown feature pose maps $\rho^*_i : \mathcal{Y} \rightarrow 2^\mathcal{Q}$, built upon the evidence available in the specific regions covered by training feature/pose pairs (see Section \ref{sec:discussion}). 

Just as their unknown counterparts $\rho^*_i$, the $\rho_i$s are inherently \emph{multi-valued}, i.e., they map elements of each approximate feature space $\Theta_i$ to \emph{sets} of training poses. The number $n_i$ of clusters can be estimated by cross-validation. Here we will set it to a fixed value for each feature space.

\subsection{Continuous mapping via belief functions} \label{sec:continuous-mapping}

The maps (\ref{eq:rhoi}) only apply to partitions of the feature range, and cannot be used directly to map individual feature values. The structure provided by the learned MoGs (\ref{eq:mog}) can nevertheless be used to build universal mappings. Given the mixture (\ref{eq:mog}), each new feature value $y_i$ can be represented by its soft assignments %\footnote{Some similarities can be noted in this respect with \cite{Howe99}, where the probabilities of short motions are modeled as a Mixture of Gaussian density, via EM. However the feature extraction there is based on 2D limb tracking on the image plane, crucially dependent on a-priori 2D patch model. The idea of building a direct map from features to poses is also central in \cite{Brand99}.}
\begin{equation}\label{eq:soft-assignment}
y_i \mapsto \Big [ \Gamma_i^1(y_i), \Gamma_i^2(y_i), \cdots, \Gamma_i^{n_i}(y_i) \Big ]
\end{equation}
to each mixture component. The density values (\ref{eq:soft-assignment}) constitute a vector of coordinates of the feature value in the feature range $\mathcal{Y}_i$: in this interpretation, the MoG approximation of $\mathcal{Y}_i$ provides an atlas of coordinate charts on the feature space itself. Rather than mapping $y$ we can use (\ref{eq:rhoi}) to map the associated coordinates (soft assignments) (\ref{eq:soft-assignment}), extending the `particle'-like information on the shape of $\rho^*_i$ given by a learnt refining (\ref{eq:rhoi}) to map any test feature value.

By normalizing (\ref{eq:soft-assignment}), each (test) feature value is associated with a probability distribution on the approximate feature space $\Theta_i$. By comparing Figures \ref{fig:implicit}-left and \ref{fig:implicit}-right (see Chapter \ref{cha:state}, Section \ref{sec:multivalued} as well), it is clear that the maps (\ref{eq:rhoi}) are multi-valued mappings linking the question $Q_1$ ``to which Gaussian component of the MoG (\ref{eq:mog}) does the new feature value $y$ belong" to the question $Q_2$ ``what is the object pose whose observed feature value is $y$". 

Now, it follows from Chapter \ref{cha:state}, Section \ref{sec:multivalued}, that the probability distribution associated with any feature value induces a belief function on the (approximate) pose range $\tilde{\mathcal{Q}}$ (Equation (\ref{eq:belvalue})). Overall, the learnt universal feature-pose mapping is a cascade of soft assignment and refining-based multi-valued mapping:
\begin{equation} \label{eq:mapping-bmr}
y_i \in \mathcal{Y}_i \stackrel{(112	)}{\mapsto} \Big [ \Gamma_i^1(y_i), \Gamma_i^2(y_i), \cdots, \Gamma_i^{n_i}(y_i) \Big ] \mapsto p_i = [p_i(\mathcal{Y}_i^1),...,p_i(\mathcal{Y}_i^{n_i})] \stackrel{(17)}{\mapsto} b_i:2^{\tilde{\mathcal{Q}}} \rightarrow [0,1]
\end{equation}
where 
\[
p_i(\mathcal{Y}_i^j) = \frac{\Gamma_i^j(y_i)}{\sum_k \Gamma_i^k(y_i)}, 
\]
associating any test feature value $y_i$ with a belief function $b_i$ on the set of training poses $\tilde{\mathcal{Q}}$.

\subsection{Training algorithm} \label{sec:evidential-models}

In the training stage the body moves in front of the camera(s), exploring its configuration space, while a sequence of training poses $\tilde{\mathcal{Q}} = \{ q_k,k = 1,...,T \}$ is provided by a source of ground truth (for instance a motion capture system, Section \ref{sec:problem}). The sample images are annotated by a bounding box indicating the location of the object within each image. At the same time:

\begin{enumerate}
\item
for each time instant $k$, a number of feature values are computed from the region of interest of each available image: $\{ y_i(k), k = 1,...,T \}$, $i = 1,...,N$;
\item
EM clustering is applied to each feature sequence $\{ y_i(k), k = 1,...,T \}$ (after setting the number of clusters $n_i$), yielding:
\begin{enumerate}
\item $N$ approximate feature spaces $\Theta_i = \{ \mathcal{Y}_i^j, j = 1,...,n_i \}$, i.e., the implicit partitions of the feature ranges $\mathcal{Y}_i$ associated with the EM clusters (Section \ref{sec:approximate-feature-spaces});
\item
$N$ maps (\ref{eq:rhoi}) $\rho_i : \mathcal{Y}_i^j \in \Theta_i \mapsto \tilde{\mathcal{Q}}_i^j \doteq \{q_k\in\tilde{\mathcal{Q}} : y_i(k)\in \mathcal{Y}_i^j\}$ mapping EM feature clusters to sets of sample training poses in the approximate pose space $\tilde{\mathcal{Q}}$.
\end{enumerate}
\end{enumerate}

As the applications (\ref{eq:rhoi}) map approximate feature spaces to disjoint partitions of the approximate pose space $\tilde{\mathcal{Q}}$ they are refinings, and $\tilde{\mathcal{Q}}$ is a common refinement (Definition \ref{def:1}) for the collection of approximate feature spaces $\Theta_1,...,\Theta_N$. 

The collection of FODs $\tilde{\mathcal{Q}}, \Theta_1,...,\Theta_N$ along with the refinings $\rho_1,...,\rho_N$ is characteristic of: the object to track, the chosen features functions $y_i$, and the actual training data.\\ We call it the \emph{evidential model} (Figure \ref{fig:model}) of the object.

% -------------------------------------------------------------------------------------------------------------------
% SUBSECTION : pose estimation
% ---------------------------------------------------------------------------------------------------------------------

\section{Regression} \label{sec:estimation}

Once an evidential model has been learned from the available training set, it can be used to provide robust estimates of the pose of the moving object when new evidence becomes available.

\subsection{Dirichlet belief function modeling of soft assignments} \label{sec:mf}

When one or more test images are acquired, new visual features $y_1,...,y_N$ are extracted. Such feature values can be mapped by the learnt universal mappings (\ref{eq:mapping-bmr}) to a collection of belief functions $b_1,...,b_N$ on the set of training poses $\tilde{\mathcal{Q}}$. From Chapter \ref{cha:state}, Section \ref{sec:credal-sets}, each $b_i$ corresponds to a convex set of probability distributions, whose width encodes the uncertainty on the pose value due to the uncertainty on the analytical form of the true, unknown feature-pose map $\rho^*_i$.

In addition, belief functions allow us to take into account the scarcity of the training samples, by introducing uncertainty on the soft assignment (\ref{eq:soft-assignment}) itself. This can be done by assigning some mass $m(\Theta_i)$ to the whole approximate feature space, prior to applying the refining $\rho_i$. This encods the fact that there are other samples out there which, if available, would alter the shape of the MoG approximation of $\mathcal{Y}_i$ in unpredictable ways.\\ Namely, we map the soft assignment (\ref{eq:soft-assignment}) to a \emph{Dirichlet belief function} \cite{jsang06normalising}, with basic probability assignment:
\begin{equation}\label{eq:dirichlet}
m_i : 2^{\Theta_i} \rightarrow [0,1], \;\;\; m_i(\mathcal{Y}_i^j) = \frac{\displaystyle \Gamma_i^j(y_i)}{\displaystyle \sum_k \Gamma_i^k(y_i)} \big ( 1 - m_i(\Theta_i) \big ).
\end{equation}
The b.p.a. (\ref{eq:dirichlet}) `discounts' \cite{jiang08new,bell93discounting} the probability distribution obtained by simply normalizing the likelihoods (\ref{eq:soft-assignment}) by assigning some mass $m_i(\Theta_i)$ to the entire FOD $\Theta_i$.

As we need to discount the limited accuracy achieved by using as coordinates in $\mathcal{Y}_i$ those derived by the MoG representation $\Theta_i$, a plausible choice is
\[
m_i(\Theta_i) = \frac{1}{n_i}.
\]
Indeed, when $n_i \rightarrow \infty$ the discount factor tends to zero, and the approximate feature space converges (in theory) to the real thing. In addition, as $n_i$ cannot be greater than the number of training pairs $T$, such a discounting factor also takes into account the limited number of training samples.
\begin{figure}[ht!]
\begin{center}
\vspace{-0mm}
\includegraphics[width = \textwidth]{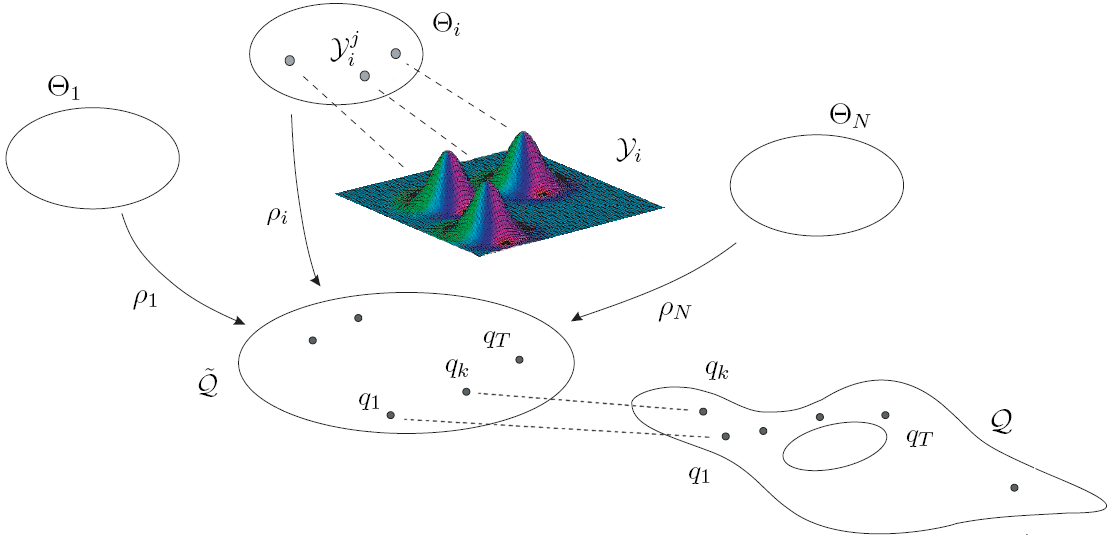}
\vspace{-0mm}
\end{center}
\caption{\label{fig:model} Evidential model. The EM clustering of each feature set collected in the training stage yields an approximate feature space $\Theta_i = \{ \mathcal{Y}_i^j, j=1,...,n_i \}$. Refining maps $\rho_i$ between each approximate feature space and $\tilde{\mathcal{Q}} = \{ q_1,...,q_T\}$ (the training approximation of the unknown pose space $\mathcal{Q}$) are learned, allowing at test time the fusion on $\tilde{\mathcal{Q}}$ of the evidence gathered on $\Theta_1,...,\Theta_N$.} \vspace{-0mm}
\end{figure}

\subsection{Cue integration} \label{sec:combination}

If we assume that the belief functions induced by test feature values are generated by `independent' sources of information they can be combined by means of Dempster's rule of combination (Definition \ref{def:dempster} or \cite{cuzzolin04smcb}), or its TBM variant the conjunctive rule of combination \cite{smets94transferable}.
\begin{definition} \label{def:conjunctive}
The \emph{conjunctive combination} of two belief functions $b_1,b_2:2^\Theta \rightarrow [0,1]$ is a new belief function $b_1 \Ocap b_2$ on the same FOD whose focal elements are all the possible intersections of focal elements
of $b_1$ and $b_2$ respectively, and whose b.p.a. is given by:
\begin{equation}\label{eq:conjunctive}
m_{b_1 \Ocap b_2}(A) = \sum_{B \cap C = A} m_{b_1}(B) \; m_{b_2}(C).
\end{equation}
\end{definition}
Definition \ref{def:conjunctive}, just like Dempster's rule, can be extended to the combination of an arbitrary number of belief functions.

While it is axiomatically justifiable as the only combination rule which meets a number of sensible requirements such as least commitment, specialization, associativity and commutativity \cite{smets07analyzing}, the conjunctive combination also amounts to assuming that the sources of evidence to merge are both reliable and independent. The current consensus is that different combination rules \cite{sentz02tech,smets93dis} are to be employed under different assumptions \cite{smets07analyzing}.\\
It is rather difficult, however, to decide in which situations the sources of information can indeed be considered independent: this is the case for features extracted from one or more views of the same object. An alternative point of view, supported by Shenoy, maintains instead that rather than employing a battery of combination rules whose applicability to a given problem is difficult to establish, we should adopt models which do meet the independence of sources assumption, as it happens in probability theory. We support this view here, and will test the adequacy of the assumption empirically in Section \ref{sec:results}.

\subsection{Belief estimate} \label{sec:belief-estimate}

The measurement belief functions 
\[
\big \{ b_i:2^{\Theta_i} \rightarrow [0,1], i=1,...,N \big \} 
\]
inferred from the test feature values $y_1,...,y_N$ via (\ref{eq:dirichlet}) are then mapped to belief functions 
\[
\big \{ b'_i:2^{\tilde{\mathcal{Q}}} \rightarrow [0,1] , i=1,...,N \big \} 
\]
on the approximate pose space $\tilde{\mathcal{Q}}$ by vacuous extension (recall Chapter \ref{cha:toe}, Definition \ref{def:vacuous}): $\forall A \subset \tilde{\mathcal{Q}}$
\begin{equation} \label{eq:vacuous-pose}
m'_i (A) = \left \{ \begin{array}{ll} m_i(A_i) & \exists A_i \subset \Theta_i \; s.t. \; A = \rho_i(A_i);  \\ 0 & otherwise. \end{array} \right.
\end{equation}
The resulting b.f.s on $\tilde{\mathcal{Q}}$ are combined by conjunctive combination (\ref{eq:conjunctive}). The result is a belief function $\hat{b} = b'_1 \Ocap \cdots \Ocap b'_N$ on $\tilde{\mathcal{Q}}$ which is right to call the \emph{belief estimate} of the object pose.

\subsubsection{Example} \label{sec:estimate-example}

It is important to understand how sophisticated a description of the object's pose a belief function is, as opposed to any estimate in the form of a `precise' probability distribution (including complex multi-modal or particle-based descriptions based on Monte-Carlo methods). A belief estimate $\hat{b}$ of the pose represents indeed an entire convex collection of probabilities (credal set) on the approximate pose space (see Chapter \ref{cha:state}, Section \ref{sec:credal-sets}). 

Suppose that the approximate pose space contains just three samples: $\tilde{\mathcal{Q}} = \{ q_1, q_2, q_3 \}$. Suppose also that the evidence combination process delivers a belief estimate $\hat{b}$ with b.p.a.:
\begin{equation}\label{eq:ex-bpa}
\begin{array}{ccccc}
\hat{m}(\{q_1,q_2\}) = 1/3, & & \hat{m}(\{q_3\}) = 1/6, & & \hat{m}(\{q_1,q_2,q_3\}) = 1/2.
\end{array}
\end{equation}
\begin{figure}[ht!]
\vspace{-0mm}
\begin{center}
\includegraphics[width = 0.75 \textwidth]{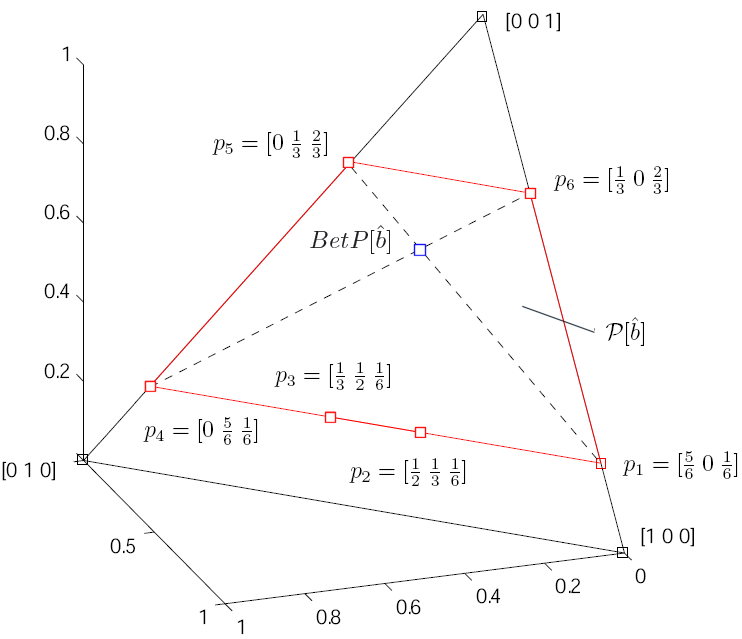}
\vspace{-0mm}
\end{center}
\caption{The convex set of probability distributions $\mathcal{P}[\hat{b}]$ (in red) associated with the belief function $\hat{b}$ (\ref{eq:ex-bpa}) on the approximate parameter space
$\tilde{\mathcal{Q}} = \{q_1,q_2,q_3\}$, displayed on the triangle of all probability distributions on $\tilde{\mathcal{Q}}$. The pignistic approximation $BetP[\hat{b}]$ (the center of mass of $\mathcal{P}[b]$, in blue) is also shown.\label{fig:ex-convex}} \vspace{-0mm}
\end{figure}
By (\ref{eq:prho}), the vertices of $\mathcal{P}[\hat{b}]$ are those probabilities generated by reassigning the mass of each focal element to any one of its singletons. There are $\prod_k |A_k|$ such possible choices, where $\{A_k\}$ is the list of focal elements of $\hat{b}$. As our belief estimate (\ref{eq:ex-bpa}) has 3 focal events of size $1,2$ and 3, the corresponding credal set $\mathcal{P}[\hat{b}]$ will be the convex closure of $1 \cdot 2 \cdot 3 = 6$ probability distributions, namely:
\[
\begin{array}{cccccc}
p_1 = [ \frac{5}{6} \; 0 \; \frac{1}{6} ], & p_2 = [ \frac{1}{2} \; \frac{1}{3} \;
\frac{1}{6} ], & p_3 = [\frac{1}{3} \; \frac{1}{2} \; \frac{1}{6} ], &
p_4 = [0 \; \frac{5}{6} \; \frac{1}{6} ], & p_5 = [0 \; \frac{1}{3} \; \frac{2}{3}
], & p_6 = [\frac{1}{3} \; 0 \; \frac{2}{3} ].
\end{array}
\]
%For example, $p_4$ is obtained by assigning the mass of $\{q_1,q_2\}$ to $q_2$, that of $\{q_1,q_2,q_3\}$ also to
%$q_2$ (whose probability is then $1/2 + 1/3 = 5/6$) and that of $\{q_3\}$ ($1/6$), of course, to $q_3$ itself.\\
Figure \ref{fig:ex-convex} shows the credal set (a polygon) associated with the belief estimate (\ref{eq:ex-bpa}) in the simplex of all probability distributions on $\tilde{\mathcal{Q}}$ (a triangle in this case). Here each probability distribution is represented as a point of $\mathbb{R}^3$. The larger the polygon, the greater the uncertainty of the estimate. In the figure $\mathcal{P}[\hat{b}]$ covers almost all the probability simplex, displaying a large degree of imprecision of the estimate due to lack of evidence.

\subsection{Computing expected pose estimates} \label{sec:mean-pose-estimates}

Point-wise information on the object's pose can be extracted from $\hat{b}$ in two different ways.

\subsubsection{Extracting a set of extremal point-wise estimates} \label{sec:extremal}

Each of the vertices (\ref{eq:prho}) of the credal set associated with the belief estimate $\hat{b}$ is a probability distribution on the approximate pose space $\tilde{\mathcal{Q}}$. We can then compute the associated expected pose as:
\begin{equation}\label{eq:mean-value}
\hat{q} = \sum_{k = 1}^T p(q_k) q_k.
\end{equation}
The set of such `extremal' estimates describes therefore an entire polytope of expected pose values in the object's pose space $\mathcal{Q}$. In the example, the expected poses for the vertices $p_1,p_4,p_5,p_6$ of $\mathcal{P}[b]$ are: 
\[
\hat{q}[p_1] = \frac{5}{6} q_1 + \frac{1}{6} q_3, \hspace{5mm} \hat{q}[p_4] = \frac{5}{6} q_2 + \frac{1}{6} q_3, \hspace{5mm} \hat{q}[p_5] = \frac{1}{3} q_2 + \frac{2}{3} q_3, \hspace{5mm} \hat{q}[p_1] = \frac{1}{3} q_1 + \frac{2}{3} q_3.
\]

\subsubsection{Extracting a point-wise estimate} \label{sec:pointwise}

An alternative way to extract a \emph{single} pose estimate $\hat{q}$ from the belief estimate consists in approximating $\hat{b}$ with a probability $\hat{p}$ on $\tilde{\mathcal{Q}}$, and then computing its mean value as above. The problem has been indeed extensively studied. In particular, Smets' \emph{pignistic function} \cite{smets05ijar} (see Chapter \ref{cha:state}, Equation (\ref{eq:pignistic-probability})):
\begin{equation}\label{eq:pignistic}
BetP[b](x) = \sum_{A \supseteq x} \frac{m_b(A)}{|A|} \;\;\; \forall x \in \Theta
\end{equation}
has been proposed within the framework of the Transferable Belief Model (\cite{smets05ijar}, Chapter \ref{cha:state} Section \ref{sec:tbm}) as the unique transformation which meets a number sensible of rationality principles. Geometrically, $BetP$ is nothing but the barycenter of the convex set of probabilities $\mathcal{P}[b]$ associated with $b$ (see Figure \ref{fig:ex-convex}). As such, it is quite consistent with the interpretation of belief functions as credal sets of probabilities.\\ Although other transforms such as the `relative plausibility of singletons' \cite{voorbraak89efficient,cuzzolin10amai,cuzzolin12ijar} and the `intersection probability' \cite{cuzzolin07smcb} have been proposed (compare Chapter \ref{cha:state}, Section \ref{sec:transformation}), the performances of the different pointwise transformations in the human pose tests presented here have been proven to be empirically comparable.

In the following, therefore, we will simply adopt the pignistic transform.

\subsection{Handling of conflict} \label{sec:conflict}

The mass the conjunctive combination (\ref{eq:conjunctive}) assigns to the empty set measures the extent to which the pieces of evidence to combine are in conflict. In the case of the evidential model, this mass is assigned by the input b.f.s (learned from the available feature values) to contradictory (disjoint) focal elements. 

In our pose estimation scenario, conflict can arise when combining feature evidence via $b'_1 \Ocap \cdots \Ocap b'_N$ for basically two reasons:
\begin{enumerate}
\item
the object is localized in an imprecise way (due to limitations of the trained detector), so that background features conflicting with the foreground information are also extracted; 
\item
occlusions are present, generating conflict for similar reasons. 
\end{enumerate}
A critical case is that in which all the focal elements of a particular measurement belief function have empty intersection with those of the other b.f.s to combine -- all the mass is assigned to $\emptyset$, and no estimation is possible.

When modeling the scarcity of training pairs via Dirichlet belief functions (Section \ref{sec:mf}), however, this extreme scenario never materializes, as each individual b.f. always has $\Theta_i$ as a focal element. In case of disagreement, then, some mass is always assigned to the focal elements of the remaining belief functions. As we argue in Section \ref{sec:inference}, combining Dirichlet belief functions amounts to assume that all the \emph{partial} combinations of feature evidence should be given some credit. Maybe, the reasoning goes, only a \emph{subset} of features is telling the truth \cite{schubert98fast}. Under the assumption that most features come from the foreground, this brings robustness to localization errors and presence of occlusions. 

In the following, therefore, we do not employ any explicit conflict resolution mechanism.

\subsection{Pose estimation algorithm} \label{sec:alg-pose}

Let us summarize the whole pose estimation procedure. Given an evidential model of the moving body with $N$ feature spaces, and given at time $t$ one or more test images, possibly coming from different cameras:

\begin{enumerate}
\item
the object detector learned during training is applied to the test image(s), returning for each of them a bounding box roughly containing the object of interest;
\item
$N$ feature values are extracted from the resulting bounding boxes, as during training;
\item
the likelihoods $\{\Gamma_i^j(y_i(t)), j=1,...,n_i\}$ of each feature value $y_i(t)$ with respect to the appropriate learned Mixture of Gaussian distribution on $\mathcal{Y}_i$ are computed (\ref{eq:soft-assignment});
\item
for each feature $i = 1,...,N$, a separate belief function
\[
b_i(t):2^{\Theta_i}\rightarrow [0,1] 
\]
on the appropriate feature space $\Theta_i$ is built from the set of likelihoods $\{\Gamma_i^j(y_i(t)), j = 1,...,n_i\}$ as in
Section \ref{sec:mf};
\item
all the resulting b.f.s $\{ b_i(t):2^{\Theta_i}\rightarrow [0,1] ,i=1,...,N \}$ are projected onto $\tilde{\mathcal{Q}}$ by vacuous extension (\ref{eq:vacuous-pose}), yielding a set of belief functions on $\tilde{\mathcal{Q}}$:
\[
\big \{ b'_i:2^{\tilde{\mathcal{Q}}} \rightarrow [0,1] , i=1,...,N \big \}
\]
\item
their conjunctive combination $\hat{b}(t) \doteq b'_1(t) \Ocap \cdots \Ocap b'_N(t)$ is computed via (\ref{eq:conjunctive});
\item
either:
\begin{enumerate}
\item the pignistic transform (\ref{eq:pignistic}) is applied to $\hat{b}(t)$, yielding a distribution on $\tilde{\mathcal{Q}}$ from which an expected pose estimate $\hat{q}(t)$ is obtained by (\ref{eq:mean-value}), or:
\item
the vertices (\ref{eq:prho}) of the convex set of probabilities $\mathcal{P}[\hat{b}(t)]$ associated with the current belief estimate $\hat{b}(t)$ are computed, and a mean pose estimate (\ref{eq:mean-value}) obtained for each one of them.
\end{enumerate}
\end{enumerate}

% -------------------------------------------- COMPUTATIONAL COST

\subsection{Computational cost} \label{sec:computational}

\subsubsection{Learning} 

EM's computational cost is easy to assess, as the algorithm usually takes a constant number of steps to converge, $c \sim 5-10$, while at each step the whole observation sequence of length $T$ is processed, yielding $O(cNnT)$ (where again $N$ is the number of features, $n$ the average number of EM clusters, $T$ the number of samples collected in the training stage). This is quite acceptable for real-world applications, since this has to be done just once in the training session.\\ In the experiments of Section \ref{sec:results} the whole learning procedure in Matlab required some 17.5 seconds for each execution of EM on a rather old Athlon 2.2 GHz processor with $N = 5$ features, $n_i = n = 5$ states for each feature space, and $T = 1726$.

\subsubsection{Estimation}

Although the conjunctive combination (\ref{eq:conjunctive}) is exponential in complexity if naively implemented, fast implementations of $\Ocap$ exist, under additional constraints \cite{Denoeux02ijar}. Numerous approximation schemes have been proposed, based on Monte-Carlo techniques \cite{moral99montecarlo}. Furthermore, the particular form of the belief functions we use in the estimation process needs to be taken into account. Dirichlet b.f.s (\ref{eq:dirichlet}) have $n_i + 1$ non-zero focal elements, reducing the computational complexity of their pairwise combination from $O(2^{2n})$ (associated with the mass multiplication of all possible $2^n$ focal elements of the first b.f. and all the focal elements of the second b.f.) to $O(n^2)$. The computational cost of the other steps of the algorithm is negligible when compared to that of belief combination.

% -------------------------------------------- THEORETICAL ASSESSMENT

\section{Assessing evidential models} \label{sec:assessing}

A number of aspects of the evidential model architecture are strictly related to fundamental questions of the example-based pose estimation problem: 
\begin{enumerate}
\item
whether the model is self-consistent, i.e., whether it produces the correct ground truth pose values when presented with the training feature data; 
\item
what resolutions $\{n_i, i = 1,...,N\}$ of the features' MoG representations are adequate to guarantee a sufficient accuracy of the learned feature-pose mapping, and through the latter of the estimation process itself; 
\item
whether the training set of poses $\mathcal{\tilde{Q}}$ is a proper approximation of the unknown parameter space $\mathcal{Q}$ (see Figure \ref{fig:model}).
\end{enumerate}

As it turns out, those issues are related to discussing, respectively: 1) whether $\mathcal{\tilde{Q}}$ is the minimal refinement (Theorem \ref{the:minimal}) of the approximate feature spaces $\Theta_i$; 2) whether the selected features space are independent, in a way which we will precise in the following; 3) whether a flag can be derived to indicate the need to update the evidential model by adding more training poses.

\subsection{Model consistency and $\tilde{\mathcal{Q}}$ as minimal refinement} \label{sec:model-resolution}

In order for the model to return the correct ground truth pose when presented with a set of training feature values $\{y_i(k), i=1,...,N\}$ it is necessary that each sample in the training set $\mathcal{\tilde{Q}}$ be characterized by a distinct set of feature MoG components. Namely, no two training poses $q_1, q_2$ are allowed to be associated with feature components falling in the same cluster for each approximate feature space: 
\[
\not\exists q_1,q_2 \; s.t. \; y_i(q_1), y_i(q_2) \in \mathcal{Y}_i^{j_i} \; \forall i 
\]
for the same $j_1,...,j_N$.  

Imagine that the $N$ feature vector components $y_1,...,y_N$ generated by a test image are such that:
$
y_1 \in \mathcal{Y}_1^{j_1}, \cdots, y_N \in \mathcal{Y}_N^{j_N}.
$
Each piece of evidence $y_i \in \mathcal{Y}_i^{j_i}$ implies that the object's pose lies within the subset $\rho_i(\mathcal{Y}_i^{j_i})$ of the training set $\tilde{\mathcal{Q}}$. The estimated pose must then fall inside the set:
\begin{equation}\label{eq:intersections}
\rho_1(\mathcal{Y}_1^{j_1}) \cap \cdots \cap \rho_N(\mathcal{Y}_N^{j_N}) \subset
\tilde{\mathcal{Q}}.
\end{equation}
Sample object poses in the same intersection of the above form are \emph{indistinguishable} under the given evidential model. The collection of all the non-empty intersections of the form (\ref{eq:intersections}) is nothing but the {minimal refinement} $\Theta_1 \otimes \cdots \otimes \Theta_N$ of the FODs $\Theta_1,...,\Theta_N$ (recall Theorem \ref{the:minimal}). 

It follows that:
\begin{theorem}
Any two poses of the training set can be distinguished under the evidential model iff $\tilde{\mathcal{Q}}$ is the minimal refinement of $\Theta_1,...,\Theta_N$.
\end{theorem}
\begin{proof}

$\Rightarrow$: if any two sample poses can be distinguished under the model, i.e., for all $k, k'$
\begin{equation}\label{eq:proof1}
q_{k'} \not \in \rho_1(\mathcal{Y}_1^{j_1}) \cap \cdots \cap \rho_N(\mathcal{Y}_N^{j_N}) \ni q_k,
\end{equation}
it follows that each intersection of the form (\ref{eq:intersections}) cannot contain more than one sample pose, otherwise there would exist a pair violating (\ref{eq:proof1}) (note that the intersection can instead be empty). Furthermore, each sample pose $q_k$ falls within such an intersection, the one associated with the visual words $\mathcal{Y}_1^{j_1}, \cdots, \mathcal{Y}_N^{j_N}$ s.t. $y_1(q_k) \in \mathcal{Y}_1^{j_1}$, ..., $y_N(q_k) \in \mathcal{Y}_1^{j_N}$. Hence, the minimal refinement of $\Theta_1$,...,$\Theta_N$ has as elements (\ref{eq:intersections}) all and only the individual sample poses (elements of $\tilde{\mathcal{Q}}$): therefore, $\tilde{\mathcal{Q}} = \Theta_1 \otimes \cdots \otimes \Theta_N$.

$\Leftarrow$: if $\tilde{\mathcal{Q}}$ is the minimal refinement of $\Theta_1$,...,$\Theta_N$ then for all $q_k\in \tilde{\mathcal{Q}}$ we have that $\{ q_k \} = \rho_1(\mathcal{Y}_1^{j_1}) \cap \cdots \cap \rho_N(\mathcal{Y}_N^{j_N})$ holds for some unique selection of feature components $\mathcal{Y}_1^{j_1}, \cdots, \mathcal{Y}_N^{j_N}$, distinct for each training pose. Any two different sample poses belong therefore to different intersections of the form (\ref{eq:intersections}), i.e., they can be distinguished under the model.
\end{proof}

The `self-consistency' of the model can then be measured by the ratio between the cardinality of the minimal refinement of $\Theta_1,...,\Theta_N$, and that of the actual approximate parameter space $\tilde{\mathcal{Q}}$:
\[
\frac{1}{T} \leq \frac{|\bigotimes_i \Theta_i|}{|\tilde{\mathcal{Q}}|} \leq 1.
\]
It is hence desirable, in the training stage, to select a collection of features which brings the minimal refinement $\Theta_1\otimes \cdots \otimes \Theta_N$ as close as possible to $\tilde{\mathcal{Q}}$: sometimes the addition of new features will be desirable in order to resolve any ambiguities.

\subsection{Complementarity of features and model quantization} \label{sec:model-complementarity}

When the approximate feature spaces $\Theta_i$ are independent (Definition \ref{def:indep}), for each combination of feature clusters $\mathcal{Y}_1^{j_1}, \cdots, \mathcal{Y}_N^{j_N}$ there exists a {unique} sample pose $q_k$ characterized by feature values in those clusters: 
\begin{equation} \label{eq:combination}
\{ q_k \} = \rho_1(\mathcal{Y}_1^{j_1}) \cap \cdots \cap \rho_N(\mathcal{Y}_N^{j_N}). 
\end{equation}
In this case different cues carry \emph{complementary} pieces of information about the object's pose -- to resolve an individual sample pose $q_k$ you need to measure all its feature values.\\ When the approximate feature spaces are \emph{not} independent, on the other hand, two situations may materialize: while in some cases fewer than $N$ feature values may be enough to resolve some training poses, in general each combination of feature values will yield a whole set of training poses.

If the model is self-consistent ($|\tilde{\mathcal{Q}}|=|\otimes_i \Theta_i|$, Section \ref{sec:model-resolution}) and the chosen features are complementary (i.e., they are such that $|\otimes_i \Theta_i| = \prod_i |\Theta_i|$), we have that $T = |\tilde{\mathcal{Q}}| \sim n_1\times ... \times n_N$: assuming $n_i = const = n$ this yields $n \sim \sqrt[N]{T}$. Given a realistic sampling of the parameter space with $T = 20000$, the use of $N = 9$ complementary features allows us to require no more than $\sqrt[9]{20000} \sim 3$ MoG components for each feature space in order to ensure a decent accuracy of the estimate.

This shows the clear advantage of encoding feature-pose maps \emph{separately}: as long as the chosen features are uncorrelated, a relatively coarse MoG representation for each feature space allows us to achieve a reasonable resolution in terms of pose estimates\footnote{In analogy to what proposed in \cite{viola01cvpr} or \cite{Meynet08}, where trees of classifiers are used for face pose estimation.}.

\subsection{On conflict and the relation between approximate and actual pose space} \label{sec:distribution}

Ideally, the set $\tilde{\mathcal{Q}}$ of training poses, as an approximation of the actual pose space $\mathcal{Q}$, should be somehow `dense' in $\mathcal{Q}$: $\forall q\in{{\mathcal{Q}}}$ there should be a sample $q_k$ such that $\|q-q_k\|<\epsilon$ for some $\epsilon$ small enough. Clearly, such a condition is hard to impose.\\ The distribution of the training poses within $\mathcal{Q}$ has nevertheless a number of consequences on the estimation process:

1) as the true pose space $\mathcal{Q}$ is typically non-linear, while the pose estimate is a linear combination of sample poses (see Section \ref{sec:belief-estimate}), the pointwise estimate can be non-admissible (fall outside $\mathcal{Q}$). This can be fixed by trying to make the feature spaces independent, as in that case every sample pose $q_k$ is characterized by a different combination (\ref{eq:combination}) of feature clusters.\\ Under this assumption, any set of test feature values $y_1 \in \mathcal{Y}_1^{j_1},..., y_N \in \mathcal{Y}_N^{j_N}$ generates a belief estimate in which a single sample pose $q_k$ is dominant. As a consequence, its credal set (Section \ref{sec:extremal}) is of limited extension around a single sample pose, and the risk of non-admissibility is reduced.

2) there can exist regions of $\mathcal{Q}$ characterized by combinations of feature clusters which are not in the model:
\[
\exists q \in \mathcal{Q} : \forall \mathcal{Y}_1^{j_1} \in \Theta_1, ..., \mathcal{Y}_N^{j_N} \in \Theta_N \;\;\; q \not \in \rho_1(\mathcal{Y}_1^{j_1}) \cap \cdots \cap \rho_N(\mathcal{Y}_N^{j_N}).
\]
This generates high level of conflict $m(\emptyset)$ in the conjunctive combination (\ref{eq:conjunctive}) (although combination is always guaranteed for Dirichlet belief functions, see above), a flag of the inadequacy of the current version of the evidential model. This calls, whenever new ground truth information can be provided, for an {update of the model} by incorporating the sample poses causing the problem.

% ---------------------------------------------------------------------------------------------------------------------
% SECTION : experiments
% ---------------------------------------------------------------------------------------------------------------------

\section{Results on human pose estimation} \label{sec:results}

We tested our Belief Modeling Regression technique in a rather challenging setup, involving the pose estimation of human arms and legs from two well separated views. While the bottom line of the evidential approach is doing the best we can with the available examples, regardless the dimensionality of the pose space, and without having at our disposal prior information on the object at hand, we ran test on articulated objects (one arm and a pair of legs) with a reasonably limited number of degrees of freedom to show what can be achieved in such a case. We demonstrate how the BMR technique outperforms competitors such as Relevant Vector Machines and GPR.

\subsection{Setup: two human pose estimation experiments} \label{sec:experimental-setup}

To collect the necessary ground truth we used a marker-based motion capture system \cite{Rosales00,Howe99} built by E-motion, a Milan firm. The number of markers used was 3 for the arm (yielding a pose space $\mathcal{Q}\subset \mathbb{R}^9$, using as pose components the 3D coordinates of the marker), and 6 for the pair of legs ($\mathcal{Q}\subset \mathbb{R}^{18}$). The person was filmed by two uncalibrated DV cameras (Figure \ref{fig:capture}).\\ In the training stage of the first experiment we asked the subject to make his arm follow a trajectory (approximately) covering the pose space of the arm itself, keeping his wrist locked and standing on a fixed spot on the floor to limit the intrinsic dimensionality of the pose space (resulting in 2 d.o.f.s for the shoulder and 3 for the elbow).\\ In the second experiment we tracked the subject's legs, assuming that the person was walking normally on the floor, and collected a training set by sampling a random walk on a small rectangular section of the floor. This is similar to what is done in other works, where the set of examples are taken for a particular family of motions/trajectories, normally associated with action categories such as the walking gait.\\ The length of the training sequences was 1726 frames for the arm and 1952 frames for the legs.

While the number of degrees of freedom was limited by constraining the articulated object (person) to performing motions of a specific class (walking versus brandishing an arm), the tests are sufficiently complex to allow us to illustrate the traits of the BMR approach to pose estimation. In addition, in both experiments the background was highly non-static, with people coming in and out the scene and flickering monitors. The object of interest would also occlude itself a number of times on at least one of the two views (e.g. sometimes one leg would occlude the other when seen from the left camera), making the experimental setup quite realistic.

\subsection{Automatic annotation of training images} \label{sec:annotation}

    % localization annotation
Under the assumptions listed in Section \ref{sec:problem} in the training stage the images ought to be annotated via a bounding box, providing a rough localization of the unknown object.%, in order for the evidential model not to fit the background together with the foreground. An off-the-shelf object detector is then trained using the bounding boxes provided during training, and applied to test images to locate the object, in terms of a bounding box.
\begin{figure}[ht!]
\vspace{-0mm}
\begin{center}
\includegraphics[width = \textwidth]{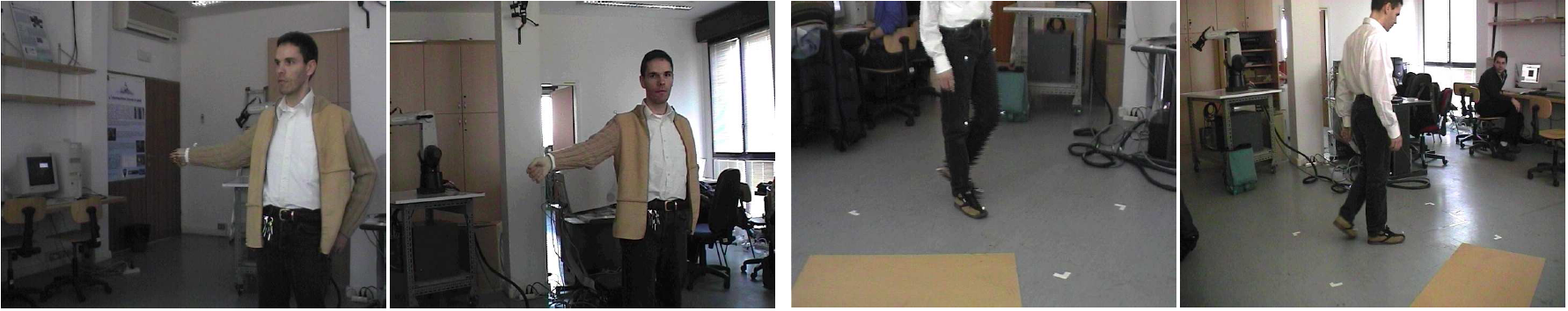}
\vspace{-0mm}
\end{center}
\caption{\label{fig:capture} Two human body-part pose estimation experiments. Left: training images of a person standing still and moving his right arm. Right: training images of the person walking inside a rectangle on the floor.}\vspace{-0mm}
\end{figure}

To simulate this annotation process, and isolate the performance of the proposed example based estimation approach from that of the object detector employed, in these tests we used color-based segmentation to separate the object of interest from the non-static background, implemented via a colorimetric analysis of the body of interest (Figure \ref{fig:feature}). Pixels were clustered in the RGB space; the cluster associated with the yellow sweater (in the arm experiment) or the black pants (legs one) was detected, and pixels belonging to that cluster assigned to the foreground. Finally, the minimal bounding box containing the  silhouette of the segmented foreground pixels was detected. 

Note that \emph{this is just a way of automatically generate, rather than manually construct, the bounding box annotation required in the assumptions of the initial scenario}: the notion that no a-priori information on the object of interest needs to be employed still holds.

\subsection{Feature extraction and modeling} \label{sec:features}

For these tests we decided to build an extremely simple feature vector for each image directly from the bounding box, as the collection $\max(row)$, $\min(row)$, $\max(col)$, $\min(col)$ of the row and column indexes defining the box (Figure \ref{fig:feature}).
\begin{figure}[ht!]
\vspace{-0mm}
\begin{center}
\includegraphics[width = 1 \textwidth]{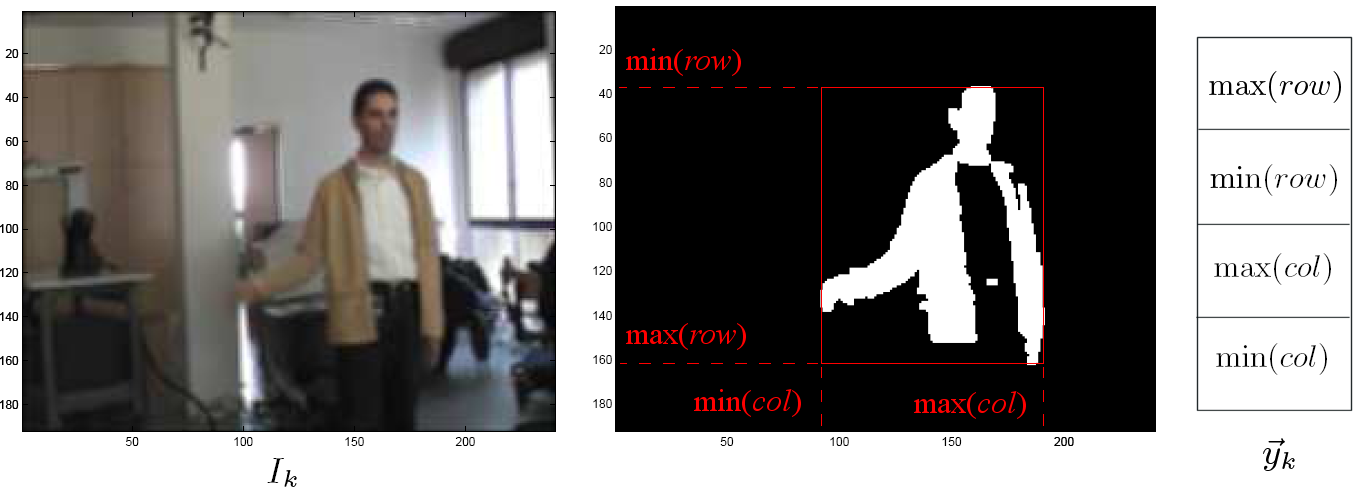}
\end{center}
\vspace{-0mm}
\caption{\label{fig:feature} Feature extraction process. Left: a training image $I_k$ in the arm experiment. Middle: the object of interest is color segmented and the bounding box containing the foreground is detected to simulate localization annotation. Right: the row and column indices of the vertices of the bounding box are collected in a feature vector $\vec{y}_k$.} \vspace{-0mm}
\end{figure}
As two views were available at all times, at each time instant two feature vectors of dimension 4 were computed from the two views.

    % different evidential models
In the arm experiment we built \emph{three} different evidential models from these vectors: one using $N=2$ features ($\max(row)$ and $\max(col)$) from the left view only, and a Mixture of Gaussians with $n_i = n = 5$ components for both feature spaces; a second model for the right view only, with $N=3$ feature spaces (associated with the components $\max(row)$, $\min(col)$ and $\max(col)$) and $n_i = n = 5$ MoG components for each feature space; an overall model in which both the 2 features from the left view and the 3 features from the right one were considered, yielding a model with $N=5$ feature spaces with the same MoG representation.

In the leg experiment, instead, we built two models with $N = 6$ feature spaces (the $\max(row)$, $\min(col)$ and $\max(col)$ feature components from both views), but characterized by a different number of Gaussian components ($n=4$ or $n=5$, respectively) to test the influence of the quantization level on the quality of the mapping (and therefore of the estimates).

\subsection{Performance} \label{sec:performance}

To measure the accuracy of the estimates produced by the different evidential models, we acquired a testing sequence for each of the two experiments and compared the results with the ground truth provided by the motion capture equipment. %As discussed in Section \ref{sec:assumptions}, the bounding box annotation provided in the training session can be used to train a localization algorithm, allowing us to extract features from the (estimated) foreground during testing. In these tests, this step has been simulated again via the colorimetric analysis exposed in Section \ref{sec:annotation}, so that the same features computed during training could be extracted from the resulting bounding boxes in the test images as well (Figure \ref{fig:feature} again).
\begin{figure}[ht!]
\vspace{-0mm}
\begin{center}
\begin{tabular}{cc}
\includegraphics[width = 0.48 \textwidth]{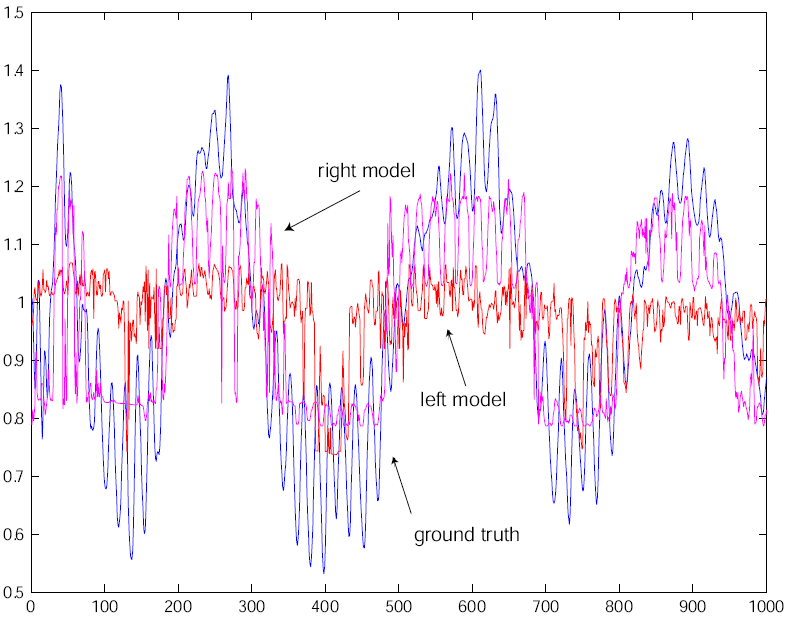}
&
\includegraphics[width = 0.48 \textwidth]{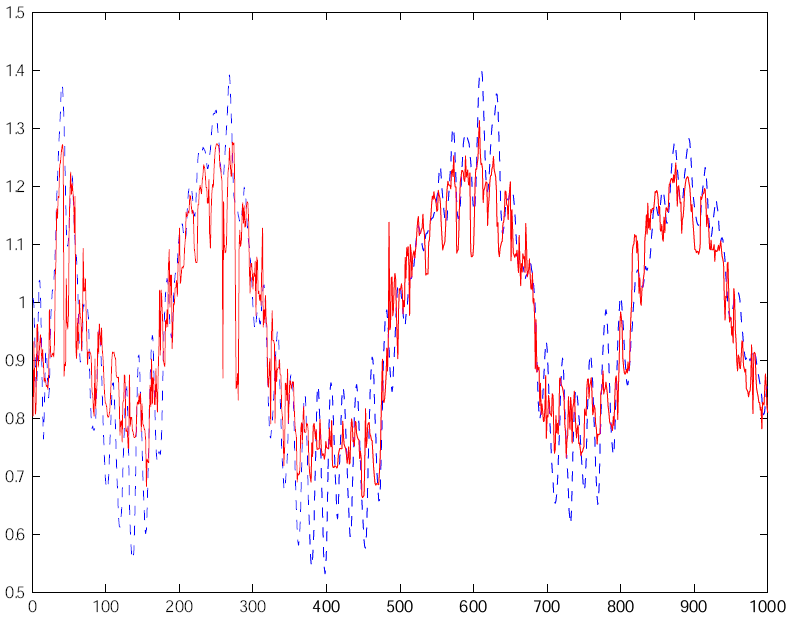}
\\
\includegraphics[width = 0.48 \textwidth]{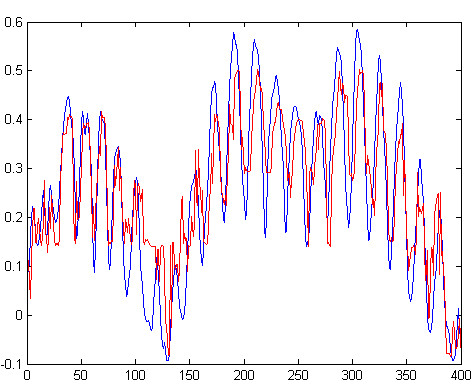}
&
\includegraphics[width = 0.48 \textwidth]{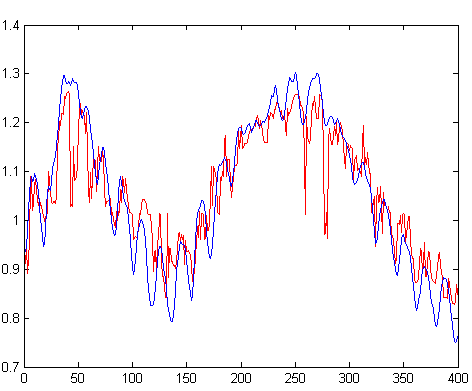}
\end{tabular}
\vspace{-0mm}
\end{center}
\caption{\label{fig:ambi} Top left: pose estimates of component 9 of the pose vector (Y coordinate of the hand marker) produced by the left (red) and right (magenta) model compared to the ground truth (blue), plotted against time. Top right: the sequence of pose estimates yielded by the overall model (which uses features computed in both left and right images) is plotted in (solid) red against the ground truth in (dashed) blue. Bottom: performance of the overall model on components 1 (left) and 6 (right) of the pose vector, for the first 400 frames of the test sequence. The pignistic function is here used to compute the pointwise estimates.} \vspace{-0mm}
\end{figure}

\subsubsection{Arm experiment} \label{sec:arm-experiment}

In the arm experiment the test sequence was 1000 frames long. Pointwise pose estimates were extracted from belief estimates via pignistic transform (\ref{eq:pignistic}). As the anecdotal evidence of Figure \ref{fig:ambi}-top left indicates, the estimates of the single-view models were of rather poor quality. Indeed, recalling the discussion of Section \ref{sec:model-resolution}, the minimal refinements $\bigotimes\Theta_i$ for the left-view and the right-view models were of size 22 and 80 respectively, signalling a poor model resolution. In opposition, the estimates obtained by exploiting image evidence from both views (Figure \ref{fig:ambi}-top right) were clearly better than a simple selection of the best partial estimate at each instant. This was confirmed by a minimal refinement $\bigotimes\Theta_i$ for the overall model with cardinality equal to 372 (the $N=5$ features encoded by a MoG with $n=5$ components were enough to resolve 372 of the 1700+ sample poses), with 139 sample poses individually resolved by some particular combination of the $N=5$ feature values. Figure \ref{fig:ambi}-bottom illustrates similar results for components 1 and 6 of the pose vector, in the same experiment.

We also measured the Euclidean distance between real and expected 3D locations of each marker over the whole testing sequence. For the arm experiment, the average estimation errors were 17.3, 7.95, 13.03, and 2.7 centimeters for the markers `hand', `wrist', `elbow' and `shoulder', respectively. As during testing the features were extracted from the estimated foreground, and no significant occlusions were present, the conflict between the different feature components was negligible throughout the test sequence.

\subsubsection{Lower and upper estimates associated with the credal estimate} \label{sec:interval-pose}

%as we argued in Section \ref{sec:belief-estimate}, the expressive power of belief estimates goes beyond the pointwise pose estimates one can obtain from them.
As each belief estimate $\hat{b}$ amounts to a convex set $\mathcal{P}[\hat{b}]$ of probability distributions on $\tilde{\mathcal{Q}}$, an expected pose estimate can be computed for each of its vertices (\ref{eq:prho}).
\begin{figure}[ht!]
\vspace{-0mm}
\begin{center}
\begin{tabular}{c}
\includegraphics[width = 0.9 \textwidth]{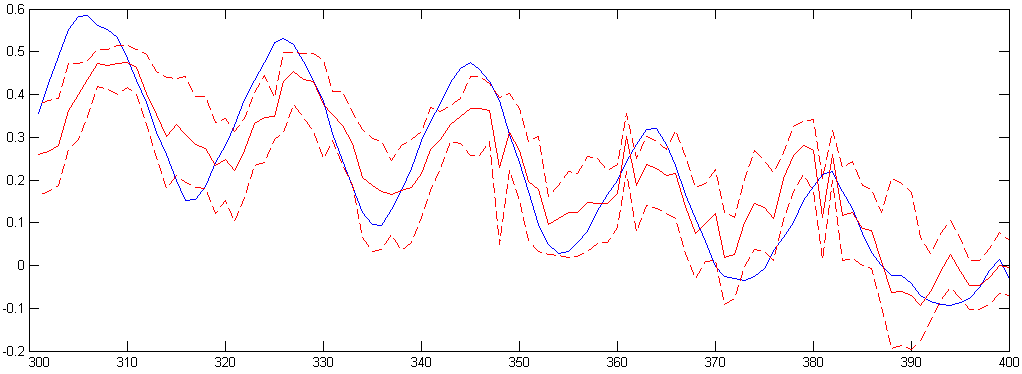} \vspace{-0mm} \\ \includegraphics[width = 0.9 \textwidth]{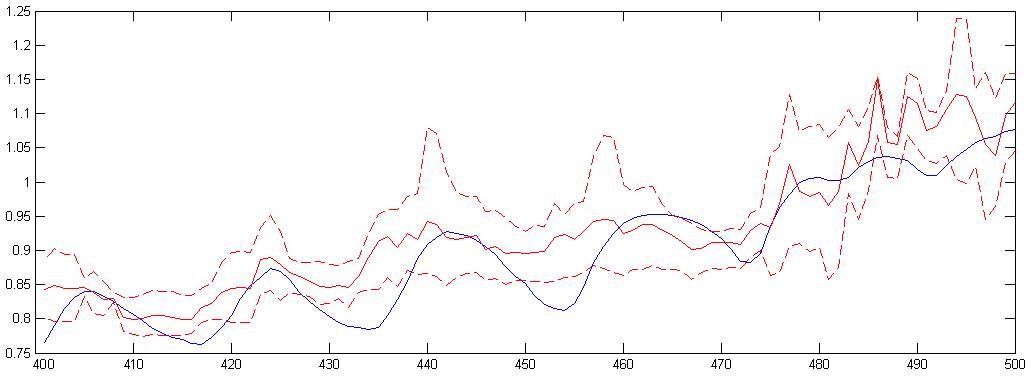} \vspace{-0mm} \\ \includegraphics[width = 0.9 \textwidth]{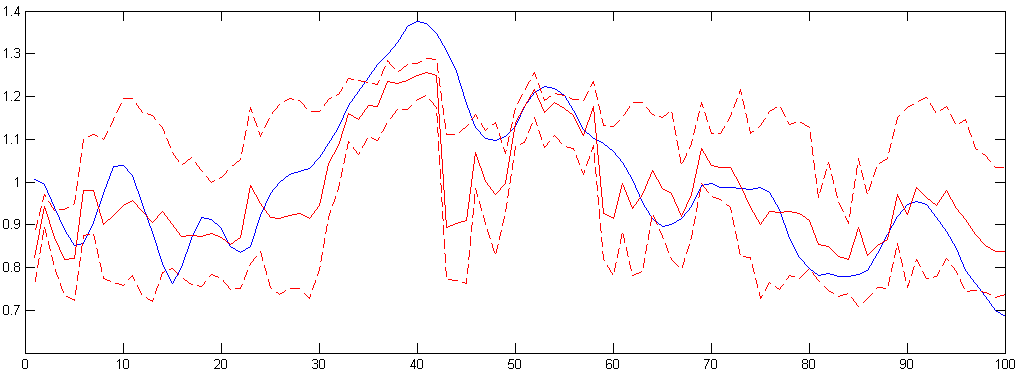} \vspace{-0mm}
\end{tabular}
\end{center} \vspace{-0mm}
\caption{\label{fig:interval} Plots of lower and upper expected pose estimates ((\ref{eq:mean-value}), in dashed red) generated by the credal sets associated with the sequence of belief estimates $\hat{b}(k)$, versus the pignistic estimate (solid red) and the ground truth (in blue). Top: component 1 of the pose vector, test sequence from $k=300$ to $k=399$. Middle: component 6, test frames from $k=400$ to $k=499$. Bottom: component 9, test frames from $k=1$ to $k=100$.} \vspace{-0mm}
\end{figure}
The BMR approach can therefore provide a robust pose estimate, for instance by computing for each instant $t$ the maximal and minimal expected value (over the vertices of $\mathcal{P}[\hat{b}]$) of each component of the pose vector.

Figure \ref{fig:interval} plots these upper and lower bounds to the expected pose values in the arm experiment, for three different components of the pose vector, over three distinct subsequences of the test sequence.\\ As it can be clearly observed, even for the rather poor (feature-wise) evidential model built here, most of the time the true pose falls within the provided interval of expected pose estimates. Quantitatively, the percentage of test frames in which this happens for the twelve pose components is 49.25\%, 44.92\%, 49.33\%, 50.50\%, 48.50\%, 48.33\%, 49.17\%, 54.42\%, 49.67\%, 51.50\%, 39.33\% and 43.50\%, respectively. We can also measure the average Euclidean distance between the true pose estimate and the \emph{boundary} of the interval of expected poses, for the four markers and along the entire test sequence: we obtain average 3D distances of 7.84cm, 3.85cm, 5.78cm and 2.07cm for the four markers, respectively. These give a better indication of the robustness of the BMR approach than errors measured with respect to a central expected pose estimate (see Figure \ref{fig:visual}-right for a comparison).
\begin{figure}[ht!]
\vspace{-0mm}
\begin{center}
\includegraphics[width = \textwidth]{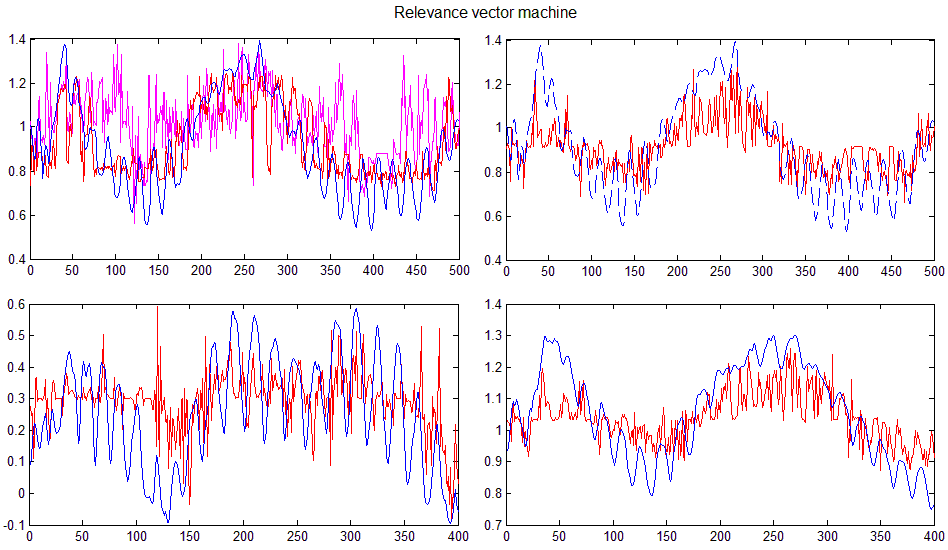}
\end{center}
\vspace{-0mm}
\caption{\label{fig:rvm} Top left: pose estimates of component 9 of the pose vector (Y coordinate of the hand) produced by an RVM using only the left (red) and right (magenta) features, compared to the ground truth (blue), plotted against time. Top right: pose estimates yielded by a RVM regression model which uses features computed in both left and right images plotted in (solid) red against the ground truth in (dashed) blue. Bottom: performance of the overall model on components 1 (left) and 6 (right) of the pose vector, for the first 400 frames of the test sequence.} \vspace{-0mm}
\end{figure}

Note that in these tests the pose estimate interval was computed \emph{using just a subset of the true vertices} of the belief estimate, for mere computational reasons. The true interval is indeed wider, and therefore associated with even lower average estimation errors. %Further discounting $\hat{b}$ based on the size $T$ of the training sequence would introduce even more caution and therefore robustness.

\subsubsection{Comparison with GP and RVM regression} It is interesting to compare BMR's performance with that of two well established regression approaches: {Gaussian Process Regression} \cite{bo09cvpr,rudovic11iccv} and \emph{Relevant Vector Machines} (RVMs) \cite{tipping01rvm}. The latter are used to build feature-pose maps in, for instance, \cite{agarwal06pami} and \cite{Thayananthan06eccv}. Figure \ref{fig:rvm} shows the estimates produced by a RVM on the same test sequences and components of Figure \ref{fig:ambi}. It is clear from a visual comparison of Figures \ref{fig:rvm} and \ref{fig:ambi} that our approach significantly outperforms a standard RVM implementation. Quantitatively, the average Euclidean distances between real and estimated 3D location of each marker over the whole arm testing sequence were, in the RVM tests, 31.2, 13.6, 23.0, and 4.5 centimeters for the markers `hand', `wrist', `elbow' and `shoulder', respectively.

Figure \ref{fig:gpr} shows instead the estimates produced by Gaussian Process Regression for the same experimental setting of Figures \ref{fig:rvm} and \ref{fig:ambi}. A visual inspection of Figures \ref{fig:gpr} and \ref{fig:ambi} shows a rather comparable performance with that of the BMR approach, although the partial models obtained from left and right view features only seem to perform relatively poorly.\\
Quantitatively, however, the average Euclidean distances between real and estimated 3D location of each marker over the whole arm testing sequence were, in the GPR tests, 25.0, 10.6, 18.6, and 7.0 centimeters for the markers `hand', `wrist', `elbow' and `shoulder', respectively, showing how our belief-theoretical approach clearly outperforms this competitor as well.
\begin{figure}[ht!]
\vspace{-0mm}
\begin{center}
\includegraphics[width = \textwidth]{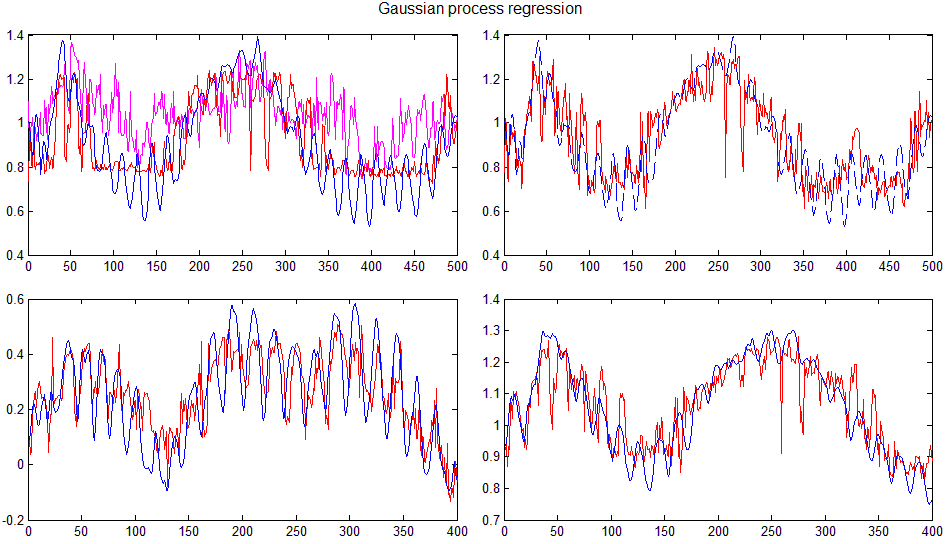}
\end{center}
\vspace{-0mm}
\caption{\label{fig:gpr} GPR pose estimates in the same experimental setting of Figures \ref{fig:ambi} and \ref{fig:rvm}.} \vspace{-0mm}
\end{figure}

Figure \ref{fig:gpr-confidence} plots the confidence intervals of the estimates produced by Gaussian Process Regression for the same test sequences of Figure \ref{fig:interval}. A confidence level of 95\% (corresponding to an interval of two standard deviations) is used.\\ We want to stress, however, the difference between the confidence band (shown in Figure \ref{fig:gpr-confidence}) associated with a \emph{single} Gaussian distribution on the outputs (poses) (such as the prediction function $p(q|y,\tilde{Q},\tilde{y})$ of a GPR) which is characterized by a \emph{single} mean estimate and a (co)-variance, and the \emph{interval} of \emph{expected (mean) poses} associated with a belief estimate (which amounts to entire family of probability distributions) shown in Figure \ref{fig:interval}.\\ This is a consequence of the second-order uncertainty encoded by belief functions, as opposed to single classical probability distributions. Indeed, for each vertex of the credal estimate produced by BMR we could also compute (besides an expectation) a covariance and a confidence band: the cumulated confidence bands for all Probability Distribution Functions (PDFs) in the credal estimate would be a fairer comparison for the single confidence band depicted in Figure \ref{fig:gpr-confidence}, and would better illustrate the approach's robustness.

\subsubsection{Testing models of different resolutions in the legs experiment}

Figure \ref{fig:legs} shows BMR's performance in the leg experiment, for a 200-frame-long test sequence. Again, the pignistic transform was adopted to extract a pointwise pose estimate at each time instant. The estimates generated by two models with the same number of feature spaces ($N=6$) but different number of MoG components per feature space ($n=5$, red; $n=4$, magenta) are shown, to analyze the effect of quantization on the model. 
\begin{figure}[ht!]
\vspace{-0mm}
\begin{center}
\includegraphics[width = 0.75 \textwidth]{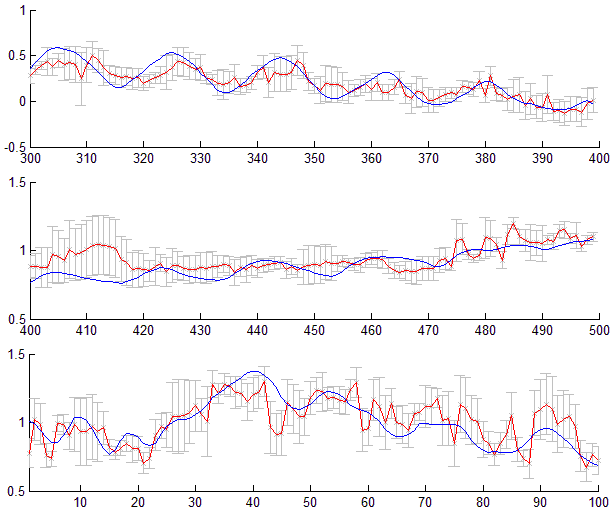}
\end{center}
\vspace{-0mm}
\caption{\label{fig:gpr-confidence} The confidence intervals (at two standard deviations) associated with GPR estimates (in red) are plotted here against the ground truth (in blue) for the same test sequences of Figure \ref{fig:interval}.} \vspace{-0mm}
\end{figure}

The results were a bit less impressive (but still good), mainly due to the difficulty of automatically segmenting a pair of black pants against a dark background (see Figure \ref{fig:capture}-right). Again, this cannot be considered an issue of the BMR approach, as annotation is supposed to be given in the training stage. A quantitative assessment returned average estimation errors (for the pignistic expected pose estimate and the model with $n=5$) of 25.41, 19.29, 21.84, 19.88, 23.00, and 22.71 centimeters, respectively, for the six markers (located on thigh, knee and toe for each leg). Consider that the cameras were located at a distance of about three meters. No significant differences in accuracy could be observed when reducing the number of MoG components to 4. As in the arm experiment, no significant conflict was reported. In this sense these tests did not allow us to illustrate the ability of the evidential approach to detect foreground features in the case of occlusions or imprecise localization: more challenging tests will need to be run in the near future.
\begin{figure}[ht!]
\vspace{-0mm}
\begin{center}
\includegraphics[width =\textwidth]{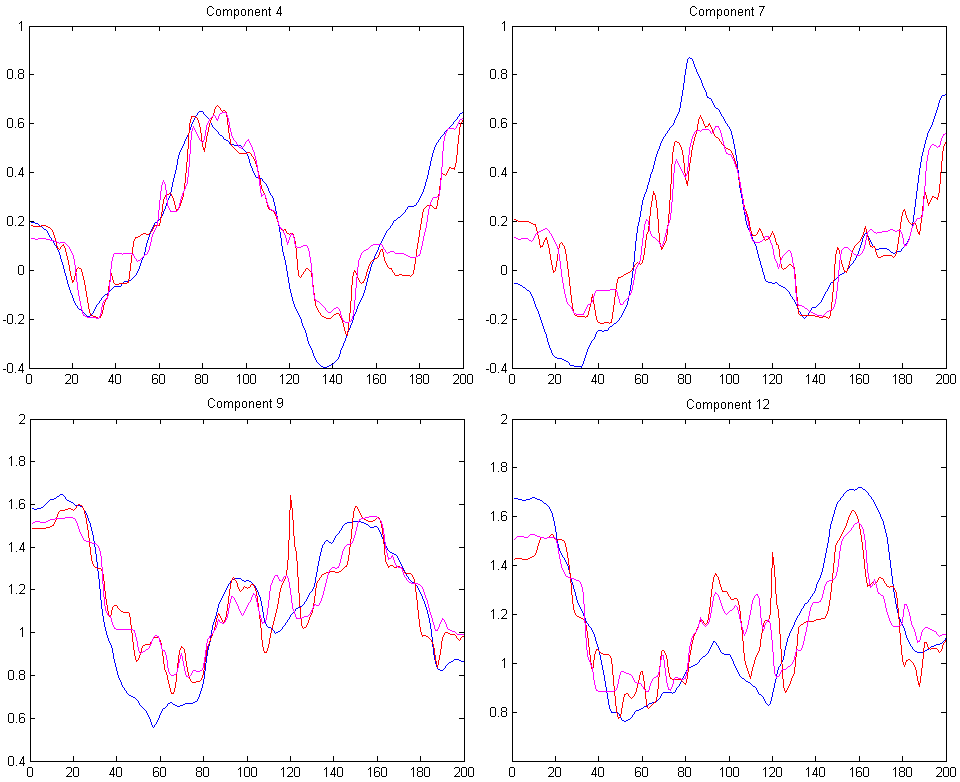}
\end{center}
\vspace{-0mm}
\caption{\label{fig:legs} Performance of two versions of the two-view evidential model with $N = 6$ feature spaces, in the leg experiment, on a test sequence of length 200. The pignistic expected pose is computed for a number of MoG components equal to $n_i = n = 5$ for each feature space (red), and a model with $n_i = n = 4$ (magenta), and plotted versus the ground truth (blue). The estimates for components 4, 7, 9 and 12 of the 18-dimensional pose vector (the 3D coordinates of each of the 6 markers) are shown.} \vspace{-0mm}
\end{figure}

    % visual estimates
\subsubsection{When ground truth is not available: visual estimates}

When ground truth is not available in the training stage, the pignistic probability $\hat{p}$ on $\tilde{\mathcal{Q}}$ extracted from the belief estimate $\hat{b}$ can be used to render, given a test image, a visual estimate in terms of the weighted sum of sample images:
\begin{equation} \label{eq:vis}
\hat{I} = \sum_{k=1,...,T} \hat{p}(q_k) \cdot I(k). 
\end{equation}
Figure \ref{fig:visual} compares the results of this visual estimate with the corresponding, real, test image. The accuracy of this visual reconstruction can be easily appreciated. Some fuzzyness is present, due to the fact that the visual estimate is the extrapolation of possibly many sample images, and expresses the degree to which the estimate is inaccurate.
\begin{figure}[ht!]
\vspace{-0mm}
\begin{center}
\includegraphics[width = 0.95\textwidth]{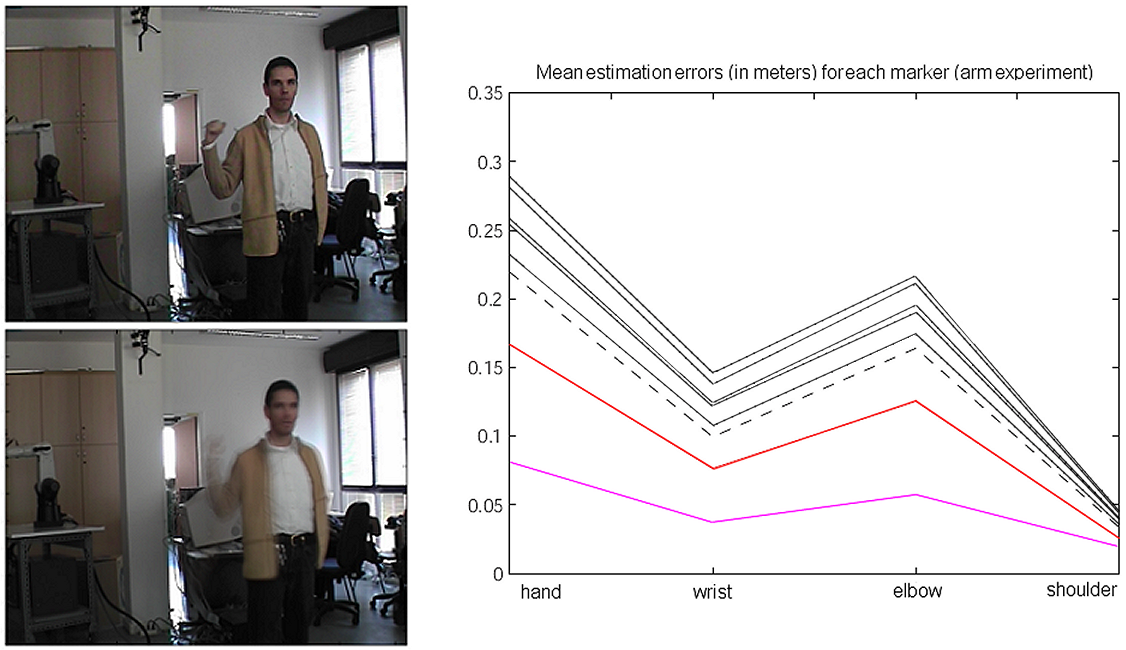}
\end{center}
\vspace{-0mm}
\caption{\label{fig:visual} Left: visual comparison between a real test image of the arm experiment (top) and the
corresponding visual reconstruction (bottom) (\ref{eq:vis}). Right: mean estimation errors for each of the four markers throughout the test sequence of the arm experiment. The errors delivered by the multiple feature space model are compared to those produced by a number of single (vectorial) feature space model. When each scalar feature component is considered separately and combined by conjunctive rule (solid red), rather than being piled up in a single observation vector, the performance is significantly superior. The dashed black line corresponds to a single feature space with $n=30$ Gaussian components. Solid black lines are associated with a quantization level of $n = 20,10,7,5$ and 3, respectively. In magenta the average 3D distance from the interval of expected poses (Section \ref{sec:interval-pose}) delivered by the evidential model of Figure \ref{fig:ambi} is plotted.} \vspace{-0mm}
\end{figure}

\subsubsection{Conjunctive combination versus vectorization} \label{sec:comparison-hmm}

Finally, it is interesting to assess the advantage of combining a number of separate belief functions for each component of the feature vector, rather than piling up all the features in a single observation vector. As a term of comparison, therefore, we applied the same estimation scheme to a \emph{single} feature space, composed by whole feature vectors, rather than the collection of spaces associated with individual feature components. We applied EM to the set of training feature vectors, with a varying number $n$ of MoG clusters.

Figure \ref{fig:visual}-right plots the different average estimation errors for the four markers in the arm experiment along the whole testing sequence of length 1000, as produced by the two-view, multiple feature space evidential model versus a single feature space one generated by applying EM to whole feature vectors. The pignistic function was again used here to compute the point-wise expected estimate. The solid red line represents the performance of the multiple feature space model, versus a number of black lines associated with single feature space models with a number of MoG clusters $n$ equal to 3, 5, 7, 10, 20 and 30, respectively. 

In a way, these tests \emph{compare the efficacy of the conjunctive combination of belief functions to that of vectorization as a data fusion mechanism}. Not only the former proves to be superior, but the plot suggests that, after a certain threshold, increasing the number of MoG components does not improve estimation performance anymore. Figure \ref{fig:compare_single} visually compares the quality of the estimates for two components (2 and 4) of the pose vector on a 100-frame long sub-sequence of the testing sequence. Even though (in principle) there is no reason why quantizing a single, vectorial feature space should yield poor performances, in practice it is impossible to learn the parameters of a Mixture of Gaussians with a number of states comparable to the product $n_1 \cdot ... \cdot n_N$ of the number of clusters of the $N$ separate feature spaces. The EM algorithm is unable to converge: the best we can get to is a few dozen states, a number insufficient to guarantee an adequate estimation accuracy.
\begin{figure}[ht!]
\vspace{-0mm}
\begin{center}
%\begin{tabular}{cc}
%\hspace{-10mm} \epsfxsize = 0.56 \hsize \epsfbox{confronto_arm_comp2_900_1000_2.png} & \hspace{-12mm}
%\epsfxsize = 0.56 \hsize \epsfbox{confronto_arm_comp4_900_1000_2.png}
%\end{tabular}
\includegraphics[width = \textwidth]{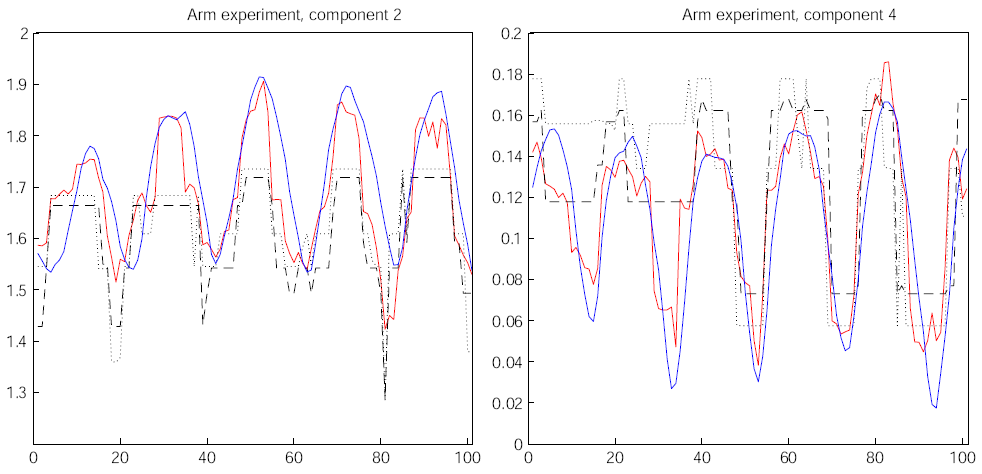}
\end{center}
\vspace{-0mm}
\caption{Visual comparison between the estimates yielded by the belief combination of scalar features (solid red line) and those produced by a single (vectorial) feature space with $n=20$ (dotted black) and $n=30$ (dashed black) Gaussian components, versus the provided ground truth (solid blue), in the arm experiment. A 100-frame long interval of the testing sequence is considered. Results for components 2 and 4 of the pose vector are shown. \label{fig:compare_single}} \vspace{-0mm}
\end{figure}

\section{Discussion} \label{sec:discussion}

We wish to conclude by discussing the methodological justification of the proposed regression framework, in the light of the problem to solve and in comparison with similar approaches, in particular Gaussian Process regression.

\subsection{Justification}

In the scenario depicted in Section \ref{sec:problem}, any regression method we design needs to represent the feature-to-pose mapping $y \mapsto q$, which is unknown.\\ Consider first the case of a single feature function. 

    % naive solutions
\subsubsection{Naive interpolation} 

The training data $\{ \tilde{\mathcal{Q}}, \tilde{\mathcal{Y}} \}$ already provides us with a first, rough approximation of the unknown mapping. A naive regression approach may, for instance, apply to any test feature value $y$ a simple linear interpolator 
\[
y \mapsto \sum_k w_k q_k 
\]
with coefficients $w_k$ depending on some distance $d(y,y_k)$ between $y$ and each training feature $y_k$. We obtain a one-to-one, piecewise linear map (see Figure \ref{fig:justification}-left) which, when the training samples are dense in the unknown pose space $\mathcal{Q}$, deliver a decent approximation of the (also unknown) feature-to-pose mapping. 

Such a naive interpolator, however, does not allow us to express any uncertainty due to lack of training information. Also, although the source of ground truth provides a single pose value $q_k$ for each sample feature value $y_k$, (self-)occlusions and projection ambiguities mean that each observed feature value $y$ (including the sample feature values $y_k$) can be generated by a continuum $\mathcal{Q}(y)$ of admissible poses. In particular, this is true for the extremely simple bounding box features implemented in Section \ref{sec:features}.\\ When presented with a training feature value $y_k$ during testing, our naive interpolator associates it with the corresponding training pose $q_k$, which is in fact only one of a continuous set of poses $\mathcal{Q}(y_k)$ that could have generated that particular feature value.

\begin{figure}[ht!]
\begin{center}
\includegraphics[width = \textwidth]{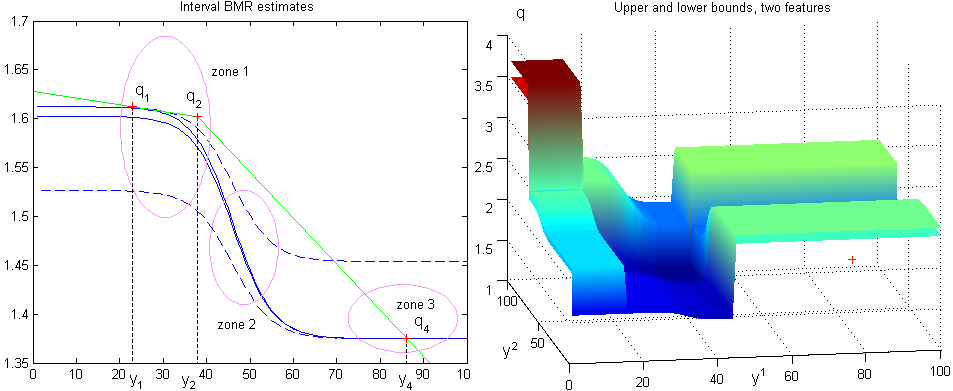}
\end{center}
\caption{Left: lower and upper bounds to the pose estimate generated by a single-feature evidential model. We picked component $c=2$ of the sample poses $q_1$, $q_2$ and $q_4$ of the training sequence of the arm experiment of Section \ref{sec:results}, and built a single-feature evidential model using as training feature values $y_1=23$, $y_2=38$ and $y_4=86$ and $n=2$ EM clusters (corresponding to $\{q_1,q_2\}$ and to $\{q_4\}$). A simple linear interpolator (in green) yields a 1-1, piecewise feature-to-pose map which does a decent job when the samples are dense in $\mathcal{Q}$. Using Bayesian belief functions to encode feature values, the uncertainty on the feature value in each cluster is smoothly propagated to the entire range of feature values, obtaining the solid blue lower and upper bounds. Using Dirichlet belief functions delivers wider, more cautious bounds (dashed blue). Right: the conjunctive combination (\ref{eq:conjunctive}) of multiple features generates a framework with much expressive power in terms of the family of mappings modeled. Here the complex, but still smooth, shape of the lower and upper bounds generated by an evidential model with two feature spaces in the same toy experiment is shown. \label{fig:justification}}
\end{figure}

    % cluster-refining framework
\subsubsection{Intervals of pose estimates in the cluster-refining framework} 

The most interesting and realistic situation is that in which the training samples are sparse in $\mathcal{Q}$. Even assuming that the sought map is one-to-one (which is not), any regressor will be uncertain about the value of the pose far from the available samples. What we need, is to be able to: 1- express the uncertainty on the value of the map far from the samples, but also 2- express the fact that to any $y$ may correspond an entire set $\mathcal{Q}(y)$ of poses.
The evidential modeling framework of Section \ref{sec:approximate-feature-spaces} addresses, to some extent, both these questions, as: i) it provides \emph{an interval} of (rather than pointwise) estimates in the regions of $\mathcal{Q}$ covered by samples (clusters); ii) it describes the uncertainty on the value of the pose far from the samples. Let us see how.

In the regions covered by samples the interval of possible poses is estimated on the basis of the interval $[q_{min},q_{max}]$ of sample pose values in the cluster $\tilde{\mathcal{Q}}_k$, where: 
\[
\begin{array}{cc}
q_{min} = \min_{q_k \in \tilde{\mathcal{Q}}_k} q_k, & q_{max} = \max_{q_k \in \tilde{\mathcal{Q}}_k} q_k. 
\end{array}
\]
The rationale is that if close feature values $y_k$ yield different poses $q_k$s, this may signal what we called `inherent' ambiguity (in that region of $\mathcal{Q}\times \mathcal{Y}$ each feature value $y$ may be generated by a wide interval $\mathcal{Q}(y)$ of admissible poses): see zone 1 in Figure \ref{fig:justification}-left. 

Far from the clusters (zone 2 of Figure \ref{fig:justification}-left) this interval uncertainty is propagated via Equation (\ref{eq:soft-assignment}), by assigning a total weight $\sum_k w_k = \Gamma^j(y)/Z$ to the ensemble of samples of each cluster (equivalently, by defining a belief function on the collection $\tilde{\mathcal{Q}}$ of sample poses).\\ Given an interpolator function $\mathcal{I} : \mathcal{Y} \rightarrow \mathcal{Q}$:
\begin{equation} \label{eq:interpolator}
\hat{q} = \mathcal{I}(\{ p_k,q_k \}, y ) 
\end{equation}
mapping a feature value $y$ to a pose vector $\hat{q}$ (given a certain probability distribution $\{p_k\}$ on the training poses $\{q_k\}$), this translates into an interval of admissible poses for each test feature $y$. 

The overall effect of the cluster-refining framework which is the building block of our evidential model is a `robustification' of the estimates produced by the chosen interpolator function $\mathcal{I}$ (observe Figure \ref{fig:justification}, zone 2 for the case of the expectation interpolator of Equation (\ref{eq:mean-value})). Still, isolated values which form clusters on their own are taken at face value (zone 3, as in GPR): this is an undesirable but unlikely result of EM clustering, which takes place whenever the number of clusters $n$ is much smaller than the number of training poses $T$.

    % expressive power in terms of a family of mapping
\subsubsection{Expressive power in terms of a family of mapping} 

In a single-feature evidential model, then, the learned refining does not constitute an approximation of the true feature-pose map under the model, but determines a constraint on the latter associated with a whole family of feature-pose mappings compatible with the given training observations. Such admissible maps are those and only those which would generate the learned refinings given the same training data. They form an $\infty$-dimensional family, bounded by an upper and a lower admissible feature-to-pose functions. We can prove that these lower and upper mappings are smooth, due to the smoothness of the Gaussian likelihoods $\Gamma$. 

For sake of simplicity we consider here a single feature model, and set to zero the mass $m(\Theta)=0$ of the approximate feature space $\Theta$ (compare Section (\ref{sec:mf})).

\begin{theorem} \label{the:bounds}
Suppose the interpolator function (\ref{eq:mean-value}) is used to infer a pose estimate $\hat{q}(y)$ from a feature value $y$, given a probability distribution $\{ p_k, k=1,...,T \}$ on the set of training poses $\tilde{\mathcal{Q}} = \{ q_k, k=1,...,T \}$. Then, for each component $q^c$ of the pose vector $q$, both the upper bound $\sup \hat{q}^c(y)$ and the lower bound $\inf \hat{q}^c(y)$ to the admissible pose estimates under a single-feature evidential model for all possible test feature values $y \in \mathcal{Y}$ are smooth functions of $y$.
\end{theorem}

\begin{proof}
We only prove the statement for the upper bound. A dual proof can be easily derived for the lower bound. The former quantity reads as:
\[
\sup_{p \in \mathcal{P}[\hat{b}(y)]} \sum_{k=1}^T p_k(y) q^c_k,
\]
where $\hat{b}(y)$ is the belief estimate generated by a test feature value $y$, and $\mathcal{P}[\hat{b}(y)]$ is the corresponding credal set. Since we consider a single feature model with $m(\Theta)=0$, $\hat{b}(y)$ has $n$ focal elements $\tilde{\mathcal{Q}}^1,...,\tilde{\mathcal{Q}}^n$ with mass $m(\tilde{\mathcal{Q}}^j) = \Gamma^j(y)/Z$, with $Z$ a normalization factor. Each is the image of a EM cluster in the feature space. Together they form a disjoint partition of $\tilde{\mathcal{Q}}$, so that:
\[
\sum_{q_k \in \tilde{\mathcal{Q}}^j} p_k(y) = m(\tilde{\mathcal{Q}}^j) = \frac{\Gamma^j(y)}{Z}. 
\]
Therefore, we can decompose the upper bound as:
\[
\sup \hat{q}^c(y) = \sup \sum_{q_k \in \tilde{\mathcal{Q}}} p_k(y) q^c_k = \sup \Big ( \sum_{j=1}^n \sum_{q_k \in \tilde{\mathcal{Q}}^j} p_k(y) q^c_k \Big ) = \sum_{j=1}^n \sup \Big ( \sum_{q_k \in \tilde{\mathcal{Q}}^j} p_k(y) q^c_k \Big ).
\]
But
\[
\sup \Big ( \sum_{q_k \in \tilde{\mathcal{Q}}^j} p_k(y) q^c_k \Big ) = \frac{\Gamma^j(y)}{Z} \sup_{q_k \in \tilde{\mathcal{Q}}^j} q^c_k,
\]
for the $\sup$ is obtained by assigning all mass $\frac{\Gamma^j(y)}{Z}$ to the sample $q_k$ with the largest pose component value. The quantity $\sup_{q_k \in \tilde{\mathcal{Q}}^j} q^c_k$ does not depend on the test feature value $y$, but is a function of the samples in the considered cluster $j$. Therefore
\[
\sup \hat{q}^c(y) = \frac{1}{Z} \sum_{j=1}^n \Gamma^j(y) \sup_{q_k \in \tilde{\mathcal{Q}}^j} q^c_k
\]
is a smooth function, the linear combination of the smooth functions $\Gamma^j(y)$ with coefficients $\sup_{q_k \in \tilde{\mathcal{Q}}^j} q^c_k$.
\end{proof}

These lower and upper bounds are depicted as solid blue lines in the example of Figure \ref{fig:justification}-left. Within those smooth bounds, any one-to-many mapping is admissible, even discontinuous ones: a quite realistic situation, for the actual pose space $\mathcal{Q}$ can have holes composed by non-admissible poses, generating discontinuities in the feature-pose map.\\ The width of this family of mappings is a function of the number $n$ of EM clusters:
\[
\sup \hat{q}(y) - \inf \hat{q}(y) =  \frac{1}{Z} \sum_{j=1}^n \Gamma^j(y) \Big ( \sup_{q_k \in \tilde{\mathcal{Q}}^j} q_k - \inf_{q_k \in \tilde{\mathcal{Q}}^j} q_k \Big ).
\]
A low $n$ amounts to a cautious approach in which training feature values are not `trusted', and the inherent ambiguity (number of training samples in $\mathcal{Q}(y)$) is higher. Many clusters ($n \rightarrow |\tilde{\mathcal{Q}}|$) indicate that we much trust the one-to-one mapping provided by the interpolator over the samples.

\subsubsection{Effect of $m_i(\Theta_i)$ and sparsity of samples} 

An additional element in our regression framework is constituted by the mass $m_i(\Theta_i)$ assigned to the whole approximate feature space $\Theta_i$ (Section \ref{sec:mf}). Its effect on the family of admissible maps is to further expand the band of estimates, depending on the number of available training samples (as $m_i(\Theta_i)=1/T$). In Figure \ref{fig:justification}-left the expanded upper and lower bounds due to the use of Dirichlet belief functions are depicted as dashed blue lines. It can be appreciated how these are still smooth functions of $y$.

%The inherent ambiguity of the mapping is described instead, as we said, by the learnt refining.

\subsubsection{Choice of an interpolation operator}

The shape of the family of mappings represented by a single-feature evidential model (i.e., of its lower and upper bounds) is also determined by the choice of the interpolation operator $\hat{q} = \mathcal{I}(\{ p_k,q_k \}, y )$. In Section \ref{sec:estimation} the interpolator function was the expectation operator $\hat{q} = \sum_k p_k q_k$, but other choices are of course possible.\\
Different interpolators generate different families of feature-pose mappings. %(Figure \ref{fig:elements}-left).

\subsubsection{Fusion of individual features} 

An additional layer of sophistication is introduced by the combination of distinct features via the conjunctive combination of the associated belief functions. This produces a rather complex families of compatible feature-to-pose mappings. Figure \ref{fig:justification}-right illustrates the shape of the lower and upper bounds generated by an evidential model with $N=2$ feature spaces. It can be noted that, despite their more complex shape, these bounds are still smooth.

\subsection{Differences and similarities with Gaussian Process Regression} \label{sec:comparison-gpr}

It can be interesting to compare the behavior of BMR with that of the classical \emph{Gaussian Process Regression} (GPR) \cite{rasmussen06gpr}. The latter assumes that any finite set of observations are drawn from a multivariate Gaussian distribution. According to~\cite{rasmussen06gpr}, a Gaussian process is defined as `{a collection of random variables, any finite number of which have (consistent) joint Gaussian distribution'. It is then completely specified by a mean $m(\mathbf{s})$ and a covariance $k(\mathbf{s},\mathbf{s}^\prime)$ function over the samples' (observations) domain, and it can be seen as a distribution over functions: 
\begin{equation} \label{eq:gp}
\mathbf{\zeta}(\mathbf{s})  \sim \mathcal{GP}_{j}(m(\mathbf{s}),k(\mathbf{s},\mathbf{s}^\prime)),
\end{equation}
where $m(\mathbf{s}) = E[\mathbf{\zeta}(\mathbf{s})]$, and  
\begin{equation} \label{eq:cov}
k(\mathbf{s},\mathbf{s}^\prime) =  E[(\mathbf{\zeta}(\mathbf{s}) -  m(\mathbf{s})) (\mathbf{\zeta}(\mathbf{s}^\prime) -m(\mathbf{s}^\prime))].
\end{equation}

If the covariance function (\ref{eq:cov}) depends on a set of hyperparameters, given a training set of noisy observations $\{(s_k=y_k,\zeta_k = q_k)\}_{K=1,...,T}$, and assuming the prediction noise to be Gaussian, we can find the optimal hyperparameters of the Gaussian Process $\mathcal{GP}$ which best fits the data by maximizing the log marginal likelihood (see~\cite{rasmussen06gpr} for more details).

With the optimal hyperparameters, we obtain a Gaussian prediction distribution in the space of targets (poses):
\begin{equation}
\label{eq:GPRreference}
\mathcal{N}\big (\mathbf{k}(s^{\ast}, \mathbf{s})^{T}[K + \sigma_{noise}^{2}I]^{-1}\varPsi^{\prime}, %\nonumber %\\
k(s^{\ast},s^{\ast})+\sigma_{noise}^{2}-\mathbf{k}(s^{\ast},\mathbf{s})^{T}[K + \sigma_{noise}^{2}I]^{-1}\mathbf{k}(s^{\ast},\mathbf{s}) \big),
\end{equation}
where $K$ is the covariance matrix calculated from the training image features $\mathbf{s}$ and $\sigma_{noise}$ is the covariance of the Gaussian noise. This is equivalent to having an entire family of regression models, all of which agree with the sample observations. %(see Figure \ref{fig:implicit}-right). %Gaussian Processes have been recently used for pose estimation, as in \cite{gong11gpr,rudovic11iccv}.

% actual comparison
Both GP and BM Regression model a family of feature-to-pose mappings, albeit of a rather different nature. In Gaussian Process Regression, mappings are one-to-one, and a Gaussian Process amounts to a probability distribution over the set of mappings. The form of the family of mappings actually modeled is determined by the choice of a covariance function, which also determines a number of characteristics of the mappings such as periodicity, continuity, etcetera. After conditioning a Gaussian Process by the training data, we obtain a prediction function (Equation~(\ref{eq:GPRreference})) on $\mathcal{Q}$ which follows a Gaussian distribution (given a test observation and the trained model parameters). The predicted mean and variance vary according to the test observations. In particular the training samples are assumed correct and trustworthy: as a result, the posterior GP has zero uncertainty there.

In opposition, Belief Modeling Regression produces a random set, an entire convex set of discrete but arbitrary PDFs, but on the set of sample poses $\tilde{\mathcal{Q}}$, rather than on $\mathcal{Q}$. As we have seen, given an interpolation function this random set corresponds to a constrained family of mappings, rather than a distribution over the possible maps as in GPR. The resulting mappings are arbitrary and interval-like, as long as they meet the upper and lower constraints, or, equivalently, as long as they generate the learned refinings under the training data. The shape of the family of mappings does depend on the chosen interpolation operator, while its width is a function of the number of clusters $n$ and the mass of the whole feature space $\Theta$. A feature of BMR is that uncertainty is present even in correspondence of sample feature values (see above).

Different is the treatment of the uncertainty induced by the scarcity of samples (i.e., far from the samples). In GPR the standard deviation of the prediction function is influenced by both the type of prior GP selected and the distance from the samples. In BMR the width of the interval of pose estimates is influenced by both the number $n_i$ of EM feature clusters, and the mass $m(\Theta_i)$ Dirichlet belief functions assign to the whole (approximate) feature space.

\subsection{Different inference mechanisms} \label{sec:inference}

Dirichlet belief functions are not the only possible way of inferring a belief function from a set of likelihoods. Another option is to normalize the likelihoods (\ref{eq:mog}) generated by the MoG, obtaining a probability (or \emph{Bayesian}  b.f.) on $\Theta_i = \{ \mathcal{Y}_i^1, \cdots, \mathcal{Y}_i^{n_i}\}$:
\begin{equation}\label{eq:bayesian-mf}
m_i(\mathcal{Y}_i^j) =  \Gamma_i^j(y_i) \Big/ \sum_k \Gamma_i^k(y_i).
\end{equation}
Alternatively, the likelihood values can be used to build a \emph{consonant} belief function (see Chapter \ref{cha:toe}, Definition \ref{def:consonant}), i.e., a b.f. whose focal elements $A_1 \subset \cdots \subset A_m$ are nested, as in \cite{Shafer76}:
\begin{equation}\label{eq:shafer}
b_i(A) = 1 -  \max_{j : \mathcal{Y}_i^j \in A^c }{ \Gamma_i^j(y_i)} \Big /
\max_{j}{\Gamma_i^j(y_i)}.
\end{equation}
%Figure \ref{fig:mf} compares the pignistic (\ref{eq:pignistic}) point-wise pose estimates returned by
The three different Bayesian (\ref{eq:bayesian-mf}), consonant (\ref{eq:shafer}), and Dirichlet (\ref{eq:dirichlet}) inference algorithms seem to produce comparable results in terms of pointwise estimates, at least under the experimental setting of Section \ref{sec:results}, characterized by low conflict. Significant differences emerge, however, if we investigate the nature of the belief estimate the different inference techniques generate.

In the Bayesian case, as the belief functions on the individual feature spaces $\Theta_i$ have disjoint (singleton) focal elements, their projection onto $\tilde{\mathcal{Q}}$ also has disjoint focal elements. The conjunctive combination of all such b.f.s yields again a belief estimate $\hat{b}$ whose focal elements are disjoint (Figure \ref{fig:appendix1}-left). This means that a region of the pose space is supported by $\hat{b}$ only to the extent by which it is supported by \emph{all} the individual features. 

If the belief functions built on the available feature spaces are Dirichlet, their projections onto $\tilde{\mathcal{Q}}$ all have the whole $\tilde{\mathcal{Q}}$ as a focal element. Therefore, their conjunctive combination (\ref{eq:conjunctive}) will have as f.e.s not only all the intersections of the form
$
\rho_1(A_1) \cap \cdots \cap \rho_N(A_N)
$
for all possible selections of a single focal element $A_i$ for each measurement function $b_i$, but also all the intersections
$
\rho_{i_1}(A_{i_1}) \cap \cdots \cap \rho_{i_m}(A_{i_m})
$
(where $i_1,...,i_m$ index {any subset} of features), and the whole approximate pose space $\tilde{\mathcal{Q}}$ (Figure \ref{fig:appendix1}-middle). This is equivalent to say that all {partial} combinations of feature evidence are given some credit, for maybe only a {subset} of features is telling the truth. When conflict among different feature models is present, this amounts to a cautious approach in which \emph{the most consensual group of features is given support}. The more so whenever the remaining features are highly discounted as less reliable ($m_i(\Theta_i)$ is high). 

Finally, in the consonant case the conjunctive combination of single-feature belief functions yields a belief estimate $\hat{b}$ whose focal elements also form chains of nested sets of poses: one can say that the resulting belief estimate is `multi-modal', with a focus on a few regions of the (approximate) pose space (Figure \ref{fig:appendix1}-right).

\begin{figure}[ht!]
\vspace{-0mm}
\begin{center}
\includegraphics[width=\textwidth]{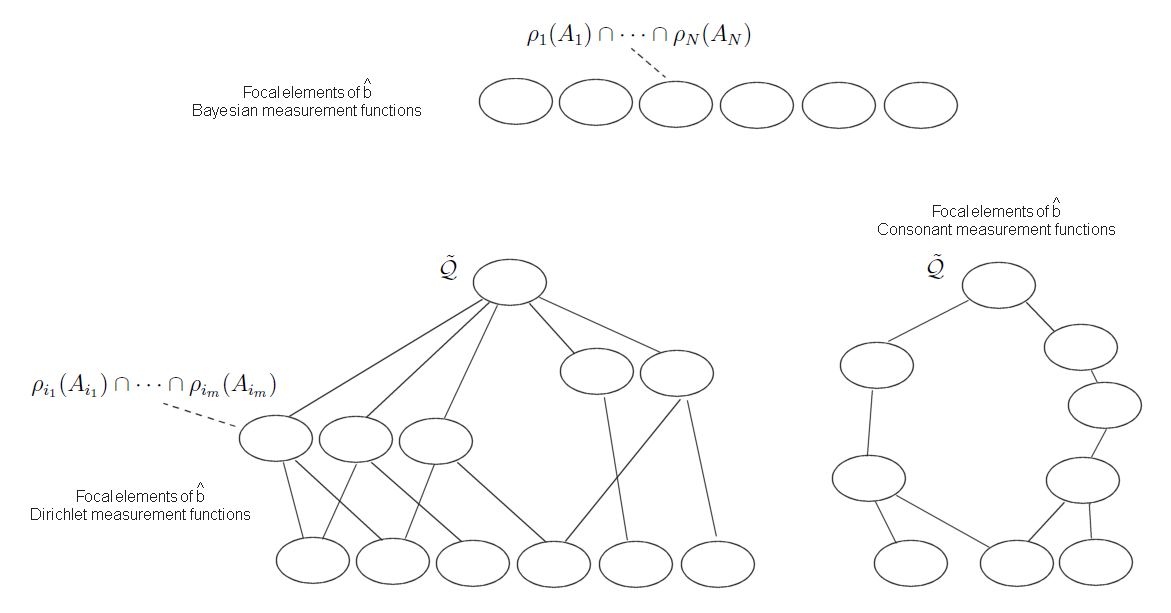}
\vspace{-0mm}
\end{center}
\caption{Left: when using Bayesian belief functions to encode feature values all the focal elements of the belief estimate $\hat{b}$ on $\tilde{\mathcal{Q}}$ are disjoint, namely intersections $\rho_1(A_1) \cap \cdots \cap \rho_N(A_N)$ of one focal element per feature. Focal elements are depicted as ellipses, while edges between them indicate inclusion $\subset$. Middle: in the Dirichlet case all the intersections $\rho_{i_1}(A_{i_1}) \cap \cdots \cap \rho_{i_m}(A_{i_m})$ generated by groups of $m$ features are also focal elements of $\hat{b}$ with nonzero mass. Right: in the consonant case a number of chains of nested f.e.s describe a multi-modal support for different regions of the pose space. \label{fig:appendix1}}
\vspace{-0mm}
\end{figure}

It is interesting to compare these three approaches by looking at the associated credal sets as well. The size of the credal set represented by a belief estimate is a function of two distinct sources of uncertainty: that associated with the belief function $b_i$ we build on each approximate feature space $\Theta_i$, and the multi-valued mapping from features to poses. Even when the former is a probability (Bayesian case), the multi-valued mapping still induces a belief function on $\tilde{\mathcal{Q}}$. 

Consider a restricted, toy model obtained from just the first four training poses $q_1, q_2, q_3$ and $q_4$ in the arm experiment. This way $\tilde{\mathcal{Q}}$ has size 4, and the probability simplex there (see Figure \ref{fig:ex-convex} again) is 3-dimensional and can be visualized.
\begin{figure}[ht!]
\vspace{-0mm}
\begin{center}
\includegraphics[width=\textwidth]{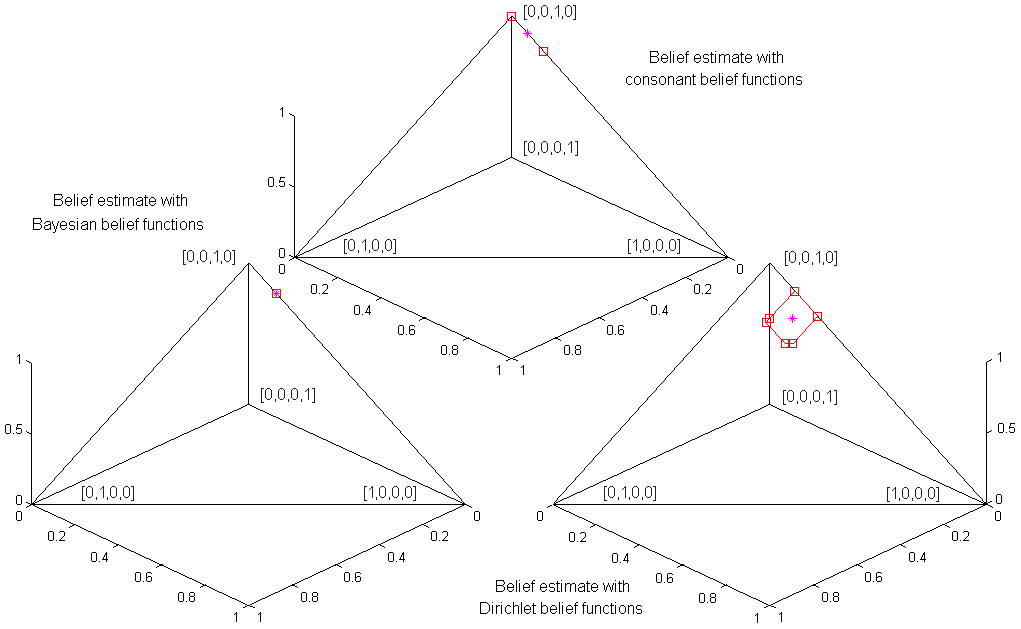}
\vspace{-0mm}
\end{center}
\caption{Left: belief estimate, represented by a credal set in the simplex of all probability distributions on $\tilde{\mathcal{Q}}$, generated by using Bayesian measurement belief functions for frame $k=13$ of the arm experiment, for a model learned from the first 4 sample poses only. Middle: belief estimate for the same frame, generated using consonant measurement functions. Right: belief estimate produced via Dirichlet measurement functions. \label{fig:appendix2}}
\vspace{-0mm}
\end{figure}
Figure \ref{fig:appendix2} depicts the credal sets generated by the three inference mechanisms (under the above toy model) in correspondence of frame 13 of the test feature sequence of the arm experiment. One can note how the credal set in the Bayesian case is the narrower (in fact in this example it reduces to a single point, although not in general), while it is the widest in the Dirichlet case. The latter amounts therefore to a more cautious approach to estimation allowing a wider uncertainty band around the central expected value (which is the one we adopt in the tests, Figure \ref{fig:interval}). In the Bayesian case, instead, all the uncertainty in the belief estimate comes from the multi-valued nature of refining maps. 

While the size and shape of the credal set varies, the pignistic probability (magenta star) is pretty close in the Bayesian and Dirichlet cases. Empirical evidence seems therefore to suggest that, when conflict is limited, pointwise estimates in the three cases are fairly close (while differing in the attached degree of uncertainty).

\section{Towards evidential tracking} \label{sec:tracking}

To conclude, we outline feasible options for extending the proposed belief-theoretical approach to fully fledged tracking, in which the temporal information provided by a time series of feature values is exploited to help pose estimation. Suppose you have a belief estimate $\hat{b}(t)$ of the pose at time $t$, and a fresh set of features at time $t+1$. The simplest way of ensuring the temporal consistence of the estimated pose is to combine the current estimate $\hat{b}(t)$ with the evidence provided by the new feature values. Namely, the latter will induce a belief estimate $\hat{b}(t+1)$ via the algorithm of Section \ref{sec:alg-pose}; this has then to be combined with the old estimate by conjunctive combination, yielding an overall, `smoothed' version of the estimate: $\hat{b}(t) \Ocap \hat{b}(t+1)$.\\ This approach, however, can easily lead to a drifting of the estimates, as no motion model whatsoever is employed. In addition, it can be argued that in this way features at time $t$ condition the estimates at time $t+1$ just as feature at time $t+1$ do, which is wrong as a matter of principle.

The use of a motion model encoding the dynamics of the object to track is more sensible: however, if this model were to be a-priori we would violate the assumptions of the example-based scenario of Section \ref{sec:problem}. The way to go is \emph{learning a motion model from the training set}, in the same way as we learn feature-pose maps from it. Assuming that the temporal dependency satisfies a Markovian-like condition, i.e., that the pose at time $t+1$ only depends on the pose at time $t$, the following framework can be formulated.

\subsection{Learning a motion model from the training set}

Consider a frame of discernment $\tilde{\dot{\mathcal{Q}}} = \tilde{{\mathcal{Q}}} \times \tilde{{\mathcal{Q}}}$ whose elements $(q_k,q_{k'})$ can be interpreted as transitions $q_k \mapsto q_{k'}$ from sample pose $q_k$ to sample pose $q_{k'}$. This frame is trivially partitioned into $T$ disjoint subsets $\tilde{\dot{\mathcal{Q}}} = \tilde{\mathcal{Q}}_1 \cup \cdots \cup \tilde{\mathcal{Q}}_T$, each of them $\tilde{\mathcal{Q}}_k$ associated with a sample pose $q_k$, and collecting all possible $q_k \mapsto q_{k'}$ transitions originating from $q_k$. We can then mine the information carried by the training set, and infer for each element of this partition a belief function $m_k : 2^{\tilde{\mathcal{Q}}_k} \rightarrow [0,1]$ with the following b.p.a.: 
\[
m_k (q_k \mapsto q_{k'}) = (1 - \epsilon) \cdot \frac{\# transitions \; from \; q_k \; to \; q_{k'}}{\# times \; q_k \; appears \; in \; \tilde{{\mathcal{Q}}} }, \hspace{5mm} m_k (\tilde{\mathcal{Q}}_k) = \epsilon. 
\]
The discounting factor $\epsilon$ is a measure of how well the motion model learned from the training set approximates the true, unknown model of the object's dynamics: in the ideal case $m_k (\tilde{\mathcal{Q}}_k) = 0$. As this is achieved only by collecting an infinite number of samples, $\epsilon = \frac{1}{\# times \; q_k \; appears \; in \; \tilde{{\mathcal{Q}}}}$ is a reasonable albeit imperfect choice for such a factor. The training set also provides a-priori information on the sample poses themselves (i.e., on the elements $\tilde{\mathcal{Q}}_k$ of the considered disjoint partition of $\tilde{\dot{\mathcal{Q}}}$), which can be encoded as a probability distribution on the partition itself: 
\[
m_0 (\tilde{\mathcal{Q}}_k) = \frac{\# times \; q_k \; appears \; in \; \tilde{{\mathcal{Q}}}}{T}.
\]

\subsection{Tracking process}

Given the above belief functions with b.p.a.s $m_1$, ..., $m_T$ defined on the individual elements $\tilde{\mathcal{Q}}_1$, ..., $\tilde{\mathcal{Q}}_T$ of the partition, and the a-priori distribution $m_0$ on the latter, we need to derive a single belief function on the transition frame $\tilde{\dot{\mathcal{Q}}}$. This amount to solving the total belief theorem \cite{cuzzolin01thesis}, formulated in Chapter \ref{cha:total}, Theorem \ref{the:total-belief}. 

%Consider a set $\Theta$ and a disjoint partition $\Omega$ of $\Theta$. Suppose a (conditional) belief function is defined on each element of the partition, and a (a-priori) b.f. is given on the partition $\Omega$ itself. We seek a total belief function on the whole of $\Theta$ whose restriction to $\Omega$ coincides with the a-priori, and whose conditional versions (in our case under conjunctive combination) coincide with the given ones for all the elements $\Theta_i$ of the partition $\Omega$ of $\Theta$ (see Figure \ref{fig:tracking}-left).
\begin{figure}[t!]
\begin{center}
\includegraphics[width=0.5 \textwidth]{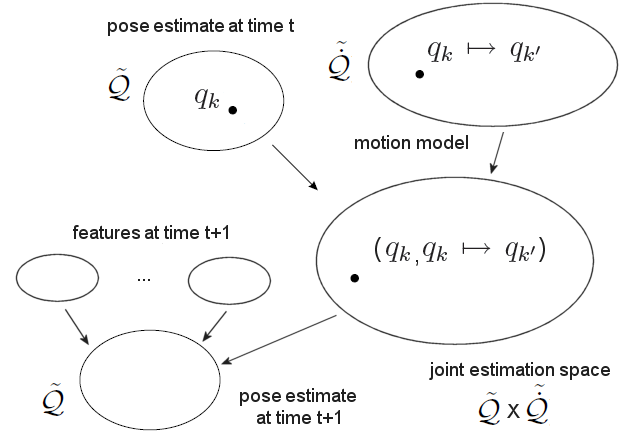}
\end{center}
\caption{Diagram of the proposed evidential tracking process. \label{fig:tracking}}
\end{figure}

Once such a total belief function on $\tilde{\dot{\mathcal{Q}}}$ representing the learned motion model is obtained, it can be combined with the current belief pose estimate $\hat{b}(t)$ (which is defined on $\tilde{\mathcal{Q}}$) on the joint estimation space $\tilde{\dot{\mathcal{Q}}} \times \tilde{\mathcal{Q}}$, the Cartesian product of the two (Figure \ref{fig:tracking}). The resulting belief function can later be projected back onto the approximate pose space $\tilde{\mathcal{Q}}$, where it represents \emph{the predicted pose given the pose at time $t$ and the learned motion model}.\\ Finally, the latter is combined with the belief functions inferred from the available feature measurements at time $t+1$, yielding a belief estimate of the pose $\hat{q}(t+1)$ which incorporates both the current feature evidence and the predictions based on the motion model learned in the training stage. 

\section{Conclusive comments}

In this conclusing Chapter we have illustrated a novel approach to example-based pose estimation, in which the available evidence comes in the form of a training set of images containing sample poses of an unspecified object, whose location within those images is provided. Ground truth is available in the training stage in the form of the configurations of these sample poses. An evidential model of the object is learned from the training data, under weak likelihood models built separately for each feature, and is exploited to estimate the pose of the object in any test image. Framing the problem within belief calculus is natural as feature-pose maps induce belief functions in the pose space, and it allows to exploit the available, limited evidence without additional assumptions, with the goal of producing the most sensible possible estimate with an attached degree of reliability.

The approach has been tested in a fairly challenging human pose recovery setup and shown to outperform popular competitors, demonstrating its potential even in the presence of poor feature representations. These results open a number of interesting directions: a proper empirical testing of object localization algorithms in conjunction with the proposed Belief Modeling Regression approach; an efficient conflict resolution mechanism able to discriminate as much as possible foreground from background features; the testing of the framework on higher-dimensional pose spaces; the full development of the outlined evidential tracking approach.

\part{Conclusions}

\chapter{Conclusions} \label{cha:conclusions}

The aim of this Book was not to prove that the generalization of probability theory due to the work of Dempster and Shafer in the first place, but also of Smets, Shenoy, Denoeux, Kohlas and others, is \emph{the} right way to cope with uncertainty in all practical situations arising in computer vision and other fields of applied science. As we have largely shown in our extensive review of Chapter \ref{cha:state}, the debate on the philosophical foundations of uncertainty theory is still raging. To date, however, the most widespread view supports the notion that no one formalism can be seen as superior to all others -- the choice, in fact, should depend on the specific application at hand \cite{smets98which,scozzafava94subjective}. Partecipating in this debate, at any rate, is beyond the scope of this Book.\\ We hope instead to have at least succeeded in giving a flavor of the richness of the mathematical formalism of the theory of evidence, in terms of both its applicability to many important problems (in particular within computer vision), and of the variety of sophisticated theoretical issues generated by the greater mathematical complexity and `internal structure' belief functions possess when compared to classical probabilities. Once again we would like to stress that all the theoretical advances presented here are direct consequences of the formulation of evidential solutions to the vision applications of Part III, which have been illustrated in detail later in the volume only for pedagogical reasons.

The number of facts and properties yet to be understood remains daunting, and only a brief mentioning of some of them is possible here. The geometric analysis of Chapter \ref{cha:geo} is only at its initial stage, even though interesting results have already been achieved. We now have a picture of the behavior of belief functions as geometrical objects, but the questions which initially motivated this approach are still to be addressed. General expressions for the natural consonant and probabilistic transformations belief functions based on the principle of `external', Dempster-based behavior proposed in Section \ref{sec:approx} are to be found. The geometric `language' we introduced, made possible by the commutativity of Dempster's rule and convex closure, appears promising -- especially for what concerns the canonical decomposition of a generic separable support function. Some steps in this direction have been recently taken in \cite{cuzzolin2019springer}.

A fascinating consequence of our algebraic study of families of compatible frames, given their strong connection to the very definition of support function, could be a new interpretation of the notion of families of compatible support functions -- an issue closely related to the future potential solution of the conflict problem via the formulation of a pseudo Gram-Schmidt algorithm we suggested in Chapter \ref{cha:independence}. The canonical decomposition problem itself could be approached from a different angle by integrating geometrical and algebraic tools to study the geometric interplays of belief spaces associated with different compatible frames. The notion of \emph{lattice of convex hulls} \cite{Stern} could prove very useful in the pursue of this line of research.

The notion of series of random variables or `random process' is widely used in a variety of engineering applications, including control engineering and computer vision -- think of the Kalman filter formalism. The geometric form of conditional subspaces derived in Chapter \ref{cha:geo} could be employed to understand the properties and features of what it is natural to call \emph{series of belief functions}, namely objects of the form:
\[
\lim_{n\rightarrow \infty} (b_1 \oplus \cdots \oplus b_n).
\]
A study of their asymptotic properties would kickstart the formulation of the evidential analogous of a random process.

The generalization of the total probability theorem illustrated in Chapter \ref{cha:total} is, in our view, a good first step towards a satisfactory description of the combination of conditional functions in the evidential framework. Our analysis of the restricted case provides important hints on the general treatment of the total belief theorem, in terms of the multiplicity, structure and relationships of the solutions. In the future we intend to investigate in more detail several alternative views of the total belief problem, an analysis which could enrich our understanding of the theory of belief functions and its connections with apparently unrelated fields of mathematics.\\ Homology theory, for instance, provides a natural framework in which to pose the graph-theoretical problems which arise from the search of candidate total functions. Indeed, the collection of linear systems associated with candidate solutions form a \emph{simplicial complex} (a structured collection of simplices, see \cite{Stern}). The transformations (\ref{eq:column-transformation}) associated with the edges of a solution graph resemble the formal sum of a `{chain}':
\[
c_k=\sum_i g_i\sigma_i,
\]
where $\sigma_i$ is a $k$-dimensional simplex of the complex and $g_i\in G$ is an arbitrary element of a group of transformations $G$.

Conditional subspaces constitute a bridge between the operations of conditioning with respect to \emph{functions} and conditioning with respect to \emph{events}. It suffices to recall that:
\[
\langle b \rangle = Cl(b \oplus b_A, \; A \subset \mathcal{C}_b),
\]
where $b \oplus b_A = b|_A$ is the belief function $b$ conditioned by the event $A$. Via this fact, the total belief theorem can be formulated as a geometric problem as well.

Finally, control engineers could easily spot a different interpretation of candidate total solutions in term of {positive linear systems}. The link is apparent whenever one compares the graph-theoretical layout of candidate solutions with the `{influence graphs}' of linear systems with all positive coefficients, or rearranges their $A$ matrices into matrices with binary $0/1$ entries. Potentially elegant theoretical results could then be achieved by expressing problems concerning conditional generalized probabilities in terms of well-known system-theoretical issues such as, for instance, controllability.

The regression framework proposed in Chapter \ref{cha:pose} to solve the example-based pose estimation problem is just an example of the potential the application of evidential reasoning to difficult, real-world problems can express. Researchers in the theory of evidence tend (at times) to form a small clique of insiders, focussing on theoretical questions of little interest to the outside world. Attempts to generate impact on problems relevant to larger academic communities are limited, as attested by the diminishing numbers of belief functions papers published at UAI, IJACAI or ECAI. The Belief Modeling Regression framework shows that belief functions can compete with and outperform popular machine learning apparata such as GPs and RVMs to tackle problems of widespread interest. 

We hope this will encourange others to rise up to the challenge and compete with more established formalisms on their own ground.

\addtocontents{toc}{}


\begin{thebibliography}{100}

\bibitem{bell96generalized}
D.~Bell, J.~Guan, and S.~K. Lee, ``Generalized union and project operations for
  pooling uncertain and imprecise information,'' {\em Data and Knowledge
  Engineering}, vol.~18, pp.~89--117, 1996.

\bibitem{klir95principles}
G.~J. Klir, ``Principles of uncertainty: What are they? why do we need them?,''
  {\em Fuzzy Sets and Systems}, vol.~74, pp.~15--31, 1995.

\bibitem{Walley96}
P.~Walley, ``Measures of uncertainty in expert systems,'' {\em Artificial
  Intelligence}, vol.~83, pp.~1--58, 1996.

\bibitem{Kong86b}
A.~P. Dempster and A.~Kong, ``Uncertain evidence and artificial analysis,''
  tech. rep., S-108, Department of Statistics, Harvard University, 1986.

\bibitem{resconi93integration}
G.~Resconi, G.~Klir, U.~S. Clair, and D.~Harmanec, ``On the integration of
  uncertainty theories,'' {\em Fuzziness and Knowledge-Based Systems 1},
  pp.~1--18, 1993.

\bibitem{dubois82several}
D.~Dubois and H.~Prade, ``On several representations of an uncertain body of
  evidence,'' in {\em Fuzzy Information and Decision Processes} (M.~M. Gupta
  and E.~Sanchez, eds.), pp.~167--181, North Holland, Amsterdam, 1982.

\bibitem{shafer1981b}
G.~Shafer, ``Two theories of probability,'' in {\em Philosophy of Science
  Association Proceedings 1978} (P.~Asquith and I.~Hacking, eds.), vol.~2, East
  Lansing (MI): Philosophy of Science Association, 1981.

\bibitem{black87shafer}
P.~Black, ``Is {S}hafer general {B}ayes?,'' in {\em Proceedings of the Third
  AAAI Uncertainty in Artificial Intelligence Workshop}, pp.~2--9, 1987.

\bibitem{sheridan1991}
F.~K.~J. Sheridan, ``A survey of techniques for inference under uncertainty,''
  {\em Artificial Intelligence Review}, vol.~5, pp.~89--119, 1991.

\bibitem{Kohlas94}
J.~Kohlas and P.-A. Monney, ``Theory of evidence - a survey of its mathematical
  foundations, applications and computational anaylsis,'' {\em ZOR-
  Mathematical Methods of Operations Research}, vol.~39, pp.~35--68, 1994.

\bibitem{Levi83}
I.~Levi, ``Consonance, dissonance and evidentiary mechanism,'' in {\em
  Festschrift for Soren Hallden}, pp.~27--42, Theoria, 1983.

\bibitem{zadeh1986}
L.~A. Zadeh, ``Is probability theory sufficient for dealing with uncertainty in
  {AI}: a negative view,'' in {\em Uncertainty in Artificial Intelligence}
  (L.~N. Kanal and J.~F. Lemmer, eds.), vol.~2, pp.~103--116, Amsterdam:
  North-Holland, 1986.

\bibitem{Smets88}
P.~Smets, ``Belief functions,'' in {\em Non-Standard Logics for Automated
  Reasoning} (P.~Smets, A.~Mamdani, D.~Dubois, and H.~Prade, eds.),
  pp.~253--286, Academic Press, London, 1988.

\bibitem{Shafer85a}
G.~Shafer, ``Nonadditive probability,'' in {\em Encyclopedia of Statistical
  Sciences} (Kotz and Johnson, eds.), pp.~6, 271--276, Wiley, 1985.

\bibitem{yager99modeling}
R.~R. Yager, ``Modeling uncertainty using partial information,'' {\em
  Information Sciences}, vol.~121, pp.~271--294, 1999.

\bibitem{cozman00reasoning}
F.~G. Cozman and S.~Moral, ``Reasoning with imprecise probabilities,'' {\em
  International Journal of Approximate Reasoning}, vol.~24, pp.~121--123, 2000.

\bibitem{polkowski96mereology}
L.~Polkowski and A.~Skowron, ``Rough mereology: A new paradigm for approximate
  reasoning,'' {\em International Journal of Approximate Reasoning}, vol.~15,
  pp.~333--365, 1996.

\bibitem{shafer78bernoulli}
G.~Shafer, ``Nonadditive probabilites in the work of {B}ernoulli and
  {L}ambert,'' {\em Arch. History Exact Sci.}, vol.~19, pp.~309--370, 1978.

\bibitem{klir1988}
G.~J. Klir and T.~A. Folger, {\em Fuzzy Sets, Uncertainty and Information}.
\newblock Englewood Cliffs (NJ): Prentice Hall, 1988.

\bibitem{krause1993}
P.~Krause and D.~Clark, {\em Representing Uncertain Knowledge}.
\newblock Dordrecht: Kluwer, 1993.

\bibitem{goodman85uncertainty}
I.~R. Goodman and H.~T. Nguyen, {\em Uncertainty Models for Knowledge-based
  systems}.
\newblock New York: North Holland, 1985.

\bibitem{smithson1989}
M.~J. Smithson, {\em Ignorance and Uncertainty: Emerging Paradigms}.
\newblock New York (NY): Springer, 1989.

\bibitem{Grabish95}
M.~Grabisch, H.~T. Nguyen, and E.~A. Walker, {\em Fundamentals of uncertainty
  calculi with applications to fuzzy inference}.
\newblock Kluwer Academic Publishers, 1995.

\bibitem{klir95book}
G.~J. Klir and B.~Yuan, {\em Fuzzy sets and fuzzy logic: theory and
  applications}.
\newblock Upper Saddle River, NJ: Prentice Hall PTR, 1995.

\bibitem{km95book}
J.~Kohlas and P.-A. Monney, {\em A Mathematical Theory of Hints. An Approach to
  {D}empster-{S}hafer Theory of Evidence}, vol.~425 of {\em Lecture Notes in
  Economics and Mathematical Systems}.
\newblock Springer-Verlag, 1995.

\bibitem{shafer01book}
G.~Shafer and V.~Vovk, {\em Probability and Finance: It's Only a Game!}
\newblock New York: Wiley, 2001.

\bibitem{halpern03book}
J.~Halpern, {\em Reasoning About Uncertainty}.
\newblock MIT Press, 2003.

\bibitem{definetti74}
B.~D. Finetti, {\em Theory of Probability}.
\newblock Wiley, London, 1974.

\bibitem{kuhr2007finetti}
J.~K{\\"u}hr and D.~Mundici, ``{De Finetti theorem and Borel states in [0,
  1]-valued algebraic logic},'' {\em International journal of approximate
  reasoning}, vol.~46, no.~3, pp.~605--616, 2007.

\bibitem{dubois88possibility}
D.~Dubois and H.~Prade, {\em Possibility theory}.
\newblock New York: Plenum Press, 1988.

\bibitem{matheronrandom}
G.~Matheron, {\em Random Sets and Integral Geometry}.
\newblock Wiley Series in Probability and Mathematical Statistics.

\bibitem{ross86random}
D.~Ross, ``Random sets without separability,'' {\em Annals of Probability},
  vol.~14:3, pp.~1064--1069, July 1986.

\bibitem{walley91book}
P.~Walley, {\em Statistical Reasoning with Imprecise Probabilities}.
\newblock New York: Chapman and Hall, 1991.

\bibitem{hendon96product}
E.~Hendon, H.~J. Jacobsen, B.~Sloth, and T.~Tranaes, ``The product of
  capacities and belief functions,'' {\em Mathematical Social Sciences},
  vol.~32, pp.~95--108, 1996.

\bibitem{denneberg00totally}
D.~Denneberg, ``Totally monotone core and products of monotone measures,'' {\em
  International Journal of Approximate Reasoning}, vol.~24, pp.~273--281, 2000.

\bibitem{wang97choquet}
Z.~Wang and G.~J. Klir, ``Choquet integrals and natural extensions of lower
  probabilities,'' {\em International Journal of Approximate Reasoning},
  vol.~16, pp.~137--147, 1997.

\bibitem{pal92uncertainty1}
N.~Pal, J.~Bezdek, and R.~Hemasinha, ``Uncertainty measures for evidential
  reasoning i: a review,'' {\em International Journal of Approximate
  Reasoning}, vol.~7, pp.~165--183, 1992.

\bibitem{pal92uncertainty2}
N.~Pal, J.~Bezdek, and R.~Hemasinha, ``Uncertainty measures for evidential
  reasoning i: a review,'' {\em International Journal of Approximate
  Reasoning}, vol.~8, pp.~1--16, 1993.

\bibitem{Dubois92}
D.~Dubois and H.~Prade, ``Evidence, knowledge, and belief functions,'' {\em
  International Journal of Approximate Reasoning}, vol.~6, pp.~295--319, 1992.

\bibitem{Fine77}
T.~L. Fine, ``Review of a mathematical theory of evidence,'' {\em Bulletin of
  the American Mathematical Society}, vol.~83, pp.~667--672, 1977.

\bibitem{zadeh84review}
L.~A. Zadeh, ``A mathematical theory of evidence (book review),'' {\em AI
  Magazine}, vol.~5:3, pp.~81--83, 1984.

\bibitem{diaconis78review}
P.~Diaconis, ``Review of 'a mathematical theory of evidence','' {\em Journal of
  American Statistical Society}, vol.~73:363, pp.~677--678, 1978.

\bibitem{lemmers86confidence}
J.~F. Lemmers, ``Confidence factors, empiricism, and the {D}empster-{S}hafer
  theory of evidence,'' in {\em Uncertainty in Artificial Intelligence} (L.~N.
  Kanal and J.~F. Lemmers, eds.), pp.~167--196, North Holland, Amsterdam, 1986.

\bibitem{dubois87principle}
D.~Dubois and H.~Prade, ``The principle of minimum specificity as a basis for
  evidential reasoning,'' in {\em Uncertainty in Knowledge-Based Systems}
  (B.~Bouchon and R.~R. Yager, eds.), pp.~75--84, Springer-Verlag, Berlin,
  1987.

\bibitem{stein93ds}
R.~Stein, ``The {D}empster-{S}hafer theory of evidential reasoning,'' {\em AI
  Expert}, vol.~8:8, pp.~26--31, August 1993.

\bibitem{spillman90managing}
R.~Spillman, ``Managing uncertainty with belief functions,'' {\em AI Expert},
  vol.~5:5, pp.~44--49, May 1990.

\bibitem{neapolitan93interpretation}
R.~E. Neapolitan, ``The interpretation and application of belief functions,''
  {\em Applied Artificial Intelligence}, vol.~7:2, pp.~195--204, April-June
  1993.

\bibitem{strat89explaining}
T.~M. Strat and J.~D. Lowrance, ``Explaining evidential analysis,'' {\em
  International Journal of Approximate Reasoning}, vol.~3, pp.~299--353, 1989.

\bibitem{Wasserman92}
L.~A. Wasserman, ``Comments on shafer's `perspectives on the theory and
  practice of belief functions`,'' {\em International Journal of Approximate
  Reasoning}, vol.~6, pp.~367--375, 1992.

\bibitem{aitchinson68discussion}
J.~Aitchinson, ``Discussion on professor {D}empster's paper,'' {\em Journal of
  the Royal Statistical Society B}, vol.~30, pp.~234--237, 1968.

\bibitem{zaffalon04incomplete}
G.~de~Cooman and M.~Zaffalon, ``Updating beliefs with incomplete
  observations,'' {\em Artif. Intell.}, vol.~159, no.~1-2, pp.~75--125, 2004.

\bibitem{Shafer76}
G.~Shafer, {\em A Mathematical Theory of Evidence}.
\newblock Princeton University Press, 1976.

\bibitem{Dempster67}
A.~P. Dempster, ``Upper and lower probability inferences based on a sample from
  a finite univariate population,'' {\em Biometrika}, vol.~54, pp.~515--528,
  1967.

\bibitem{Dempster68a}
A.~P. Dempster, ``Upper and lower probabilities generated by a random closed
  interval,'' {\em Annals of Mathematical Statistics}, vol.~39, pp.~957--966,
  1968.

\bibitem{Dempster69}
A.~P. Dempster, ``Upper and lower probabilities inferences for families of
  hypothesis with monotone density ratios,'' {\em Annals of Mathematical
  Statistics}, vol.~40, pp.~953--969, 1969.

\bibitem{Dempster08b}
A.~P. Dempster, ``A generalization of bayesian inference,'' in {\em Classic
  Works of the Dempster-Shafer Theory of Belief Functions}, pp.~73--104, 2008.

\bibitem{hajek92deriving}
P.~Hajek, ``Deriving {D}empster's rule,'' in {\em Proceeding of IPMU'92},
  pp.~73--75, 1992.

\bibitem{wilson91prior}
N.~Wilson, ``The representation of prior knowledge in a {D}empster-{S}hafer
  approach,'' in {\em TR/Drums Conference}, Blanes, 1991.

\bibitem{bernardo94}
J.~M. Bernardo and F.~M. Smith, {\em Bayesian Theory}.
\newblock Wiley, 1994.

\bibitem{dempster90bayes}
A.~P. Dempster, ``{B}ayes, {F}ischer, and belief fuctions,'' in {\em Bayesian
  and Likelihood Methods in Statistics and Economics} (S.~J.~P. S.~Geisser, J.
  S.~Hodges and A.~Zellner, eds.), 1990.

\bibitem{lewis76}
D.~Lewis, ``Probabilities of conditionals and conditional probabilities,'' {\em
  Philosophical Review}, vol.~85, pp.~297--315, 1976.

\bibitem{kolmogorov}
A.~N. Kolmogorov, {\em Foundations of the theory of probability}.
\newblock Chealsea Pub. Co., Oxford, 1950.

\bibitem{Rudin}
W.~Rudin, {\em Real and Complex Analysis}.
\newblock McGraw-Hill, 1987.

\bibitem{Rosenthal}
K.~I. Rosenthal, {\em Quantales and their applications}.
\newblock Longman house, Burnt Mill, Harlow, Essex, UK: Longman scientific and
  technical, 1990.

\bibitem{Augustin96}
T.~Augustin, ``Modeling weak information with generalized basic probability
  assignments,'' in {\em Data Analysis and Information Systems - Statistical
  and Conceptual Approaches} (H.~H. Bock and W.~Polasek, eds.), pp.~101--113,
  Springer, 1996.

\bibitem{dubois86set}
D.~Dubois and H.~Prade, ``A set theoretical view of belief functions,'' {\em
  International Journal of Intelligent Systems}, vol.~12, pp.~193--226, 1986.

\bibitem{denneberg99interaction}
D.~Denneberg and M.~Grabisch, ``Interaction transform of set functions over a
  finite set,'' {\em Information Sciences}, vol.~121, pp.~149--170, 1999.

\bibitem{kruse91tool}
R.~Kruse, E.~Schwecke, and F.~Klawonn, ``On a tool for reasoning with mass
  distribution,'' in {\em Proceedings of the 12th International Joint
  Conference on Artificial Intelligence (IJCAI91)}, vol.~2, pp.~1190--1195,
  1991.

\bibitem{kruse91reasoning}
R.~Kruse, D.~Nauck, and F.~Klawonn, ``Reasoning with mass,'' in {\em
  Uncertainty in Artificial Intelligence} (P.~S. B.~D.~D\'Ambrosio and P.~P.
  Bonissone, eds.), pp.~182--187, Morgan Kaufmann, San Mateo, CA, 1991.

\bibitem{Stern}
M.~Stern, {\em Semimodular lattices}.
\newblock Cambridge University Press, 1999.

\bibitem{grabish06moebius}
M.~Grabisch, ``The {M}oebius transform on symmetric ordered structures and its
  application to capacities on finite sets,'' {\em Discrete Mathematics},
  vol.~287 (1-3), pp.~17--34, 2004.

\bibitem{cuzzolin10ida}
F.~Cuzzolin, ``Three alternative combinatorial formulations of the theory of
  evidence,'' {\em Intelligent Data Analysis}, vol.~14, no.~4, pp.~439--464,
  2010.

\bibitem{dubois86unicity}
D.~Dubois and H.~Prade, ``On the unicity of {D}empster's rule of combination,''
  {\em International Journal of Intelligent Systems}, vol.~1, pp.~133--142,
  1986.

\bibitem{Sikorski}
R.~Sikorski, {\em Boolean algebras}.
\newblock Springer Verlag, 1964.

\bibitem{kohlas93c}
J.~Kohlas, ``Support and plausibility functions induced by filter-valued
  mappings,'' {\em Int. J. of General Systems}, vol.~21, no.~4, pp.~343--363,
  1993.

\bibitem{hajek92state}
P.~Hajek and D.~Harmanec, ``On belief functions (the present state of
  {D}empster-{S}hafer theory),'' in {\em Advanced topics in AI} (Marik, ed.),
  Springer-Verlag, 1992.

\bibitem{pearl89reasoning}
J.~Pearl, ``Reasoning with belief functions: a critical assessment,'' tech.
  rep., UCLA, Technical Report R-136, 1989.

\bibitem{lee88comparison}
C.-H. Lee, ``A comparison of two evidential reasoning schemes,'' {\em
  Artificial Intelligence}, vol.~35, pp.~127--134, 1988.

\bibitem{Pearl90}
J.~Pearl, ``Reasoning with belief functions: an analysis of compatibility,''
  {\em International Journal of Approximate Reasoning}, vol.~4, pp.~363--389,
  1990.

\bibitem{yen92computing}
J.~Yen, ``Computing generalized belief functions for continuous fuzzy sets,''
  {\em International Journal of Approximate Reasoning}, vol.~6, pp.~1--31,
  1992.

\bibitem{getler92failure}
J.~J. Gertler and K.~C. Anderson, ``An evidential reasoning extension to
  quantitative model-based failure diagnosis,'' {\em IEEE Transactions on
  Systems, Man, and Cybernetics}, vol.~22:2, pp.~275--289, March/April 1992.

\bibitem{rakar99transferable}
A.~Rakar, A.~Jurii, and P.~Ball\'e, ``Transferable belief model in fault
  diagnosis,'' {\em Engineering Applications of Artificial Intelligence},
  vol.~12, pp.~555--567, 1999.

\bibitem{demotier06smcc}
S.~Demotier, W.~Schon, and T.~Denoeux, ``Risk assessment based on weak
  information using belief functions: a case study in water treatment,'' {\em
  IEEE Transactions on Systems, Man and Cybernetics, Part C}, vol.~36(3),
  pp.~382-- 396, May 2006.

\bibitem{lesh86evidential}
S.~A. Lesh, {\em An evidential theory approach to judgement-based decision
  making}.
\newblock {PhD} dissertation, Department of Forestry and Environmental Studies,
  Duke University, December 1986.

\bibitem{boucher90speech}
L.~Boucher, T.~Simons, and P.~Green, ``Evidential reasoning and the combination
  of knowledge and statistical techniques in syllable based speech
  recognition,'' in {\em Proceedings of the NATO Advanced Study Institute,
  Speech Recognition and Understanding. Recent Advances, Trends and
  Applications} (R.~Laface, P.; De~Mori, ed.), pp.~487--492, Cetraro, Italy,
  1-13 July 1990.

\bibitem{cleynenbreugel91road}
J.~V. Cleynenbreugel, S.~A. Osinga, F.~Fierens, P.~Suetens, and A.~Oosterlinck,
  ``Road extraction from multitemporal satellite images by an evidential
  reasoning approach,'' {\em Pattern Recognition Letters}, vol.~12:6,
  pp.~371--380, June 1991.

\bibitem{Fua86}
P.~Fua, ``Using probability density functions in the framework of evidential
  reasoning,'' {\em Uncertainty in Knowledge-Based Systems, Lectures Notes in
  Computer science}, vol.~286, pp.~243--252, 1986.

\bibitem{cross91compatibilityIEEE}
V.~Cross and T.~Sudkamp, ``Compatibility and aggregation in fuzzy evidential
  reasoning,'' in {\em Proceedings of IEEE}, pp.~1901--1906, 1991.

\bibitem{ramer96comparative}
A.~Ramer, ``Text on evidence theory: comparative review,'' {\em International
  Journal of Approximate Reasoning}, vol.~14, pp.~217--220, 1996.

\bibitem{kofler94algorithmic}
E.~T. Kofler and C.~T. Leondes, ``Algorithmic modifications to the theory of
  evidential reasoning,'' {\em Journal of Algorithms}, vol.~17:2, pp.~269--279,
  September 1994.

\bibitem{dempster90construction}
A.~P. Dempster, ``Construction and local computation aspects of network belief
  functions,'' in {\em Influence Diagrams, Belief Nets and Decision Analysis}
  (R.~M. Oliver and J.~Q. Smith, eds.), Wiley, Chirichester, 1990.

\bibitem{wang92exponential}
S.~Wang and M.~Valtorta, ``On the exponential growth rate of
  {D}empster-{S}hafer belief functions,'' in {\em Proceedings of the SPIE -
  Applications of Artificial Intelligence X: Knowledge-Based Systems},
  vol.~1707, pp.~15--24, Orlando, FL, USA, 22-24 April 1992.

\bibitem{roesmer00nonstandard}
C.~Roesmer, ``Nonstandard analysis and {D}empster-shafer theory,'' {\em
  International Journal of Intelligent Systems}, vol.~15, pp.~117--127, 2000.

\bibitem{jaffray89coherent}
J.~Y. Jaffray, ``Coherent bets under partially resolving uncertainty and belief
  functions,'' {\em Theory and Decision}, vol.~26, pp.~99--105, 1989.

\bibitem{shenoy90axioms}
P.~P. Shenoy and G.~Shafer, ``Axioms for probability and belief functions
  propagation,'' in {\em Uncertainty in Artificial Intelligence, 4} (L.~N.~K.
  R.~D.~Shachter, T. S.~Lewitt and J.~F. Lemmer, eds.), pp.~159--198, North
  Holland, Amsterdam, 1990.

\bibitem{klawonn90axiomatic}
F.~Klawonn and E.~Schweke, ``On the axiomatic justification of {D}empster's
  rule of combination,'' {\em International Journal of Intelligent Systems},
  vol.~7, pp.~469--478, 1990.

\bibitem{Ginsberg84}
M.~L. Ginsberg, ``Non-monotonic reasoning using {D}empster's rule,'' in {\em
  Proc. 3rd National Conference on AI (AAAI-84)}, pp.~126--129, 1984.

\bibitem{zadeh86simple}
L.~Zadeh, ``A simple view of the {D}empster-{S}hafer theory of evidence and its
  implication for the rule of combination,'' {\em AI Magazine}, vol.~7, no.~2,
  pp.~85--90, 1986.

\bibitem{fagin91new}
R.~Fagin and J.~Halpern, ``A new approach to updating beliefs,'' in {\em Proc.
  of UAI}, pp.~347--374, 1991.

\bibitem{dasilva92algorithms}
W.~T. da~Silva and R.~L. Milidiu, ``Algorithms for combining belief
  functions,'' {\em International Journal of Approximate Reasoning}, vol.~7,
  pp.~73--94, 1992.

\bibitem{sudkamp92consistency}
T.~Sudkamp, ``The consistency of {D}empster-{S}hafer updating,'' {\em
  International Journal of Approximate Reasoning}, vol.~7, pp.~19--44, 1992.

\bibitem{smets86bayes}
P.~Smets, ``Bayes' theorem generalized for belief functions,'' in {\em
  Proceedings of ECAI-86}, vol.~2, pp.~169--171, 1986.

\bibitem{Voorbraak91}
F.~Voorbraak, ``On the justification of {D}empster's rule of combination,''
  {\em Artificial Intelligence}, vol.~48, pp.~171--197, 1991.

\bibitem{Zadeh86}
L.~A. Zadeh, ``A simple view of the {D}empster-{S}hafer theory of evidence and
  its implications for the rule of combination,'' {\em AI Magazine}, vol.~7:2,
  pp.~85--90, 1986.

\bibitem{Shafer86}
G.~Shafer, ``The combination of evidence,'' {\em International Journal of
  Intelligent Systems}, vol.~1, pp.~155--179, 1986.

\bibitem{Wilson92}
N.~Wilson, ``The combination of belief: when and how fast?,'' {\em
  International Journal of Approximate Reasoning}, vol.~6, pp.~377--388, 1992.

\bibitem{Shafer81}
G.~Shafer, ``Constructive probability,'' {\em Synthese}, vol.~48, pp.~309--370,
  1981.

\bibitem{Shafer87d}
G.~Shafer, ``Probability judgment in artificial intelligence and expert
  systems,'' {\em Statistical Science}, vol.~2, pp.~3--44, 1987.

\bibitem{Shafer90}
G.~Shafer, ``Perspectives on the theory and practice of belief functions,''
  {\em International Journal of Approximate Reasoning}, vol.~4, pp.~323--362,
  1990.

\bibitem{smets87versus}
P.~Smets, ``Upper and lower probability functions versus belief functions,'' in
  {\em Proceedings of the International Symposium on Fuzzy Systems and
  Knowledge Engineering}, pp.~17--21, Guangzhou, China, 1987.

\bibitem{Shafer87c}
G.~Shafer, ``Belief functions and possibility measures,'' in {\em Analysis of
  Fuzzy Information 1: Mathematics and logic} (Bezdek, ed.), pp.~51--84, CRC
  Press, 1987.

\bibitem{Srivastava89}
G.~Shafer and R.~Srivastava, ``The {B}ayesian and belief-function formalism: A
  general perspective for auditing,'' {\em Auditing: A Journal of Practice and
  Theory}, 1989.

\bibitem{lowrance88automated}
J.~D. Lowrance, ``Automated argument construction,'' {\em Journal of
  Statistical Planning Inference}, vol.~20, pp.~369--387, 1988.

\bibitem{yagernonmonotonicity}
R.~R. Yager, ``Nonmonotonicity and compatibility relations in belief
  structures,''

\bibitem{goutsias97random}
J.~Goutsias, R.~P. Mahler, and H.~T. Nguyen, {\em Random sets: theory and
  applications ({IMA} {V}olumes in {M}athematics and {I}ts {A}pplications,
  {V}ol. 97)}.
\newblock Springer-Verlag, December 1997.

\bibitem{Goutsias98}
J.~Goutsias, ``Modeling random shapes: an introduction to random closed set
  theory,'' tech. rep., Department of Electrical and Computer Engineering, John
  Hopkins University, Baltimore, JHU/ECE 90-12, April 1998.

\bibitem{Nguyen97}
H.~T. Nguyen and T.~Wang, ``Belief functions and random sets,'' in {\em
  Applications and Theory of Random Sets, The IMA Volumes in Mathematics and
  its Applications, Vol. 97}, pp.~243--255, Springer, 1997.

\bibitem{Nguyen78}
H.~T. Nguyen, ``On random sets and belief functions,'' {\em J. Mathematical
  Analysis and Applications}, vol.~65, pp.~531--542, 1978.

\bibitem{Hestir91}
H.~T. Hestir, H.~T. Nguyen, and G.~S. Rogers, ``A random set formalism for
  evidential reasoning,'' in {\em Conditional Logic in Expert Systems},
  pp.~309--344, North Holland, 1991.

\bibitem{Shafer87b}
G.~Shafer, P.~P. Shenoy, and K.~Mellouli, ``Propagating belief functions in
  qualitative {M}arkov trees,'' {\em International Journal of Approximate
  Reasoning}, vol.~1, pp.~(4), 349--400, 1987.

\bibitem{smets92TBMrandom}
P.~Smets, ``The transferable belief model and random sets,'' {\em International
  Journal of Intelligent Systems}, vol.~7, pp.~37--46, 1992.

\bibitem{Ruspini87}
E.~H. Ruspini, ``Epistemic logics, probability and the calculus of evidence,''
  in {\em Proc. 10th Intl. Joint Conf. on AI (IJCAI-87)}, pp.~924--931, 1987.

\bibitem{Ruspini92}
E.~H. Ruspini, J.~D. Lowrance, and T.~M. Strat, ``Understanding evidential
  reasoning,'' {\em International Journal of Approximate Reasoning}, vol.~6,
  pp.~401--424, 1992.

\bibitem{Fagin88}
R.~Fagin and J.~Halpern, ``Uncertainty, belief and probability,'' in {\em Proc.
  Intl. Joint Conf. in AI (IJCAI-89)}, pp.~1161--1167, 1988.

\bibitem{laskey87beliefs}
K.~B. Laskey, ``Beliefs in belief functions: an examination of {S}hafer's
  canonical examples,'' in {\em AAAI Third Workshop on Uncertainty in
  Artificial Intelligence}, pp.~39--46, Seattle, 1987.

\bibitem{kyburg87bayesian}
H.~Kyburg, ``Bayesian and non-{B}ayesian evidential updating,'' {\em Artificial
  Intelligence}, vol.~31, no.~3, pp.~271--294, 1987.

\bibitem{seidenfeld97some}
T.~Seidenfeld, ``Some static and dynamic aspects of rubust {B}ayesian theory,''
  in {\em Random Sets: Theory and Applications} (Goutsias, Malher, and Nguyen,
  eds.), pp.~385--406, Springer, 1997.

\bibitem{levi80enterprise}
I.~Levi, {\em The enterprise of knowledge}.
\newblock MIT Press, 1980.

\bibitem{zaffalon-treebased}
M.~Zaffalon and E.~Fagiuoli, ``Tree-based credal networks for classification.''

\bibitem{cuzzolin2010credal}
F.~Cuzzolin, ``{Credal semantics of Bayesian transformations in terms of
  probability intervals},'' {\em Systems, Man, and Cybernetics, Part B:
  Cybernetics, IEEE Transactions on}, vol.~40, no.~2, pp.~421--432, 2010.

\bibitem{antonucci10ipmu}
A.~Antonucci and F.~Cuzzolin, ``Credal sets approximation by lower
  probabilities: Application to credal networks,'' in {\em Proc. of IPMU 2010},
  2010.

\bibitem{chateauneuf1989}
A.~Chateauneuf and J.~Y. Jaffray, ``Some characterizations of lower
  probabilities and other monotone capacities through the use of {M}\"obius
  inversion,'' {\em Mathematical Social Sciences}, vol.~17, pp.~263--283, 1989.

\bibitem{cuzzolin08jelia}
F.~Cuzzolin, ``On the credal structure of consistent probabilities,'' in {\em
  Logics in Artificial Intelligence}, vol.~5293/2008, pp.~126--139, Springer
  Berlin / Heidelberg, 2008.

\bibitem{wallner2005}
A.~Wallner, ``Maximal number of vertices of polytopes defined by
  f-probabilities,'' in {\em ISIPTA 2005 -- Proceedings of the Fourth
  International Symposium on Imprecise Probabilities and Their Applications}
  (F.~G. Cozman, R.~Nau, and T.~Seidenfeld, eds.), pp.~126--139, SIPTA, 2005.

\bibitem{Lowrance82}
J.~D. Lowrance and T.~D. Garvey, ``Evidential reasoning: A developing
  concept,'' in {\em Proceedings of the Internation Conference on Cybernetics
  and Society} (I.~of~Electrical and E.~Engineers, eds.), pp.~6--9, 1982.

\bibitem{benferhat95belief}
S.~Benferhat, A.~Saffiotti, and P.~Smets, ``Belief functions and default
  reasoning,'' in {\em Procs. of the 11th Conf. on Uncertainty in AI. Montreal,
  Canada}, pp.~19--26, 1995.

\bibitem{smets94what}
P.~Smets, ``What is {D}empster-{S}hafer's model ?,'' in {\em Advances in the
  Dempster-Shafer Theory of Evidence} (F.~M. Yager~R.R. and K.~J., eds.),
  pp.~5--34, Wiley, 1994.

\bibitem{smets91updating}
P.~Smets, ``About updating,'' in {\em Proceedings of the 7th conference on
  Uncertainty in Artificial Intelligence} (B.~D\'ambrosio, P.~Smets, and
  B.~P.~P. and, eds.), pp.~378--385, 1991.

\bibitem{williams82discussion}
P.~M. Williams, ``Discussion of shafer's paper,'' {\em Journal of the Royal
  Statistical Society B}, vol.~44, pp.~322--352, 1982.

\bibitem{williams1978}
P.~M. Williams, ``On a new theory of epistemic probability,'' {\em British
  Journal for the Philosophy of Science}, vol.~29, pp.~375--387, 1978.

\bibitem{wilson92howmuch}
N.~Wilson, ``How much do you believe,'' {\em International Journal of
  Approximate Reasoning}, vol.~6, pp.~345--365, 1992.

\bibitem{provan92validity}
G.~Provan, ``The validity of {D}empster-{S}hafer belief functions,'' {\em
  International Journal of Approximate Reasoning}, vol.~6, pp.~389--399, 1992.

\bibitem{smets97normative}
P.~Smets, ``The normative representation of quantified beliefs by belief
  functions,'' {\em Artificial Intelligence}, vol.~92, pp.~229--242, 1997.

\bibitem{shapley71cores}
L.~Shapley, ``Cores of convex games,'' {\em Int. J. Game Theory}, vol.~1,
  pp.~11�--26, 1971.

\bibitem{smets1993no}
P.~Smets, ``{No Dutch book can be built against the TBM even though update is
  not obtained by Bayes rule of conditioning},'' in {\em Workshop on
  probabilistic expert systems, Societa Italiana di Statistica, Roma},
  pp.~181--204, 1993.

\bibitem{wakker99dempster}
P.~P. Wakker, ``Dempster-belief functions are based on the principle of
  complete ignorance,'' in {\em Proceedings of the 1st International Sysmposium
  on Imprecise Probabilites and Their Applications}, pp.~535--542, Ghent,
  Belgium, 29 June - 2 July 1999.

\bibitem{Shafer04comments}
G.~Shafer, ``Comments on "constructing a logic of plausible inference: a guide
  to cox's theorem", by kevin s. van horn,'' {\em Int. J. Approx. Reasoning},
  vol.~35, no.~1, pp.~97--105, 2004.

\bibitem{smets88versus}
P.~Smets, ``Transferable belief model versus {B}ayesian model,'' in {\em
  Proceedings of ECAI 1988} (K.~Y., ed.), pp.~495--500, Pitman, London, 1988.

\bibitem{shafer82bayes}
G.~Shafer, ``Bayes's two arguments for the rule of conditioning,'' {\em Annals
  of Statistics}, vol.~10:4, pp.~1075--1089, December 1982.

\bibitem{shenoy88axiomatic}
P.~P. Shenoy and G.~Shafer, ``An axiomatic framework for {B}ayesian and belief
  function propagation,'' in {\em Proceedings of the AAAI Workshop of
  Uncertainty in Artificial Intelligence}, pp.~307--314, 1988.

\bibitem{halpern92twoviews}
J.~Y. Halpern and R.~Fagin, ``Two views of belief: belief as generalized
  probability and belief as evidence,'' {\em Artificial Intelligence}, vol.~54,
  pp.~275--317, 1992.

\bibitem{smets93quantifying}
P.~Smets, ``Quantifying beliefs by belief functions : An axiomatic
  justification,'' in {\em Proceedings of the 13th International Joint
  Conference on Artificial Intelligence, IJCAI93}, pp.~598--603, 1993.

\bibitem{smets92concept}
P.~Smets, ``The concept of distinct evidence,'' in {\em Proceedings of the 4th
  Conference on Information Processing and Management of Uncertainty in
  Knowledge-Based Systems (IPMU 92)}, pp.~789--794, Palma de Mallorca, 6-10
  July 92.

\bibitem{smets92resolving}
P.~Smets, ``Resolving misunderstandings about belief functions','' {\em
  International Journal of Approximate Reasoning}, vol.~6, pp.~321--34, 1992.

\bibitem{den99reasoning}
T.~Denoeux, ``Reasoning with imprecise belief structures,'' {\em International
  Journal of Approximate Reasoning}, vol.~20, pp.~79--111, 1999.

\bibitem{campos05nlp}
F.~Campos and F.~de~Souza, ``Extending {D}empster-{S}hafer theory to overcome
  counter intuitive results,'' in {\em Proceedings of IEEE NLP-KE '05}, vol.~3,
  pp.~729-- 734, 2005.

\bibitem{grabish06lattices}
M.~Grabisch, ``Belief functions on lattices,'' {\em Int. J. of Intelligent
  Systems}, 2006.

\bibitem{Lowrance90}
J.~D. Lowrance, T.~Garvey, and T.~M. Strat, ``A framework for evidential
  reasoning systems,'' in {\em Readings in uncertain reasoning} (Shafer and
  Pearl, eds.), pp.~611--618, Morgan Kaufman, 1990.

\bibitem{Zarley88a}
D.~Zarley, Y.~Hsia, and G.~Shafer, ``Evidential reasoning using {DELIEF},'' in
  {\em Proc. Seventh National Conference on Artificial Intelligence}, vol.~1,
  pp.~205--209, 1988.

\bibitem{Laskey88}
K.~Laskey and P.~Lehner, ``Belief manteinance: an integrated approach to
  uncertainty management,'' in {\em Proceeding of the Seventh National
  Conference on Artificial Intelligence (AAAI-88)}, vol.~1, pp.~210--214, 1988.

\bibitem{Hajek96}
P.~Hajek, ``Getting belief functions from kripke models,'' {\em International
  Journal of General Systems}, vol.~24, pp.~325--327, 1996.

\bibitem{baldwin90general}
J.~F. Baldwin, ``Towards a general theory of evidential reasoning,'' in {\em
  Proceedings of the 3rd International Conference on Information Processing and
  Management of Uncertainty in Knowledge-Based Systems (IPMU'90)}
  (B.~Bouchon-Meunier, R.~Yager, and L.~Zadeh, eds.), pp.~360--369, Paris,
  France, 2-6 July 1990.

\bibitem{baldwin91combining}
J.~F. Baldwin, ``Combining evidences for evidential reasoning,'' {\em
  International Journal of Intelligent Systems}, vol.~6:6, pp.~569--616,
  September 1991.

\bibitem{wang94robust}
C.-C. Wang and H.-S. Don, ``A robust continuous model for evidential
  reasoning,'' {\em Journal of Intelligent and Robotic Systems: Theory and
  Applications}, vol.~10:2, pp.~147--171, June 1994.

\bibitem{wang94polar}
C.-C. Wanga and H.-S. Don, ``A polar model for evidential reasoning,'' {\em
  Information Sciences}, vol.~77:3-4, pp.~195--226, March 1994.

\bibitem{an93relation}
Z.~An, D.~A. Bell, and J.~G. Hughes, ``Relation-based evidential reasoning,''
  {\em International Journal of Approximate Reasoning}, vol.~8, pp.~231--251,
  1993.

\bibitem{Kramosil96nonnumerical}
I.~Kramosil, ``Expert systems with non-numerical belief functions,'' {\em
  Problems of control and information theory}, vol.~16, pp.~39--53, 1996.

\bibitem{andersen96linear}
K.~A. Andersen and J.~N. Hooker, ``A linear programming framework for logics of
  uncertainty,'' {\em Decision Support Systems}, vol.~16, pp.~39--53, 1996.

\bibitem{smarandache2005introduction}
F.~Smarandache and J.~Dezert, ``{An introduction to the DSm theory for the
  combination of paradoxical, uncertain and imprecise sources of
  information},'' in {\em Proceedings of the 13th International Congress of
  Cybernetics and Systems}, pp.~6--10, 2005.

\bibitem{Lowrance86}
J.~D. Lowrance, T.~D. Garvey, and T.~M. Strat, ``A framework for
  evidential-reasoning systems,'' in {\em Proceedings of the National
  Conference on Artificial Intelligence} (A.~A. for Artificial~Intelligence,
  ed.), pp.~896--903, 1986.

\bibitem{lamata94calculus}
M.~Lamata and S.~Moral, ``Calculus with linguistic probabilites and belief,''
  in {\em Advances in the Dempster-Shafer Theory of Evidence}, pp.~133--152,
  Wiley, New York, 1994.

\bibitem{smets90PAMI}
P.~Smets, ``The combination of evidence in the transferable belief models,''
  {\em IEEE Transactions on PAMI}, vol.~12, pp.~447--458, 1990.

\bibitem{smets94transferable}
P.~Smets and R.~Kennes, ``The {T}ransferable {B}elief {M}odel,'' {\em
  Artificial Intelligence}, vol.~66, pp.~191--234, 1994.

\bibitem{smets95axiomatic}
P.~Smets, ``The axiomatic justification of the transferable belief model,''
  tech. rep., TR/IRIDIA/1995-8.1, Universite' Libre de Bruxelles, 1995.

\bibitem{smets98quantified}
P.~Smets, ``The transferable belief model for quantified belief
  representation,'' in {\em Handbook of Defeasible Reasoning and Uncertainty
  Management Systems, Vol. 1: Quantified Representation of Uncertainty and
  Imprecision} (G.~D. and S.~Ph., eds.), pp.~267--301, Kluwer, Doordrecht,
  1998.

\bibitem{smets97TBMbelief}
P.~Smets and R.~Kruse, ``The transferable belief model for belief
  representation,'' in {\em Uncertainty Management in information systems: from
  needs to solutions} (M.~A. and S.~Ph., eds.), pp.~343--368, Kluwer, Boston,
  1997.

\bibitem{ubf}
P.~Smets, ``The nature of the unnormalized beliefs encountered in the
  transferable belief model,'' in {\em Proceedings of the 8th Annual Conference
  on Uncertainty in Artificial Intelligence (UAI-92)}, (San Mateo, CA),
  pp.~292--29, Morgan Kaufmann, 1992.

\bibitem{cuzzolin14tfs}
F.~Cuzzolin, ``Lp consonant approximations of belief functions,'' {\em IEEE
  Transactions on Fuzzy Systems}, vol.~22.

\bibitem{smets92nature}
P.~Smets, ``The nature of the unnormalized beliefs encountered in the
  transferable belief model,'' in {\em Proceedings of the 8th Conference on
  Uncertainty in Artificial Intelligence (AI92)} (D.~B. Dubois~D.,
  Wellmann~M.P. and S.~Ph., eds.), pp.~292--297, 1992.

\bibitem{smets91other}
P.~Smets, ``The transferable belief model and other interpretations of
  {D}empster-{S}hafer's model,'' in {\em Uncertainty in Artificial
  Intelligence, volume 6} (P.~Bonissone, M.~Henrion, L.~Kanal, and J.~Lemmer,
  eds.), pp.~375--383, North-Holland, Amsterdam, 1991.

\bibitem{smets90constructing}
P.~Smets, ``Constructing the pignistic probability function in a context of
  uncertainty,'' in {\em Uncertainty in Artificial Intelligence, 5}
  (M.~Henrion, R.~Shachter, L.~Kanal, and J.~Lemmer, eds.), pp.~29--39,
  Elsevier Science Publishers, 1990.

\bibitem{smets98application}
P.~Smets, ``The application of the transferable belief model to diagnostic
  problems,'' {\em Int. J. Intelligent Systems}, vol.~13, pp.~127--158, 1998.

\bibitem{smets92reliability}
P.~Smets, ``The transferable belief model for expert judgments and reliability
  problems,'' {\em Reliability Engineering and System Safety}, vol.~38,
  pp.~59--66, 1992.

\bibitem{dubois01using}
D.~Dubois, M.~Grabisch, H.~Prade, and P.~Smets, ``Using the transferable belief
  model and a qualitative possibility theory approach on an illustrative
  example: the assessment of the value of a candidate,'' {\em Intern. J.
  Intell. Systems}, 2001.

\bibitem{snow98vulnerability}
P.~Snow, ``The vulnerability of the transferable belief model to dutch books,''
  {\em Artificial Intelligence}, vol.~105, pp.~345--354, 1998.

\bibitem{Kramosil96a}
I.~Kramosil, ``{D}empster-{S}hafer theory with indiscernible states and
  observations,'' {\em International Journal of General Systems}, vol.~25,
  pp.~147--152, 1996.

\bibitem{kramosil93boolean}
I.~Kramosil, ``Toward a boolean-valued {D}empster-{S}hafer theory,'' in {\em
  LOGICA '92} (S.~V., ed.), pp.~110--131, Prague, 1993.

\bibitem{kramosil97nonstandard}
I.~Kramosil, ``Belief functions with nonstandard values,'' in {\em Proceedings
  of Qualitative and Quantitative Practical Reasoning} (D.~Gabbay, R.~Kruse,
  A.~Nonnengart, and H.~J. Ohlbach, eds.), pp.~380--391, Bonn, June 1997.

\bibitem{kramosil94definability}
I.~Kramosil, ``Definability of belief functions over countable sets by
  real-valued random variables,'' in {\em IPMU. Information Processing and
  Management of Uncertainty in Knowledge-Based Systems} (S.~V., ed.), vol.~3,
  pp.~49--50, Paris, July 1994.

\bibitem{kramosil94stronglaw}
I.~Kramosil, ``Strong law of large numbers for set-valued random variables,''
  in {\em Proceedings of the 3rd Workshop on Uncertainty Processing in Expert
  Systems}, pp.~122--142, Prague, University of Economics, September 1994.

\bibitem{kramosil97belief}
I.~Kramosil, ``Belief functions generated by signed measures,'' {\em Fuzzy Sets
  and Systems}, vol.~92, pp.~157--166, 1997.

\bibitem{kramosil96jordan}
I.~Kramosil, ``Jordan decomposition of signed belief functions,'' in {\em
  Proceedings IPMU'96. Information Processing and Management of Uncertainty in
  Knowledge-Based Systems}, pp.~431--434, Granada, Universidad de Granada, July
  1996.

\bibitem{kramosil97dempster}
I.~Kramosil, ``A probabilistic analysis of {D}empster combination rule,'' in
  {\em The Logica. Yearbook 1997} (C.~Timothy, ed.), pp.~174--187, Prague,
  1997.

\bibitem{kramosil98dempstersigned}
I.~Kramosil, ``Dempster combination rule for signed belief functions,'' {\em
  International Journal of Uncertainty, Fuzziness and Knowledge-Based Systems},
  vol.~6:1, pp.~79--102, February 1998.

\bibitem{kramosil97probabilistic}
I.~Kramosil, ``Probabilistic analysis of {D}empster-{S}hafer theory. part
  one,'' tech. rep., Academy of Science of the Czech Republic, Technical Report
  716, 1997.

\bibitem{kramosil98probabilistic}
I.~Kramosil, ``Probabilistic analysis of {D}empster-{S}hafer theory. part
  two.,'' tech. rep., Academy of Science of the Czech Republic, Technical
  Report 749, 1998.

\bibitem{Kyburg87}
H.~E. Kyburg, ``Bayesian and non-{B}ayesian evidential updating,'' {\em
  Artificial Intelligence}, vol.~31, pp.~271--293, 1987.

\bibitem{lemmer91conditions}
J.~F. Lemmer and J.~H.~E.~Kyburg, ``Conditions for the existence of belief
  functions corresponding to intervals of belief,'' in {\em Proceedings of the
  Ninth National Conference on Artificial Intelligence, (AAAI-91)},
  pp.~488--493, Anaheim, CA, USA, 14-19 July 1991.

\bibitem{yu94conditional}
C.~Yu and F.~Arasta, ``On conditional belief functions,'' {\em International
  Journal of Approxiomate Reasoning}, vol.~10, pp.~155--172, 1994.

\bibitem{Chateauneuf89}
A.~Chateauneuf and J.~Jaffray, ``Some characterization of lower probabilities
  and other monotone capacities through the use of {M}oebius inversion,'' {\em
  Math. Soc. Sci.}, vol.~17, pp.~263--283, 1989.

\bibitem{Jaffray92}
J.~Y. Jaffray, ``Bayesian updating and belief functions,'' {\em IEEE
  Transactions on Systems, Man and Cybernetics}, vol.~22, pp.~1144--1152, 1992.

\bibitem{gilboa93updating}
I.~Gilboa and D.~Schmeidler, ``Updating ambiguous beliefs,'' {\em Journal of
  economic theory}, vol.~59, pp.~33--49, 1993.

\bibitem{denneberg94conditioning}
D.~Denneberg, ``Conditioning (updating) non-additive probabilities,'' {\em Ann.
  Operations Res.}, vol.~52, pp.~21--42, 1994.

\bibitem{itoh95new}
M.~Itoh and T.~Inagaki, ``A new conditioning rule for belief updating in the
  {D}empster-{S}hafer theory of evidence,'' {\em Transactions of the Society of
  Instrument and Control Engineers}, vol.~31:12, pp.~2011--2017, 1995.

\bibitem{spies94conditional}
M.~Spies, ``Conditional events, conditioning, and random sets,'' {\em IEEE
  Transactions on Systems, Man, and Cybernetics}, vol.~24, pp.~1755--1763,
  1994.

\bibitem{slobodova97lncs}
A.~Slobodova, ``Multivalued extension of conditional belief functions,'' in
  {\em Qualitative and quantitative practical reasoning}, vol.~1244/1997,
  pp.~568--573, Springer.

\bibitem{slobodova94conditional}
A.~Slobodova, ``Conditional belief functions and valuation-based systems,''
  tech. rep., Slovak Academy of Sciences, 1994.

\bibitem{smets93belief}
P.~Smets, ``Belief functions : the disjunctive rule of combination and the
  generalized {B}ayesian theorem,'' {\em International Journal of Approximate
  Reasoning}, vol.~9, pp.~1--35, 1993.

\bibitem{hsia91characterizing}
Y.~T. Hsia, ``Characterizing belief functions with minimal commitment,'' in
  {\em Proceedings of IJCAI-91}, pp.~1184--1189, 1991.

\bibitem{klawonn92dynamic}
F.~Klawonn and P.~Smets, ``The dynamic of belief in the transferable belief
  model and specialization-generalization matrices,'' in {\em Proceedings of
  UAI'92}, pp.~130--137.

\bibitem{xu94evidential}
H.~Xu and P.~Smets, ``Evidential reasoning with conditional belief functions,''
  in {\em Proceedings of the 10th Uncertainty in Artificial Intelligence}
  (L.~de~Mantaras~R. and P.~D., eds.), pp.~598--605, 1994.

\bibitem{xu96reasoning}
H.~Xu and P.~Smets, ``Reasoning in evidential networks with conditional belief
  functions,'' {\em International Journal of Approximate Reasoning}, vol.~14,
  pp.~155--185, 1996.

\bibitem{smets93jeffrey}
P.~Smets, ``Jeffrey's rule of conditioning generalized to belief functions,''
  in {\em Proceedings of UAI'93}, pp.~500--505.

\bibitem{perea09amodel}
A.~Perea, ``A model of minimal probabilistic belief revision,'' {\em Theory and
  Decision}, vol.~67, no.~2, pp.~163--222, 2009.

\bibitem{suppes1977}
P.~Suppes and M.~Zanotti, ``On using random relations to generate upper and
  lower probabilities,'' {\em Synthese}, vol.~36, pp.~427--440, 1977.

\bibitem{jeffrey65book}
R.~Jeffrey, {\em The logic of decision}.
\newblock Mc Graw - Hill, 1965.

\bibitem{shafer81jeffrey}
G.~Shafer, ``Jeffrey's rule of conditioning,'' {\em Philosophy of Sciences},
  vol.~48, pp.~337--362, 1981.

\bibitem{jeffrey1988conditioning}
R.~Jeffrey, ``{Conditioning, kinematics, and exchangeability},'' {\em
  Causation, chance, and credence}, vol.~1, pp.~221--255, 1988.

\bibitem{klopotek99lncs}
M.~Klopotek and S.~Wierzchon, ``An interpretation for the conditional belief
  function in the theory of evidence,'' in {\em LNCS}, vol.~1609/1999,
  pp.~494--502, Springer, 1999.

\bibitem{tang05dempster}
Y.~Tang and J.~Zheng, ``Dempster conditioning and conditional independence in
  evidence theory,'' in {\em AI 2005: Advance in Artificial Intelligence},
  vol.~3809/2005, pp.~822--825, Springer Berlin/Heidelberg, 2005.

\bibitem{lehrer05updating}
E.~Lehrer, ``Updating non-additive probabilities - a geometric approach,'' {\em
  Games and Economic Behavior}, vol.~50, pp.~42--57, 2005.

\bibitem{krantz83priors}
D.~H. Krantz and J.~Miyamoto, ``Priors and likelihood ratios as evidence,''
  {\em Journal of the American Statistical Association}, vol.~78, pp.~418--423,
  June 1983.

\bibitem{chateauneuf00ambiguity}
A.~Chateauneuf and J.-C. Vergnaud, ``Ambiguity reduction through new
  statistical data,'' {\em International Journal of Approximate Reasoning},
  vol.~24, pp.~283--299, 2000.

\bibitem{Smets97FAPR2}
P.~Smets and R.~Cooke, ``How to derive belief functions within probabilistic
  frameworks?,'' in {\em Proceedings of the International Joint Conference on
  Qualitative and Quantitative Practical Reasoning (ECSQARU / FAPR '97)}, Bad
  Honnef, Germany, 9-12 June 1997.

\bibitem{bryson98qualitative}
N.~Bryson and A.~Mobolurin, ``Qualitative discriminant approach for generating
  quantitative belief functions,'' {\em IEEE Transactions on Knowledge and Data
  Engineering}, vol.~10, pp.~345--348, 1998.

\bibitem{ngwenyama98generating}
O.~K. Ngwenyama and N.~Bryson, ``Generating belief functions from qualitative
  preferences: An approach to eliciting expert judgments and deriving
  probability functions,'' {\em Data and Knowledge Engineering}, vol.~28,
  pp.~145--159, 1998.

\bibitem{seidenfeld78}
T.~Seidenfeld, ``Statistical evidence and belief functions,'' in {\em
  Proceedings of the Biennial Meeting of the Philosophy of Science
  Association}, vol.~1978, pp.~478--489, 1978.

\bibitem{durham92statistical}
S.~D. Durham, J.~S. Smolka, and M.~Valtorta, ``Statistical consistency with
  {D}empster's rule on diagnostic trees having uncertain performance
  parameters,'' {\em International Journal of Approximate Reasoning}, vol.~6,
  pp.~67--81, 1992.

\bibitem{watada94logical}
J.~Watada, Y.~Kubo, and K.~Kuroda, ``Logical approach: to evidential reasoning
  under a hierarchical structure,'' in {\em Proceedings of the International
  Conference on Data and Knowledge Systems for Manufacturing and Engineering},
  vol.~1, pp.~285--290, Hong Kong, 2-4 May 1994.

\bibitem{dempster1966}
A.~P. Dempster, ``New methods for reasoning towards posterior distributions
  based on sample data,'' {\em Annals of Mathematical Statistics}, vol.~37,
  pp.~355--374, 1966.

\bibitem{Wasserman90prior}
L.~Wasserman, ``Prior envelopes based on belief functions,'' {\em Annals of
  Statistics}, vol.~18, pp.~454--464, 1990.

\bibitem{Beran71}
R.~J. Beran, ``On distribution-free statistical inference with upper and lower
  probabilities,'' {\em Annals of Mathematical Statistics}, vol.~42,
  pp.~157--168, 1971.

\bibitem{dutta1985}
A.~Dutta, ``Reasoning with imprecise knowledge in expert systems,'' {\em
  Information Sciences}, vol.~37, pp.~3--24, 1985.

\bibitem{Chen95}
Y.~Y. Chen, ``Statistical inference based on the possibility and belief
  measures,'' {\em Transactions of the American Mathematical Society},
  vol.~347, pp.~1855--1863, 1995.

\bibitem{chen1995}
Y.~Y. Chen, ``Statistical inference based on the possibility and belief
  measures,'' {\em Transactions of the American Mathematical Society},
  vol.~347, pp.~1855--1863, 1995.

\bibitem{Wasserman90}
L.~A. Wasserman, ``Belief functions and statistical inference,'' {\em Canadian
  Journal of Statistics}, vol.~18, pp.~183--196, 1990.

\bibitem{shafer82parametric}
G.~Shafer, ``Belief functions and parametric models,'' {\em Journal of the
  Royal Statistical Society, Series B}, vol.~44, pp.~322--352, 1982.

\bibitem{aregui08constructing}
A.~Aregui and T.~Denoeux, ``Constructing consonant belief functions from sample
  data using confidence sets of pignistic probabilities,'' {\em International
  Journal of Approximate Reasoning}, vol.~49, no.~3, pp.~575--594, 2008.

\bibitem{Walley87}
P.~Walley, ``Belief function representations of statistical evidence,'' {\em
  The Annals of Statistics}, vol.~15, pp.~1439--1465, 1987.

\bibitem{acker00belief}
C.~V. den Acker, ``Belief function representation of statistical audit
  evidence,'' {\em International Journal of Intelligent Systems}, vol.~15,
  pp.~277--290, 2000.

\bibitem{srivastava94integrating}
R.~P. Srivastava and G.~Shafer, ``Integrating statistical and nonstatistical
  audit evidence using belief functions: a case of variable sampling,'' {\em
  International Journal of Intelligent Systems}, vol.~9:6, pp.~519--539, June
  1994.

\bibitem{hummel88statistical}
R.~Hummel and M.~Landy, ``A statistical viewpoint on the theory of evidence,''
  {\em IEEE Transactions on PAMI}, pp.~235--247, 1988.

\bibitem{liu97method}
J.~Liu and M.~C. Desmarais, ``Method of learning implication networks from
  empirical data: algorithm and monte-carlo simulation-based validation,'' {\em
  IEEE Transactions on Knowledge and Data Engineering}, vol.~9, pp.~990--1004,
  1997.

\bibitem{wong94representation}
S.~K.~M. Wong and P.~Lingras, ``Representation of qualitative user preference
  by quantitative belief functions,'' {\em IEEE Transactions on Knowledge and
  Data Engineering}, vol.~6:1, pp.~72--78, February 1994.

\bibitem{wong90generation}
S.~Wong and P.~Lingas, ``Generation of belief functions from qualitative
  preference relations,'' in {\em Proceedings of the Third International
  Conference IPMU}, pp.~427--429, 1990.

\bibitem{seidenfeld07isipta}
T.~Seidenfeld, M.~Schervish, and J.~Kadane, ``Coherent choice functions under
  uncertainty,'' in {\em Proceedings of ISIPTA'07}, 2007.

\bibitem{einhorn1986}
H.~J. Einhorn and R.~M. Hogarth, ``Decision making under ambiguity,'' {\em
  Journal of Business}, vol.~59, pp.~S225--S250, 1986.

\bibitem{smets2002decision}
P.~Smets, ``{Decision making in a context where uncertainty is represented by
  belief functions},'' {\em Belief functions in business decisions},
  pp.~17--61, 2002.

\bibitem{smets2005decision}
P.~Smets, ``{Decision making in the TBM: the necessity of the pignistic
  transformation},'' {\em International Journal of Approximate Reasoning},
  vol.~38, no.~2, pp.~133--147, 2005.

\bibitem{Caselton92}
W.~F. Caselton and W.~Luo, ``Decision making with imprecise probabilities:
  {D}empster-{S}hafer theory and application,'' {\em Water Resources Research},
  vol.~28, pp.~3071--3083, 1992.

\bibitem{Jaffray94}
J.~Y. Jaffray and P.~P. Wakker, ``Decision making with belief functions:
  compatibility and incompatibility with the sure-thing principle,'' {\em
  Journal of Risk and Uncertainty}, vol.~8, pp.~255--271, 1994.

\bibitem{Strat90}
T.~Strat, ``Decision analysis using belief functions,'' {\em International
  Journal of Approximate Reasoning}, vol.~4, pp.~391--417, 1990.

\bibitem{klir94dynamic}
G.~J. Klir, ``Dynamic decision making with belief functions,'' in {\em Measures
  of uncertainty in the {D}empster-{S}hafer theory of evidence} (M.~F.
  R.~R.~Yager and J.~Kacprzyk, eds.), pp.~35--49, Wiley, New York, 1994.

\bibitem{horiuchi98decision}
T.~Horiuchi, ``Decision rule for pattern classification by integrating interval
  feature values,'' {\em IEEE Transactions on Pattern Analysis and Machine
  Intelligence}, vol.~20, pp.~440--448, 1998.

\bibitem{beynon00alternative}
M.~Beynon, B.~Curry, and P.~Morgan, ``The {D}empster-{S}hafer theory of
  evidence: approach to multicriteria decision modeling,'' {\em OMEGA: The
  International Journal of Management Science}, vol.~28, pp.~37--50, 2000.

\bibitem{smets01decision}
P.~Smets, ``Decision making in a context where uncertainty is represented by
  belief functions,'' in {\em Belief Functions in Business Decisions} (S.~R.,
  ed.), pp.~495--504, Physica-Verlag, 2001.

\bibitem{strat94decisionanalysis}
T.~M. Strat, ``Decision analysis using belief functions,'' in {\em Advances in
  the Dempster-Shafer Theory of Evidence}, Wiley, New York, 1994.

\bibitem{maluf97monotonicity}
D.~A. Maluf, ``Monotonicity of entropy computations in belief functions,'' {\em
  Intelligent Data Analysis}, vol.~1, pp.~207--213, 1997.

\bibitem{schubert94thesis}
J.~Schubert, {\em Cluster-based Specification Techniques in {D}empster-{S}hafer
  Theory for an Evidential Intelligence Analysis of MultipleTarget Tracks}.
\newblock {PhD} dissertation, Royal Institute of Technology, Sweden, 1994.

\bibitem{schubert95onrho}
J.~Schubert, ``On \'rho\' in a decision-theoretic apparatus of
  {D}empster-{S}hafer theory,'' {\em International Journal of Approximate
  Reasoning}, vol.~13, pp.~185--200, 1995.

\bibitem{elouedi00classification}
Z.~Elouedi, K.~Mellouli, and P.~Smets, ``Classification with belief decision
  trees,'' in {\em Proceedings of the Nineth International Conference on
  Artificial Intelligence: Methodology, Systems, Architectures: AIMSA 2000},
  Varna, Bulgaria, 2000.

\bibitem{elouedi00decision}
Z.~Elouedi, K.~Mellouli, and P.~Smets, ``Decision trees using belief function
  theory,'' in {\em Proceedings of the Eighth International Conference IPMU:
  Information Processing and Management of Uncertainty in Knowledge-based
  Systems}, vol.~1, pp.~141--148, Madrid, 2000.

\bibitem{Xu92a}
H.~Xu, ``A decision calculus for belief functions in valuation-based systems,''
  in {\em Proceedings of the 8th Uncertainty in Artificial Intelligence} (D.~D.
  W. M. P.~D. B. and S.~Ph., eds.), pp.~352--359, 1992.

\bibitem{xu96transferable}
H.~Xu, Y.~Hsia, and P.~Smets, ``Transferable belief model for decision making
  in valuation based systems,'' {\em IEEE Transactions on Systems, Man, and
  Cybernetics}, vol.~26:6, pp.~698--707, 1996.

\bibitem{xu96decision}
H.~Xu, Y.-T. Hsia, and P.~Smets, ``The transferable belief model for decision
  making in the valuation-based systems,'' {\em IEEE Transactions on Systems,
  Man, and Cybernetics}, vol.~26A, pp.~698--707, 1996.

\bibitem{vonneumann44}
J.~von Neumann and O.~Morgenstern, {\em Theory of Games and Economic Behavior}.
\newblock Princeton University Press, 1944.

\bibitem{smets05ijar}
P.~smets, ``Decision making in the tbm: the necessity of the pignistic
  transformation,'' {\em International Journal of Approximate Reasoning},
  vol.~38(2), pp.~133--147, February 2005.

\bibitem{yaffray88application}
J.~Y. Jaffray, ``Application of linear utility theory for belief functions,''
  in {\em Uncertainty and Intelligent Systems}, pp.~1--8, Springer-Verlag,
  Berlin, 1988.

\bibitem{jaffray89linear}
J.~Y. Jaffray, ``Linear utility theory for belief functions,'' {\em Operation
  Research Letters}, vol.~8, pp.~107--112, 1989.

\bibitem{yaffray94dynamic}
J.~Y. Jaffray, ``Dynamic decision making with belief functions,'' in {\em
  Advances in the Dempster-Shafer Theory of Evidence} (M.~F. R.~R.~Yager and
  J.~Kacprzyk, eds.), pp.~331--352, Wiley, New York, 1994.

\bibitem{troffaes07}
M.~Troffaes, ``Decision making under uncertainty using imprecise
  probabilities,'' {\em International Journal of Approximate Reasoning},
  vol.~45, no.~1, pp.~17--29, 2007.

\bibitem{denoeux08ai}
T.~Denoeux, ``Conjunctive and disjunctive combination of belief functions
  induced by non distinct bodies of evidence,'' {\em Artificial Intelligence}.

\bibitem{ristic04ipmu}
B.~Ristic and P.~Smets, ``Belief function theory on the continuous space with
  an application to model based classification,'' in {\em IPMU},
  pp.~1119--1126, 2004.

\bibitem{caron06ijar}
F.~Caron, B.~Ristic, E.~Duflos, and P.~Vanheeghe, ``Least committed basic
  belief density induced by a multivariate gaussian: {F}ormulation with
  applications,'' {\em International Journal of Approximate Reasoning},
  vol.~48.

\bibitem{denoeux08flairs}
T.~Denoeux, ``A new justification of the unnormalized dempster's rule of
  combination from the {L}east {C}ommitment {P}rinciple,'' in {\em Proceedings
  of FLAIRS'08, Special Track on Uncertaint Reasoning}, 2008.

\bibitem{chau93upper}
C.~W.~R. Chau, P.~Lingras, and S.~K.~M. Wong, ``Upper and lower entropies of
  belief functions using compatible probability functions,'' in {\em
  Proceedings of the 7th International Symposium on Methodologies for
  Intelligent Systems (ISMIS'93)} (Z.~Komorowski, J.;~Ras, ed.), pp.~306--315,
  Trondheim, Norway, 15-18 June 1993.

\bibitem{yang94evidential}
J.-B. Yang and M.~G. Singh, ``An evidential reasoning approach for
  multiple-attribute decision making with uncertainty,'' {\em IEEE Transactions
  on Systems, Man, and Cybernetics}, vol.~24:1, pp.~1--18, January 1994.

\bibitem{xu97valuation}
H.~Xu, ``Valuation-based systems for decision analysis using belief
  functions,'' {\em Decision Support Systems}, vol.~20, pp.~165--184, 1997.

\bibitem{orponen90dempster}
P.~Orponen, ``{D}empster's rule of combination is np-complete,'' {\em
  Artificial Intelligence}, vol.~44, pp.~245--253, 1990.

\bibitem{clarke91efficient}
M.~Clarke and N.~Wilson, ``Efficient algorithms for belief functions based on
  the relationship between belief and probability,'' in {\em Proceedings of the
  European Conference on Symbolic and Quantitative Approaches to Uncertainty}
  (P.~Kruse, R.;~Siegel, ed.), pp.~48--52, Marseille, France, 15-17 October
  1991.

\bibitem{Barnett81}
J.~Barnett, ``Computational methods for a mathematical theory of evidence,'' in
  {\em Proc. of the 7th National Conference on Artificial Intelligence
  (AAAI-88)}, pp.~868--875, 1981.

\bibitem{guan94computational}
J.~W. Guan, D.~A. Bell, and Z.~Guan, ``Evidential reasoning in expert systems:
  computational methods,'' in {\em Proceedings of the Seventh International
  Conference on Industrial and Engineering Applications of Artificial
  Intelligence and Expert Systems (IEA/AIE-94)} (F.~Anger, R.~Rodriguez, and
  M.~Ali, eds.), pp.~657--666, Austin, TX, USA, 31 May - 3 June 1994.

\bibitem{kennes91fast}
R.~Kennes and P.~Smets, ``Fast algorithms for {D}empster-{S}hafer theory,'' in
  {\em Uncertainty in Knowledge Bases, Lecture Notes in Computer Science 521}
  (L.~Z. B.~Bouchon-Meunier, R.R.~Yager, ed.), pp.~14--23, Springer-Verlag,
  Berlin, 1991.

\bibitem{kennes91computational}
R.~Kennes and P.~Smets, ``Computational aspects of the moebius
  transformation,'' in {\em Uncertainty in Artificial Intelligence 6}
  (P.~Bonissone, M.~Henrion, L.~Kanal, and J.~Lemmer, eds.), pp.~401--416,
  Elsevier Science Publishers, 1991.

\bibitem{kennes92computational}
R.~Kennes, ``Computational aspects of the moebius transformation of graphs,''
  {\em IEEE Transactions on Systems, Man, and Cybernetics}, vol.~22,
  pp.~201--223, 1992.

\bibitem{Xu92}
H.~Xu, ``An efficient tool for reasoning with belief functions,'' in {\em Proc.
  of the 4th International Conference on Information Proceeding and Management
  of Uncertainty in Knowledge-Based Systems}, pp.~65--68, 1992.

\bibitem{km91}
J.~Kohlas and P.-A. Monney, ``Propagating belief functions through constraint
  systems,'' {\em Int. J. Approximate Reasoning}, vol.~5, pp.~433--461, 1991.

\bibitem{kohlas89b}
J.~Kohlas, ``Modeling uncertainty with belief functions in numerical models,''
  {\em Europ. J. of Operational Research}, vol.~40, pp.~377--388, 1989.

\bibitem{Xu94}
H.~Xu and R.~Kennes, ``Steps towards an efficient implementation of
  {D}empster-{S}hafer theory,'' in {\em Advances in the Dempster-Shafer Theory
  of Evidence} (R.~R. Yager, M.~Fedrizzi, and J.~Kacprzyk, eds.), pp.~153--174,
  John Wiley and Sons, Inc., 1994.

\bibitem{Thoma91}
H.~M. Thoma, ``Belief function computations,'' in {\em Conditional Logic in
  Expert Systems}, pp.~269--308, North Holland, 1991.

\bibitem{bissig97fastdivision}
R.~Bissig, J.~Kohlas, and N.~Lehmann, ``Fast-division architecture for
  {D}empster-{S}hafer belief functions,'' in {\em Qualitative and Quantitative
  Practical Reasoning, First International Joint Conference on Qualitative and
  Quantitative Practical Reasoning; ECSQARU--FAPR'97} (D.~Gabbay, R.~Kruse,
  A.~Nonnengart, and H.~Ohlbach, eds.), Springer, 1997.

\bibitem{gordon85method}
J.~Gordon and E.~H. Shortliffe, ``A method for managing evidential reasoning in
  a hierarchical hypothesis space,'' {\em Artificial Intelligence}, vol.~26,
  pp.~323--357, 1985.

\bibitem{Shafer87a}
G.~Shafer and R.~Logan, ``Implementing {D}empster's rule for hierarchical
  evidence,'' {\em Artificial Intelligence}, vol.~33, pp.~271--298, 1987.

\bibitem{Lehmann99}
N.~Lehmann and R.~Haenni, ``An alternative to outward propagation for
  {D}empster-{S}hafer belief functions,'' in {\em Proceedings of The Fifth
  European Conference on Symbolic and Quantitative Approaches to Reasoning with
  Uncertainty - Ecsqaru ( Lecture Notes in Computer Science Series)}, London,
  5-9 July 1999.

\bibitem{yager86arithmetic}
R.~R. Yager, ``Arithmetic and other operations on {D}empster-{S}hafer
  structures,'' {\em International Journal of Man-Machine Studies}, vol.~25,
  pp.~357--366, 1986.

\bibitem{Xu94a}
H.~Xu, ``Computing marginals from the marginal representation in {M}arkov
  trees,'' in {\em Proc. of the 5th International Conference on Information
  Proceeding and Management of Uncertainty in Knowledge-Based Systems},
  pp.~275--280, 1994.

\bibitem{xu95computing}
H.~Xu, ``Computing marginals from the marginal representation in {M}arkov
  trees,'' {\em Artificial Intelligence}, vol.~74, pp.~177--189, 1995.

\bibitem{almond95book}
R.~G. Almond, {\em Graphical Belief Modeling}.
\newblock Chapman and Hall/CRC, 1995.

\bibitem{yaghlane06ipmu}
B.~B. Yaghlane and K.~Mellouli, ``Belief function propagation in directed
  evidential networks,'' in {\em IPMU}, 2006.

\bibitem{bergsten93dempster}
U.~Bergsten and J.~Schubert, ``Dempster's rule for evidence ordered in a
  complete directed acyclic graph,'' {\em International Journal of Approximate
  Reasoning}, vol.~9, pp.~37--73, 1993.

\bibitem{Kong86a}
A.~Kong, {\em Multivariate belief functions and graphical models}.
\newblock {PhD} dissertation, Harvard University, Department of Statistics,
  1986.

\bibitem{Mellouli86}
K.~Mellouli, {\em On the propagation of beliefs in networks using the
  {D}empster-{S}hafer theory of evidence}.
\newblock {PhD} dissertation, University of Kansas, School of Business, 1986.

\bibitem{mellouli97pooling}
K.~Mellouli and Z.~Elouedi, ``Pooling experts opinion using {D}empster-{S}hafer
  theory of evidence,'' in {\em Proceedings of IEEE}, pp.~1900--1905, 1997.

\bibitem{almond90thesis}
R.~G. Almond, {\em Fusion and propagation of graphical belief models: an
  implementation and an example}.
\newblock {PhD} dissertation, Department of Statistics, Harvard University,
  1990.

\bibitem{lepar98uai}
V.~Lepar and P.~Shenoy, ``A comparison of lauritzen-spiegelhalter, hugin, and
  shenoy-shafer architectures for computing marginals of probability
  distributions,'' in {\em UAI}, pp.~328--337, 1998.

\bibitem{cozman2000}
F.~G. Cozman, ``Credal networks,'' {\em Artificial Intelligence}, vol.~120,
  pp.~199--233, 2000.

\bibitem{wilson91montecarlo}
N.~Wilson, ``A {M}onte-{C}arlo algorithm for {D}empster-{S}hafer belief,'' in
  {\em Proc. of UAI}, pp.~414�--417, 1991.

\bibitem{moral99montecarlo}
S.~Moral and A.~Salmeron, ``A {M}onte-{C}arlo algorithm for combining
  {D}empster-{S}hafer belief based on approximate pre-computation,'' in {\em
  Symbolic and Quantitative Approaches to Reasoning with Uncertainty
  (ECSQARU'09), LNCS}, vol.~1638, pp.~305�--315, 1999.

\bibitem{kramosil98montecarlo}
I.~Kramosil, ``Monte-carlo estimations for belief functions,'' in {\em
  Proceedings of the Fourth International Conference on Fuzzy Sets Theory and
  Its Applications} (A.~Heckerman, D.;~Mamdani, ed.), vol.~16, pp.~339--357,
  Liptovsky Jan, Slovakia, 2-6 Feb. 1998.

\bibitem{resconi98speed-up}
G.~Resconi, A.~van~der Wal, and D.~Ruan, ``Speed-up of the monte carlo method
  by using a physical model of the {D}empster-{S}hafer theory,'' {\em
  International Journal of Intelligent Systems}, vol.~13, pp.~221--242, 1998.

\bibitem{Deutscher99}
B.~B. J.~Deutscher, B.~North and A.~Blake, ``Tracking through singularities and
  discontinuities by random sampling,'' in {\em Proceedings of ICCV'99},
  pp.~1144--1149, 1999.

\bibitem{Deutscher00}
J.~Deutscher, A.~Blake, and I.~Reid, ``Articulated body motion capture by
  annealed particle filtering,'' in {\em Proceedings of the IEEE Conference on
  Computer Vision and Pattern Recognition CVPR'00, Hilton Head Island, SC,
  USA}, pp.~126--133, July 2000.

\bibitem{daniel06ijis}
M.~Daniel, ``On transformations of belief functions to probabilities,'' {\em
  International Journal of Intelligent Systems}, vol.~21, no.~3, pp.~261--282,
  2006.

\bibitem{weiler94approximation}
T.~Weiler, ``Approximation of belief functions,'' {\em IJUFKS}, vol.~11, no.~6,
  pp.~749--777, 2003.

\bibitem{Kramosil95}
I.~Kramosil, ``Approximations of believeability functions under incomplete
  identification of sets of compatible states,'' {\em Kybernetika}, vol.~31,
  pp.~425--450, 1995.

\bibitem{bauer97approximation}
M.~Bauer, ``Approximation algorithms and decision making in the
  {D}empster-{S}hafer theory of evidence--an empirical study,'' {\em
  International Journal of Approximate Reasoning}, vol.~17, pp.~217--237, 1997.

\bibitem{Yaghlane01ecsqaru}
A.~B. Yaghlane, T.~Denoeux, and K.~Mellouli, ``Coarsening approximations of
  belief functions,'' in {\em Proceedings of ECSQARU'2001} (S.~Benferhat and
  P.~Besnard, eds.), pp.~362--373, 2001.

\bibitem{Denoeux01ijufk}
T.~Denoeux, ``Inner and outer approximation of belief structures using a
  hierarchical clustering approach,'' {\em Int. Journal of Uncertainty,
  Fuzziness and Knowledge-Based Systems}, vol.~9(4), pp.~437--460, 2001.

\bibitem{Denoeux02ijar}
T.~Denoeux and A.~B. Yaghlane, ``Approximating the combination of belief
  functions using the fast moebius transform in a coarsened frame,'' {\em
  International Journal of Approximate Reasoning}, vol.~31(1-2), pp.~77--101,
  October 2002.

\bibitem{Haenni02ijar}
R.~Haenni and N.~Lehmann, ``Resource bounded and anytime approximation of
  belief function computations,'' {\em International Journal of Approximate
  Reasoning}, vol.~31(1-2), pp.~103--154, October 2002.

\bibitem{BaroniV03}
P.~Baroni and P.~Vicig, ``Transformations from imprecise to precise
  probabilities,'' in {\em ECSQARU}, pp.~37--49, 2003.

\bibitem{daniel06on}
M.~Daniel, ``On transformations of belief functions to probabilities,'' {\em
  International Journal of Intelligent Systems, special issue on Uncertainty
  Processing}.

\bibitem{daniel04ipmu}
M.~Daniel, ``Consistency of probabilistic transformations of belief
  functions,'' in {\em IPMU}, pp.~1135--1142, 2004.

\bibitem{smets88beliefversus}
P.~Smets, ``Belief functions versus probability functions,'' in {\em
  Uncertainty and Intelligent Systems} (S.~L. Bouchon~B. and Y.~R., eds.),
  pp.~17--24, Springer Verlag, Berlin, 1988.

\bibitem{wilson93pignistic}
N.~Wilson, ``Decision making with belief functions and pignistic
  probabilities,'' in {\em Proceedings of ECSQARU}, pp.~364--371, Granada,
  1993.

\bibitem{burger09gene}
T.~Burger and A.~Caplier, ``{A Generalization of the Pignistic Transform for
  Partial Bet},'' in {\em Proceedings of the 10th European Conference on
  Symbolic and Quantitative Approaches to Reasoning with Uncertainty (ECSQARU),
  Verona, Italy, July 1-3}, pp.~252--263, Springer-Verlag New York Inc, 2009.

\bibitem{dezert2004generalized}
J.~Dezert, F.~Smarandache, and M.~Daniel, ``{The generalized pignistic
  transformation},'' {\em Arxiv preprint cs/0409007}, 2004.

\bibitem{voorbraak89efficient}
F.~Voorbraak, ``A computationally efficient approximation of
  {D}empster-{S}hafer theory,'' {\em International Journal on Man-Machine
  Studies}, vol.~30, pp.~525--536, 1989.

\bibitem{Cobb03isf}
B.~R. Cobb and P.~P. Shenoy, ``A comparison of bayesian and belief function
  reasoning,'' {\em Information Systems Frontiers}, vol.~5(4), pp.~345--358,
  2003.

\bibitem{Cobb03ecsqaru}
B.~R. Cobb and P.~P. Shenoy, ``A comparison of methods for transforming belief
  function models to probability models,'' in {\em Proceedings of ECSQARU'2003,
  Aalborg, Denmark}, pp.~255--266, July 2003.

\bibitem{cuzzolin10amai}
F.~Cuzzolin, ``Geometry of relative plausibility and relative belief of
  singletons,'' {\em Annals of Mathematics and Artificial Intelligence},
  vol.~59, pp.~47--79, May 2010.

\bibitem{cuzzolin08pricai}
F.~Cuzzolin, ``Dual properties of the relative belief of singletons,'' in {\em
  Proceedings of the Tenth Pacific Rim Conference on Artificial Intelligence
  (PRICAI'08), Hanoi, Vietnam, December 15-19 2008}, 2008.

\bibitem{cuzzolin2008semantics}
F.~Cuzzolin, ``{Semantics of the relative belief of singletons},'' {\em
  Interval/Probabilistic Uncertainty and Non-Classical Logics}, pp.~201--213,
  2008.

\bibitem{cuzzolin2008dual}
F.~Cuzzolin, ``{Dual properties of the relative belief of singletons},'' {\em
  PRICAI 2008: Trends in Artificial Intelligence}, pp.~78--90, 2008.

\bibitem{cuzzolin08unclog-semantics}
F.~Cuzzolin, ``Semantics of the relative belief of singletons,'' in {\em
  International Workshop on Uncertainty and Logic UNCLOG'08, Kanazawa, Japan},
  2008.

\bibitem{haenni08aggregating}
R.~Haenni, ``Aggregating referee scores: an algebraic approach,'' in {\em
  COMSOC'08, 2nd International Workshop on Computational Social Choice}
  (U.~Endriss and W.~Goldberg, eds.), pp.~277--288, 2008.

\bibitem{dezert}
J.~Dezert and F.~Smarandache, ``A new probabilistic transformation of belief
  mass assignment,'' 2007.

\bibitem{sudano01fusion}
J.~J. Sudano, ``Pignistic probability transforms for mixes of low- and high-
  probability events,'' in {\em Proceedings of the International Conference on
  Information Fusion}, 2001.

\bibitem{sudano01IPT}
J.~J. Sudano, ``Inverse pignistic probability transforms,'' in {\em Proceedings
  of the International Conference on Information Fusion}, 2002.

\bibitem{sudano03icif}
J.~Sudano, ``Equivalence between belief theories and na�ve bayesian fusion
  for systems with independent evidential data,'' in {\em Proceedings of the
  Sixth International Conference on Information Fusion (ISIF'03)}, 2003.

\bibitem{cuzzolin07smcb}
F.~Cuzzolin, ``Two new {B}ayesian approximations of belief functions based on
  convex geometry,'' {\em IEEE Transactions on Systems, Man, and Cybernetics -
  Part B}, vol.~37, no.~4, pp.~993--1008, 2007.

\bibitem{cuzzolin08smcc}
F.~Cuzzolin, ``A geometric approach to the theory of evidence,'' {\em IEEE
  Transactions on Systems, Man and Cybernetics part C}, vol.~38, no.~4,
  pp.~522--534, 2008.

\bibitem{cuzzolin11isipta-consonant}
F.~Cuzzolin, ``Lp consonant approximations of belief functions in the mass
  space,'' in {\em submitted to ISIPTA'11, Innsbruck, Austria}, 2011.

\bibitem{dubois93possibility-probability}
D.~Dubois, H.~Prade, and S.~Sandri, ``On possibility-probability
  transformations,'' in {\em Fuzzy Logic: State of the Art} (R.~Lowen and
  M.~Lowen, eds.), pp.~103--112, Kluwer Academic Publisher, 1993.

\bibitem{grabish97fss}
M.~Grabisch, ``K-order additive discrete fuzzy measures and their
  representation,'' {\em Fuzzy sets and systems}, vol.~92, pp.~167--189, 1997.

\bibitem{miranda06ejor}
P.~Miranda, M.~Grabisch, and P.~Gil, ``Dominance of capacities by k-additive
  belief functions,'' {\em European Journal of Operational Research}, vol.~175,
  pp.~912--930, 2006.

\bibitem{burger10approx}
T.~Burger, ``Defining new approximations of belief function by means of
  dempster's combination,'' in {\em Proceedings of the Workshop on the theory
  of belief functions}, 2010.

\bibitem{burger2010barycenters}
T.~Burger and F.~Cuzzolin, ``{The barycenters of the k-additive dominating
  belief functions \& the pignistic k-additive belief functions},'' 2010.

\bibitem{tessem93approximations}
B.~Tessem, ``Approximations for efficient computation in the theory of
  evidence,'' {\em Artificial Intelligence}, vol.~61:2, pp.~315--329, 1993.

\bibitem{wang92continuous}
C.-C. Wang and H.-S. Don, ``A continuous belief function model for evidential
  reasoning,'' in {\em Proceedings of the Ninth Biennial Conference of the
  Canadian Society for Computational Studies of Intelligence} (R.~Glasgow,
  J.;~Hadley, ed.), pp.~113--120, Vancouver, BC, Canada, 11-15 May 1992.

\bibitem{shafer79allocations}
G.~Shafer, ``Allocations of probability,'' {\em Annals of Probability},
  vol.~7:5, pp.~827--839, 1979.

\bibitem{honda06entropy}
A.~Honda and M.~Grabisch, ``Entropy of capacities on lattices and set
  systems,'' {\em To appear in Information Science}, 2006.

\bibitem{kohlas97allocation}
J.~Kohlas, ``Allocation of arguments and evidence theory,'' {\em Theoretical
  Computer Science}, vol.~171, pp.~221--246, 1997.

\bibitem{kb95}
J.~Kohlas and P.~Besnard, ``An algebraic study of argumentation systems and
  evidence theory,'' Tech. Rep. 95--13, Institute of Informatics, University of
  Fribourg, 1995.

\bibitem{kb94}
J.~Kohlas and H.~Brachinger, ``Argumentation systems and evidence theory,'' in
  {\em Advances in Intelligent Computing -- IPMU'94, Paris}
  (B.~Bouchon-Meunier, R.~Yager, and L.~Zadeh, eds.), pp.~41--50, Springer,
  1994.

\bibitem{Strat84}
T.~M. Strat, ``Continuous belief functions for evidential reasoning,'' in {\em
  Proceedings of the National Conference on Artificial Intelligence}
  (I.~of~Electrical and E.~Engineers, eds.), pp.~308--313, August 1984.

\bibitem{smets05real}
P.~Smets, ``Belief functions on real numbers,'' {\em International Journal of
  Approximate Reasoning}, vol.~40, no.~3, pp.~181--223, 2005.

\bibitem{aregui07isipta}
A.~Aregui and T.~Denoeux, ``Constructing predictive belief functions from
  continuous sample data using confidence bands,'' in {\em Proceedings of
  ISIPTA}, 2007.

\bibitem{kohlas95foundations}
J.~Kohlas, ``Mathematical foundations of evidence theory,'' in {\em
  Mathematical Models for Handling Partial Knowledge in Artificial
  Intelligence} (G.~Coletti, D.~Dubois, and R.~Scozzafava, eds.), pp.~31--64,
  Plenum Press, 1995.

\bibitem{kohlas94representation}
J.~Kohlas and P.-A. Monney, ``Representation of evidence by hints,'' in {\em
  Advances in the Dempster-Shafer Theory of Evidence} (R.~Yager, J.~Kacprzyk,
  and M.~Fedrizzi, eds.), pp.~473--492, John Wiley, New York, 1994.

\bibitem{vakili93}
Vakili, ``Approximation of hints,'' tech. rep., Institute for Automation and
  Operation Research, University of Fribourg, Switzerland, Tech. Report 209,
  1993.

\bibitem{kohlas95modelbased}
J.~Kohlas, P.-A. Monney, R.~Haenni, and N.~Lehmann, ``Model-based diagnostics
  using hints,'' in {\em Symbolic and Quantitative Approaches to Uncertainty,
  European Conference ECSQARU95, Fribourg} (C.~Fridevaux and J.~Kohlas, eds.),
  pp.~259--266, Springer, 1995.

\bibitem{maccheroni05annals}
F.~Maccheroni and M.~Marinacci, ``A strong law of large numbers for
  capacities,'' {\em The Annals of Probability}, vol.~33, no.~3,
  pp.~1171--1178, May 2005.

\bibitem{miranda03extreme}
E.~Miranda, I.~Couso, and P.~Gil, ``Extreme points of credal sets generated by
  2-alternating capacities,'' {\em International Journal of Approximate
  Reasoning}, vol.~33, pp.~95--115, 2003.

\bibitem{bruning02stat}
M.~Bruning and D.~Denneberg, ``Max-min $\sigma$-additive representation of
  monotone measures,'' {\em Statistical Papers}, vol.~34, pp.~23--35, 2002.

\bibitem{gilboa94additive}
I.~Gilboa and D.~Schmeidler, ``Additive representations of non-additive
  measures and the choquet integral,'' {\em Annals of Operations Research},
  vol.~52, no.~1, pp.~43--65, 1994.

\bibitem{smets95canonical}
P.~Smets, ``The canonical decomposition of a weighted belief,'' in {\em
  Proceedings of the International Joint Conference on AI, IJCAI�95},
  pp.~1896--1901, Montr\'eal, Canada, 1995.

\bibitem{kramosil97measure}
I.~Kramosil, ``Measure-theoretic approach to the inversion problem for belief
  functions,'' in {\em Proceedings of IFSA'97, Seventh International Fuzzy
  Systems Association World Congress}, vol.~1, pp.~454--459, Prague, Academia,
  June 1997.

\bibitem{Walley82frequentist}
P.~Walley and T.~L. Fine, ``Towards a frequentist theory of upper and lower
  probability,'' {\em The Annals of Statistics}, vol.~10, pp.~741--761, 1982.

\bibitem{denoeux06ipmu}
T.~Denoeux, ``Construction of predictive belief functions using a frequentist
  approach,'' in {\em IPMU}, 2006.

\bibitem{liu95model}
L.~Liu, ``Model combination using {G}aussian belief functions,'' tech. rep.,
  School of Business, University of Kansas, Lawrence, KS, 1995.

\bibitem{shafer92note}
G.~Shafer, ``A note on {D}empster's {G}aussian belief functions,'' tech. rep.,
  School of Business, University of Kansas, Lawrence, KS, 1992.

\bibitem{liu96theory}
L.~Liu, ``A theory of gaussian belief functions,'' {\em International Journal
  of Approximate Reasoning}, vol.~14, pp.~95--126, 1996.

\bibitem{dempster90normal}
A.~P. Dempster, ``Normal belief functions and the {K}alman filter,'' tech.
  rep., Department of Statistics, Harvard Univerisity, Cambridge, MA, 1990.

\bibitem{liu99local}
L.~Liu, ``Local computation of gaussian belief functions,'' {\em International
  Journal of Approximate Reasoning}, vol.~22, pp.~217--248, 1999.

\bibitem{yao98interpretations}
Y.~Y. Yao and P.~J. Lingras, ``Interpretations of belief functions in the
  theory of rough sets,'' {\em Information Sciences}, vol.~104(1-2),
  pp.~81--106, 1998.

\bibitem{maass06philosophical}
S.~Maass, ``A philosophical foundation of non-additive measure and
  probability,'' {\em Theory and decision}, vol.~60, pp.~175--191, 2006.

\bibitem{Gardenfors}
P.~Gardenfors, {\em Knowledge in Flux: Modeling the Dynamics of Epistemic
  States}.
\newblock MIT Press, Cambridge, MA, 1988.

\bibitem{joslyn98towards}
C.~Joslyn and L.~Rocha, ``Towards a formal taxonomy of hybrid uncertainty
  representations,'' {\em Information Sciences}, vol.~110, pp.~255--277, 1998.

\bibitem{klir99fuzzy}
G.~J. Klir, ``On fuzzy-set interpretation of possibility theory,'' {\em Fuzzy
  Sets and Systems}, vol.~108, pp.~263--273, 1999.

\bibitem{denoeux99reasoning}
T.~Denoeux, ``Reasoning with imprecise belief structures,'' {\em International
  Journal of Approximate Reasoning}, vol.~20, pp.~79--111, 1999.

\bibitem{Walley91coherent}
P.~Walley, ``Coherent lower (and upper) probabilities,'' tech. rep., University
  of Warwick, Coventry (U.K.), Statistics Research Report 22, 1981.

\bibitem{walley00towards}
P.~Walley, ``Towards a unified theory of imprecise probability,'' {\em
  International Journal of Approximate Reasoning}, vol.~24, pp.~125--148, 2000.

\bibitem{smets98which}
P.~Smets, ``Probability, possibility, belief: Which and where ?,'' in {\em
  Handbook of Defeasible Reasoning and Uncertainty Management Systems, Vol. 1:
  Quantified Representation of Uncertainty and Imprecision} (G.~D. and S.~Ph.,
  eds.), pp.~1--24, Kluwer, Doordrecht, 1998.

\bibitem{Keynes21pir}
J.~M. Keynes, ``Fundamental ideas,'' {\em A Treatise on Probability, Ch. 4},
  1921.

\bibitem{Walley91}
P.~Walley, {\em Statistical Reasoning with Imprecise Probabilities}.
\newblock London: Chapman and Hall, 1991.

\bibitem{tessem92interval}
B.~Tessem, ``Interval probability propagation,'' {\em IJAR}, vol.~7,
  pp.~95--120, 1992.

\bibitem{decampos94}
L.~de~Campos, J.~Huete, and S.~Moral, ``Probability intervals: a tool for
  uncertain reasoning,'' {\em Int. J. Uncertainty Fuzziness Knowledge-Based
  Syst.}, vol.~1, pp.~167--196, 1994.

\bibitem{moral93partially}
S.~Moral and L.~M. de~Campos, ``Partially specified belief functions,'' in {\em
  Proceedings of the Ninth Conference on Uncertainty in Artificial
  Intelligence} (A.~Heckerman, D.;~Mamdani, ed.), pp.~492--499, Washington, DC,
  USA, 9-11 July 1993.

\bibitem{unclog08book}
V.-N. Huynh, Y.~Nakamori, H.~Ono, J.~Lawry, V.~Kreinovich, and H.~Nguyen, eds.,
  {\em Interval / Probabilistic Uncertainty and Non-Classical Logics}.
\newblock Springer, 2008.

\bibitem{dubois83unfair}
D.~Dubois and H.~Prade, ``Unfair coins and necessity measures: towards a
  possibilistic interpretation of histograms,'' {\em Fuzzy Sets and Systems},
  vol.~10, no.~1, pp.~15--20, 1983.

\bibitem{dubois87properties}
D.~Dubois and H.~Prade, ``Properties of measures of information in evidence and
  possibility theories,'' {\em Fuzzy Sets and Systems}, vol.~24, pp.~161--182,
  1987.

\bibitem{joslyn91towards}
C.~Joslyn, ``Towards an empirical semantics of possibility through maximum
  uncertainty,'' in {\em Proc. IFSA 1991} (R.~Lowen and M.~Roubens, eds.),
  vol.~A, pp.~86--89, 1991.

\bibitem{smets90possibility}
P.~Smets, ``The transferable belief model and possibility theory,'' in {\em
  Proceedings of NAFIPS-90} (K.~Y., ed.), pp.~215--218, 1990.

\bibitem{dubois90}
D.~Dubois and H.~Prade, ``Consonant approximations of belief functions,'' {\em
  International Journal of Approximate Reasoning}, vol.~4, pp.~419--449, 1990.

\bibitem{cuzzolin09ecsqaru-outer}
F.~Cuzzolin, ``Complexes of outer consonant approximations,'' in {\em
  Proceedings of ECSQARU'09, Verona, Italy}, 2009.

\bibitem{cuzzolin10fss}
F.~Cuzzolin, ``The geometry of consonant belief functions: simplicial complexes
  of necessity measures,'' {\em Fuzzy Sets and Systems}, vol.~161, no.~10,
  pp.~1459--1479, 2010.

\bibitem{cliff92minimal}
C.~Joslyn and G.~Klir, ``Minimal information loss possibilistic approximations
  of random sets,'' in {\em Proc. 1992 FUZZ-IEEE Conference} (J.~Bezdek, ed.),
  pp.~1081--1088, 1992.

\bibitem{joslyn97possibilistic}
C.~Joslyn, ``Possibilistic normalization of inconsistent random intervals,''
  {\em Advances in Systems Science and Applications}, pp.~44--51, 1997.

\bibitem{cuzzolin04ipmu}
F.~Cuzzolin, ``Simplicial complexes of finite fuzzy sets,'' in {\em Proceedings
  of the $10^{th}$ International Conference on Information Processing and
  Management of Uncertainty IPMU'04, Perugia, Italy}, pp.~1733--1740, 2004.

\bibitem{smets81degree}
P.~Smets, ``The degree of belief in a fuzzy event,'' {\em Information
  Sciences}, vol.~25, pp.~1--19, 1981.

\bibitem{klir97constructing}
G.~J. Klir, W.~Zhenyuan, and D.~Harmanec, ``Constructing fuzzy measures in
  expert systems,'' {\em Fuzzy Sets and Systems}, vol.~92, pp.~251--264, 1997.

\bibitem{feriet1982}
J.~K. de~F\'eriet, ``Interpretation of membership functions of fuzzy sets in
  terms of plausibility and belief,'' in {\em Fuzzy Information and Decision
  Processes} (M.~M. Gupta and E.~Sanchez, eds.), pp.~93--98, Amsterdam:
  North-Holland, 1982.

\bibitem{lee95fuzzy}
E.~S. Lee and Q.~Zhu, {\em Fuzzy and Evidential Reasoning}.
\newblock Heidelberg: Physica-Verlag, 1995.

\bibitem{heilpern97representation}
S.~Heilpern, ``Representation and application of fuzzy numbers,'' {\em Fuzzy
  Sets and Systems}, vol.~91, pp.~259--268, 1997.

\bibitem{yager99class}
R.~R. Yager, ``Class of fuzzy measures generated from a {D}empster-{S}hafer
  belief structure,'' {\em International Journal of Intelligent Systems},
  vol.~14, pp.~1239--1247, 1999.

\bibitem{palacharla94understanding}
P.~Palacharla and P.~C. Nelson, ``Understanding relations between fuzzy logic
  and evidential reasoning methods,'' in {\em Proceedings of Third IEEE
  International Conference on Fuzzy Systems}, vol.~1, pp.~1933--1938, 1994.

\bibitem{romer95applicability}
C.~Roemer and A.~Kandel, ``Applicability analysis of fuzzy inference by means
  of generalized {D}empster-{S}hafer theory,'' {\em IEEE Transactions on Fuzzy
  Systems}, vol.~3:4, pp.~448--453, November 1995.

\bibitem{Renaud99}
S.~Petit-Renaud and T.~Denoeux, ``Handling different forms of uncertainty in
  regression analysis: a fuzzy belief structure approach,'' in {\em Proceedings
  of The Fifth European Conference on Symbolic and Quantitative Approaches to
  Reasoning with Uncertainty - Ecsqaru ( Lecture Notes in Computer Science
  Series)}, London, 5-9 July 1999.

\bibitem{lucas99generalization}
L.~Caro and A.~B. Nadjar, ``Generalization of the {D}empster-{S}hafer theory: a
  fuzzy-valued measure,'' {\em IEEE Transactions on Fuzzy Systems}, vol.~7,
  pp.~255--270, 1999.

\bibitem{yen90generalizing}
J.~Yen, ``Generalizing the {D}empster-{S}hafer theory to fuzzy sets,'' {\em
  IEEE Transactions on Systems, Man, and Cybernetics}, vol.~20:3, pp.~559--569,
  1990.

\bibitem{mahler95combining}
R.~P. Mahler, ``Combining ambiguous evidence with respect to ambiguous a priori
  knowledge. part ii: Fuzzy logic,'' {\em Fuzzy Sets and Systems}, vol.~75,
  pp.~319--354, 1995.

\bibitem{Palacharla94b}
P.~Palacharla and P.~Nelson, ``Evidential reasoning in uncertainty for data
  fusion,'' in {\em Proceedings of the Fifth International Conference on
  Information Processing and Management of Uncertainty in Knowledge-Based
  Systems}, vol.~1, pp.~715--720, 1994.

\bibitem{yager86entailment}
R.~R. Yager, ``The entailment principle {D}empster-{S}hafer granules,'' {\em
  International Journal of Intelligent Systems}, vol.~1, pp.~247--262, 1986.

\bibitem{yager95including}
R.~R. Yager and D.~P. Filev, ``Including probabilistic uncertainty in fuzzy
  logic controller modeling using {D}empster-{S}hafer theory,'' {\em IEEE
  Transactions on Systems, Man, and Cybernetics}, vol.~25:8, pp.~1221--1230,
  1995.

\bibitem{yager96normalization}
R.~R. Yager, ``On the normalization of fuzzy belief structures,'' {\em
  International Journal of Approximate Reasoning}, vol.~14, pp.~127--153, 1996.

\bibitem{smets91patterns}
P.~Smets, ``Patterns of reasoning with belief functions,'' {\em Journal of
  Applied Non-Classical Logic}, vol.~1:2, pp.~166--170, 1991.

\bibitem{cholvy2009using}
L.~Cholvy, ``{Using logic to understand relations between DSmT and
  Dempster-Shafer Theory},'' {\em Symbolic and Quantitative Approaches to
  Reasoning with Uncertainty}, pp.~264--274, 2009.

\bibitem{HRWW08a}
R.~Haenni, J.~Romeijn, G.~Wheeler, and J.~Williamson, ``Possible semantics for
  a common framework of probabilistic logics,'' in {\em {UncLog'08},
  International Workshop on Interval/Probabilistic Uncertainty and
  Non-Classical Logics} (V.~N. Huynh, Y.~Nakamori, H.~Ono, J.~Lawry,
  V.~Kreinovich, and H.~T. Nguyen, eds.), no.~46 in Advances in Soft Computing,
  (Ishikawa, Japan), pp.~268--279.

\bibitem{paris08unclog}
J.~B. Paris, D.~Picado-Muino, and M.~Rosefield, ``Information from inconsistent
  knowledge: A probability logic approach,'' in {\em Interval / Probabilistic
  Uncertainty and Non-classical Logics, Advances in Soft Computing} (V.-N.
  Huynh, Y.~Nakamori, H.~Ono, J.~Lawry, V.~Kreinovich, and H.~Nguyen, eds.),
  vol.~46, Springer-Verlag, Berlin - Heidelberg, 2008.

\bibitem{batens00}
D.~Batens, C.~Mortensen, and G.~Priest, ``Frontiers of paraconsistent logic,''
  in {\em Studies in logic and computation} (J.~V. Bendegem, ed.), vol.~8,
  Research Studies Press, 2000.

\bibitem{Saffiotti_abelief-function}
A.~Saffiotti, ``A belief-function logic,'' in {\em Universit Libre de
  Bruxelles}, pp.~642--647, MIT Press.

\bibitem{Saffiotti90}
A.~Saffiotti, ``A hybrid framework for representing uncertain knowledge,'' in
  {\em Procs. of the 8th AAAI Conf. Boston, MA}, pp.~653--658, 1990.

\bibitem{Saffiotti91}
A.~Saffiotti, ``A hybrid belief system for doubtful agents,'' in {\em
  Uncertatiny in Knowledge Bases, Lecture Notes in Computer Science 251},
  pp.~393--402, Springer-Verlag, 1991.

\bibitem{haenni05isipta}
R.~Haenni, ``Towards a unifying theory of logical and probabilistic
  reasoning,'' in {\em Proceedings of ISIPTA'05}, 2005.

\bibitem{provan90logicbased}
G.~M. Provan, ``A logic-based analysis of {D}empster-{S}hafer theory,'' {\em
  International Journal of Approximate Reasoning}, vol.~4, pp.~451--495, 1990.

\bibitem{harmanec94qualitative}
D.~Harmanec and P.~Hajek, ``A qualitative belief logic,'' {\em International
  Journal of Uncertainty, Fuzziness and Knowledge-Based Systems}, 1994.

\bibitem{kohlas87b}
J.~Kohlas, ``The logic of uncertainty. potential and limits of probability.
  theory for managing uncertainty in expert systems,'' Tech. Rep. 142,
  Institute for Automation and Operations Research, University of Fribourg,
  1987.

\bibitem{Mates72}
B.~Mates, {\em Elementary Logic}.
\newblock Oxford University Press, 1972.

\bibitem{Saffiotti92}
A.~Saffiotti, ``A belief function logic,'' in {\em Proceedings of the 10th AAAI
  Conf. San Jose, CA}, pp.~642--647, 1992.

\bibitem{benferhat95tech}
S.~Benferhat, A.~Saffiotti, and P.~Smets, ``Belief functions and default
  reasonings,'' tech. rep., Universite' Libre de Bruxelles, Technical Report
  TR/IRIDIA/95-5, 1995.

\bibitem{hunter87versus}
D.~Hunter, ``{D}empster-{S}hafer versus probabilistic logic,'' in {\em
  Proceedings of the Third AAAI Uncertainty in Artificial Intelligence
  Workshop}, pp.~22--29, 1987.

\bibitem{harmanec94modal}
D.~Harmanec, G.~Klir, and G.~Resconi, ``On modal logic inpterpretation of
  {D}empster-{S}hafer theory,'' {\em International Journal of Intelligent
  Systems}, vol.~9, pp.~941--951, 1994.

\bibitem{resconi96interpretations}
G.~Resconi, G.~J. Klir, D.~Harmanec, and U.~S. Clair, ``Interpretations of
  various uncertainty theories using models of modal logic: a summary,'' {\em
  Fuzzy Sets and Systems}, vol.~80, pp.~7--14, 1996.

\bibitem{harmanec96modal}
D.~Harmanec, G.~Klir, and Z.~Wang, ``Modal logic inpterpretation of
  {D}empster-{S}hafer theory: an infinite case,'' {\em International Journal of
  Approximate Reasoning}, vol.~14, pp.~81--93, 1996.

\bibitem{tsiporkova99evidence}
E.~Tsiporkova, B.~D. Baets, and V.~Boeva, ``Evidence theory in multivalued
  models of modal logic,'' {\em Journal of Applications of Nonclassical Logic},
  1999.

\bibitem{tsiporkova99dempster}
E.~Tsiporkova, B.~D. Baets, and V.~Boeva, ``Dempster's rule of conditioning
  traslated into modal logic,'' {\em Fuzzy Sets and Systems}, vol.~102,
  pp.~317--383, 1999.

\bibitem{dupin94penalty}
F.~D. de~Saint~Cyr, J.~Lang, and N.~Schiex, ``Penalty logic and its link with
  {D}empster-{S}hafer theory,'' in {\em Proceedings of UAI'94}, pp.~204--211,
  1994.

\bibitem{bundy85incidence}
A.~Bundy, ``Incidence calculus: A mechanism for probability reasoning,'' {\em
  Journal of automated reasoning}, vol.~1, pp.~263--283, 1985.

\bibitem{liu98method}
W.~Liu, D.~McBryan, and A.~Bundy, ``Method of assigning incidences,'' {\em
  Applied Intelligence}, vol.~9, pp.~139--161, 1998.

\bibitem{smets99practical}
P.~Smets, ``Practical uses of belief functions,'' in {\em Uncertainty in
  Artificial Intelligence 15} (L.~K. B. and P.~H., eds.), pp.~612--621, 1999.

\bibitem{shenoy94using}
P.~P. Shenoy, ``Using {D}empster-{S}hafer's belief function theory in expert
  systems,'' in {\em Advances in the Dempster-Shafer Theory of Evidence} (M.~F.
  R.~R.~Yager and J.~Kacprzyk, eds.), pp.~395--414, Wiley, New York, 1994.

\bibitem{boston00signal}
J.~R. Boston, ``A signal detection system based on {D}empster-{S}hafer theory
  and comparison to fuzzy detection,'' {\em IEEE Transactions on Systems, Man,
  and Cybernetics - Part C: Applications and Reviews}, vol.~30:1, pp.~45--51,
  February 2000.

\bibitem{ip91exchange}
H.~H.~S. Ip and H.-M. Wong, ``Evidential reasonign in foreign exchange rates
  forecasting,'' in {\em Proceedings of IEEE}, pp.~152--159, 1991.

\bibitem{soh98seaice}
L.-K. Soh, C.~Tsatsoulis, T.~Bowers, and A.~Williams, ``Representing sea ice
  knowledge in a {D}empster-{S}hafer belief system,'' in {\em Proceedings of
  IEEE}, pp.~2234--2236, 1998.

\bibitem{besserer93multiple}
B.~Besserer, S.~Estable, and B.~Ulmer, ``Multiple knowledge sources and
  evidential reasoning for shape recognition,'' in {\em Proceedings of IEEE},
  pp.~624--631, 1993.

\bibitem{simpson90application}
W.~R. Simpson and J.~W. Sheppard, ``The application of evidential reasoning in
  a portable maintenance aid,'' in {\em Proceedings of the IEEE Systems
  Readiness Technology Conference} (V.~Jorrand, P.;~Sgurev, ed.), pp.~211--214,
  San Antonio, TX, USA, 17-21 September 1990.

\bibitem{ferrari89coupling}
C.~Ferrari and G.~Chemello, ``Coupling fuzzy logic techniques with evidential
  reasoning for sensor data interpretation,'' in {\em Proceedings of
  Intelligent Autonomous Systems 2} (T.~Kanade, F.~Groen, and L.~Hertzberger,
  eds.), vol.~2, pp.~965--971, Amsterdam, Netherlands, 11-14 December 1989.

\bibitem{foucher99multiscale}
S.~Foucher, J.-M. Boucher, and G.~B. Benie, ``Multiscale and multisource
  classification using {D}empster-{S}hafer theory,'' in {\em Proceedings of
  IEEE}, pp.~124--128, 1999.

\bibitem{bergsten97applying}
U.~Bergsten, J.~Schubert, and P.~Svensson, ``Applying data mining and machine
  learning techniques to submarine intelligence analysise,'' in {\em
  Proceedings of the Third International Conference on Knowledge Discovery and
  Data Mining (KDD'97)} (D.~Heckerman, H.~Mannila, D.~Pregibon, and
  R.~Uthurusamy, eds.), pp.~127--130, Newport Beach, USA, 14-17 August 1997.

\bibitem{quost06ipmu1}
B.~Quost, T.~Denoeux, and M.~Masson, ``One-against-all classifier combination
  in the framework of belief functions,'' in {\em IPMU}, 2006.

\bibitem{Burger06}
T.~Burger, O.~Aran, and A.~Caplier, ``Modeling hesitation and conflict: A
  belief-based approach for multi-class problems,'' {\em Machine Learning and
  Applications, Fourth International Conference on}, pp.~95--100, 2006.

\bibitem{Aran09}
O.~Aran, T.~Burger, A.~Caplier, and L.~Akarun, ``A belief-based sequential
  fusion approach for fusing manual and non-manual signs,'' {\em Pattern
  Recognition}, vol.~42, pp.~812--822, May 2009.

\bibitem{kessentini10ipmu}
Y.~Kessentini, T.~Burger, and T.~Paquet, ``Evidential ensemble hmm classifier
  for handwriting recognition,'' in {\em Proceedings of IPMU}, 2010.

\bibitem{mas09}
M.-H. Masson and T.~Denoeux, ``Belief functions and cluster ensembles,'' in
  {\em ECSQARU}, pp.~323--334, July 2009.

\bibitem{Cuzzolin2018maxent}
F.~Cuzzolin, ``Generalised maximum entropy classifiers,'' in {\em Proceedings
  of BELIEF 2018}, 2018.

\bibitem{denoeux95knearest}
T.~Denoeux, ``A k-nearest neighbour classification rule based on
  {D}empster-{S}hafer theory,'' {\em IEEE Transactions on Systems, Man, and
  Cybernetics}, vol.~25:5, pp.~804--813, 1995.

\bibitem{zouhal98evidence}
L.~M. Zouhal and T.~Denoeux, ``Evidence-theoretic k-nn rule with parameter
  optimization,'' {\em IEEE Transactions on Systems, Man and Cybernetics Part
  C: Applications and Reviews}, vol.~28, pp.~263--271, 1998.

\bibitem{lehegarat97application}
S.~L. Hegarat-Mascle, I.~Bloch, and D.~Vidal-Madjar, ``Application of
  {D}empster-{S}hafer evidence theory to unsupervised clasification in
  multisource remote sensing,'' {\em IEEE Transactions on Geoscience and Remote
  Sensing}, vol.~35:4, pp.~1018--1031, July 1997.

\bibitem{binaghi99fuzzy}
E.~Binaghi and P.~Madella, ``Fuzzy {D}empster-{S}hafer reasoning for rule-based
  classifiers,'' {\em International Journal of Intelligent Systems}, vol.~14,
  pp.~559--583, 1999.

\bibitem{fixen95modified}
D.~Fixen and R.~P.~S. Mahler, ``The modified {D}empster-{S}hafer approach to
  classification,'' {\em IEEE Transactions on Systems, Man, and Cybernetics -
  Part A: Systems and Humans}, vol.~27:1, pp.~96--104, January 1997.

\bibitem{fixsen97modified}
D.~Fixsen and R.~P. Mahler, ``Modified {D}empster-{S}hafer approach to
  classification,'' {\em IEEE Transactions on Systems, Man, and Cybernetics
  Part A:Systems and Humans.}, vol.~27, pp.~96--104, 1997.

\bibitem{denoeux95neural}
T.~Denoeux, ``An evidence-theoretic neural network classifier,'' in {\em
  Proceedings of the 1995 IEEE International Conference on Systems, Man, and
  Cybernetics (SMC'95)}, vol.~3, pp.~712--717, October 1995.

\bibitem{wang98majority}
C.-C. Wang and H.-S. Don, ``The majority theorem of centralized multiple bams
  networks,'' {\em Information Sciences}, vol.~110, pp.~179--193, 1998.

\bibitem{loonis95multi}
P.~Loonis, E.-H. Zahzah, and J.-P. Bonnefoy, ``Multi-classifiers neural network
  fusion versus {D}empster-{S}hafer's orthogonal rule,'' in {\em Proceedings of
  IEEE}, pp.~2162--2165, 1995.

\bibitem{denoeux97analysis}
T.~Denoeux, ``Analysis of evidence-theoretic decision rules for pattern
  classification,'' {\em Pattern Recognition}, vol.~30:7, pp.~1095--1107, 1997.

\bibitem{ng98equalisation}
S.~N. Geok and S.~Harcharan, ``Data equalisation with evidence combination for
  pattern recognition,'' {\em Pattern Recognition Letters}, vol.~19,
  pp.~227--235, 1998.

\bibitem{schubert98neural}
J.~Schubert, ``A neural network and iterative optimization hybrid for
  {D}empster-{S}hafer clustering,'' in {\em Proceedings of EuroFusion98
  International Conference on Data Fusion (EF'98)} (J.~O. M.~Bedworth, ed.),
  pp.~29--36, Great Malvern, UK, 6-7 October 1998.

\bibitem{schubert98fast}
J.~Schubert, ``Fast {D}empster-{S}hafer clustering using a neural network
  structure,'' in {\em Proceedings of the Seventh International Conference on
  Information Processing and Management of Uncertainty in Knowledge-based
  Systems (IPMU'98)}, pp.~1438--1445, Universit\'e de La Sorbonne, Paris,
  France, 6-10 July 1998.

\bibitem{schubert99fast}
J.~Schubert, ``Fast {D}empster-{S}hafer clustering using a neural network
  structure,'' in {\em Information, Uncertainty and Fusion} (R.~R.~Y.
  B.~Bouchon-Meunier and L.~A. Zadeh, eds.), pp.~419--430, Kluwer Academic
  Publishers (SECS 516), Boston, MA, 1999.

\bibitem{schubert99simultaneous}
J.~Schubert, ``Simultaneous {D}empster-{S}hafer clustering and gradual
  determination of number of clusters using a neural network structure,'' in
  {\em Proceedings of the 1999 Information, Decision and Control Conference
  (IDC'99)}, pp.~401--406, Adelaide, Australia, 8-10 February 1999.

\bibitem{schubert97creating}
J.~Schubert, ``Creating prototypes for fast classification in
  {D}empster-{S}hafer clustering,'' in {\em Proceedings of the International
  Joint Conference on Qualitative and Quantitative Practical Reasoning (ECSQARU
  / FAPR '97)}, Bad Honnef, Germany, 9-12 June 1997.

\bibitem{wilkinson90evidential}
G.~G. Wilkinson and J.~Megier, ``Evidential reasoning in a pixel classification
  hierarchy-a potential method for integrating image classifiers and expert
  system rules based on geographic context,'' {\em International Journal of
  Remote Sensing}, vol.~11:10, pp.~1963--1968, October 1990.

\bibitem{wang91neural}
C.-C. Wang and H.-S. Don, ``Evidential reasoning using neural networks,'' in
  {\em Proceedings of IEEE}, pp.~497--502, 1991.

\bibitem{mohiddin94evidential}
S.~M. Mohiddin and T.~S. Dillon, ``Evidential reasoning using neural
  networks,'' in {\em Proceedings of IEEE}, pp.~1600--1606, 1994.

\bibitem{giacinto97application}
G.~Giacinto, R.~Paolucci, and F.~Roli, ``Application of neural networks and
  statistical pattern recognition algorithms to earthquake risk evaluation,''
  {\em Pattern Recognition Letters}, vol.~18, pp.~1353--1362, 1997.

\bibitem{wesley86cv}
L.~P. Wesley, ``Evidential knowledge-based computer vision,'' {\em Optical
  Engineering}, vol.~25, pp.~363--379, 1986.

\bibitem{pinz96active}
A.~Pinz, M.~Prantl, H.~Ganster, and H.~Kopp-Borotschnig, ``Active fusion - a
  new method applied to remote sensing image interpretation,'' {\em Pattern
  Recognition Letters}, vol.~17, pp.~1349--1359, 1996.

\bibitem{boshra99accommodating}
M.~Boshra and H.~Zhang, ``Accommodating uncertainty in pixel-based verification
  of 3-d object hypotheses,'' {\em Pattern Recognition Letters}, vol.~20,
  pp.~689--698, 1999.

\bibitem{li88evidential}
Z.~Li and L.~Uhr, ``Evidential reasoning in a computer vision system,'' in {\em
  Uncertainty in Artificial Intelligence 2} (Lemmer and Kanal, eds.),
  pp.~403--412, North Holland, Amsterdam, 1988.

\bibitem{guironnet06eusipco}
M.~Guironnet, D.~Pellerin, and M.~Rombaut, ``Camera motion classification based
  on the transferable belief model,'' in {\em Proceedings of EUSIPCO'06,
  Florence, Italy}, 2006.

\bibitem{Burger08}
T.~Burger, O.~Aran, A.~Urankar, L.~Akarun, and A.~Caplier, ``A dempster-shafer
  theory based combination of classifiers for hand gesture recognition,'' {\em
  Computer Vision and Computer Graphics - Theory and Applications, Lecture
  Notes in Communications in Computer and Information Science}, 2008.

\bibitem{Kes09}
Y.~Kessentini, T.~Paquet, and A.~B. Hamadou, ``Off-line handwritten word
  recognition using multi-stream hidden markov models,'' {\em Pattern
  Recognition Letters}, vol.~30, no.~1, pp.~60--70, 2010.

\bibitem{ayoun01data}
A.~Ayoun and P.~Smets., ``Data association in multi-target detection using the
  transferable belief model,'' {\em Intern. J. Intell. Systems}, vol.~16(10),
  pp.~1167--1182, 2001.

\bibitem{martinerie92dataassociation}
F.~Martinerie and P.~Foster, ``Data association and tracking from distributed
  sensors using hidden {M}arkov models and evidential reasoning,'' in {\em
  Proceedings of 31st Conference on Decision and Control}, pp.~3803--3804,
  Tucson, December 1992.

\bibitem{bogler87}
P.~Bogler, ``Shafer�{D}empster reasoning with applications to multisensor
  target identification systems,'' {\em IEEE Transactions on Systems, Man and
  Cybernetics}, vol.~17, no.~6, pp.~968--�977, 1987.

\bibitem{lohmann91evidential}
G.~Lohmann, ``An evidential reasoning approach to the classification of
  satellite images,'' in {\em Symbolic and Qualitative Approaches to
  Uncertainty} (R.~Kruse and P.~Siegel, eds.), pp.~227--231, Springer-Verlag,
  Berlin, 1991.

\bibitem{ip94facial}
H.~H.~S. Ip and R.~C.~K. Chiu, ``Evidential reasonign for facial gesture
  recognition from cartoon images,'' in {\em Proceedings of IEEE},
  pp.~397--401, 1994.

\bibitem{borotschnig98comparison}
H.~Borotschnig, L.~Paletta, M.~Prantl, and A.~Pinz, ``A comparison of
  probabilistic, possibilistic and evidence theoretic fusion schemes for active
  object recognition,'' {\em Computing}, vol.~62, pp.~293--319, 1999.

\bibitem{vasseur99perceptual}
P.~Vasseur, C.~Pegard, E.~Mouaddib, and L.~Delahoche, ``Perceptual organization
  approach based on {D}empster-{S}hafer theory,'' {\em Pattern Recognition},
  vol.~32, pp.~1449--1462, 1999.

\bibitem{smets78theory}
P.~Smets, ``Theory of evidence and medical diagnostic,'' {\em Medical
  Informatics Europe}, vol.~78, pp.~285--291, 1978.

\bibitem{smets79medical}
P.~Smets, ``Medical diagnosis : Fuzzy sets and degree of belief,'' in {\em MIC
  79} (J.~Willems, ed.), pp.~185--189, Wiley, 1979.

\bibitem{chen93medical}
S.-Y. Chen, W.-C. Lin, and C.-T. Chen, ``Evidential reasoning based on
  {D}empster-{S}hafer theory and its application to medical image analysis,''
  in {\em Proceedings of SPIE - Neural and Stochastic Methods in Image and
  Signal Processing II}, vol.~2032, pp.~35--46, San Diego, CA, USA, 12-13 July
  1993.

\bibitem{liu93datafusion}
L.~J. Liu, J.~Y. Yang, and J.~F. Lu, ``Data fusion for detection of early stage
  lung cancer cells using evidential reasoning,'' in {\em Proceedings of the
  SPIE - Sensor Fusion VI}, vol.~2059, pp.~202--212, Boston, MA, USA, 7-8
  September 1993.

\bibitem{deutsch91knowledge}
M.~Deutsch-McLeish, P.~Yao, F.~Song, and T.~Stirtzinger,
  ``Knowledge-acquisition methods for finding belief functions with an
  application to medical decision making,'' in {\em Proceedings of the
  International Symposium on Artificial Intelligence} (H.~Cantu-Ortiz, F.J.;
  Terashima-Marin, ed.), pp.~231--237, Cancun, Mexico, 13-15 November 1991.

\bibitem{bloch96aspects}
I.~Bloch, ``Some aspects of {D}empster-{S}hafer evidence theory for
  classification of multi-modality medical images taking partial volume effect
  into account,'' {\em Pattern Recognition Letters}, vol.~17, pp.~905--919,
  1996.

\bibitem{chen92spatial}
S.-Y. Chen, W.-C. Lin, and C.-T. Chen, ``Spatial reasoning based on
  multivariate belief functions,'' in {\em Proceedings of IEEE}, pp.~624--626,
  1992.

\bibitem{aran2009sequential}
O.~Aran, T.~Burger, A.~Caplier, and L.~Akarun, ``{Sequential Belief-Based
  Fusion of Manual and Non-manual Information for Recognizing Isolated
  Signs},'' {\em Gesture-Based Human-Computer Interaction and Simulation},
  pp.~134--144, 2009.

\bibitem{mascle98introduction}
I.~B. S.~Le H\'egarat-Mascle and D.~Vidal-Madjar, ``Introduction of
  neighborhood information in evidence theory and application to data fusion of
  radar and optical images with partial cloud cover,'' {\em Pattern
  Recognition}, vol.~31, pp.~1811--1823, 1998.

\bibitem{reece97qualitative}
S.~Reece, ``Qualitative model-based multisensor data fusion and parameter
  estimation using infinity -norm {D}empster-{S}hafer evidential reasoning,''
  in {\em Proceedings of the SPIE - Signal Processing, Sensor Fusion, and
  Target Recognition VI} (A.~Heckerman, D.;~Mamdani, ed.), vol.~3068,
  pp.~52--63, Orlando, FL, USA, 21-24 April 1997.

\bibitem{smets00fusion}
P.~Smets, ``Data fusion in the transferable belief model,'' in {\em Proc. 3rd
  Intern. Conf. Inforation Fusion}, pp.~21--33, Paris, France 2000.

\bibitem{an93structure}
P.~An and W.~M. Moon, ``An evidential reasoning structure for integrating
  geophysical, geological and remote sensing data,'' in {\em Proceedings of
  IEEE}, pp.~1359--1361, 1993.

\bibitem{filippidis99fuzzy}
A.~Filippidis, ``Fuzzy and {D}empster-{S}hafer evidential reasoning fusion
  methods for deriving action from surveillance observations,'' in {\em
  Proceedings of the Third International Conference on Knowledge-Based
  Intelligent Information Engineering Systems}, pp.~121--124, Adelaide,
  September 1999.

\bibitem{hong92recursive}
L.~Hong, ``Recursive algorithms for information fusion using belief functions
  with applications to target identification,'' in {\em Proceedings of IEEE},
  pp.~1052--1057, 1992.

\bibitem{buede97target}
D.~M. Buede and P.~Girardi, ``Target identification comparison of {B}ayesian
  and {D}empster-{S}hafer multisensor fusion,'' {\em IEEE Transactions on
  Systems, Man, and Cybernetics Part A: Systems and Humans.}, vol.~27,
  pp.~569--577, 1997.

\bibitem{leung00bayesian}
H.~Leung and J.~Wu, ``{B}ayesian and {D}empster-{S}hafer target identification
  for radar surveillance,'' {\em IEEE Transactions on Aerospace and Electronic
  Systems}, vol.~36:2, pp.~432--447, April 2000.

\bibitem{wesley93autonomous}
L.~P. Wesley, ``Autonomous locative reasoning: an evidential approach,'' in
  {\em Proceedings of IEEE}, pp.~700--707, 1993.

\bibitem{xia97driven}
Y.~Xia, S.~Iyengar, and N.~Brener, ``An event driven integration reasoning
  scheme for handling dynamic threats in an unstructured environment,'' {\em
  Artificial Intelligence}, vol.~95, pp.~169--186, 1997.

\bibitem{golshani96dynamic}
F.~Golshani, E.~Cortes-Rello, and T.~H. Howell, ``Dynamic route planning with
  uncertain information,'' {\em Knowledge-based Systems}, vol.~9, pp.~223--232,
  1996.

\bibitem{abel88lattice}
S.~Abel, ``The sum-and-lattice points method based on an evidential reasoning
  system applied to the real-time vehicle guidance problem,'' in {\em
  Uncertainty in Artificial Intelligence 2} (Lemmer and Kanal, eds.),
  pp.~365--370, 1988.

\bibitem{Pagac98}
D.~Pagac, E.~M. Nebot, and H.~Durrant-Whyte, ``An evidential approach to
  map-bulding for autonomous vehicles,'' {\em IEEE Transactions on Robotics and
  Automation}, vol.~14, No 4, pp.~623--629, August 1998.

\bibitem{gambino97tbm}
F.~Gambino, G.~Ulivi, and M.~Vendittelli, ``The transferable belief model in
  ultrasonic map building,'' in {\em Proceedings of IEEE}, pp.~601--608, 1997.

\bibitem{murphy98dempster}
R.~R. Murphy, ``Dempster-{S}hafer theory for sensor fusion in autonomous mobile
  robots,'' {\em IEEE Transactions on Robotics and Automation}, vol.~14,
  pp.~197--206, 1998.

\bibitem{lim94resolving}
E.-P. Lim, J.~Srivastava, and S.~Shekar, ``Resolving attribute incompatibility
  in database integration: an evidential reasoning approach,'' in {\em
  Proceedings of IEEE}, pp.~154--163, 1994.

\bibitem{dubitzky99towards}
W.~Dubitzky, A.~G. B�chner, J.~G. Hughes, and D.~A. Bell, ``Towards
  concept-oriented databases,'' {\em Data and Knowledge Engineering}, vol.~30,
  pp.~23--55, 1999.

\bibitem{mcclean00background}
S.~McClean, B.~Scotney, and M.~Shapcott, ``Using background knowledge in the
  aggregation of imprecise evidence in databases,'' {\em Data and Knowledge
  Engineering}, vol.~32, pp.~131--143, 2000.

\bibitem{mcclean97evidence}
S.~McClean and B.~Scotney, ``Using evidence theory for the integration of
  distributed databases,'' {\em International Journal of Intelligent Systems},
  vol.~12, pp.~763--776, 1997.

\bibitem{websterii99vadidation}
L.~W. II, J.-G. Chen, S.~S. Tan, C.~Watson, and A.~de~Korvin, ``Vadidation of
  authentic reasoning expert systems,'' {\em Information Sciences}, vol.~117,
  pp.~19--46, 1999.

\bibitem{Biswas89}
G.~Biswas and T.~S. Anand, ``Using the {D}empster-{S}hafer scheme in a
  mixed-initiative expert system shell,'' in {\em Uncertainty in Artificial
  Intelligence, volume 3} (L.~Kanal, T.~Levitt, and J.~Lemmer, eds.),
  pp.~223--239, North-Holland, 1989.

\bibitem{iancu97prosum}
I.~Iancu, ``Prosum-prolog system for uncertainty management,'' {\em
  International Journal of Intelligent Systems}, vol.~12, pp.~615--627, 1997.

\bibitem{guan90rule}
J.~Guan, D.~A. Bell, and V.~R. Lesser, ``Evidential reasoning and rule
  strengths in expert systems,'' in {\em Proceedings of AI and Cognitive
  Science '90} (N.~McTear, M.F.;~Creaney, ed.), pp.~378--390, Ulster, UK, 20-21
  September 1990.

\bibitem{xu96some}
H.~Xu and P.~Smets, ``Some strategies for explanations in evidential
  reasoning,'' {\em IEEE Transactions on Systems, Man and Cybernetics},
  vol.~26:5, pp.~599--607, 1996.

\bibitem{Strat87}
T.~M. Strat, ``The generation of explanations within evidential reasoning
  systems,'' in {\em Proceedings of the Tenth Joint Conference on Artificial
  Intelligence} (I.~of~Electrical and E.~Engineers, eds.), pp.~1097--1104,
  1987.

\bibitem{ramasso2007forward}
E.~Ramasso, M.~Rombaut, and D.~Pellerin, ``{Forward-Backward-Viterbi procedures
  in the Transferable Belief Model for state sequence analysis using belief
  functions},'' {\em Symbolic and Quantitative Approaches to Reasoning with
  Uncertainty}, pp.~405--417, 2007.

\bibitem{haduong06climate}
M.~Ha-Duong, ``Hierarchical fusion of expert opinion in the transferable belief
  model, application on climate sensivity,'' Working Papers halshs-00112129-v3,
  HAL, 2006.

\bibitem{shipley99project}
M.~F. Shipley, C.~A. Dykman, and A.~de~Korvin, ``Project management: using
  fuzzy logic and the {D}empster-{S}hafer theory of evidence to select team
  members for the project duration,'' in {\em Proceedings of IEEE},
  pp.~640--644, 1999.

\bibitem{gillett00monetary}
P.~R. Gillett, ``Monetary unit sampling: a belief-function implementation for
  audit and accounting applications,'' {\em International Journal of
  Approximate Reasoning}, vol.~25, pp.~43--70, 2000.

\bibitem{Cuzzolin99}
F.~Cuzzolin and R.~Frezza, ``An evidential reasoning framework for object
  tracking,'' in {\em SPIE - Photonics East 99 - Telemanipulator and
  Telepresence Technologies VI} (M.~R. Stein, ed.), vol.~3840, pp.~13--24,
  19-22 September 1999.

\bibitem{cuzzolin2000e}
F.~Cuzzolin and R.~Frezza, ``Sequences of belief functions and model-based data
  association,'' in {\em submitted to the IAPR Workshop on Machine Vision
  Applications (MVA2000)}, November 28-30, 2000.

\bibitem{diaz06fusion}
J.~Diaz, M.~Rifqi, and B.~Bouchon-Meunier, ``A similarity measure between basic
  belief assignments,'' in {\em Proceedings of FUSION'06}.

\bibitem{shi10distance}
C.~Shi, Y.~Cheng, Q.~Pan, and Y.~Lu, ``A new method to determine evidence
  distance,'' in {\em Proceedings of the 2010 International Conference on
  Computational Intelligence and Software Engineering (CiSE)}, pp.~1--4, 2010.

\bibitem{khatibi10new}
V.~Khatibi and G.~Montazer, ``A new evidential distance measure based on belief
  intervals,'' {\em Scientia Iranica - Transactions D: Computer Science and
  Engineering and Electrical Engineering}, vol.~17, no.~2, pp.~119--132, 2010.

\bibitem{jousselme10brest}
A.-L. Jousselme and P.~Maupin, ``On some properties of distances in evidence
  theory,'' in {\em Proceedings of BELIEF'10, Brest, France}, 2010.

\bibitem{cuzzolin10brest}
F.~Cuzzolin, ``Geometric conditioning of belief functions,'' in {\em
  Proceedings of BELIEF'10, Brest, France}, 2010.

\bibitem{cuzzolin-geometric-conditioning}
F.~Cuzzolin, ``Geometric conditional belief functions in the belief space,'' in
  {\em Proceedings of ISIPTA'11, Innsbruck, Austria}, 2011.

\bibitem{cuzzolin11isipta-conditional}
F.~Cuzzolin, ``Geometric conditional belief functions in the belief space,'' in
  {\em submitted to ISIPTA'11, Innsbruck, Austria}, 2011.

\bibitem{wang91geometrical}
C.-C. Wang and H.-S. Don, ``A geometrical approach to evidential reasoning,''
  in {\em Proceedings of IEEE}, pp.~1847--1852, 1991.

\bibitem{black96examination}
P.~Black, {\em An examination of belief functions and other monotone
  capacities}.
\newblock {PhD} dissertation, Department of Statistics, Carnegie Mellon
  University, 1996.
\newblock Pgh. PA 15213.

\bibitem{black97geometric}
P.~Black, ``Geometric structure of lower probabilities,'' in {\em Random Sets:
  Theory and Applications} (Goutsias, Malher, and Nguyen, eds.), pp.~361--383,
  Springer, 1997.

\bibitem{cuzzolin02fsdk}
F.~Cuzzolin, ``Geometry of {D}empster's rule,'' in {\em Proceedings of FSDK02},
  Singapore, 18-22 November 2002.

\bibitem{cuzzolin04smcb}
F.~Cuzzolin, ``Geometry of {D}empster's rule of combination,'' {\em IEEE
  Transactions on Systems, Man and Cybernetics part B}, vol.~34, no.~2,
  pp.~961--977, 2004.

\bibitem{monney91planar}
P.-A. Monney, ``Planar geometric reasoning with the thoery of hints,'' in {\em
  Computational Geometry. Methods, Algorithms and Applications, Lecture Notes
  in Computer Science, vol. 553} (H.~Bieri and H.~Noltemeier, eds.),
  pp.~141--159, 1991.

\bibitem{Ha}
V.~Ha and P.~Haddawy, ``Theoretical foundations for abstraction-based
  probabilistic planning,'' in {\em Proc. of the $12^{th}$ Conference on
  Uncertainty in Artificial Intelligence}, pp.~291--298, August 1996.

\bibitem{cuzzolin03isipta}
F.~Cuzzolin, ``Geometry of upper probabilities,'' in {\em Proceedings of the
  $3^{rd}$ Internation Symposium on Imprecise Probabilities and Their
  Applications (ISIPTA'03)}, July 2003.

\bibitem{Novikov_russian}
B.~A. Dubrovin, S.~P. Novikov, and A.~T. Fomenko, {\em Sovremennaja geometrija.
  Metody i prilozenija}.
\newblock Moscow: Nauka, 1986.

\bibitem{Socolovsky94}
H.~Garcia-Compe\'an, J.~L\'opez-Romero, M.~Rodriguez-Segura, and M.~Socolovsky,
  ``Principal bundles, connections and {BRST} cohomology,'' tech. rep., Los
  Alamos National Laboratory, hep-th/9408003, July 1994.

\bibitem{Gould}
H.~Gould, {\em Combinatorial identities}.
\newblock Morgantown, W.Va., 1972.

\bibitem{cuzzolin14annals}
F.~Cuzzolin, ``On the fiber bundle structure of the space of belief
  functions,'' {\em Annals of Combinatorics}, vol.~18.

\bibitem{miranda04ipmu}
P.~Miranda, M.~Grabisch, and P.~Gil, ``On some results of the set of dominating
  k-additive belief functions,'' in {\em IPMU}, pp.~625--632, 2004.

\bibitem{baroni04ipmu}
P.~Baroni, ``Extending consonant approximations to capacities,'' in {\em
  Proceedings of IPMU}, pp.~1127--1134, 2004.

\bibitem{cuzzolin05hawaii}
F.~Cuzzolin, ``On the properties of relative plausibilities,'' in {\em
  Proceedings of the International Conference of the IEEE Systems, Man, and
  Cybernetics Society (SMC'05), Hawaii, USA}.

\bibitem{cuzzolin06ipmu}
F.~Cuzzolin, ``The geometry of relative plausibilities,'' in {\em Proceedings
  of the $11^{th}$ International Conference on Information Processing and
  Management of Uncertainty IPMU'06, special session on ``Fuzzy measures and
  integrals, capacities and games}.

\bibitem{cuzzolin01thesis}
F.~Cuzzolin, {\em Visions of a generalized probability theory}.
\newblock {PhD} dissertation, Universit\`a di Padova, 19 February 2001.

\bibitem{cuzzolin01space}
F.~Cuzzolin and R.~Frezza, ``Geometric analysis of belief space and conditional
  subspaces,'' in {\em Proceedings of ISIPTA'01, Cornell University, June
  2001}.

\bibitem{cuzzolin2008geometric}
F.~Cuzzolin, ``{A geometric approach to the theory of evidence},'' {\em IEEE
  Transactions on Systems, Man, and Cybernetics, Part C: Applications and
  Reviews}, vol.~38, no.~4, pp.~522--534, 2008.

\bibitem{Cuzzolin2018geo}
F.~Cuzzolin, ``General geometry of belief function combination,'' in {\em
  Proceedings of BELIEF 2018}, 2018.

\bibitem{rota97book}
D.~A. Klain and G.-C. Rota, {\em Introduction to Geometric Probability}.
\newblock Cambridge University Press, 1997.

\bibitem{ha98geometric}
V.~Ha and P.~Haddawy, ``Geometric foundations for interval-based
  probabilities,'' in {\em {KR}'98: Principles of Knowledge Representation and
  Reasoning} (A.~G. Cohn, L.~Schubert, and S.~C. Shapiro, eds.), pp.~582--593,
  1998.

\bibitem{hunter06fusion}
A.~Hunter and W.~Liu, ``Fusion rules for merging uncertain information,'' {\em
  Information Fusion}, vol.~7, no.~1, pp.~97--134, 2006.

\bibitem{cuzzolin13fusion}
F.~Cuzzolin and W.~Gong, ``Belief modeling regression for pose estimation,'' in
  {\em Proceedings of FUSION 2013, Istanbul, Turkey}, pp.~1398--1405, 2013.

\bibitem{dubois92various}
D.~Dubois and H.~Prade, ``On the combination of evidence in various
  mathematical frameworks,'' in {\em Reliability Data Collection and Analysis}
  (J.~flamm and T.~Luisi, eds.), pp.~213--241, 1992.

\bibitem{wierman01measuring}
M.~Wierman, ``Measuring conflict in evidence theory,'' in {\em Proceedings of
  the Joint 9th IFSA World Congress, Vancouver, BC, Canada}, vol.~3,
  pp.~1741--1745, 2001.

\bibitem{lefevre02if}
E.~Lefevre, O.~Colot, and P.~Vannoorenberghe, ``Belief functions combination
  and conflict management,'' {\em Information Fusion Journal}, vol.~3, no.~2,
  pp.~149--162, 2002.

\bibitem{cattaneo03isipta}
M.~E. G.~V. Cattaneo, ``Combining belief functions issued from dependent
  sources.,'' in {\em ISIPTA}, pp.~133--147, 2003.

\bibitem{josang03strategies}
A.~Josang, M.~Daniel, and P.~Vannoorenberghe, ``Strategies for combining
  conflicting dogmatic beliefs,'' in {\em Proceedings of Fusion 2003}, vol.~2,
  pp.~1133--1140, 2003.

\bibitem{1163941}
W.~Liu, ``Analyzing the degree of conflict among belief functions,'' {\em
  Artif. Intell.}, vol.~170, no.~11, pp.~909--924, 2006.

\bibitem{Szasz}
G.~Szasz, {\em Introduction to lattice theory}.
\newblock New York and London: Academic Press, 1963.

\bibitem{Jacobson}
N.~Jacobson, {\em Basic Algebra {I}}.
\newblock New York: Freeman and Company, 1985.

\bibitem{Whitney35}
H.~Whitney, ``On the abstract properties of linear dependence,'' {\em American
  Journal of Mathematics}, vol.~57, pp.~509--533, 1935.

\bibitem{harary69}
F.~Harary and W.~T. Tutte, ``Matroids versus graphs,'' in {\em The many facets
  of graph theory, Lecture Notes in Math., Vol. 110}, pp.~155--170,
  Springer-Verlag, Berlin, 1969.

\bibitem{Oxley}
J.~G. Oxley, {\em Matroid theory}.
\newblock Great Clarendon Street, Oxford, UK: Oxford University Press, 1992.

\bibitem{cuzzolin08isaim-matroid}
F.~Cuzzolin, ``Boolean and matroidal independence in uncertainty theory,'' in
  {\em Proceedings of ISAIM 2008, Fort Lauderdale, Florida}, 2008.

\bibitem{cuzzolin14algebraic}
F.~Cuzzolin, ``Chapter 12: An algebraic study of the notion of independence of
  frames,'' in {\em Mathematics of Uncertainty Modeling in the Analysis of
  Engineering and Science Problems} (S.~Chakraverty, ed.), IGI Publishing,
  2014.

\bibitem{Kohlas95}
J.~Kohlas and P.-A. Monney, {\em A Mathematical Theory of Hints - An Approach
  to the {D}empster-{S}hafer Theory of Evidence}.
\newblock Lecture Notes in Economics and Mathematical Systems, Springer-Verlag,
  1995.

\bibitem{cuzzolin05amai}
F.~Cuzzolin, ``Algebraic structure of the families of compatible frames of
  discernment,'' {\em Annals of Mathematics and Artificial Intelligence},
  vol.~45(1-2), pp.~241--274, 2005.

\bibitem{Rosenbaum}
A.~Beutelspacher and U.~Rosenbaum, {\em Projective geometry}.
\newblock Cambridge: Cambridge University Press, 1998.

\bibitem{dilworth44}
R.~P. Dilworth, ``Dependence relations in a semimodular lattice,'' {\em Duke
  Math. J.}, vol.~11, pp.~575--587, 1944.

\bibitem{cuzzolin01bcc}
F.~Cuzzolin, ``Lattice modularity and linear independence,'' in {\em 18th
  British Combinatorial Conference, Brighton, UK}, 2001.

\bibitem{yaghlane00independence}
B.~B. Yaghlane, P.~Smets, and K.~Mellouli, ``Independence concepts for belief
  functions,'' in {\em Proceedings of Information Processing and Management of
  Uncertainty (IPMU'2000)}, 2000.

\bibitem{yager87new}
R.~Yager, ``On the {D}empster-{S}hafer framework and new combination rules,''
  {\em Information Sciences}, vol.~41, pp.~93--138, 1987.

\bibitem{deutsch90study}
M.~Deutsch-McLeish, ``A study of probabilities and belief functions under
  conflicting evidence: comparisons and new method,'' in {\em Proceedings of
  the 3rd International Conference on Information Processing and Management of
  Uncertainty in Knowledge-Based Systems (IPMU'90)} (B.~Bouchon-Meunier,
  R.~Yager, and L.~Zadeh, eds.), pp.~41--49, Paris, France, 2-6 July 1990.

\bibitem{liu06analyzing}
W.~Liu, ``Analyzing the degree of conflict among belief functions,'' {\em
  Artificial Intelligence}, vol.~170, pp.~909--924, 2006.

\bibitem{murphy00combining}
C.~K. Murphy, ``Combining belief functions when evidence conflicts,'' {\em
  Decision Support Systems}, vol.~29, pp.~1--9, 2000.

\bibitem{carlson05tech}
J.~Carlson and R.~Murphy, ``Use of {D}empster-{S}hafer conflict metric to adapt
  sensor allocation to unknown environments,'' tech. rep., Safety Security
  Rescue Research Center, University of South Florida, 2005.

\bibitem{sentz02tech}
K.~Sentz and S.~Ferson, ``Combination of evidence in {D}empster-{S}hafer
  theory,'' tech. rep., SANDIA Tech. Report, SAND2002-0835, April 2002.

\bibitem{Cuzzolin2000}
F.~Cuzzolin and R.~Frezza, ``Integrating feature spaces for object tracking,''
  in {\em Proc. of the International Symposium on the Mathematical Theory of
  Networks and Systems (MTNS2000)}, 21-25 June 2000.

\bibitem{cuzzolin05isipta}
F.~Cuzzolin and R.~Frezza, ``Evidential modeling for pose estimation,'' in {\em
  Proceedings of the $4^{rd}$ Internation Symposium on Imprecise Probabilities
  and Their Applications (ISIPTA'05)}, Pittsburgh, July 2005.

\bibitem{vanderwaerden37}
B.~L. van~der Waerden, {\em Moderne Algebra, Vol. 1}.
\newblock Berlin: Springer-Verlag, 1937.

\bibitem{maclane38}
S.~M. Lane, ``A lattice formulation for transcendence degrees and p-bases,''
  {\em Duke Math. J.}, vol.~4, pp.~455--468, 1938.

\bibitem{teichmuller36}
O.~Teichmuller, ``p-algebren,'' {\em Deutsche Math.}, vol.~1, pp.~362--388,
  1936.

\bibitem{Birkhoff35}
G.~Birkhoff, ``Abstract linear dependence and lattices,'' {\em American Journal
  of Mathematics}, vol.~57, pp.~800--804, 1935.

\bibitem{DaveyDefence}
S.~Davey and S.~Colgrove, ``A unified probabilistic data assotiation filter
  with multiple models,'' Tech. Rep. DSTO-TR-1184, Surveillance System
  Division, Electonic and Surveillance Reserach Lab., 2001.

\bibitem{Karlsson01}
R.~Karlsoon and F.~Gustafsson, ``Monte carlo data association for multiple
  target tracking,'' in {\em IEEE Workshop on Target Tracking}, 2001.

\bibitem{smets2004kalman}
P.~Smets and B.~Ristic, ``{Kalman filter and joint tracking and classification
  in the TBM framework},'' in {\em Proceedings of the Seventh International
  Conference on Information Fusion}, vol.~1, pp.~46--53, Citeseer, 2004.

\bibitem{ristic06if}
B.~Ristic and P.~Smets, ``The {TBM} global distance measure for the association
  of uncertain combat {ID} declarations,'' {\em Information Fusion}, vol.~7(3),
  pp.~276--284, 2006.

\bibitem{HagerCVPR98}
C.~Rasmussen and G.~Hager, ``Joint probabilistic techniques for tracking
  multi-part objects,'' in {\em Int. Conf. on Computer Vision and Pattern
  Recognition}, 1998.

\bibitem{HagerPAMI2001}
C.~Rasmussen and G.~Hager, ``Probabilistic data association methods for
  tracking complex visual objects,'' {\em IEEE Transaction on Patter Analysis
  and Machine Intelligence}, vol.~23, pp.~560--576, 2001.

\bibitem{Shalom88}
Y.~Bar-Shalom and T.~E. Fortmann, {\em Tracking and Data Association}.
\newblock Academic Press, Inc., 1988.

\bibitem{Bloem95}
E.~Bloem and H.~Blom, ``Joint probabilistic data association methods avoiding
  track coalescence,'' in {\em Proceedings of CDC'95}.

\bibitem{Blake96}
M.~Isard and A.~Blake, ``Contour tracking by stochastic propagation of
  conditional density,'' in {\em Proceedings of ECCV'96}, pp.~343--356, 1996.

\bibitem{Jung97}
S.~Jung and K.~Wohn, ``Tracking and motion estimation of the articulated
  object: a hierarchical kalman filter approach,'' {\em Real-Time Imaging},
  vol.~3, pp.~415--432, 1997.

\bibitem{Moore95}
R.~Elliot, L.~Aggoun, and J.~Moore, {\em Hidden {M}arkov models: estimation and
  control}.
\newblock 1995.

\bibitem{Slobodova97}
A.~Slobodova, ``Multivalued extension of conditional belief functions,'' in
  {\em Proceedings of the International Joint Conference on Qualitative and
  Quantitative Practical Reasoning (ECSQARU / FAPR '97)}, Bad Honnef, Germany,
  9-12 June 1997.

\bibitem{kohlas88b}
J.~Kohlas, ``Conditional belief structures,'' {\em Probability in Engineering
  and Information Science}, vol.~2, no.~4, pp.~415--433, 1988.

\bibitem{Zhou2017uai}
C.~Zhou and F.~Cuzzolin, ``The total belief theorem,'' in {\em The total belief
  theorem}, 2017.

\bibitem{Farina}
S.~Rinaldi and L.~Farina, {\em I sistemi lineari positivi: teoria e
  applicazioni}.
\newblock Citt\'a Studi Edizioni.

\bibitem{Sidenbladh00a}
H.~Sidenbladh, M.~Black, and D.~Fleet, ``Stochastic tracking of 3{D} human
  figures using 2d image motion,'' in {\em ECCV'00}, 2000.

\bibitem{Sminchisescu03kinematic}
C.~Sminchisescu and B.~Triggs, ``Kinematic jump processes for monocular 3d
  human tracking,'' in {\em Proc. of CVPR}, vol.~1, pp.~69�--76, 2003.

\bibitem{agarwal06pami}
A.~Agarwal and B.~Triggs, ``Recovering 3d human pose from monocular images,''
  {\em IEEE Trans. PAMI}, vol.~28, no.~1, pp.~44--�58, 2006.

\bibitem{bb71611}
A.~Elgammal and C.~Lee, ``Inferring 3d body pose from silhouettes using
  activity manifold learning,'' pp.~II: 681--688, 2004.

\bibitem{946721}
G.~Shakhnarovich, P.~Viola, and T.~Darrell, ``Fast pose estimation with
  parameter-sensitive hashing,'' in {\em ICCV '03: Proceedings of the Ninth
  IEEE International Conference on Computer Vision}, (Washington, DC, USA),
  p.~750, IEEE Computer Society, 2003.

\bibitem{1068941}
C.~Sminchisescu, A.~Kanaujia, Z.~Li, and D.~Metaxas, ``Discriminative density
  propagation for 3d human motion estimation,'' in {\em Proc. of CVPR}, vol.~1,
  pp.~390--397, 2005.

\bibitem{1099953}
T.-P. Tian, R.~Li, and S.~Sclaroff, ``Articulated pose estimation in a learned
  smooth space of feasible solutions,'' in {\em CVPR '05: Proceedings of the
  2005 IEEE Computer Society Conference on Computer Vision and Pattern
  Recognition (CVPR'05) - Workshops}, (Washington, DC, USA), p.~50, IEEE
  Computer Society, 2005.

\bibitem{bb71628}
R.~Poppe and M.~Poel, ``Comparison of silhouette shape descriptors for
  example-based human pose recovery,'' pp.~541--546, 2006.

\bibitem{bb33622}
Y.~Zheng, X.~Zhou, B.~Georgescu, S.~Zhou, and D.~Comaniciu, ``Example based
  non-rigid shape detection,'' pp.~IV: 423--436, 2006.

\bibitem{niyogi96afgr}
S.~Niyogi and W.~Freeman, ``Example-based head tracking,'' in {\em Proceedings
  of the Second International Conference on Automatic Face and Gesture
  Recognition}, pp.~374--378, 1996.

\bibitem{Athitsos04cvpr}
V.~Athitsos, J.~Alon, S.~Sclaroff, and G.~Kollios, ``Boostmap: A method for
  efficient approximate similarity rankings,'' in {\em Proc. of CVPR}, vol.~2,
  pp.~268�--275, 2004.

\bibitem{rosales00human-motion}
R.~Rosales and S.~Sclaroff, ``Specialized mappings and the estimation of human
  body pose from a single image,'' in {\em IEEE Workshop on Human Motion},
  pp.~19�--24, 2000.

\bibitem{dalal05cvpr}
N.~Dalal and B.~Triggs, ``Histograms of oriented gradients for human
  detection,'' in {\em Proc. of CVPR}, pp.~886�--893, 2005.

\bibitem{viola01cvpr}
P.~Viola and M.~Jones, ``Rapid object detection using a boosted cascade of
  simple features,'' in {\em Proc. of CVPR}, vol.~1, pp.~511�--518, 2001.

\bibitem{Hel-Or95}
M.~W. Y.~Hel-Or, ``Pose estimation by fusing noisy data of different
  dimensions,'' {\em IEEE PAMI}, vol.~17, pp.~195--201, 1995.

\bibitem{Darrell98}
M.~H. T.~Darrell, G.~Gordon and J.~Woodfill, ``Integrated person tracking using
  stereo, color, and pattern detection,'' in {\em CVPR'98}, pp.~601--608, 1998.

\bibitem{Moeslund00}
T.~Moeslund and E.~Granum, ``3{D} human pose estimation using 2{D}-data and an
  alternative phase space representation,'' in {\em Workshop on Human Modeling,
  Analysis and Synthesis at CVPR2000, Hilton Head Island}, June 2000.

\bibitem{Sminchisescu01}
C.~Sminchisescu and B.~Triggs, ``Covariance scaled sampling for monocular 3{D}
  body tracking,'' in {\em Proceedings of the IEEE Conference on Computer
  Vision and Pattern Recognition CVPR'01, Hawaii}, December 2001.

\bibitem{Sidenbladh03}
H.~Sidenbladh and M.~Black, ``Learning the statistics of people in images and
  video,'' {\em IJCV}, vol.~54, pp.~189--209, 2003.

\bibitem{Thayananthan06eccv}
A.~Thayananthan, R.~Navaratnam, B.~Stenger, P.~Torr, and R.~Cipolla,
  ``Multivariate relevance vector machines for tracking,'' in {\em Proc. of
  ECCV}, vol.~3, pp.~124--�138, 2006.

\bibitem{mori06pami}
G.~Mori and J.~Malik, ``Recovering 3d human body configurations using shape
  contexts,'' {\em IEEE Trans. PAMI}, vol.~28, no.~7.

\bibitem{rasmussen06gpr}
C.~Rasmussen and C.~Williams, {\em Gaussian Processes for Machine Learning}.
\newblock MIT Press, 2006.

\bibitem{gong2018tfs}
W.~Gong and F.~Cuzzolin, ``A belief-theoretical approach to example-based pose
  estimation,'' {\em IEEE Transactions on Fuzzy Systems}, vol.~26, no.~2,
  pp.~598--611, 2018.

\bibitem{Meynet08}
J.~Meynet, T.~Arsan, J.~C. Mota, and J.-P. Thiran, ``Fast multi-view face
  tracking with pose estimation,'' in {\em Proc. of EUSIPCO}, 2008.

\bibitem{melkonyan06ijar}
T.~Melkonyan and R.~Chambers, ``Degree of imprecision: Geometric and algebraic
  approaches,'' {\em forthcoming in the International Journal of Approximate
  Reasoning}, 2006.

\bibitem{berger90jspim}
Berger, ``Robust bayesian analysis: Sensitivity to the prior,'' {\em Journal of
  Statistical Planning and Inference}, vol.~25, pp.~303--328, 1990.

\bibitem{seidenfeld93dilation}
T.~Seidenfeld and L.~Wasserman, ``Dilation for convex sets of probabilities,''
  {\em Annals of Statistics}, vol.~21, pp.~1139--1154, 1993.

\bibitem{felzenszwalb-2010}
P.~Felzenszwalb, R.~Girshick, D.~McAllester, and D.~Ramanan, ``Object detection
  with discriminatively trained part based models,'' {\em PAMI}, vol.~32,
  no.~9, pp.~1627--1645, 2010.

\bibitem{moore-veryfast}
A.~Moore, ``Very fast em-based mixture model clustering using multiresolution
  kd-trees,'' in {\em Advances in Neural Information Processing Systems}
  (M.~Kearns and D.~Cohn, eds.), (340 Pine Street, 6th Fl., San Francisco, CA
  94104), pp.~543--549, Morgan Kaufman, April 1999.

\bibitem{Rosales00}
R.~Rosales and S.~Sclaroff, ``Learning and synthesizing human body motion and
  posture,'' in {\em Fourth Int. Conf. on Automatic Face and Gesture
  Recognition, Grenoble, France}, March 2000.

\bibitem{jsang06normalising}
A.~Jsang and S.~Pope, ``Normalising the consensus operator for belief fusion,''
  2006.

\bibitem{jiang08new}
W.~Jiang, A.~Zhang, and Q.~Yang, ``A new method to determine evidence
  discounting coefficient,'' in {\em LNCS}, vol.~5226/2008, pp.~882--887, 2008.

\bibitem{bell93discounting}
D.~A. Bell and J.~W. Guan, ``Discounting and combination operations in
  evidential reasoning,'' in {\em Uncertainty in Artificial Intelligence.
  Proceedings of the Ninth Conference (1993)} (A.~Heckerman, D.;~Mamdani, ed.),
  pp.~477--484, Washington, DC, USA, 9-11 July 1993.

\bibitem{smets07analyzing}
P.~Smets, ``Analyzing the combination of coflicting belief functions,'' {\em
  Information Fusion}, vol.~8, no.~4, pp.~387--412, 2007.

\bibitem{smets93dis}
P.~Smets, ``Belief functions: the disjunctive rule of combination and the
  generalized {B}ayesian theorem,'' {\em International Journal of Approximate
  reasoning}, vol.~9, pp.~1--35, 1993.

\bibitem{cuzzolin12ijar}
F.~Cuzzolin, ``On the relative belief transform,'' {\em International Journal
  of Approximate Reasoning}, 2012.

\bibitem{Howe99}
N.~Howe, M.~Leventon, and W.~Freeman, ``Bayesian reconstruction of 3{D} human
  motion from single-camera video,'' in {\em Neural Information Processing
  Systems, Denver, Colorado}, November 1999.

\bibitem{bo09cvpr}
L.~Bo and C.~Sminchisescu, ``Structured output-associative regression,'' in
  {\em Proceedings of CVPR}, 2009.

\bibitem{rudovic11iccv}
O.~Rudovic and M.~Pantic, ``Shape-constrained gaussian process regression for
  facial-point-based head-pose normalization,'' in {\em Proceedings of ICCV},
  2011.

\bibitem{tipping01rvm}
M.~E. Tipping, ``Sparse bayesian learning and the relevance vector machine,''
  {\em Journal of Machine Learning Research}, vol.~1, pp.~211--244, 2001.

\bibitem{scozzafava94subjective}
R.~Scozzafava, ``Subjective probability versus belief functions in artificial
  intelligence,'' {\em International Journal of General Systems}, vol.~22:2,
  pp.~197--206, 1994.

\bibitem{cuzzolin2019springer}
F.~Cuzzolin, {\em The geometry of uncertainty}.
\newblock Springer-Verlag, 2019.

\end{thebibliography}
\end{document}